\documentclass[11pt]{ociamthesis}  
\usepackage[utf8]{inputenc}
\usepackage[T1]{fontenc}
\usepackage{textcomp}

\usepackage{amsmath}
\usepackage{amsthm}
\usepackage{mathtools}
\usepackage[inline]{enumitem}
\usepackage{bm}
\usepackage{hyperref}
\usepackage{inconsolata}
\usepackage{booktabs}
\usepackage{faktor}
\usepackage{url}
\usepackage{graphicx}
\usepackage[charter]{mathdesign}
\usepackage{bm}
\usepackage{thmtools}
\usepackage{algorithm}
\usepackage{algorithmic}
\usepackage{tikz}
\let\mathbb\relax
\usepackage{bbold}
\usepackage{appendix}
\usepackage{caption}
\usepackage{siunitx}
\usepackage{adjustbox}
\usepackage{wrapfig}
\usepackage{multirow}
\usepackage{tikz-layers}
\usepackage{subcaption}
\usepackage{rotating}
\usepackage{xcolor}
\usepackage{arydshln}
\usepackage{longtable}
\usepackage{tikz-cd}
\usepackage{makecell}
\usepackage{tabularx}
\usepackage{cancel}
\usepackage[mathcal]{eucal}
\usepackage{nicefrac}

\usepackage{setspace}

\allowdisplaybreaks

\mathtoolsset{showonlyrefs=true}

\newcommand{\ba}{\mathbf{a}}
\newcommand{\bb}{\mathbf{b}}
\newcommand{\bc}{\mathbf{c}}
\newcommand{\bh}{\mathbf{h}}
\newcommand{\bi}{\mathbf{i}}
\newcommand{\be}{\mathbf{e}}
\newcommand{\bj}{\mathbf{j}}
\newcommand{\bk}{\mathbf{k}}
\newcommand{\bl}{\mathbf{l}}
\newcommand{\bs}{\mathbf{s}}
\newcommand{\bt}{\mathbf{t}}
\newcommand{\bp}{\mathbf{p}}
\newcommand{\bq}{\mathbf{q}}

\newcommand{\bu}{\mathbf{u}}
\newcommand{\bv}{\mathbf{v}}
\newcommand{\bw}{\mathbf{w}}
\newcommand{\bx}{\mathbf{x}}
\newcommand{\by}{\mathbf{y}}
\newcommand{\bz}{\mathbf{z}}

\newcommand{\bW}{\mathbf{W}}
\newcommand{\bX}{\mathbf{X}}

\newcommand{\bZ}{\mathbf{Z}}
\newcommand{\bA}{\mathbf{A}}
\newcommand{\bL}{\mathbf{L}}

\newcommand{\bP}{\mathbf{P}}

\newcommand{\bell}{{\boldsymbol{\ell}}}

\newcommand{\one}{\mathbb{1}}

\newcommand{\node}{\text{node}}

\newcommand{\wA}{\widetilde{A}}
\newcommand{\wB}{\widetilde{B}}
\newcommand{\wC}{\widetilde{C}}
\newcommand{\wcL}{\widetilde{\mathcal{L}}}
\newcommand{\wcK}{\widetilde{\mathcal{K}}}
\newcommand{\wP}{\widetilde{P}}

\newcommand{\tensalg}{H} 
\newcommand{\tensalgps}{H} 
\newcommand{\tensalgmr}{H^{n \times n}} 

\newcommand{\inc}{\text{inc}}
\newcommand{\zs}{\text{zs}}
\newcommand{\tp}{\text{tp}}

\newcommand{\GTN}{\texttt{G2TN} }
\newcommand{\GTAN}{\texttt{G2T(A)N} }

\newcommand{\middashrule}{\hdashline\noalign{\vskip0.5ex}}

\newcommand{\cB}{\mathcal{B}}

\newcommand{\cE}{\mathcal{E}}

\newcommand{\cG}{\mathcal{G}}
\newcommand{\cH}{\mathcal{H}}
\newcommand{\cI}{\mathcal{I}}
\newcommand{\cK}{\mathcal{K}}
\newcommand{\cL}{\mathcal{L}}
\newcommand{\cM}{\mathcal{M}}
\newcommand{\cN}{\mathcal{N}}

\newcommand{\cV}{\mathcal{V}}
\newcommand{\cX}{\mathcal{X}}

\newcommand{\bbE}{\mathbb{E}}
\newcommand{\bbV}{\mathbb{V}}
\newcommand{\bbP}{\mathbb{P}}
\newcommand{\bbN}{\mathbb{N}}
\newcommand{\bbR}{\mathbb{R}}
\newcommand{\R}{\mathbb{R}}
\newcommand{\bbZ}{\mathbb{Z}}

\newcommand{\norm}[1]{\left\lVert #1 \right\rVert}
\newcommand{\expe}[1]{\bbE\left[ #1 \right]}
\newcommand{\prob}[1]{\bbP\left[ #1 \right]}
\newcommand{\abs}[1]{\left\lvert #1 \right\rvert}
\newcommand{\bracks}[1]{\left\lbrack #1 \right\rbrack}
\newcommand{\curls}[1]{\left\{ #1 \right\}}
\newcommand{\pars}[1]{\left( #1 \right)}
\newcommand{\inner}[2]{\left\langle #1, #2 \right\rangle}
\newcommand{\KL}[2]{\mathsf{D}_\texttt{KL}\left\lbrack #1 \, \Vert \, #2 \right\rbrack}

\newcommand{\given}{\, \vert \,}
\newcommand{\setgiven}{\, : \,}

\newcommand{\T}[1]{\mathsf{T}((#1))}
\newcommand{\TV}[1]{\mathsf{T}(#1)}
\renewcommand{\d}{\mathrm{d}}
\renewcommand{\L}{L}
\newcommand{\spn}{\texttt{span}}

\newcommand{\med}{\operatorname{med}}
\newcommand{\Jac}{\operatorname{J}}
\DeclareMathOperator{\Span}{span}
\DeclareMathOperator{\sh}{Sh}
\newcommand{\vs}{\bbR^d}

\newcommand{\onevar}{{1\texttt{-var}}}
\newcommand{\Lattice}[2]{\mathsf{Lattice}_{{#1},{#2}}(\cX)}
\newcommand{\Paths}{\mathsf{Paths}}
\newcommand{\Seq}{\mathsf{Seq}}
\newcommand{\len}[1]{{\mathsf{L}_{#1}}}
\newcommand{\hor}[1]{{\mathsf{T}_{#1}}}

\theoremstyle{plain}
\newtheorem{theorem}{Theorem}[chapter]
\newtheorem{proposition}[theorem]{Proposition}
\newtheorem{lemma}[theorem]{Lemma}
\newtheorem{corollary}[theorem]{chapter}
\newtheorem{definition}[theorem]{Definition}
\newtheorem{remark}[theorem]{Remark}
\newtheorem{example}[theorem]{Example}

\newcommand{\texp}[1][]{\ifthenelse{\equal{#1}{}}{\exp_\otimes}{\exp_{\otimes_{#1}}}}
\renewcommand{\S}{\mathsf{S}}

\newcommand{\kernel}{\mathsf{k}} 
\newcommand{\sigkernel}[1][]{{\ifthenelse{\equal{#1}{}}{\kernel_\S}{\kernel_{\S_{#1}}}}} 
\newcommand{\sigkernelp}[2][]{{\ifthenelse{\equal{#1}{}}{\kernel^{(#2)}_\S}{\kernel^{(#2)}_{\S_{#1}}}}} 
\newcommand{\rffkernel}{\tilde{\kernel}}
\newcommand{\rffsigkernel}[1][]{\ifthenelse{\equal{#1}{}}{\rffkernel_\S}{\rffkernel_{\S_{#1}}}}
\newcommand{\rffsigkernelp}[2][]{\ifthenelse{\equal{#1}{}}{\rffkernel^{(#2)}_\S}{\rffkernel^{(#2)}_{\S_{#1}}}}
\newcommand{\rffsigkernelhat}[1][]{\ifthenelse{\equal{#1}{}}{\hat{\kernel}_\S}{\hat{\kernel}_{\S_{#1}}}}

\newcommand{\rffsigkernelDP}[1][]{\ifthenelse{\equal{#1}{}}{\rffkernel^{\texttt{DP}}_\S}{\rffkernel^{\texttt{DP}}_{\S_{#1}}}}
\newcommand{\rffsigkernelTRP}[1][]{\ifthenelse{\equal{#1}{}}{\rffkernel^\texttt{TRP}_\S}{\rffkernel^\texttt{TRP}_{\S_{#1}}}}

\newcommand{\test}[1][]{\ifthenelse{\equal{#1}{}}{given}{not given}}

\newcommand{\fm}{\varphi} 
\newcommand{\fms}{\widetilde{\varphi}} 
\newcommand{\fmn}{\Phi} 
\newcommand{\fmg}{\Psi} 
\newcommand{\fmnz}{\hat{\Phi}}
\newcommand{\nf}{f} 

\newcommand{\kernelfeatures}{\varphi}
\newcommand{\signature}[1][]{\ifthenelse{\equal{#1}{}}{\kernelfeatures_\S}{\kernelfeatures_{\S_{#1}}}}

\newcommand{\rff}{\tilde{\kernelfeatures}}

\newcommand{\rffsig}[1][]{\ifthenelse{\equal{#1}{}}{\rff_{\S}}{\rff_{\S_{#1}}}}
\newcommand{\rffsighat}[1][]{\ifthenelse{\equal{#1}{}}{\hat{\kernelfeatures}_{\S}}{\hat{\kernelfeatures}_{\S_{#1}}}}
\newcommand{\rffsigDP}[1][]{\ifthenelse{\equal{#1}{}}{\rff_{\S}^{\texttt{DP}}}{\rff_{\S_{#1}}^{\texttt{DP}}}}
\newcommand{\rffsigTRP}[1][]{\ifthenelse{\equal{#1}{}}{\rff_{\S}^{\texttt{TRP}}}{\rff_{\S_{#1}}^{\texttt{TRP}}}}

\newcommand{\iid}{\text{i.i.d.}}

\newcommand{\Dp}[1]{D^{(p)}}
\newcommand{\Ep}[1]{E^{(p)}}
\newcommand{\Hil}{{\cH_\kernel}}
\newcommand{\HilT}{\mathrm{T}(\cH)}
\newcommand{\HilRFF}{\tilde{\Hil}}

\newcommand{\HilRFFT}{\HilRFF_{\S}}
\newcommand{\HilRFFTDP}{\HilRFF_{\S}^{\texttt{DP}}}
\newcommand{\HilRFFTTRP}{\HilRFF_{\S}^{\texttt{TRP}}}

\newcommand{\dimRFF}{\tilde{d}}

\newcommand{\dimTRP}{\dimRFF_{\texttt{TRP}}}

\newcommand{\p}{\prime}
\newcommand{\GP}{\texttt{GP}}
\newcommand{\FCN}[1]{\ensuremath{\texttt{FCN}_{#1}}}

\newcommand{\LStwoTwidth}[2]{\ensuremath{\texttt{LS2T}_{#1}^{#2}}}

\newcommand{\FCNLStwoTwidth}[3]{\ensuremath{\texttt{FCN}_{#1}\text{-}\texttt{LS2T}_{#2}^{#3}}}
\newcommand{\SVM}{\texttt{SVM}}

\newcommand{\RKHS}{\texttt{RKHS}}
\newcommand{\RFF}{\texttt{RFF}}
\newcommand{\RBF}{\texttt{RBF}}
\newcommand{{\RFSF}}{{\texttt{RFSF}}}
\newcommand{\RFSFD}{\texttt{RFSF-DP}}
\newcommand{\RFSFT}{\texttt{RFSF-TRP}}
\newcommand{\KS}{\texttt{KSig}}
\newcommand{\KSP}{\texttt{KSigPDE}}
\newcommand{\GAK}{\texttt{GAK}}
\newcommand{\RWS}{\texttt{RWS}}

\newcommand{\TRP}{\texttt{TRP}}
\newcommand{\CP}{\texttt{CP}}
\newcommand{\pr}{\texttt{Pr}}

\renewcommand{\d}{\,\mathrm{d}}

\newcommand{\diag}{\operatorname{diag}}

\renewcommand{\b}{\mathbf}

\renewcommand{\c}{\mathcal}

\newcommand{\cU}{\mathcal{U}}

\newcommand{\bmu}{\boldsymbol{\mu}}
\newcommand{\bomega}{\boldsymbol{\omega}}
\newcommand{\blambda}{\boldsymbol{\lambda}}
\newcommand{\bOmega}{\boldsymbol{\Omega}}

\newcommand{\spc}{\hspace{3pt}}

\usepackage{l3draw,xparse}
\ExplSyntaxOn
\fp_new:N \l__expl_shuffle_fp

\cs_new_protected:Nn \expl_shuffle:
 {
  \draw_begin:
  \draw_linewidth:n { \l__expl_shuffle_fp }
  \draw_cap_round:
  \draw_join_round:
  \draw_path_moveto:n { 0.0ex , 1.0ex }
  \draw_path_lineto:n { 0.0ex , 0.0ex }
  \draw_path_lineto:n { 1.75ex , 0.0ex }
  \draw_path_lineto:n { 1.75ex , 1.0ex }
  \draw_path_use_clear:n { stroke }
  \draw_path_moveto:n { 0.875ex , 0.0ex }
  \draw_path_lineto:n { 0.875ex , 1.0ex }
  \draw_path_use_clear:n { stroke }
  \draw_end:
 }
\NewDocumentCommand{\shuffle}{}
 {
  \mathrel
   {
    \fp_set:Nn \l__expl_shuffle_fp { 0.08ex }
    \text{$\mspace{2mu}$ \expl_shuffle: $\mspace{2mu}$}
   }
 }
\ExplSyntaxOff

\usetikzlibrary{shapes}
\usetikzlibrary{calc,positioning,decorations.pathreplacing,shapes.misc,arrows.meta,3d}

\pgfdeclarehorizontalshading{red}{150bp}{rgb(0bp)=(1,0.76,0.77);
rgb(37.5bp)=(1,0.76,0.77);
rgb(62.5bp)=(0.95,0.55,0.56);
rgb(100bp)=(0.95,0.55,0.56)}

\pgfdeclarehorizontalshading{grey}{150bp}{rgb(0bp)=(0.89,0.89,0.89);
rgb(37.5bp)=(0.89,0.89,0.89);
rgb(50bp)=(0.86,0.86,0.86);
rgb(50.25bp)=(0.82,0.82,0.82);
rgb(62.5bp)=(0.97,0.97,0.97);
rgb(100bp)=(1,1,1)}

\pgfdeclarehorizontalshading{blue}{150bp}{rgb(0bp)=(0.76,0.89,1);
rgb(37.5bp)=(0.76,0.89,1);
rgb(62.5bp)=(0.58,0.79,0.9);
rgb(100bp)=(0.58,0.79,0.9)}

\pgfdeclarehorizontalshading{green}{150bp}{rgb(0bp)=(0.71,0.94,0.73);
rgb(37.5bp)=(0.71,0.94,0.73);
rgb(62.5bp)=(0.49,0.79,0.48);
rgb(100bp)=(0.49,0.79,0.48)}

\tikzset{outer sep=0.1pt, inner sep=0,
	node/.style={rectangle, draw=black, minimum width=42, minimum height=20, outer sep=1.0pt, inner sep=0},
	every picture/.style={line width=0.75pt}}

 \tikzset{
    rect_sharp/.style={rectangle, draw=black, minimum width=30, minimum height=20, outer sep=2.5pt, inner sep=2.0pt},
    rect/.style={rectangle, rounded corners=2.5pt, draw=black, minimum width=30, minimum height=20, outer sep=2.5pt, inner sep=2.0pt},
	rect2/.style={rectangle, rounded corners=2.5pt, draw=black, minimum width=40, minimum height=20, outer sep=2.5pt, inner sep=2.0pt},
	every picture/.style={line width=0.75pt},
    ncbar angle/.initial=90,
    ncbar/.style={
        to path=(\tikztostart)
        -- ($(\tikztostart)!#1!\pgfkeysvalueof{/tikz/ncbar angle}:(\tikztotarget)$)
        -- ($(\tikztotarget)!($(\tikztostart)!#1!\pgfkeysvalueof{/tikz/ncbar angle}:(\tikztotarget)$)!\pgfkeysvalueof{/tikz/ncbar angle}:(\tikztostart)$)
        -- (\tikztotarget)
    },
    ncbar/.default=0.5cm,
}

\tikzset{square left brace/.style={ncbar=0.15cm}}
\tikzset{square right brace/.style={ncbar=-0.2cm}}

\tikzset{
    seq/.style={rectangle, text centered, minimum height=0.5cm, minimum width=3.5cm, rotate=90, outer sep=2.5pt, inner sep=0pt},
    wideseq/.style={rectangle, text centered, minimum height=0.75cm, minimum width=3.5cm, rotate=90, outer sep=2.5pt, inner sep=0pt},
    seq2/.style={rectangle, text centered, minimum height=0.5cm, minimum width=3.5cm, outer sep=1.0pt, inner sep=0pt},
    wideseq2/.style={rectangle, text centered, minimum height=0.75cm, minimum width=3.5cm, outer sep=1.0pt, inner sep=0pt}
}

\tikzset{
    sqr/.style={rectangle, text centered, minimum height=1.0cm, minimum width=1.0cm, outer sep=0.25cm, inner sep=0pt},
}

\tikzset{
    circ/.style={circle, draw=black, minimum size=25, outer sep=2.5pt, inner sep=0.0pt},
    circ2/.style={circle, minimum width=0.4cm, minimum height=0.4cm, outer sep=2.5pt, inner sep=0pt}
}

\tikzset{
    cross/.style={cross out, draw=black, minimum size=12.5pt, inner sep=0pt, outer sep=0pt, rotate=45},
    cross/.default={1pt}
}

\definecolor{grey}{HTML}{F5F5F5}
\definecolor{darkgrey}{HTML}{888888}
\definecolor{orange}{HTML}{FFF2CC}
\definecolor{brown}{HTML}{D6B656}
\definecolor{lightyellow}{HTML}{fff0d2}
\definecolor{yellow}{HTML}{d0ab49}
\definecolor{lightblue}{HTML}{e6f1f8}
\definecolor{blue}{HTML}{69a7d0}
\definecolor{lightgreen}{HTML}{e9f5e9}
\definecolor{green}{HTML}{7cb959}
\definecolor{lightpurple}{HTML}{f4ebf9}
\definecolor{purple}{HTML}{b793c8}
\definecolor{darkerpurple}{HTML}{9673A6}

\let\oldparagraph=\paragraph
\renewcommand\paragraph[1]{\oldparagraph{#1.}}

\title{Scalable Machine Learning Algorithms\\     
       using Path Signatures}   

\author{Csaba T{\'o}th}             
\college{Corpus Christi College}  

\degree{Doctor of Philosophy in Mathematics}     
\degreedate{Michaelmas 2024}         

\begin{document}

\baselineskip=18pt plus1pt

\setcounter{secnumdepth}{3}
\setcounter{tocdepth}{3}

\maketitle                  
\pagenumbering{gobble}
\chapter*{Acknowledgements}
First and foremost, I would like to thank my supervisor Prof.~Harald Oberhauser for his ongoing support and guidance throughout my DPhil journey. I am deeply grateful to Harald for the thought-provoking conversations we had and for always supporting me in exploring my ideas.
I would like to extend my gratitude to Prof.~Terry Lyons and
the DataSig team, who welcomed me as one of their own, and provided me with opportunities to present my work at various workshops and conferences. I would also like to thank the Mathematical Institute of the University of Oxford for generously providing me with a scholarship that allowed me to undertake this journey, and for providing the computational resources required for my applied projects.

I am grateful to my collaborators and colleagues for the many discussions we
had: Patric Bonnier, Darrick Lee, Zoltán Szabó, Masaki Adachi, Patrick Kidger, Cristopher Salvi, Maud Lemercier, Christina Zou, Alexander Schell; thank you for your help and collegiality. 
I am also grateful to all my friends whether located in the UK or Hungary for supporting me in all ways and standing by me potentially despite the physical distance between us. Finally, I would like to thank the infinite support of my family: my sister Réka and my mother Éva; without your love and support I would not be where or who I am today. I dedicate this thesis to you.  
\chapter*{Declarations}
This thesis is submitted to the University of Oxford in support of my application for the degree of Doctor of Philosophy. It has been composed by myself under the supervision of Prof.~Harald Oberhauser. I confirm that this thesis has not been submitted for another university degree. Parts of this thesis have appeared as published papers some of which were written with collaborators:
\begin{enumerate}
    \item The work in Chapter \ref{ch:gpsig} adapts the paper \cite{toth2020bayesian} written by me under the supervision of Harald Oberhauser published at the \textit{International Conference on Machine Learning} in 2020.
    \item The work in Chapter \ref{ch:s2t} adapts the paper \cite{toth2021seq2tens}, that was written in collaboration with Patric Bonnier under the supervision of Harald Oberhauser published at the \textit{International Conference on Learning Representations} in 2021. In particular, I wrote the paper besides the theoretical results appearing originally in Appendices A, B and C, which were written by Patric, and due to space limitations these were omitted here and deferred to the article.
    \item Chapter \ref{ch:g2t} adapts the paper \cite{toth2022capturing} written in collaboration with Darrick Lee with frequent discussions with Celia Hacker and Harald Oberhauser published at the \textit{Advances in Neural Information Processing Systems} in 2022. My contribution was the idea of using expected signatures to describe random walks on graphs, deriving the governing equations both in the tensor-valued and low-rank case, implementing and running experiments, while Darrick formalized the theory and the connection to hypo-elliptic diffusions. The main text was written with equal contribution, while the parts of the appendix written by Darrick consisting of a background and certain theoretical proofs were deferred to the article.
    \item Chapter \ref{ch:rfsf} adapts the paper \cite{toth2023random} published in \textit{SIAM Mathematics of Data Science} in 2025 written by me with frequent discussions with Zoltán Szabó and Harald Oberhauser.
    \item Chapter \ref{ch:vfsf} adapts \cite{toth2024learning} which was published at \textit{Artificial Intelligence and Statistics 2025} written by me with frequent discussions with Masaki Adachi and Harald Oberhauser. 
\end{enumerate}        
\chapter*{Abstract}

In this thesis, we consider the integration of path signatures--a mathematical object rooted in rough path theory--with scalable machine learning algorithms to address challenges in sequential and structured data modelling. The key topics considered include:
\begin{itemize}

\item \textbf{Path Signatures}: Introduced as a hierarchical and theoretically robust feature representation for sequential data, path signatures faithfully capture dynamics while offering properties like invariance to reparameterization and tree-like equivalence. Challenges like computational overheads are addressed through novel algorithms throughout the thesis.

\item \textbf{Gaussian Processes}: We demonstrate embedding signature kernels into Gaussian process models, offering an expressive probabilistic modelling approach for sequential data while tackling computational barriers via sparse variational inference techniques. This approach enhances performance on probabilistic time series classification tasks.

\item \textbf{\texttt{Seq2Tens} Framework}: Combines signature features with deep learning, using iterations of low-rank layers to mitigate computational costs while retaining expressiveness. Applications include time series classification, mortality prediction, generative modelling.

\item \textbf{Graph Representation}: Extends path signatures to graph data and connects them to hypo-elliptic diffusions, combined with low-rank techniques offering scalable architectures for capturing global and local graph structures, and outperforming conventional graph neural networks on tasks requiring long-range reasoning.

\item \textbf{Random Fourier Signature Features}: Introduces scalable random feature-based approximations for signature kernels with supporting theoretical results, overcoming computational limitations for large datasets while retaining state-of-the-art performance.

\item \textbf{Recurrent Sparse Spectrum Signature Gaussian Processes}: Combines Random Fourier Signature Features with Gaussian Processes and a forgetting mechanism for adaptive context focus in time series forecasting, bridging short- and long-term dependencies.
\end{itemize}

The topics are divided into separate, self-contained chapters that can be read independently.          
\begin{romanpages}
\tableofcontents            
\end{romanpages}            

\pagenumbering{arabic}
\chapter*{Introduction} \label{ch:intro}
\addcontentsline{toc}{chapter}{Introduction}

\section*{Unifying theme}
\addcontentsline{toc}{section}{Unifying theme}
The unifying theme of this thesis centers on the integration of path signatures, a mathematical framework rooted in rough path theory, with scalable machine learning algorithms to address challenges in sequential and structured data modelling. Sequential data, such as time series, DNA sequences, or video streams, represent a cornerstone of modern data analysis due to its prevalence across various fields, including finance, healthcare, and natural language processing. The core challenge lies in extracting meaningful representations from such data while preserving its intrinsic sequential structure and ensuring computational efficiency.

A prominent approach for exploiting sequential structure is through handcrafted feature engineering techniques, which find representations of the data so that well-established models become applicable. The dataset is then typically represented in a finite-dimensional Euclidean space. This approach is problematic in the sense that the feature extraction method often has to be tailored to the problem at hand, and some loss of information is inevitably incurred. In recent years, the signature method originating from rough path theory has become a popular method for extracting features of streamed data, that describe sequences and paths using a hierarchy of tensors with increasing degrees. Path signatures provide a robust theoretical foundation for representing sequential data as a hierarchy of tensor features, faithfully capturing the dynamics of the data. On the one hand, this has been remarkably successful due to the powerful mathematical machinery that underlies signatures allowing to provide strong theoretical guarantees. Their mathematical elegance offers unique benefits, such as invariance to reparameterization and universality for approximating continuous functions. On the other hand, this representation invokes a considerable computational burden due to the combinatorially increasing dimensionality of tensors of increasing degrees, as such their practical application in machine learning has been limited by significant computational challenges, e.g.~the polynomial growth in feature space dimension leading to bottlenecks in handling high-dimensional data.

This thesis takes a multifaceted approach to address these challenges by developing novel algorithms that retain the theoretical strengths of path signatures while significantly improving their computational scalability. Central to this work is the integration of path signatures with popular machine learning methodologies, including Gaussian processes, deep learning approaches, and kernel methods. Additionally, the thesis extends the application of path signatures beyond sequential data to the domain of graphs, showcasing their versatility and adaptability. Overall, this approach proved to be a fruitful one often leading to outperforming state-of-the-art methods in their genre on tasks such as time series classification, time series forecasting, mortality prediction in healthcare, generative modelling, and graph-based learning. This unifying theme underscores the thesis's contribution to bridging the gap between the theoretical elegance of path signatures and their practical deployment in real-world machine learning applications.

\section*{Outline}
\addcontentsline{toc}{section}{Outline}
The thesis is organized into six chapters that serve as standalone blocks readable independently, each complete with their own set of introductions, conclusions, and appendices.

\paragraph{Chapter \ref{ch:back}}
This introductory chapter provides an accessible introduction to path signatures, a mathematical object rooted in rough path theory, that serves as a powerful feature representation for sequential data. This chapter lays the foundation for understanding the algebraic, analytic, and computational properties of path signatures that make them suitable for machine learning applications.
The chapter begins by defining tensor products and constructing tensor algebras, which form the hierarchical feature space where path signatures reside. Tensor algebras, with their associative and noncommutative structure, allow for the representation of data as sequences of tensors, while the dual space and operations like the shuffle product further enhance their algebraic utility. Path signatures are then presented as sequences of iterated integrals capturing the geometry of paths. For bounded variation paths, these signatures encode the interactions of path coordinates, with practical computation simplified using tools like Chen’s identity. Although the signature admits extensions to the unbounded variation setting, our applications lie in the discrete domain, and we defer to \cite{lyons2007differential} for a review of rough path theory, one of the main themes of which is exactly that how highly oscillatory driving signals affect systems of differential equations, and how the signature can be extended to this setting.

The chapter emphasizes the key properties of path signatures: their universality as an algebra of functions that approximate continuous mappings, reparameterization invariance for robustness to changes in time sampling, and the ability to algebraically manipulate path concatenation and reversal.
Motivating examples illustrate how path signatures generalize polynomial feature maps for sequential data, providing a versatile and robust framework for sequence modelling. Challenges like the combinatorial growth of tensor dimensions are acknowledged, which is addressed via various scalable solutions later in the thesis, such as low-rank approximations and random projections. By establishing the theoretical and practical foundations of path signatures, Chapter 1 sets the stage for their integration into machine learning models.

\paragraph{Chapter \ref{ch:gpsig}}
Gaussian Processes (\texttt{GP}s) are a cornerstone of Bayesian machine learning, providing a flexible, non-parametric framework for modelling uncertainty in data. However, their application to sequential data, such as time series or trajectories, is often limited by the difficulty of constructing covariance functions that adequately capture temporal dependencies and patterns. Path signatures offer a powerful solution to this challenge. They provide a hierarchical and structured encoding of sequence information through iterated integrals, making them an ideal tool for defining expressive covariance functions in \texttt{GP}s. This work introduces a novel approach to embedding signature kernels within the \texttt{GP} framework. Thus, we are concerned with learning from sequential data in a Bayesian manner by constructing Gaussian processes with the inner product of signatures as covariance functions, continuing on the premise of \cite{kiraly2019kernels}, where a kernel trick was introduced for the inner product of signature features.

The chapter demonstrates the potential of integrating path signature features into \texttt{GP}s, offering a mathematically grounded and practical approach to Bayesian modelling of sequential data. By constructing covariance functions based on the inner product of signatures, the method captures the temporal and structural dependencies inherent in sequential data. The introduction of sparse variational inference further addresses the computational challenges, ensuring scalability even for large datasets of time series data. Empirical evaluations highlight the effectiveness of the approach, showcasing superior performance on time series classification tasks compared to conventional sequence kernels and comparable performance to state-of-the-art frequentist methods. The results underscore the versatility of signature-based \texttt{GP}s, providing a robust framework for combining theoretical rigor with practical applicability in Bayesian sequential data modelling. This integration not only enriches the representational capacity of \texttt{GP}s, but also opens up new venues for applying Bayesian methods to sequential data problems.

\paragraph{Chapter \ref{ch:s2t}}
Deep learning and Gaussian processes have emerged as powerful tools for modelling sequential data, yet both face limitations when processing complex, high-volume or high-dimensional streams. Signature features provide an effective representation of sequences but can be computationally expensive when incorporated into parametric models like neural networks as in \cite{Kidger2019DeepSig}. We introduce \texttt{Seq2Tens}, a novel framework that bridges this gap by leveraging the expressive power of signature-based features while introducing low-rank linear functionals to mitigate computational challenges. This approach allows the representation of sequential data through layers of low-rank transformations, offering a computationally efficient mechanism to extract global sequence information. By stacking multiple such layers, \texttt{Seq2Tens} achieves deep learning-style expressiveness while preserving interpretability and scalability. This chapter presents both the theoretical foundation of \texttt{Seq2Tens} and its practical utility, extending the versatility of signature-based features into new domains of deep learning.

The chapter establishes \texttt{Seq2Tens} as an advancement in the scalable deployment of signature-based features for sequential data modelling. The framework introduces low-rank constraints on linear functionals of signatures, enabling efficient computations directly in the dual space. This idea is further enhanced by the concept of deep layers, which stack multiple low-rank transformations to recover lost expressiveness due to the low-rank constraint in the style of deep learning. Experimental results validate the effectiveness of \texttt{Seq2Tens} across a range of tasks, including time series classification, mortality prediction, and generative modelling, where it consistently outperforms state-of-the-art neural network baselines. The experiments demonstrate how \texttt{Seq2Tens} not only improves computational efficiency but also preserves the expressiveness of signature features, making it a practical and robust solution for modern sequential data challenges. This work paves the way for broader adoption of signature-based methods in deep learning while addressing scalability concerns.

\paragraph{Chapter \ref{ch:g2t}}
Graphs provide a natural representation for structured data, appearing in diverse fields such as social networks, molecular chemistry, and computer vision. A major challenge in graph learning is to design representations that capture not only local neighborhood structures but also long-range dependencies and global characteristics. Building upon the foundational ideas of path signatures and their algebraic properties, The chapter extends their application to graph domains by leveraging hypo-elliptic diffusion processes. This novel approach introduces the hypo-elliptic graph Laplacian, a tensor-valued operator that encodes graph substructures through the evolution of random walks and diffusion dynamics. By integrating this operator into neural network architectures, the chapter proposes a scalable framework for capturing complex interactions in graph-structured data. This innovation bridges the gap between local aggregation and long-range structural encoding in graph neural networks, offering a theoretically grounded and computationally efficient solution for graph learning tasks.

This work can be treated as a follow-up to Chapter \ref{ch:s2t}, where we generalized the signature method using algebraic techniques, and devised scalable neural network architectures for streamed data. We present a significant extension of signature-based methods from sequential data to graph domains, addressing a critical gap in the ability to model long-range interactions and hierarchical structures in graphs by classical graph neural networks. The introduction of the hypo-elliptic graph Laplacian provides a mathematically robust mechanism to encode diffusion dynamics and random walk distributions, enabling the representation of rich graph features. The chapter further adapts low-rank algorithms from the previous chapter to design scalable neural architectures for graph learning, ensuring computational efficiency without sacrificing expressiveness. Empirical evaluations demonstrate state-of-the-art performance on challenging graph tasks compared to classic graph neural networks, such as node classification and graph-level predictions, where traditional methods struggle to capture long-range dependencies. By combining algebraic insights with practical implementation, this work lays the groundwork for future exploration of graph-specific adaptations of signature-based methodologies.

\paragraph{Chapter \ref{ch:rfsf}}
This chapter follows the premise of increased interest in the signature kernel. Signature kernels have emerged as a powerful tool for modelling sequential data, offering theoretical guarantees such as universality and invariance to reparameterization. Further, the benefits of kernelization allow us to compose the signature map with a static kernel enhancing expressiveness and allowing us to circumvent the curse of dimensionality incurred by high-degree tensors. However, their practical application faces significant computational barriers, particularly the quadratic complexity in sequence length and sample size, which limits scalability. We address this limitation by introducing Random Fourier Signature Features, a novel approximation method that combines the algebraic richness of path signatures with the efficiency of random feature maps. By constructing low-dimensional embeddings that approximate the signature kernel, this approach dramatically reduces the computational overhead while preserving key properties of the kernel with high probability. By refining this method with dimensionality reduction techniques, we end up with even more scalable variants capable of handling millions of sequences. This innovation extends the applicability of signature kernels to large-scale machine learning tasks, offering a practical solution for computationally intensive domains.

Thus, we make a contribution to signature kernels by overcoming their scalability challenges, introducing a random feature-based approximation that is both efficient and theoretically robust. The Random Fourier Signature Features approach leverages probabilistic concentration arguments to ensure accurate kernel approximations with high probability. Experimental evaluations demonstrate that this approach achieves performance on par with full-rank signature kernels on moderate datasets while significantly outperforming other random feature methods on large-scale datasets. This chapter not only addresses the computational bottlenecks of signature kernels but also establishes Random Fourier Signature Features as a scalable and versatile tool for modern machine learning, enabling the practical application of signature methods in domains requiring high-throughput data processing.

While working on this project, a general \texttt{Scikit-Learn} compatible Python package has also been produced for all kinds of signature kernel computations on \texttt{GPU} called \texttt{KSig} available at \texttt{\href{https://github.com/tgcsaba/ksig}{https://github.com/tgcsaba/ksig}}, and a preprint that serves as a user guide and documentation to it \cite{toth2025ksig}. This paper is not included as part of this thesis due to space limitations.

\paragraph{Chapter \ref{ch:vfsf}}
Time series forecasting is a crucial task in fields ranging from finance and energy to healthcare and logistics, requiring models that can capture both short-term dynamics and long-term dependencies. Traditional Gaussian processes (\texttt{GP}s) provide a powerful probabilistic framework but struggle with scalability and context adaptability for long time series data. We introduce Recurrent Sparse Spectrum Signature Gaussian Processes (\texttt{RS\textsuperscript{3}GP}), a novel approach that combines the efficiency of low-rank signature features with variational inference in Gaussian processes. \texttt{RS\textsuperscript{3}GP} incorporates a principled forgetting mechanism using a novel decay parameterization, adjusting the context length to focus on more recent time steps. This innovation enables \texttt{RS\textsuperscript{3}GP} to overcome the challenges of traditional \texttt{GP}s and signature-based models, offering a scalable and interpretable solution for time series forecasting across diverse problems.

We present \texttt{RS\textsuperscript{3}GP} as a novel approach that bridges the gap between theoretical rigour and practical applicability in time series forecasting. By leveraging the expressive power of Random Fourier Signature Features and introducing a decay-based forgetting mechanism, \texttt{RS\textsuperscript{3}GP} dynamically adjusts its focus on relevant data segments, making it particularly effective in scenarios with both short-term and long-term dependencies. The model is computationally efficient, processing long sequences in real time through efficient parallelized computations while maintaining a scalable memory footprint. Experimental evaluations highlight its superior performance over traditional \texttt{GP}s and competitive results against deep learning models on synthetic and real-world datasets. This contribution establishes \texttt{RS\textsuperscript{3}GP} as a competitive approach with the state-of-the-art for scalable, probabilistic time series forecasting. Overall, the proposed methodology opens new directions in advancing scalable and interpretable models in time series forecasting.
\begingroup
\chapter{Path Signatures} \label{ch:back}

\section{Tensors, tensor algebras, and their dual} \label{sec:back:tensors}
First, we construct the space that path signatures live in called the tensor algebra. 

\subsection{Tensor products} Here we provide a brief overview of tensor products of vector spaces, which we will use to construct our feature space called the tensor algebra .

\paragraph{Preliminary example} We start with some preliminary examples about tensors.
\begin{example}[Tensors on $\bbR^d$]
If $\bx = (x_1, \ldots, x_d) \in \bbR^d $ and $\by = (y_1, \ldots, y_e) \in \bbR^e $ are two vectors, then their tensor product $\bx \otimes \by$ is defined as the $ (d\times e) $-matrix, or degree-$2$ tensor, with entries $ (\bx \otimes \by)_{i,j} = x_i y_j $. This is also commonly called the outer product of the two vectors. The space $ \bbR^d \otimes \bbR^e $ is defined as the linear span of all degree-$2$ tensors $\bx \otimes \by $ for $\bx \in \bbR^d, \by \in \bbR^e $.
If $\bz \in \bbR^f $ is another vector, then one may form a degree-$3$ tensor $\bx \otimes \by \otimes \bz $ with shape $ (d\times e \times f) $ defined to have entries $ (\bx \otimes \by \otimes \bz)_{i,j,k} = x_iy_jz_k $. The space $\bbR^d \otimes \bbR^e \otimes \bbR^f $ is analogously defined as the linear span of all degree-$3$ tensors $\bx \otimes \by \otimes \bz $ for $ \bx \in \bbR^d, \by \in \bbR^e, \bz \in \bbR^f$, and so forth for higher orders.
\end{example}
\paragraph{Abstract tensor product} The abstract definition of the tensor product, which holds generally for infinite-dimensional vector spaces is given below.
\begin{definition}[Abstract tensor product]
If $U$ and $V$ are (not necessarily finite-dimensional) vector spaces, their tensor product $U \otimes V$ is a vector space together with a bilinear map $\otimes: U \times V \to U \otimes W$, such that any bilinear map $B: U \times V \to W$ factors through $U \otimes V$, i.e.~there exists a unique linear map $\tilde B: U \otimes V \to W$, which satisfies $B = \tilde B \circ \otimes$. Let $\cB_U$ and $\cB_V$ respectively be bases of $U$ and $V$. Then, $U \times V$ is unique up to isomorphism, and its basis is formed by the tensor product of the bases, i.e.~$\cB_{U \otimes V} = \curls{\bu \otimes \bv \in U \otimes V \setgiven \bu \in \cB_U, \bv \in \cB_V}$.
\end{definition}

\begin{example}
    Similarly to the example about $\bbR^d$ above, a coordinate-wise construction of the tensor product is given by identifying $U \otimes V$ with the set of functions from $\cB_U \times \cB_V$ to $\bbR$ that are nonzero only on finitely many elements, which is a vector space by pointwise operations on functions.  Specifically, for $\bu \in \cB_U$ and $\bv \in \cB_V$, we may identify $\bu \otimes \bv$ with the function that takes $1$ on $(\bu, \bv)$ and $0$ on every other element of $\cB_U \times \cB_V$. Then, we can extend this to $U \otimes V$ by asserting the bilinearity of $\otimes$, so that if we denote $\bx = \sum_{\bu \in \cB_U} \alpha_\bu \bu$ and $\by = \sum_{\bv \in \cB_V} \beta_\bv \bv$, we have that
\begin{align}
    \bx \otimes \by = \pars{\sum_{\bu \in \cB_U} \alpha_{\bu}\bu} \otimes \pars{\sum_{\bv \in \cB_V} \beta_{\bv} \bv} = \sum_{\bu \in \cB_U} \sum_{\bv \in \cB_V} \alpha_{\bu} \beta_{\bv} \bu \otimes \bv,
\end{align}
where finitely many coefficients are nonzero in $\curls{\alpha_\bu \in \bbR \setgiven \bu \in \cB_U}$ and $\curls{\beta_\bv \in \bbR \setgiven \bv \in \cB_V}$.
\end{example}

We refer the reader to \cite{ryan2013introduction, yokonuma1992tensor, lang2002algebra} for further details about abstract tensors.

\paragraph{Tensor products of Hilbert spaces}
It is often that one comes across the definition of the tensor product formulated in the category of vector spaces as an abstract vector space such that any bilinear map factors through it as discussed in the previous paragraph. This can then be endowed with an inner product and completed to get a tensor product in the category of Hilbert spaces. However, there also exist explicit models of tensor products of Hilbert spaces. Hence, we introduce a concrete model of the Hilbert space tensor product originally proposed by \cite{murray1936rings}.

\begin{definition}
Let $\cH_1, \dots, \cH_m$ be Hilbert spaces. To each element $(h_1, \dots, h_m) \in \cH_1 \times \cdots \cH_m$, associate the multi-linear operator $h_1 \otimes \cdots \otimes h_m: \cH_1 \times \cdots \times \cH_m \to \bbR$ defined for each $(f_1, \dots, f_m) \in \cH_1 \times \cdots \times \cH_m$ by
\begin{align}
    (h_1 \otimes \cdots \otimes h_m)(f_1, \dots, f_m) = \inner{h_1}{f_1}_{\cH_1} \cdots \inner{h_m}{f_m}_{\cH_m}.
\end{align}
Take the linear span of all such multi-linear operators to build the space
\begin{align}
    \cH_1 \otimes^\p  \cdots \otimes^\p  \cH_m = \Span\curls{h_1 \otimes \cdots \otimes h_m \setgiven h_1 \in \cH_1, \dots, h_m \in \cH_m},
\end{align}
and endow $\cH_1 \otimes^\p  \cdots \otimes^\p  \cH_m$ with an inner product via
\begin{align} \label{eq:back:inner_tensor}
    \inner{h_1 \otimes \cdots h_m}{f_1 \otimes \cdots \otimes f_m}_{\cH_1 \otimes^\p  \cdots \otimes^\p  \cH_m} = \inner{h_1}{f_1}_{\cH_1} \cdots \inner{h_m}{f_m}_{\cH_m}
\end{align}
for all $h_1, f_1 \in \cH_1, \dots, h_m, f_m \in \cH_m$, and extend by linearity to $\cH_1 \otimes^\p  \cdots \otimes^\p  \cH_m$. Taking the topological completion of this space under this inner product gives a Hilbert space denoted by $\cH_1 \otimes  \cdots \otimes \cH_m$ called the tensor product of the Hilbert spaces $\cH_1, \dots, \cH_m$.
\end{definition}

\subsection{Tensor algebras} \label{sec:back:tensalgs}
Now that we have discussed how to take tensor products of vector (Hilbert) spaces, we show how to embed a linear space $V$ into a much bigger linear space $\T{V}$, which is also an associative algebra\footnote{An algebra $A$ is a vector space $A$, where one can multiply elements together, i.e.~$\ba \bb \in A$ for $\ba, \bb \in A$.} using a so-called free construction. 
Since the tensor product is associative, we can unambiguously take tensor powers of the vector space $V$, and we denote $V^{\otimes m} = V \otimes \cdots \otimes V$. Note that if $\cB_V$ is a basis of $V$, then the basis of $V^{\otimes m}$ is given by all elements of the form
\begin{align}
    \bv_1 \otimes \cdots \otimes \bv_m \quad \text{for} \quad \bv_1, \dots, \bv_m \in \cB_V.
\end{align}

\begin{definition}[Extended tensor algebra]
We define the extended tensor algebra as formal series of tensors over a vector space $V$ indexed by their degree $m \in \bbN$,
\begin{align} \label{eq:back:free_alg}
   \T{V} = \prod_{m \geq 0} V^{\otimes m} = \curls{(t_0, \bt_1, \bt_2, \dots) \setgiven \bt_m \in V^{\otimes m} \text{ for } m \in \bbN},
\end{align}
where by convention $V^{\otimes 0} = \bbR$ is a scalar.
$\T{V}$ is again a vector space with vector addition and scalar multiplication defined coordinate-wise as
\begin{align}
    \bs + \bt = \pars{\bs_m + \bt_m}_{m \geq 0}, \quad \lambda \bs = \pars{\lambda \bs_m}_{m \geq 0}.
\end{align}
$V$ is a linear subspace of $\T{V}$ by identifying $\bv \in V$ as $(0, \bv, \mathbf{0}, \mathbf{0}, \dots) \in  \T{V}$. 
Further, $\T{V}$ is also an associative unital algebra with unit $\mathbf{1} = (1, \mathbf{0}, \dots) \in \T{V}$. It is endowed with a (noncommutative\footnote{Noncommutative means that $\ba \bb \neq \bb \ba$ in general for elements $\ba, \bb \in V$ of the algebra.} when $\dim V > 1$) product defined for $\bs, \bt \in \T{V}$
\begin{align} \label{eq:back:alg_mult}
    \bs \cdot \bt = \pars{\sum_{i=0}^m \bs_i \otimes \bt_{m-i}}_{m \geq 0} \in  \T{V}.
\end{align}
An element $\bt = (t_0, \bt_1, \dots) \in \T{V}$ is invertible if and only if $t_0 \neq 0$, with inverse
\begin{align}
    \bt^{-1} = \frac{1}{t_0} \sum_{n \geq 0} \pars{\mathbf{1} - \frac{\bt}{t_0}}^n
\end{align}
\end{definition}
\begin{definition}[Tensor algebra]
    Another commonly encountered formulation of the tensor algebra consists of sequences of tensors with only finitely many non-zero coefficients, which we refer to simply as the tensor algebra. It is defined as
    \begin{align} 
        \TV{V} = \bigoplus_{m \geq 0} V^{\otimes m} = \curls{(t_0, \bt_1, \bt_2, \dots, \bt_M, \mathbf{0}, \mathbf{0}, \dots) \setgiven \bt_m \in V^{\otimes m} \text{ for } m \in \bbN, M \in \bbN}.
    \end{align}
    All algebraic operations carry over from $\T{V}$ to $\TV{V}$. Hence, we may equivalently define $\TV{V}$ as the subalgebra of $\T{V}$ where all but finitely many projections are zero.
\end{definition}
For example, take $V = \bbR^d$, then the degree-$0$ term is a scalar, the degree-$1$ term is a vector, the degree-$2$ term is matrix, the degree-$3$ term is a $d \times d \times d$ array, and so forth.
This process of turning $V$ into an algebra $\TV{V}$ or $\T{V}$ is a free construction; informally this means that this is the minimal structure that turns $V$ into an algebra; for more details about tensor algebras, see \cite{yokonuma1992tensor, lang2002algebra, reutenauer2003free}.

\begin{definition}[Tensor algebra over a Hilbert space]
Now, let us consider the case when $V = \cH$ is a Hilbert space. 
We now define for $\bs, \bt \in \bigoplus_{m \geq 0} \cH^{\otimes m}$ their inner product as
\begin{align}\label{ref:hilbert inner product}
    \inner{\bs}{\bt}_{\bigoplus_{m \geq 0} \cH^{\otimes m}} = \sum_{m \geq 0} \inner{\bs_m}{\bt_m}_{\cH^{\otimes m}},
\end{align}
where the inner product $\inner{\bs_m}{\bt_m}_{\cH^{\otimes m}}$ on $\cH^{\otimes m}$ is as in \eqref{eq:back:inner_tensor}.
Finally, the topological completion of $\bigoplus_{m \geq 0} \cH^{\otimes m}$ in this inner product gives a Hilbert space $\TV{\cH}$, which we call the tensor algebra over the Hilbert space $\cH$, which can be equivalently defined as
\begin{align} \label{eq:back:hilt_def}
    \TV{\cH} = \{ \bt = (t_0,\bt_1,\bt_2,\ldots): \bt_m \in \cH^{\otimes m}, \,\langle \bt, \bt \rangle_{\TV{\cH}} < \infty\}.
\end{align}
\end{definition}

\subsection{Duals of tensors}
In this section, we study linear functionals acting on tensor algebras over vector spaces following the exposition in \cite{lyons2007differential}. Recall that the algebraic dual $V^\star$ for a vector space $V$ denotes the set of all linear functions from $V$ to $\bbR$, i.e. $\L(V, \bbR)$. When $V$ is finite-dimensional, $V$ and $V^\star$ are isomorphic in a basis dependent way, and in particular $\dim V = \dim V^\star$. Let $\curls{\be_1, \dots, \be_d}$ be a basis of $V$. This induces a dual basis $\{\be_1^\star, \dots \be_d^\star\}$ of $V^\star$ uniquely defined by
\begin{align}
    \inner{\be_i^\star}{\be_j} = \delta_{i,j} = \begin{cases}
        1 &\text{ if } i = j,\\
        0 &\text{ otherwise}.
    \end{cases}
\end{align}
Similarly, we know that the basis of $V^{\otimes m}$ is given by
\begin{align}
    \curls{\be_{i_1} \otimes \dots \otimes \be_{i_m} \setgiven i_1, \dots, i_m \in [d]},
\end{align}
and the basis of $(V^\star)^{\otimes m}$ is given by the elements
\begin{align}
    \curls{\be_{i_1}^\star \otimes \cdots \otimes \be_{i_m}^\star \setgiven i_1, \dots, i_m \in [d]}.
\end{align}
Now, note that $\pars{V^{\otimes m}}^\star \cong (V^\star)^{\otimes m}$, where the isomorphism is given by identifying the basis of $\pars{V^{\otimes m}}^\star$ with the basis of $(V^\star)^{\otimes m}$ such that
\begin{align}
    \be_{i_1}^\star \otimes \cdots \otimes \be_{i_m}^\star \mapsto \pars{\be_{i_1} \otimes \cdots \otimes \be_{i_m}}^\star, 
\end{align}
and this isomorphism is natural, that is, independent of the choice of basis for V, i.e.
\begin{align}
    \inner{\be_{i_1}^\star \otimes \cdots \otimes \be_{i_m}^\star}{\be_{j_1} \otimes \cdots \otimes \be_{j_m}} = \delta_{i_1, j_1} \cdots \delta_{i_m, j_m}.
\end{align}
The linear action of $(E^\star)^{\otimes m}$ on $E^{\otimes m}$ extends naturally to a linear mapping $(E^\star)^{\otimes m} \to \T{V}^\star$ defined for a $\bt = (t_0, \bt_1, \dots) \in \T{V}$ by
\begin{align}
    \inner{\be_{i_1}^\star \otimes \cdots \otimes \be_{i_m}^\star}{\bt} = \inner{\be_{i_1}^\star \otimes \cdots \otimes \be_{i_m}^\star}{\bt_m}.
\end{align}
In essence, if we think of $\bt$ as a non-commuting power series in the letters $\curls{i_1, \dots i_d}$, then this action picks up the coefficient of the word $I = (i_1, \dots i_m) \in [d]^{m}$. Analogously, by taking linear combinations of such linear functionals, the action of $\TV{V^\star}$ can be extended to a linear mapping $\TV{V^\star} \to \T{V}^\star$, and with this identification, $\TV{V^\star}$ can be identified as a strict subspace of $\T{V}^\star$. It is important that $\TV{V^\star}$ contains sequences of tensors with finitely many nonzero coefficients, since this allows to consider its action on $\T{V}$ without any issues of convergence.

Next, we introduce a product on $\TV{V^\star}$ which turns it into a commutative algebra.

\begin{definition}[Shuffle product \cite{lyons_differential_2007}]
    Let $\curls{\be_1^\star, \dots, \be_d^\star}$ be a basis of $V^\star$, and let $I = (i_1, \dots, i_n) \in [d]^n$, $J = (i_{n+1}, \dots, i_{n+m}) \in [d]^m$ be arbitrary multi-indices. The shuffle product of two elementary tensors $\be_I^\star \in (V^\star)^{\otimes n}$ and $\be_J^\star \in (V^\star)^{\otimes m}$ is defined as
    \begin{align}
        \be_I^\star \shuffle \be_J^\star \sum_{\sigma \in \sh(n, m)} \be^\star_{i_{\sigma^{-1}(1)}} \otimes \cdots \otimes \be^\star_{i_{\sigma^{-1}(n+m)}},
    \end{align}
    where the summation is taken over all $(n, m)$-shuffles, that is over the subset of all permutations of $\curls{1, 2, \dots, n + m}$, that satisfy $\sigma(1) < \dots < \sigma(n)$ and $\sigma(n+1) < \dots < \sigma(n+m)$.
    It can be checked that this construction does not depend on the choice of basis, and turns $\TV{V^\star}$ into a commutative algebra by considering its linear extension.
\end{definition}

\subsection{Motivating example: polynomial features}
Finally, before introducing path signatures, we consider a motivating example in the spirit of statistical learning, and recall how polynomials in the coordinates of a vector $\bx = (x_1, \dots, x_d) \in \bbR^d$ can be used to approximate continuous functions on compact sets, a property known as universality in the machine learning literature.

\begin{definition}[Tensor exponential]
    The tensor exponential is defined as 
    \begin{align} \label{eq:back:tensor_exp}
        \exp_{\otimes}: V \to \T{V}, \quad \exp_{\otimes}(\bx) = \pars{1, \bx, \frac{\bx^{\otimes 2}}{2!}, \dots, \frac{\bx^{\otimes m}}{m!}, \dots}.
    \end{align}
    That is, the map $\exp_\otimes$ takes a vector $\bx \in \bbR^d$ and maps it to the infinite sequence of tensors, which has as its degree-$m$ term $\exp_{\otimes_m}(\bx) = {\bx^{\otimes m}} / {m!}$.
\end{definition}

Let $V = \bbR^d$. We can build a functions $f: V \to \bbR$ by taking linear functionals of $\exp_{\otimes}$, that is for $\ell \in \TV{V^\star}$, we may consider the composition
    \begin{align}\label{eq:back:pol feature}
    f: \bx \mapsto \ell \odot \exp_{\otimes} (\bx) = \inner{\ell}{\exp_\otimes(\bx)}.
    \end{align}
Hence, any choice of $\ell \in \TV{V^\star}$ induces a function $f: V \to \bbR$, which is a linear function in terms of $\exp_\otimes$. It might be helpful to spell out the definition of $f$ in coordinates. By the definition of $\ell = (\ell_0, \ell_1, \dots, \ell_M, \mathbf{0}, \dots) \in \TV{V^\star}$, we have for some $M \in \bbN$ 
\begin{align}
    \inner{\ell}{\exp_\otimes(x)} = \sum_{m=0}^M \frac{1}{m!} \inner{\ell_m}{\bx^{\otimes m}}.
\end{align}
Spelled out in coordinates, $\bx=(x_1,\ldots,x_d)$ and $\ell_m = \sum_{I \in [d]^m} a_I \be_I^\star$, this reads as 
\begin{align}
    \inner{\ell_m}{\bx^{\otimes m}} = \sum_{i_1,\ldots,i_m \in [d]} a_{i_1,\dots,i_m} \be^\star_{i_1,\dots,i_m}(\bx) = \sum_{i_1,\dots,i_m \in [d]} a_{i_1, \dots, i_m} x_{i_1}\cdots x_{i_m}.
\end{align}
Thus putting this together, formulated in coordinates $f$ is written as 
\begin{align}
    f(\bx) &= a_0+ a_1 x_1 + \cdots a_d x_d \\
    &+ \frac{1}{2!}\pars{a_{1,1} x_1^2 + a_{1,2} x_1 x_2 + \cdots + a_{d,d} x_d^2} \\
    &+ \vdots \\
    &+ \frac{1}{M!}\pars{a_{1,\ldots,1} x_1^M +\cdots +a_{d,\ldots,d} x_d^M}
\end{align}
which is a polynomial of degree-$m$. The Weierstrass approximation theorem guarantees that continuous functions can be uniformly approximated over compact sets in terms of polynomials. This property makes polynomials a classical choice for machine learning, making it possible to learn from vector-valued data. Due to this approximation property, fit datasets well, but they can exhibit arbitrarily large oscillations between the interpolation points, and they are often replaced by other nonlinearities \cite{rasmussen2006}.

\section{Path signatures}\label{sec:back:pathsig} 

In this part, let $(V, \norm{\cdot}_V)$ be a finite-dimensional Banach space. In the following, we define the signature of a path and shed light on its properties, which make it a viable feature map for learning from streamed data. Path signatures can be seen as a natural generalization of the tensor exponential map, but instead of mapping a point in $V$ to a sequence of tensors, its domain is the space of paths. 
It generalizes many of the nice properties of polynomial features such as universality and simultaneously gives the option to ignore time-parameterization without an explicit search over all possible time changes (like in dynamic time warping approaches).

\subsection{Definition} \label{appendix:sig_properties}
\begin{definition}
    A continuous path $\bx: [0, T] \to V$ is of bounded variation if
    \begin{align}
        \norm{\bx}_\onevar = \sup_{D \subset [0, T]} \sum_{t_i \in D} \norm{\bx_{t_{i}} - \bx_{t_{i-1}}}_V < \infty,
    \end{align}
    where $D = \{0 = t_0 < t_1 < \cdots < t_n = T \setgiven n \in \bbN \}$ runs over all finite partitions of $[0, T]$.
\end{definition}
We denote the set of bounded variation paths on some arbitrary time horizon by
\begin{align*}
    \Paths(V) = \curls{\bx \in C([0,T], V): T \in \bbR_+, \bx_0 =0, \| \bx \|_\onevar<\infty },
\end{align*}

To ease the notation, let us define the continuous $m$-simplex:
\begin{align}
    \Delta_m([0, T]) = \curls{ (t_1, \ldots, t_m) \in \bbR^m \setgiven 0 < t_1 < \cdots < t_m < T },
\end{align}
and for a path $(\bx_t)_{t \in [0, T]} \in \Paths(V)$ denote its time horizon by $\hor{\bx} = T$.

\begin{definition}
A classic way to obtain a structured and hierarchical description of a path $\bx \in \Paths(V)$ is by computing a sequence of iterated integrals called the path signature given as tensors of increasing degrees such that the degree-$m$ object is
\begin{align} \label{eq:back:pathsig}
    \S_m(\bx) = \int_{\bt \in \Delta_m([0, \hor{x}])} \d\bx_{t_1} \otimes \cdots \otimes \d\bx_{t_m}
    = \int_{\bt \in \Delta_m([0, \hor{x}])} \dot \bx_{t_1} \otimes \cdots \otimes\dot \bx_{t_m} \d t_1\cdots \d t_m,
\end{align}
where the integrals are defined in the Riemann--Stieltjes sense. Formally, we refer to the map that takes a path to its iterated integrals as the path signature map, and it is defined as
\begin{align}
    \S: \Paths(V) \to \T{V}, \quad \S(\bx) = \pars{1, \S_1(\bx), \S_2(\bx), \ldots}.
\end{align}
\end{definition}

For a path $(\bx_t)_t = (x^1_{t}, \dots, x^d_{t})_t \in \Paths(V)$ that evolves in $V$, one can spell this out in coordinates: the $m$-th signature level $\S_m(\bx) \in V^{\otimes m}$ is the tensor that has as its $(i_1,\ldots, i_m) \in [d]^m$--th coordinate entry the real number computed by the Riemann--Stieltjes integral
\begin{align*} \label{eq:back:sigcoords}
    \S_m(\bx)_{i_1,\dots,i_m} = \int_{\bt \in \Delta_m([0, \hor{\bx}])} \d x^{i_1}_{t_1} \cdots \d x^{i_m}_{t_m} 
    = \int_{\bt \in \Delta_m([0, \hor{\bx}])} \dot x^{i_1}_{t_1} \cdots \dot x^{i_m}_{t_m} \d t_1\cdots \d t_m
\end{align*}
where $\dot x^i_t = \frac{\d x^i_t}{\d t}$ denotes the time derivative.  


\begin{proposition}[Factorial decay \cite{lyons2007differential}]
    Let $\bx \in \Paths(V)$ be a bounded variation path. We have for the level-$m$ path signature feature for $m \in \bbN$ that
    \begin{align}
        \norm{\S_m(\bx)}_{V^{\otimes m}} \leq \frac{\norm{\bx}_{\onevar}^m}{m!}.
    \end{align}
\end{proposition}

Next, we show that there is a precise connection between the algebra structure of $\T{\bbR^d}$ and operations on paths. Hence, we first define the concatenation operation on paths, and then state how this affects the signature.

\begin{definition}[Path concatenation]
    Let $\bx, \by \in \Paths(V)$ such that $\bx: [0, S] \to V$ and $\by: [0, T] \to V$. Their concatenation denoted by $\bx \star \by$ is defined for $t \in [0, S+T]$ as
    \begin{align}
        (\bx \star \by)_t = \begin{cases}
            \bx_t \quad&\text{if } t \in [0, S]\\
            \bx_S - \by_0 + \by_{t-S} &\text{if } t \in [S, S+T].
        \end{cases}
    \end{align}
\end{definition}

\begin{proposition}[Chen identity \cite{lyons2007differential}] \label{prop:chen}
    Let $\bx, \by \in \Paths(V)$. Then, we have that
    \begin{align}
        \S(\bx \star \by) = \S(\bx) \cdot \S(\by).
    \end{align}
\end{proposition}

Similarly, path reversal can also be represented in terms of the signature.
\begin{proposition}[Path reversal \cite{lyons2007differential}]
    Let $\bx \in \Paths(V)$ defined on $[0, T]$ taking values in $V$. Denote by $(\overleftarrow{\bx}_t)_{t \in [0, T]} = (\bx_{T - t})_{t \in [0, T]}$. Then, we have that
    \begin{align}
        \S(\overleftarrow{\bx}) = \S^{-1}(\bx).
    \end{align}
    In particular, the range of $\S: \Paths(V) \to \T{V}$ is a group.
\end{proposition}

\subsection{Properties} \label{subsec:back:props}

\paragraph{Invariance under tree-like equivalence} 
A classic result going back to Chen~\cite{chen-58} shows that the map $\bx \mapsto \S(\bx)$ is injective up to tree-like equivalence. 
Loosely speaking, tree-like equivalence is from a purely analytic point of view more natural to work with than reparameterization since tree-like equivalence between paths is analogous to Lebesgue almost sure equivalence between sets.
Howevever, from a practical point of view, the difference between paths that are tree-like equivalent and paths that differ by a reparameterization is negligible and we invite the reader to use them as synonyms throughout.
Nevertheless, we give the precise definition below and refer the interested reader to~\cite{MR2630037} for a detailed discussion. 
\begin{definition}
A bounded variation path $\bx \in \Paths(V)$ is \emph{tree-like} if there exists a continuous function $h:[0, T] \rightarrow [0,\infty)$ such that $h(0)= h(T)=0$ and for all $s < t$
\begin{align}
\norm{\bx_t - \bx(s)}_V \le h(s) + h(t) - 2 \inf_{u \in [s,t]} h(u). 
\end{align}
\end{definition}
\begin{proposition}
Let $\bx, \by \in \Paths(V)$ be two paths of bounded variation. Then 
\begin{align}
\S(\bx) = \S(\by) 
\end{align}
if and only if $\bx \star \overleftarrow \by $ is tree-like where $\star$ denotes path concatenation and $\overleftarrow{\by}$ denotes path reversal.
\end{proposition}
In particular for $\bx, \by \in \Paths(V)$ and $\ell \in \TV{V^\star}$, for any functions of the form $f(\bx) = \langle \ell, \S(\bx) \rangle$,
$f(\bx) = f(\by)$ if and only if $\bx$ and $\by$ differ by parameterization (strictly speaking, by a tree-like equivalence).
This ability to factor out time-invariance can be very powerful since the space of all possible parameterizations is huge and we never make an explicit search over possible time changes like in the DTW distance.

However, we are often interested in time parameterization, and the path signature can be made sensitive to it by adding an extra coordinate to the path $\bx \in \Paths(V)$ before computing the signature features of this enhanced path $\hat{\bx}_t = (t, \bx_t) \in \bbR \oplus V$. In other words, we make the parameterization part of the trajectory. The augmented path $\hat{\bx}$ is never tree-like since the first coordinate is strictly increasing. Hence, the map $\Paths(V) \ni \bx \mapsto \S(\hat\bx) \in \T{\bbR \oplus V}$ is injective.

\paragraph{Universality} 
One of the most attractive properties of the monomial feature map $\bx \rightarrow \texp(\bx)$ for vectors $\bx \in V$, is that any continuous function $f: V \rightarrow \bbR$ can be uniformly approximated on compact sets as a linear functional of $\texp$, that is $f(\bx) \approx \langle \ell, \texp(\bx) \rangle$ for some $\ell \in \TV{V^\star}$.
The reason is that linear combinations of monomials (polynomials) form an algebra and the Stone--Weierstrass theorem applies.
Such approximation properties of feature maps are usually referred to as ``universality'' in the ML literature.
One of the most attractive properties of the signature feature map $\bx \mapsto \S(\bx)$ for paths $\bx \in \Paths(V)$ is that a universality result holds. A central result in this direction is the shuffle identity for path signatures, which we state below.

\begin{theorem}[Shuffle identity \cite{lyons2007differential}] Let $\bx \in \Paths(V)$, and $\ell_1, \ell_2 \in \TV{V^\star}$. Then,
\begin{align}
    \inner{\ell_1}{\S(\bx)}\inner{\ell_2}{S(\bx)} = \inner{\ell_1 \shuffle \ell_2}{\S(\bx)}
\end{align}
\end{theorem}

In particular, the range of the path signature forms an algebra of real-valued functons. This means that the Stone-Weierstrass theorem applies, and linear functionals acting on the signature allow to approximate continuous functions on paths.

\begin{proposition}
For every continuous $f: \Paths(V) \rightarrow V$, $\cK\subset \Paths(V)$ compact, $\epsilon>0$ there exists a $\ell  \in \TV{V^\star}$ such that 
\begin{align}
\label{eq:back:univ}
\textstyle{\sup_{\bx \in \cK} | f(\bx) - \langle \ell, \S(\hat\bx) \rangle | < \epsilon.}
\end{align}
The analogous result holds for paths $\bx \in \Paths(V)$ without a time coordinate by replacing the domain of paths by equivalence classes of paths under tree-like equivalence. 
For a proof and many extensions, see~\cite{chevyrev2022signature}.
\end{proposition}

Informally, iterated integrals of paths can be seen as a generalization of classical monomials and from this perspective, the approximation \eqref{eq:back:univ} can be regarded as the extension of classic polynomial regression to path-valued data.
Thus at least informally it is not surprising, that vanilla signature features suffer from similar drawbacks as classic monomial features; for example, if classic monomials are replaced by other nonlinearities this often drastically improves the approximations; see Section \ref{sec:back:sigkernels} for kernelized methods.

\subsection{Examples} \label{subsec:back:examples}
\paragraph{Linear paths}
Consider the path $\bx:[0,1] \rightarrow V$ that just runs along a straight line $\bx: t \mapsto t\bv$, where $\bv \in \bbR^d$ is a given vector.
Plugging this into the definition of the path signature~\eqref{eq:back:pathsig}, we get by a direct calculation that 
\begin{align} \label{eq:back:lin_path}
\S_m(\bx) = \frac{\bv^{\otimes m}}{m!} \in V^{\otimes m}, \quad \S(\bx) = \texp(\bv) \in \T{V}
\end{align}

We see that for this special case the that $\bx$ is determined by its increment $\bx_1 - \bx_0 = \bv$, the path signature $\S(\bx)$ equals the monomial feature map $\texp(\bv)$ of the total increment $\bv = \bx_1 - \bx_0$. This shows how in this special case the signature reduces to monomial features, and this is one of the many reasons why signatures are regarded as ``polynomials of paths''. 

\paragraph{Piecewise linear paths}
In general, these integrals need to be computed by standard integration techniques, but for a piecewise linear path $\bx: [0,T] \to V$, that is partitioned into $L$ disjoint intervals, $[0,T] = \bigsqcup_{i=1}^{L} [t_{i-1},t_{i}]$, and $\bx$ is piecewise linear on each of these pieces with increments $\bx_{t_i} - \bx_{t_{i-1}} = \bv_i$ for some $\bv_i \in V$ for $i \in [L]$, these iterated integrals can be computed by Chen's identity (Proposition \ref{prop:chen}) via 
\begin{align} \label{eq:back:pathsiglin_chen}
    \S(\bx) = \S(\bx\vert_{[t_0, t_1]}) \cdot \S(\bx\vert_{[t_1, t_2]}) \cdots \S(\bx\vert_{[t_{L-1}, t_L]}) = \texp(\bv_1) \cdot \texp(\bv_2) \cdots \texp(\bv_L),
\end{align}
where $\bx\vert_{[t_i, t_{i+1}]}$ denotes the restriction of $\bx$ to $[t_i, t_{i+1}]$.
Performing the algebra multiplication from \eqref{eq:back:alg_mult}, we get that the degree-$m$ term reduces to an iterated sum of the form
\begin{align} \label{eq:back:pathsiglin_explicit}
    \S_m(\bx) = \sum_{1 \leq i_1 \leq \cdots \leq i_m \leq L} \frac{1}{\bi!} \bv_{i_1} \otimes \cdots \otimes \bv_{i_m} \in V^{\otimes m}, 
\end{align}
where the coefficient $\bi!$ is such that if there are $k \in [m]$ distinct entries among $i_1, \dots, i_m$ with number of repetitions $p_1, \dots, p_k$, then $\bi! = {p_1! \cdots p_k!}$.

\subsection{Signature kernels} \label{sec:back:sigkernels}
\paragraph{Kernel trick for path signatures}
Signature kernels provide a means to compare paths through their signatures by defining inner products in the tensor algebra. Given their structure, they offer a powerful similarity measure for sequences of potentially different lengths, making them particularly suitable for sequential data analysis. The signature kernel inherently captures the algebraic and geometric structure of the paths by inheriting the powerful properties of path signatures, such as invariance to reparameterization and universality, see \cite[Thm.~2.2.4]{cass2024lecture} for a proof that the signature kernel is universal in the kernel learning sense.

Next, we will consider the setting when $\cH$ is a Hilbert space, and our bounded variation paths evolve in this space, that is, $\Paths(\cH)$. Then, the signature kernel $\sigkernel: \Paths(\cH) \times \Paths(\cH) \to \bbR$ is defined for paths $\bx, \by \in \Paths(\cH)$ by taking the inner product of signatures
\begin{align}
    \sigkernel(\bx, \by) = \inner{\S(\bx)}{\S(\by)}_{\TV{\cH}} = \sum_{m=0}^\infty \langle \S_m(x), \S_m(y) \rangle_{\cH^{\otimes m}} 
\end{align}
where $\inner{\cdot}{\cdot}_{\TV{\cH}}$ denotes the inner product in the tensor algebra $\TV{\cH}$ over the Hilbert space $\TV{\cH}$, which itself is indeed a Hilbert space as discussed in Section \ref{sec:back:tensalgs}. In practice, one often introduces a truncation factor $M \in \bbN \cup \{\infty\}$, and considers 
\begin{align}
	\sigkernel[\leq M](x, y) = \sum_{m=0}^M \langle \S_m(x), \S_m(y) \rangle_{\cH^{\otimes m}},
\end{align}
and with some abuse of notation we may omit the truncation superscript, that is, when talking about the signature kernel $\sigkernel$, we may mean the truncated version.

First, \cite{kiraly2019kernels} developed an algorithm for finite $M$, and later \cite{salvi2021signature} formulated a \texttt{PDE} to solve for infinite $M$. The kernel trick of \cite{kiraly2019kernels} is written on the level-$m$ for $\bx, \by \in \Paths(\cH)$ as
\begin{align} \label{eq:back:pathsigkernel_m}
    \sigkernel[m](\bx, \by) = \int_{\substack{\bs \in \Delta_m([0, \hor{\bx}])\\\bt \in \Delta_m([0, \hor{\by}])}} \inner{\d \bx_{s_1}}{\d \by_{t_1}}_\cH \cdots \inner{\d \bx_{s_m}}{\d\by_{t_m}}_\cH.
\end{align}
This allows to compute $\sigkernel$ solely by inner product evaluation on $\cH$.

\paragraph{Lifting paths by static kernels}
Since the kernel trick introduced above allows to compute the inner product of path signatures only by inner product evaluation, it is often beneficial to preprocess paths by lifting them to a higher-dimensional space, which can enhance the richness of the information captured by the signature. One such tool for this is the method of reproducing kernel Hilbert spaces (\RKHS{}), which has become wide-spread in machine learning, since it allows for nonlinear learning in infinite-dimensional spaces using the kernel trick \cite{berlinet2011reproducing}.

Now let $\cX$ be a set, which will be our data space where our paths take values. Further, let $\kernel: \cX \times \cX \to \bbR$ be a so-called static kernel on $\cX$ with \RKHS{} $\Hil$. Let us denote for an $\bx \in \cX$, $\kernel_\bx = \kernel(\bx, \cdot) \in \Hil$ the reproducing kernel lift of $\bx$. Analogously, for a path $(\bx_t)_{t \in [0, T]} \in \Paths(\bx)$ denote by $\kernel_\bx = (\kernel_{\bx_t})_{t \in [0, T]} \in \Paths(\Hil)$ the path $\bx$ lifted into the \RKHS{} of $\kernel$, $\Hil$. Then, \eqref{eq:back:pathsigkernel_m} can formally be written for the lifts of $\bx, \by \in \Paths(\cX)$ as
\begin{align} \label{eq:back:pathsigkernel_m_2}
    \sigkernel[m](\kernel_{\bx}, \kernel_\by) &= \int_{\substack{\Delta_m([0, \hor{\bx}])\\\Delta_m([0, \hor{\by}])}} \inner{\d \kernel_{\bx_{s_1}}}{\d \kernel_{\by_{t_1}}}_{\Hil} \cdots \inner{\d \kernel_{\bx_{s_m}}}{\d\kernel_{\by_{t_m}}}_{\Hil} \\
    &= \int_{\substack{\Delta_m([0, \hor{\bx}])\\\Delta_m([0, \hor{\by}])}} \d \kappa(s_1, t_1), \cdots \d\kappa(s_m, t_m),
\end{align}
where $\kappa([s, t] \times [u, v]) = \kernel(\bx_t, \by_v) - \kernel(\bx_t, \by_u) - \kernel(\bx_s, \by_v) + \kernel(\bx_s, \by_u)$ is a signed Borel measure on $[0, \hor{\bx}] \times [0, \hor{\by}]$ for fixed $\bx, \by \in \Paths(\cX)$, see \cite[Thm.~2]{kiraly2019kernels} for a proof. This allows to compute the signature kernel for paths lifted into reproducing kernel Hilbert spaces, that significantly enhances the expressiveness of the signature kernel. Note that $\cH_k$ is a genuine infinite-dimensional space for non-degenerate kernels $\kernel$, for which there are various choices, see \cite[Ch.~4]{rasmussen2006}, e.g. the \RBF{} or Matérn family of kernels. Abstractly, we can think of the signature kernel as taking a static kernel $\kernel$ on the state space $\cX$ and turning it into a kernel $\sigkernel$ for paths evolving in that state space. In the following, we will treat the static kernel $\kernel$ as a hyperparameter of $\sigkernel$, and we do not make the composition explicit, i.e.~for $\bx, \by \in \Paths(\cX)$, we will write $\sigkernel(\bx, \by)$ and implicitly mean that $\bx$ and $\by$ are first lifted into an \RKHS{} before taking signatures. The standard formulation of signatures when $\cX = \cE$ is a Euclidean space can be recovered by using the linear kernel $\kernel(\bx, \by) = \inner{\bx}{\by}_\cE$ as the static kernel lift for $\bx, \by \in \cE$.




\subsection{Rough paths}
    So far we assumed that our input domain consists of paths of bounded variation paths, but in the real-world the evolution of quantities is often subject to noise, e.g.~a classical model in physics and engineering is $\bx_t = \ba_t + B_t$, where $\ba$ is a bounded variation path, but $B$ is a Brownian sample path. Since Brownian sample paths are not of bounded variation, $\bx$ is not of bounded variation.
    However, the iterated integrals of $\bx$ can still be defined as above, but one has to replace the iterated Riemann--Stieltjes integrals by Ito--Stratonovich integrals in the definition of $\S(\bx)$.
    Even rougher trajectories such as fractional Brownian motion and non-Markovian processes can be handled that way with so-called rough path integrals. 
    Rough path theory provides a systematic study that comes with a rich toolbox combining analytic and algebraic estimates, rich enough to cover the trajectories of large classes of stochastic processes; see \cite{lyons2007differential,lyons2014rough,friz2020course}.	
\endgroup

\subsection{Discrete-time setting} \label{sec:back:discrete}
In practice, we often do not have access to paths, but are only given discrete-time observations that lie in some state space $\cX$, that is, we work with the space of finite length sequences
\begin{align}
    \Seq(\cX) = \curls{\bx = (\bx_0, \dots, \bx_L) \setgiven \bx_0, \dots, \bx_L \in \cX, \bx_0 = \mathbf{0}, L \in \bbN},
\end{align}
where we denote the ``effective'' length of a sequence $\bx = (\bx_0, \dots, \bx_L) \in \Seq(\cX)$ by $\len{\bx} = L$. Notice that for sequences we use $0$-based indexing, that does not count to the length. We make this choice because signatures use sequence increments, which reduces length by $1$-step.

Below, we will assume that $\cX$ is either a topological vector space (i.e.~$\bbR^d$), or that as a preprocessing step our input sequence $\bx$ has been lifted to one using a static feature map $\varphi: \cX \to V$, i.e.~consider $\varphi(\bx)$ substituted in place of $\bx$ which we do not make explicit here.

\paragraph{Discretized signature features}
It is not obvious how to extend path signatures to the discrete-time setting, since they are defined in continuous-time via iterated integrals, and we discuss multiple choices below. First, we consider the so-called ``canonical'' choice, which proceeds by identifying $\bx$ with a continuous time path by its linear interpolation for any choice of time parameterization, and then computing the signature of this path.

To ease the notational burden, we define the following discrete-time $m$-simplices:
\begin{align}
    \Delta_m(L) = \curls{(i_1, \dots, i_m) \in \bbZ_+^m \setgiven 1 \leq i_1 < \cdots < i_m \leq L},
\end{align}
and its closed companion
\begin{align}
     \bar\Delta_m(L) = \curls{(i_1, \dots, i_m) \in \bbZ_+^m \setgiven 1 \leq i_1 \leq \cdots \leq i_m \leq L}.
\end{align}

Then, under this choice of path embedding (linear interpolation), we can compute the level~--~$m$ signature of $\bx \in \Seq(\cX)$ using the relation \eqref{eq:back:pathsiglin_explicit} by
\begin{align} \label{eq:back:sig_discrete}
    \S_m(\bx) = \sum_{\bi \in \bar\Delta_m(\len{\bx})} \frac{1}{\bi!} \delta \bx_{i_1} \otimes \cdots \otimes \delta \bx_{i_m},
\end{align}
where $\delta \bx_i = \bx_i - \bx_{i-1}$ refers to first-order (backwards) sequence differencing.

Before discussing \eqref{eq:back:sig_discrete} in detail, let us first generalize this formula. Thus, next we consider other choices for approximating iterated path integrals using discrete-time data. We introduce the order-$p$ approximation to the path signature considered in \cite[App.~B]{kiraly2019kernels}. The order-$p$ level-$m$ signature feature $\S^{(p)}: \Seq(\cX) \to \TV{\cX}$ of a sequence $\bx \in \Seq(\cX)$ is defined as
\begin{align} \label{eq:back:sig_order_p}
    \S_m^{(p)}(\bx) = \sum_{\bi \in \bar\Delta_m(\len{\bx})} \frac{1}{\bi!^{(p)}} \delta \bx_{i_1} \otimes \cdots \otimes \delta \bx_{i_m},
\end{align}
where $\bi!^{(p)}$ is such that if there are $k$ distinct indices in $\bi$ with number of repetitions $p_1, \dots, p_k$ respectively, then it is defined as
\begin{align}
    \bi!^{(p)} = \begin{cases}
        p_1! \cdots p_k! \quad &\text{for } p_1, \dots, p_k \leq p\\
        0 &\text{otherwise}.
    \end{cases}
\end{align}

Now, we notice that \eqref{eq:back:sig_order_p} is a generalization of \eqref{eq:back:sig_discrete}, since the latter can be recovered by choosing $p=m$. As mentioned before, we refer to $p$ as the order hyperparameter of the discretized signature features. Notice that \eqref{eq:back:sig_order_p} summarizes the sequence $\bx$ by a summation over its possibly contiguous length-$m$ subsequences with at most $p$-repetitions for each sequence index, such that each term is normalized by a factor depending on the number of repetitions for each unique index. In particular, we will often focus on the $p=1$ case, which is often referred to as a non-geometric approximation since there does not exist a continuous-time path such that $\S^{(1)}(\bx)$ is its path signature. Nevertheless, a key observation of this thesis is that the order-$1$ case often works well or even better than higher orders, while preserving a considerable amount of computation time and memory. This is explained in \cite{toth2021seq2tens}, which is an extension of Chapter \ref{ch:s2t}, where my collaborator Patric proved that the $p=1$ case is universal on the space of sequences given the existence a time coordinate which encodes the position within the sequence. 

In particular, the $p=1$ case can be written for a sequence $\bx \in \Seq(\cX)$ as
\begin{align} \label{eq:back:sig_order_1}
    \S_m^{(1)}(\bx) = \sum_{\bi \in \Delta_m(\len{\bx})} \delta \bx_{i_1} \otimes \cdots \otimes \delta \bx_{i_m},
\end{align}
i.e.~notice that the summation is taken exclusively over non-contiguous subsequences, and as such, can be treated as a generalization of $k$-mer features appearing in the context of string kernels, see \cite[Sec.~5]{kiraly2019kernels} for a connection to such classic kernels.

We remark that different choices of $p$ correspond to different algebraic embeddings, i.e.~replacing the tensor exponential in \eqref{eq:back:pathsiglin_chen} with a different algebra embedding. This is one of the main themes of Chapter \ref{ch:s2t}, and different choices lead to beside mildly different computational complexities and performances, to interesting algebraic questions, see \cite{diehl2020time, diehl2023generalized} for a discussion.

In the following, we may omit the order superscript when talking about signature features and kernels in the discrete-time setting, and simply treat it as a hyperparameter setting.

\paragraph{Discretized signature kernels}
The signature is a powerful feature set for nonlinear regression on streamed data. 
A computational bottleneck associated with it is the dimensionality of $\TV{\cX}$. As we are dealing with tensors, for $\cX$ finite-dimensional $\S_m(\bx)$ is a tensor of degree-$m$ which has $\pars{\dim \cX}^{m}$ coordinates that need to be computed. This can quickly become computationally expensive. For infinite-dimensional $\cX$, it is infeasible to directly compute $\S$.

The work of \cite{kiraly2019kernels} takes the setting when $\cX = \cH$ is a Hilbert space, and shows a kernel trick that allows to compute signature kernels for discrete data up to finite truncation $M \in \bbZ_+$ using dynamic programming, even when $\Hil$ is infinite-dimensional.
Subsequently, \cite{salvi2021signature} proposed a \texttt{PDE}-based algorithm to approximate the untruncated signature kernel, which was further extended in \cite{cass2021general}, and we refer to \cite{lee2023signature} for a recent overview of signature kernels.

Here, we focus on discrete-time, and continue the discussion of signature kernels. Our starting point is the generalized order-$p$ approach of \cite{kiraly2019kernels}, that results in the features \eqref{eq:back:sig_order_p}.

The signature kernel is a powerful formalism that allows to transform any static kernel on $\cX$ into a kernel for sequences that evolve in $\cX$. Let $\kernel: \cX \times \cX \to \bbR$ be a static kernel with \RKHS{} $\Hil$. Now, let $\bx, \by \in \Seq(\cX)$, and consider the discretized signature kernel of order-$p$ and level-$m$ $\sigkernelp[m]{p}: \Seq(\cX) \times \Seq(\cX) \to \bbR$ given as the inner product of order-$p$ (lifted) signatures from \eqref{eq:back:sig_order_p} 
\begin{align}
    \sigkernelp[m]{p}(\bx, \by) &= \inner{\S_m^{(p)}(\kernel_\bx)}{\S_m^{(p)}(\kernel_\bx)}_{\Hil^{\otimes m}}
    \\
    &= \inner{\sum_{\bi \in \bar\Delta_m(\len{\bx})} \frac{1}{\bi!^{(p)}} \delta \kernel_{\bx_{i_1}} \otimes \cdots \otimes \delta \kernel_{\bx_{i_m}}}{\sum_{\bj \in \bar\Delta_m(\len{\by})} \frac{1}{\bj!^{(p)}} \delta \kernel_{\by_{j_1}} \otimes \cdots \otimes \delta \kernel_{\by_{j_m}}}
    \\
    &= \sum_{\bi \in \bar\Delta_m(\len{\bx})} \sum_{\bj \in \bar\Delta_m(\len{\by})} \frac{1}{\bi!^{(p)} \bj!^{(p)}} \inner{\delta \kernel_{\bx_{i_1}}}{\delta \kernel_{\bx_{j_1}}}_{\Hil}  \cdots \inner{\delta \kernel_{\bx_{i_m}}}{\delta \kernel_{\by_{j_m}}}_\Hil
    \\
    &= \sum_{\substack{\bi \in \bar\Delta_m(\len{\bx})\\\bj \in \bar\Delta_m(\len{\by})}} \frac{1}{\bi!^{(p)}\bj!^{(p)}} \delta_{i_1, j_1} \kernel(\bx_{i_1}, \by_{j_1}) \cdots \delta_{i_m, j_m} \kernel(\bx_{i_m}, \by_{j_m}), \label{eq:back:discr_sig_kernel_p}
\end{align}
where $\delta_{i, j} \kernel(\bx_i, \by_j) = \kernel(\bx_i, \by_j) - \kernel(\bx_i, \by_{j-1}) - \kernel(\bx_{i-1}, \by_j) + \kernel(\bx_{i-1}, \by_{j-1})$ denotes a second-order (backwards) cross-differencing operator.
The key insight by \cite{kiraly2019kernels} is equation \eqref{eq:back:discr_sig_kernel_p}, i.e.~that $\sigkernel[m]$ can be computed without computing $\S_m$ itself by a kernel trick using only kernel evaluations.

The hyperparameters are: \begin{enumerate*}[label=(\arabic*)] \item the choice of the static kernel $\kernel$, for which there is a wide range of options, e.g.~for $\cX = \bbR^d$ the \RBF, exponential, polynomial or Mat{\'e}rn family of kernels; \item any hyperparameters that $\kernel$ comes with, such as the bandwidth or polynomial degree; \item truncation level-$M$; \item the choice of algebra embedding, i.e. the order-$p$; \item and the choice of kernel normalization that scales each level $\sigkernel[m]$ appropriately. It also comes with theoretical guarantees such as analytic estimates when sequences converge to paths \cite{kiraly2019kernels}, its maximum mean discrepancy (\texttt{MMD}) \cite{gretton2012kernel} metrizes classic topologies for stochastic processes, and can lead to robust statistics in the classic statistical sense (B-robustness); see \cite{chevyrev2022signature} for details. \end{enumerate*}

Although \eqref{eq:back:discr_sig_kernel_p} looks expensive to compute, \cite{kiraly2019kernels} applies dynamic programming to efficiently compute $\sigkernel$ for finite $M$ using a recursive algorithm; an alternative algorithm is the above mentioned approach of approximating the (untruncated) signature kernel $\sigkernel$ using \texttt{PDE}-discretization. Importantly, \eqref{eq:back:discr_sig_kernel_p} avoids computing tensors, and only depends on the entry-wise evaluations of the static kernel $\kernel(\bx_i, \by_j)$. Indeed, this leads to a computational cost of $O(\len{\bx} \len{\by})$ with respect to length, that is feasible to compute for sequences evolving in high-dimensional state-spaces, but only with moderate sequence length. Note that the same bottleneck applies to \texttt{PDE}-based approaches. In part, the aim of Chapter \ref{ch:rfsf} is to alleviate this quadratic cost in sequence length, while approximately enjoying the modelling capability of working within an infinite-dimensional \RKHS.

\section{Algorithms} \label{sec:back:algs}
\subsection{Notation for algorithms} \label{sec:back:notation}
We define notation based on \cite{kiraly2019kernels} for describing vectorized computations. We use $1$-based indexing for arrays. Note that elements outside the bounds of an array are treated as zeros. Let $A$ and $B$ be k-fold arrays of size $(n_1 \times \dots \times n_k)$, indexed by $i_j \in \{1, \dots, n_j\}$ for $j \in \{1, \dots, k\}$.

We define the following operations:
  \begin{enumerate}[label=(\roman*)]
  	\item  The cumulative sum along axis $j$ as:
  	\begin{align}
  		&A[:, \dots, :, \boxplus, :, \dots, :][i_1, \dots, i_{j-1}, i_j, i_{j+1}, \dots i_k] = \sum_{\kappa=1}^{i_j} A[i_1, \dots, i_{j-1}, \kappa, i_{j+1}, \dots, i_k].
  	\end{align}
  	\item The slice-wise sum along axis $j$ as:
  	\begin{align}
	  	&A[:, \dots, :, \Sigma, :, \dots, :][i_1, \dots, i_{j-1}, i_{j+1}, \dots, i_k] = \sum_{\kappa=1}^{n_j} A[i_1,\dots, i_{j-1}, \kappa, i_{j+1}, \dots i_k].
  	\end{align}
  	\item The shift along axis $j$ by $+m$ for $m \in \bbZ$ as:
  	\begin{align}
	  	&A[:, \dots, :, +m, :, \dots, :][i_1, \dots, i_j, \dots, i_k] = A[i_1,\dots, i_j-m, \dots i_k]
  	\end{align}
    \item The Hadamard product of arrays $A$ and $B$:
  	\[(A \odot B) [i_1, \dots, i_k] = A[i_1, \dots, i_k] \cdot B[i_1, \dots, i_k]. \]
    \item Now, if $A$ has shape $(n_1 \times \cdots \times n_j \times \cdots \times n_k)$ and $B$ has shape $(n_1 \times \cdots \times n_j^\prime \times \cdots \times n_k)$, then their Kronecker product along axis $j$ is defined for $i_j \in [n_j]$ and $i_j^\p \in [n_j^\prime]$ as
        \begin{align}
            (A \boxtimes_j B) [i_1, \dots, (i_j-1)n_j^\p + i_j^\p, \dots, i_k] = A[i_1, \cdots, i_j, \cdots i_k] B[i_1, \dots, i_j^\p, \dots, i_k].
        \end{align}  
  \end{enumerate}

\subsection{Algorithms for path signatures}
In this section, we provide adapted algorithms for computing signature features and kernels.

First, we provide an algorithm for computing the order-$p$ signature features given in \eqref{eq:back:sig_order_p} in Algorithm \ref{alg:back:sig_order_p}. This algorithm has not appeared explicitly in previous work, and is a combination of Algorithms 5 and 6 in \cite{kiraly2019kernels}. The output of the algorithm are flattened signature tensors of dimension $d, d^2, \dots, d^M$. Upon investigation, the time complexity of the algorithm is $O(MNLpd^M)$ and requires a space of $O(p^2NLD + Nd^M)$. This is clearly inefficient due to the polynomail complexity in $d$ and exponential in $M$, and in the later Chapters \ref{ch:s2t} and \ref{ch:rfsf} we discuss how this complexity can be reduced via various low-rank and random projection techniques.

Next, we adapt the dual algorithm for computing the order-$p$ signature kernel \eqref{eq:back:discr_sig_kernel_p} from \cite[Alg.~6]{kiraly2019kernels} under Algorithm \ref{alg:back:sigkernel_p}. The output of the algorithm consists of signature kernel matrices per level $K_1, \dots, K_m \in \bbR^{N \times N}$. The time complexity of the algorithm is $O(N^2 L^2 d + MN^2L^2p^2)$, while the space is $O(p^2 N^2 L^2)$, where $d$ is the complexity of static kernel evaluation. Due to joint quadratic complexity in $N^2$ and $L^2$, this is infeasible for high-volume ($N$ large) and long time series ($L$ large). In Chapter \ref{ch:gpsig}, we combine the signature kernel with Gaussian processes and variational inference techniques reducing the complexity to linear in $N$, allowing to scale to large datasets, although the length still remains a bottleneck. Later on, in Chapter \ref{ch:rfsf}, we introduce scalable random approximations to the signature kernel, which alleviate the quadratic complexity in both $N$ and $L$, while also scaling favourably in the ``feature size'' i.e.~in the dimensionality of the random feature map that represents the random approximate kernel.

\begin{algorithm}[t]
\begin{footnotesize}
\caption{Computing order-$p$ signature features \eqref{eq:back:sig_order_p}.}
\label{alg:back:sig_order_p}
\begin{algorithmic}[1]
    \STATE {\bfseries Input:} Time series $X = (\bx_1, \dots, \bx_N) \subset \Seq(\bbR^d)$, truncation-$M \in \bbZ_+$, order-$p \in \bbZ_+$
    \STATE Optional: Add time-parameterization $\bx_i \gets (\bx_{i, t}, t / \len{\bx_i})_{t=1}^{\len{\bx_i}}$ for all $i \in [N]$
    \STATE Tabulate to uniform length $L = \max_{j \in [N]} \len{\bx_j}$ by $\bx_i \gets (\bx_{i, 1}, \ldots, \bx_{i, \len{\bx_i}}, \ldots, \bx_{i, \len{\bx_i}})$ for all $i \in [N]$
    \STATE Initialize an array $U$ with shape $[N, L, d]$
    \STATE Fill $U$ with values $U[i, l, :] = \bx_{i,l} - \bx_{i, l-1}$ for all $i \in [N]$ and $l \in [L]$
    \STATE Accumulate into level-$1$ features $P_1 \gets U[:, \Sigma, :]$
    \STATE Initialize an array $R$ of shape $[1, N, L, d]$
    \STATE Assign values $R[1, :, :, :] \gets U$
    \FOR{$m=2$ {\bfseries to} $M$}
        \item Let $q \gets \min(p, m)$
        \item Collapse into expanding window features $P^\p \gets R[\Sigma, :, \boxplus, :]$
        \item Update with next increment $P^\p \gets P^\p[:, +1, :] \boxtimes_{-1} U$
        \STATE Initialize new array $R^\p$ of shape $[q, N, L, d]$
        \STATE Assign values $R^\p[1, :, :, :] \gets P^\p$ 
        \FOR{$r=2$ to $q$}
            \STATE Update with current increment $R^\p[r, :, :, :] \gets \frac{1}{r} R[r-1, :, :, :] \boxtimes_{-1} U$
        \ENDFOR
        \STATE Aggregate into level-$m$ features $P_m \gets R^\prime[\Sigma, :, \Sigma, :]$
        \STATE Roll array $R \gets R^\prime$
    \ENDFOR
    \STATE {\bfseries Output:} Arrays of signature features per level $P_1 \in \bbR^d, \dots, P_M$.
\end{algorithmic}
\end{footnotesize}
\end{algorithm}

\begin{algorithm}[h]\begin{footnotesize}
\caption{Computing order-$p$ signature kernels \eqref{eq:back:discr_sig_kernel_p}.}
\label{alg:back:sigkernel_p}
\begin{algorithmic}[1]
    \STATE {\bfseries Input:} Sequences $\bX=(\bx_i)_{i=1,\dots,N} \subset \Seq(\cX)$, truncation-$M \in \bbZ_+$, order-$p \in \bbZ_+$, static kernel $\kernel: \cX \times \cX \to \bbR$
    \STATE Optional: Add time-parameterization $\bx_i \gets (\bx_{i, t}, t / \len{\bx_i})_{t=1}^{\len{\bx_i}}$ for all $i \in [N]$
    \STATE Tabulate to uniform length $L = \max_{j \in [N]} \len{\bx_j}$ by $\bx_i \gets (\bx_{i, 1}, \ldots, \bx_{i, \len{\bx_i}}, \ldots, \bx_{i, \len{\bx_i}})$ for all $i \in [N]$
    \STATE Initialize an array $K$ with shape $[N, N, L, L]$
    \STATE Compute kernel increments $K[i, j, k, l] = \delta_{k, l} \kernel(\bx_{i, k}, \bx_{j, k})$ for $i, j \in [N]$ and $k, l \in [L]$
    \STATE Collapse into level-$1$ kernel entries $K_1 \gets K[:, :, \Sigma, \Sigma]$
    \STATE Initialize an array $R$ of shape $[1, 1, N, N, L, L]$
    \STATE Assign values $R[1, 1, :, :, :, :] \gets K$
    \FOR{$m=2$ {\bfseries to} $M$}
    \STATE Let $q \gets \min(p, m)$
    \STATE $K^\p \gets R[\Sigma, \Sigma, :, :, \boxplus+1, \boxplus+1] \odot K$
    \STATE Initialize new array $R^\p$ of shape $[q, q, N, N, L, L]$
    \STATE Assign values $R^\p[1, 1, :, :, :, :] \gets K^\p$
    \FOR{$r=2$ {\bfseries to} $q$}
        \STATE $R^\p[r, 1, :, :, :, :] \gets \frac{1}{r} R[r-1, \Sigma, :, :, :, \boxplus+1] \odot K$
        \STATE $R^\p[1, r, :, :, :, :] \gets \frac{1}{r} R[\Sigma, r-1, :, :, \boxplus+1, :] \odot K$
        \FOR{$s=2$ {\bfseries to} $q$}
            \STATE $R^\p[r, s, :, :, :, :] \gets \frac{1}{rs} R[r-1, s-1, :, :, :, :] \odot K$
        \ENDFOR
    \ENDFOR
    \STATE Collapse into level-$m$ kernel entries $K_m \gets R^\p[\Sigma, \Sigma, :, :, \Sigma, \Sigma]$
    \STATE Roll array $R \gets R^\p$
    \ENDFOR
    \STATE {\bfseries Output:} Signature kernel matrices per level $K_1, \dots, K_M$.
\end{algorithmic}
\end{footnotesize}
\end{algorithm}



\begingroup
\chapter{Gaussian Processes with Signature Covariances} \label{ch:gpsig}
\section{Introduction} \label{section:introduction}

The evolution of some state variable, parameter or object gives naturally gives rise to sequential data, which is defined by having a notion of order on the incoming information.
The ordering relation, or index set does not have to represent physical time, but for simplicity we will call it as such. When the index-set is discrete, we refer to the data as a sequence, and when continuous as a path.
For example, besides time series (\texttt{TS}), sources of sequential data are text \cite{Pennington2014Glove}, \texttt{DNA} \cite{Heather2016DNA}, or even topological data analysis \cite{chevyrev_nanda_oberhauser_2018}.
This ubiquity of sequential data has received special attention by the machine learning community in recent years.

In this chapter, we are focused on combining signature kernels with these approaches:

\paragraph{Bayesian approaches}
Often not only point predictions, but estimates of the associated uncertainties are required \cite{Ghahramani2013Bayesian}.
Gaussian processes (\texttt{GP}) \cite{rasmussen2006} provide flexible priors over functions of the data in nonparametric Bayesian models.
In the context of sequential data, two prominent ways to use \texttt{GP}s are: \begin{enumerate*}[label=(\arabic*)] \item using as covariance functions kernels specifically designed for sequences~\cite{lodhi2002text, cuturi2011fast, Cuturi2011AR, al2017learning}, \item  modelling the evolution in a latent space, that emits the observations, as a discrete dynamical system with a \texttt{GP} prior on the transition function, a model called the Gaussian Process State Space Model (\texttt{GPSSM}) \cite{Frigola2013MCMC, frigola2014variational, mattos2016recurrent, eleftheriadis2017identification, Doerr2018Proba, ialongo2019overcoming}. \end{enumerate*}
These two approaches are not mutually exclusive; if one models the latent system as a higher order Markov process, then sequence kernels can incorporate the effect of past states.

\paragraph{Deep learning approaches} 
Deep learning approaches, such as the celebrated \texttt{LSTM} network \cite{hochreiter1997long}, other forms of \texttt{RNN}s \cite{Cho2014}, convolutional networks, and transformers \cite{vaswani2017attention} have successfully been applied to a variety of tasks involving sequential data \cite{Sutskever2014Seq2Seq, Oord2016Wavenet, wen2022transformers}. Deep learning models can approximate any continuous function, but the cost is a large number of parameters, high variance and poor interpretability.
This leaves the door open for alternative approaches not only as competitors, but as complementary building blocks in larger models.

\paragraph{Contribution}
In principle, one can just use the signature kernel and algorithms as covariance function to define a \texttt{GP} for sequential data. 
However, the computational complexity becomes quickly prohibitive, which ultimately does not lead to scalable performance on many time series benchmarks.
We therefore develop a different approach to signature covariances that builds on two recent advances in \texttt{GP} inference: variational inference \cite{Titsias2009Variational, Hensman2015Scalable, matthews2016sparse} and inter-domain inducing points~\cite{lazaro2009inter} to alleviate the computational burden.
In particular, we show that one can use low-rank tensors as inter-domain inducing points by optimizing a variational bound.
Moreover, we use this \texttt{GP} as a building block in combination with \texttt{RNN}s to build models that combine the strenghts of signatures, Bayesian models, and deep learning architectures. 
This results in scalable inference algorithms and we use this to benchmark on standard \texttt{TS} datasets \begin{enumerate*}[label=(\roman*)] \item against popular non-Bayesian time series classifiers purely in terms of accuracy, \item against alternative Bayesian models by comparing the calibration of uncertainties for predictions \end{enumerate*}.
Code and benchmarks are publically available at \url{http://github.com/tgcsaba/GPSig}.

\section{Background and Notation}\label{sec:background}
Given data $(\bX, Y)$ consisting of $n_\bX$ inputs $\bX = (\bx_1,\ldots,\bx_{n_\bX}) \subset \cX$ with labels $Y=(y_1,\ldots,y_{n_\bX}) \subset \bbR$, the Bayesian approach is to put a prior on a set of functions $\{f \setgiven f: \cX \rightarrow \bbR\}$, update this prior by conditioning on $(\bX, Y)$, and then use the resulting posterior to make inference about the label $y_\star$ of an unseen point $\bx_\star$. 
When this is done with Gaussians, the central object is a \texttt{GP} $f=(f_{\bx})_{\bx \in \cX}$ which is specified by the mean and covariance functions.
Below we recall how covariances can be constructed from polynomial features, which we then extend to signatures.

\subsection{The feature space view}
Given a map $\varphi:\cX \hookrightarrow V$ that injects $\cX$ into a Hilbert space $\cH$, a natural way to put a prior on a function class $\cX \rightarrow \bbR$ is to consider {linear} functions of $\varphi$ as model, that is $f_{\bx}= \langle \ell, \varphi(\bx) \rangle$ to model $f(\bx_i)\approx y_i$ for some ``weights`` $\ell \in \cH$. Uncertainty about $f$ is then specified by uncertainty about $\ell$ (and the hyperparameters of $\varphi$). 
We refer to $\varphi$ as a {feature map} and to $\cH$ as {feature space}.
The advantage of the nonparametric approach is that it allows one to avoid explicitly computing features and weights, and to perform inference directly on the function $f$.
Throughout we assume that $f=(f_{\bx})_{\bx \in \cX}$ is a centered \texttt{GP}. 
Then, predictions about unseen points can then be made by Bayesian inference using Gaussian conditioning.    
If the task is classification where the labels $Y$ are discrete, such an approach can be still applied by using a \texttt{GP} $f=(f_{\bx})_{\bx \in \cX}$ as {nuiscance function} to put a {prior on the class membership probability} by specifying $ p(y=1|\bx)=\sigma(f_{\bx})$, where $\sigma$ is for example a sigmoid.

\subsection{Polynomial features} First, we review the classical case of how a Gaussian process can be constructed from polynomial features, which we will later on generalize to signatures.

Now, let $\cX \subseteq \bbR^d$ and define the monomial feature map
\begin{align}\label{eq:gpsig:moments}
    \varphi(\bx) = (1,\bx, \bx^{\otimes 2}, \bx^{\otimes 3}, \ldots, \bx^{\otimes M})
\end{align}
where $\bx^{\otimes m} \in \pars{\bbR^d}^{\otimes m}$ is a tensor. 
We refer to Section \ref{sec:back:tensors} about tensors. 
If we set $f_{\bx}= \langle  \ell, \varphi(\bx)\rangle$ and put a centered Gaussian prior on $\ell = (\ell_1,\ldots, \ell_M)$, then it follows from Section \ref{sec:back:tensors} that
\begin{align}\label{eq:gpsig:polynomial cov.}
    \expe{f_{\bx} f_{\by}} &= 1 + \sum_{m=1}^M \expe{\inner{\ell \otimes \ell}{\bx^{\otimes m} \otimes \by^{\otimes m}}}
    \\
    &= 1 + \sum_{m=1}^M \inner{\expe{\ell \otimes \ell}}{\bx^{\otimes m} \otimes \by^{\otimes m}}
    \\
    & =1+ \sum_{m=1}^M \langle \Sigma_m^{2}, \bx^{\otimes m} \otimes \by^{\otimes m} \rangle
\end{align}
where $\Sigma_m^2= \bbE[ \ell_m \otimes \ell_m ] \in \pars{\bbR^d}^{\otimes (2m)}$.
Taking $\Sigma_m^2$ to be an isotropic ``diagonal``' tensor $\Sigma_m^2 = \sigma_m^2 \sum_{i_1, \dots, i_m =1}^m (e_{i_1} \otimes \cdots \otimes e_{i_m})^{\otimes 2} $ recovers the polynomial kernel,
\begin{align} \label{eq:gpsig:poly_isotropic}
    {\bbE[f_{\bx}f_{\by}]} = \sum_{m=0}^M \sigma_m^2 \langle \bx, \by \rangle^m. 
\end{align}
Many other variations exist, for example other classes of polynomials, such as Hermite polynomials (the eigenfunctions of the classic \texttt{RBF} kernel), can increase the effectiveness, since they allow to make the associated feature expansion infinite dimensional.
However, what makes any such class of polynomials a sensible choice for $\varphi$ is that by the Stone--Weierstrass theorem, any continuous compactly supported function $\cX \rightarrow \bbR$ can be arbitrary well approximated as linear functions of $\varphi(\bx)$.
This approximation property is often called {universality} \cite{Micchelli2006Universal, sriperumbudur2011universality}.

\section{From signature features to covariances} \label{sec:gpsig:our_gp}
The signature feature map $\S$ can be seen as a generalization of the polynomial feature map $\varphi$ as defined in~\eqref{eq:gpsig:moments} from the domain $\cX \subseteq \bbR^d$ of vectors to the domain of bounded variation paths. 

Recall from Section \ref{sec:back:pathsig} that it is defined for a $(\bx_t)_{t \in [0, T]} \in \Paths(\cX)$ as
\begin{align}\label{eq:gpsig:signature} 
    \S(\bx) = \pars{\int_{\bt \in \Delta_m([0, T])} \d \bx_{t_1} \otimes \cdots \otimes \d \bx_{t_m}}_m.  
\end{align} 

\subsection{Automatic parameterization invariance determination}
A classic empirical finding that led to dynamic time warping (\texttt{DTW}) is that functions of sequences are often invariant to a certain degree of time parameterization: for example, different speakers pronounce words at different speeds. 
However, sometimes the parameterization matters, e.g.~for financial data.
Thus we do not only care about
\begin{align}
\label{eq:gpsig:functions}
\cH_{\Paths} = {\{ f:\Paths(\cX) \rightarrow \bbR \setgiven \text{continuous and compactly supported}\} }
\end{align}
but also about the subset of it that consist of parameterization invariant functions. 
To make this precise, we call $(\bx_{s})_{s \in [0, S]}$ a {reparameterization} of $(\by_t)_{t \in [0, T]}$ if there exists a a smooth increasing function $\rho: [0, S] \rightarrow [0, T]$ (the ``time change'') such that $\bx_s = y_{\rho(s)}$ for all $s \in [0, S]$.

Often the function we want to learn is invariant to some but not extreme reparameterization, so we need a more nuanced way to quantify parameterization (in)variance.
Hence, what we really want is a hyperparameter $\tau \ge 0$ that signifies the degree of parameterization invariance: for $\tau=0$ all the prior mass should concentrate on the ``extreme case'' that is the subset of \eqref{eq:gpsig:functions} consisting parameterization invariant functions; and as $\tau$ gets increased the probability mass should spread out to parameterization sensitive functions. 
This would allow to infer the degree of parameterization invariance by automatic relevance determination (\texttt{ARD}). 

To accomplish this, we augment our paths with a hyperparameter $\tau \ge 0$ such that
\begin{align}
    (\bx^{(\tau)}_t)_{t \in [0, T]} = \pars{\tau \cdot t, \bx_t)}_{t \in [0, T]} \in \Paths(\hat\cX),
\end{align}
where $(\bx_t)_{t \in [0, T]} \in \Paths(\cX)$ and $\hat\cX \subseteq \bbR^{1+d}$.
This simply makes the parameterization part of the trajectory by adding an extra coordinate, and allows the magnitude of $\tau$ to scale the effect of parameterization sensitivity.
Since signatures distinguish different trajectories (but not the speed at which we run through them), it follows that for $\tau>0$ and $\bx, \by \in \Paths(\cX)$ 
\begin{align}
    \S(\bx^{(\tau)}) = \Phi(\by^{(\tau)}) \text{ if and only if } \bx = \by,
\end{align}
while for $\tau = 0$ we have, 
\begin{align}
{\S(\bx^{(0)}) = \S(\by^{(0)}) \text{ if and only if } \bx \sim \by}
\end{align}
since the extra coordinate is ``switched off''. 
Here, $\sim$ denotes tree-like equivalence, but we invite the reader to read $\bx \sim \by$ as saying that $\bx$ is a reparameterization of $\by$.
This is strictly speaking not true and the precise mathematical statement is given in Section \ref{subsec:back:props}, but note that for real-world data, tree-like equivalence is synonymous with reparameterization.

In the rest of the chapter, we omit the dependence on $\tau$ from the notation, but it should be noted that we are still working with the parameterization augmented path $\bx^{(\tau)}$ in place of $\bx$.

\begin{table}[t]
	\begin{center}
		\begin{sc}
        \resizebox{\textwidth}{!}{
		\begin{tabular}{lccc}
		\toprule
		& Vectors & Paths & Sequences \\
		\midrule
		Domain  & $\cX$ & $\Paths(\cX)$ & $\Seq(\cX)$\\
		Features & $\varphi(\bx)=\pars{\bx^{\otimes m}}_{m}$ & $\S(\bx)=\pars{\int_{\bt \in \Delta_m([0, \hor{\bx}])} \d\bx_{t_1} \otimes \cdots \otimes\d\bx_{t_m}}_{m}$ & $\S^{(1)}(\bx) = \pars{\sum_{\bi \in \Delta_m(\len{\bx})} \delta \bx_{i_1} \otimes \cdots \delta \bx_{i_m}}_m$\\
		Feature space & $\prod_{m}\pars{\bbR^d}^{\otimes m}$ & $\prod_m\pars{\bbR^d}^{\otimes m}$ & $\prod_m\pars{\bbR^d}^{\otimes m}$\\
		Functions & $f:\bbR^d \rightarrow \bbR$ & $f:\Paths(\cX) \rightarrow \bbR$ & $f: \Seq(\cX) \to \bbR$\\ 
		Covariance & $\sum_m \sigma_m^2\langle \bx,\by \rangle^m$& $\sum_m\sigma^2_m \int_{\substack{\bs \in \Delta_m([0, \hor{\bx}])\\ \bt \in \Delta_m([0, \hor{\by}])}} \prod_{k=1}^m \inner{d\bx_{s_m}}{d\by_{t_m}}$ & $\sum_m \sigma_m^2 \sum_{\substack{\bi \in \Delta_m(\len{\bx})\\\bj \in \Delta_m(\len{\by})}} \prod_{k=1}^m\inner{\delta \bx_{i_k}}{\delta \by_{j_k}}$
        \\
		\bottomrule
		\end{tabular}}
	\end{sc}
	\end{center}
	\label{table:gpsig:vectors_vs_paths}
	\caption{Comparison of polynomial and signature features}
	\end{table}

\subsection{Signature covariances}
In this section, we focus on deriving the signature covariances with paths lifted by the linear kernel for simplicity, but the construction extends to other variations, e.g.~the \texttt{RBF} kernel, by replacing the coordinate-wise expansion of paths with the eigenfunctions of the kernel, which forms an orthonormal basis of the \texttt{RKHS}. The formulation gets more technical since in this case the linear functional $\ell$ becomes infinite dimensional, and equation \eqref{eq:gpsig:karhunen} becomes the Karhunen-Loeve expansion for \texttt{GP}s \cite{adler2009random}.
Following the feature space view, we now argue in complete analogy to the case of the classical polynomial feature map $\varphi$ for $\cX=\bbR^d$, and define a centered \texttt{GP} $f=(f_{\bx})_{\bx \in \Paths(\cX)}$ by putting a centered Gaussian prior on $\ell=(\ell_1,\ldots,\ell_M)$ and setting
\begin{align} \label{eq:gpsig:karhunen}
    f_{\bx}= \inner{\ell}{\S(\bx)} \quad \text{for } \bx \in \Paths(\cX). 
\end{align}
The \texttt{GP} is hence fully specified by its covariance function 
\begin{align}\label{eq:gpsig:signature covariance}
    \sigkernel(\bx, \by)= \sum_{m=0}^M \inner{\Sigma_m^2}{\int_{\bs \in \Delta_m([0, \hor{\bx}])} \d\bx_{s_1} \otimes \cdots \otimes\d\bx_{s_m} \otimes \int_{\bt \in \Delta_m([0, T])} \d\by_{t_1} \otimes \cdots \otimes\d\by_{t_m}}_{m}
\end{align}
that has $(\tau, M, \Sigma_1^2,\ldots,\Sigma_M^2)$ as hyperparameters.
In particular, choosing an isotropic covariance structure for $\Sigma_m^2$ completely analogously to how it is done above equation \eqref{eq:gpsig:poly_isotropic} gives that
\begin{align} \label{eq:gpsig:isotropic_sig_cov}
    \sigkernel(\bx, \by) = \sum_{m=0}^M \sigma_m^2 \int_{\substack{\bs \in \Delta_m([0, \hor{\bx}])\\\bt\in\Delta_m{([0, \hor{\by}])}}} \inner{\d\bx_{s_1}}{\d\by_{t_1}} \cdots \inner{\d\bx_{s_m}}{\d\by_{t_m}},
\end{align}

Finally, as we are given data not as paths, but as discrete-time observations $\bx, \by \in \Seq(\cX)$, in practice we use the order-$1$ discretized signature kernel (see Section \ref{sec:back:discrete}) given as 

\begin{align} \label{eq:gpsig:discrete_sig_cov}
    \sigkernelp{1}(\bx, \by) = \sum_{m=0}^M \sigma_m^2 \sum_{\substack{\bi \in \Delta_m(\len{\bx}) \\ \bj \in \Delta_m(\len{\by})}} \inner{\delta \bx_{i_1}}{\delta \by_{j_1}} \cdots \inner{\delta \bx_{i_m}}{\delta \by_{j_m}},
\end{align}
and in forth going we suppress the dependence on the order-$p=1$ hyperparameter. 

\subsection{Regularity of the resulting Gaussian process}
Next, we want to find out the regularity properties of the \texttt{GP} with covariance \eqref{eq:gpsig:isotropic_sig_cov}. However, the index set $\Paths(\cX)$ is a very large space so some care is needed. 
In Appendix~\ref{app:gpsig:gpsig}, we compute covering numbers that yield explicit bounds on the modulus of continuity in terms of $\norm{\bx}_\onevar$. 
\begin{theorem} \label{thm:continuity}
    Let $L > 0$ and $\Paths_L(\cX)=\{\bx \in \Paths(\cX): \| \bx \|_{\onevar} \le L\}$.
    There exists a centered \texttt{GP} $f=(f_{\bx})_{\bx \in \Paths_L(\cX)}$ with $\sigkernel(\bx,\by)$ as defined in~\eqref{eq:gpsig:isotropic_sig_cov} as covariance function and that has continuous sample paths $\bx \mapsto f_{\bx}$.
    Further, an explicit bound on its modulus of continuity in terms of $L$ is given in equation~\eqref{eq:gpsig:modulus}. 
\end{theorem}
We now have a well-defined \texttt{GP} for Bayesian inference for sequences at hand that inherits many of the attractive properties of signature features. 
To turn this into useful models for large \texttt{TS} benchmarks we develop efficient inference algorithms in the next section.

	\section{Sparse variational inducing tensors}\label{sec:sparse var tensor}
To reiterate, we are given data $(\bX, Y)$ consisting of $n_\bX$ sequences $\bX = (\bx_1,\ldots,\bx_{n_\bX}) \subset \Seq(\cX)$ of maximal length $\len{\bX}=\max_{\bx \in \bX}\len{\bx}$ that evolve in $\bbR^d$ with labels $Y=(y_1,\ldots,y_{n_\bX})$, and the task is to predict labels $y_\star$ of unseen points $\bx_\star$.
For sequences, the sample size $n_\bX$ and associated covariance matrix inversion is not the only compational bottleneck but also the maximal length $\len{\bX}$, and the dimension $d$ of the state space matter: $n_\bX$, $\len{\bX}$ and $d$ can be simultaneously large. 

In this section, we introduce a sparse variational inference scheme to approximate the posterior, that locates the inducing points in a space other than the data-domain; this approach is usually called \emph{inter-domain} sparse variational inference \cite{lazaro2009inter, matthews2016sparse}. This allows for more efficient data-representation and faster inference. 
Key to our approach is that signature features take values in a well-understood subset of the feature space $\TV{\bbR^d}$.
This allows as us to augment the index set with structured tensors, and locate inducing points in this larger set.

\subsection{Variational inference}
As is well-known, inference for \texttt{GP}s scales as $O(n_\bX^3)$, see Section 3.3.~in~\cite{rasmussen2006}.
This first led to {sparse} models, \cite{quinonero2005unifying}, that select a subset $\bZ=\{\bz_1,\ldots,\bz_{n_\bZ}\}$ of $\bX$ consisting of $n_\bZ\ll n_\bX$ points, and subsequently to {pseudo-inputs}, \cite{snelson2006sparse}, that select points $\bZ$ that are not necessarily in $\bX$.
This was a big step towards complexity reduction, but pseudo-inputs are prone to overfitting, \cite{MatthewsDPhil}.
A different idea is to treat $\bZ$ as parameters of a variational approximation \cite{Titsias2009Variational} and not as model parameters; that is the points $\bZ$ are choosen simultaneously with the hyperparameters of the \texttt{GP} by maximising a lower bound on the log-marginal likelhood $\log p(Y)$, the so-called {evidence lower bound} (ELBO), given as
\begin{align} \label{eq:gpsig:elbo}
	\log p(Y) \geq \bbE_{q(f_{\bx})}[\log p(Y \vert f_{\bx})] - \KL{q(f_\bZ)}{p(f_\bZ)},
\end{align}
where $f_{\bx}$ and $f_\bZ$ denotes the \texttt{GP} evaluated at the data-points and the inducing locations. Typically, $q(f_\bZ)$ is given a free-form multivariate Gaussian to be learnt from the data, and then extended to other indices of the \texttt{GP} by \textit{prior conditional matching}, i.e. $q(f_{\bx} \vert f_\bZ) = p(f_{\bx} \vert f_\bZ)$.	Initially applied to regression, this was extended to classification~\cite{chai2012variational, Hensman2015Scalable}.
Among its advantages are that it gives a nonparametric approximation to the true posterior, adding inducing points only improves the approximation, and any optimization method can be used to maximize the \texttt{ELBO}, most importantly, stochastic optimization; see \cite{Hensman2013GaussianPF, bauer2016understanding,bui2016unifying}.

\paragraph{Inter-domain approaches} Another idea is to go beyond the original index set and place inducing points $\bZ$ in a different space $\cX'$, that is, given a centered \texttt{GP} $g=(g_\bx )_{\bx \in \cX}$ one augments the original index set $\cX$ by a set $\cX'$ to define a new \texttt{GP} $(g_\bx)_{\bx \in \cX \cup \cX'}$ and then locates the inducing points in this bigger model.
This was suggested in~\cite{lazaro2009inter} in the context of integral transforms, which was extended in~\cite{Hensman2016Fourier}, and studied in more generality in~\cite{matthews2016sparse}.
In general, it is not obvious how to find a useful augmentation set $\cX'$ and define the covariance enlarged to $\cX\cup \cX'$.
\subsection{Inter-domain low-rank tensors}
\paragraph{A feature space augmentation}
Given any \texttt{GP} with a covariance function $\kernel(\bx,\by)=\langle \Phi(\bx), \Phi(\by) \rangle$ where $\Phi$ is explicitly known\footnote{Mercer's Theorem guarantees the existence of $\Phi$, but not in a sufficiently explicit form.}, we propose that a natural augmentation candidate is the ``feature space'' $\cX'=\operatorname{span}\{\Phi(\bx): \bx \in \cX\}$ itself.
The covariance function $\kernel$ of $g$ can be simply extended to $\cX\cup \cX'$ by linearity,
\begin{align}\label{eq:gpsig:augment}
\kernel(\bx,\bz)=\kernel(\bz,\bx)=\alpha \kernel(\bx, \bx') + \beta \kernel(\bx, \bx'') 
\end{align}
for $\bx \in \cX$, $ \bz=\alpha \Phi(\bx')+\beta \Phi(\bx'') \in \cX'$, $\alpha,\beta \in \bbR$; analogous for $\kernel(\bz,\bz')$ with $\bz,\bz' \in \cX'$. 
For our \texttt{GP},
\begin{align} \label{eq:gpsig:sig_span}
    \cX'=\operatorname{span}\{\Phi_{\tau}(\bx):\bx \in \Paths(\cX)\} \subset \TV{\bbR^d}.
\end{align}
We can thus extend our signature covariance~\eqref{eq:gpsig:signature covariance} to $\Seq(\cX) \cup \TV{\bbR^d}$ by~\eqref{eq:gpsig:augment}. 
This provides a flexible class of inducing point locations $\bZ$ by optimizing over elements of the tensor algebra $ \bZ \subset \TV{\bbR^d}$.
We coin these inducing point locations as \textit{inducing tensors}.

\paragraph{Consistency of augmentation}	
A subtle point about augmenting the index set is that maximizing the \texttt{ELBO} in \eqref{eq:gpsig:elbo} is not necessarily equivalent anymore to minimizing a rigorously defined KL divergence between the true posterior process and its approximation over the unaugmented index set. In \cite{matthews2016sparse}, a sufficient condition given for this to hold is that the prior \texttt{GP} evaluated at the newly added indices is deterministic conditioned on the original \texttt{GP}.
In the case of~\eqref{eq:gpsig:augment}, this is easily seen to be true, since the augmented indices arise as linear combinations of (the lifts of) elements in the original index set.
Therefore, the corresponding \texttt{GP} evaluations arise as linear combinations of evaluations of the original process by the fact that the (pre-)feature space $\cX^\p$ is a \textit{representation} of the \textit{Hilbert space generated by the process} \cite[Sec.~3]{berlinet2011reproducing}.

\paragraph{Representation of inducing tensors}
We define our inducing tensors as in the low-rank format, specifically as rank-$1$ tensors such each inducing point is represented as 
\[ \label{eq:gpsig:sparse_tens}
    \bz = (\bz_m)_{m=0,\ldots,M} \in \prod_{m=0}^M (\bbR^d)^{\otimes m},
\]
where $z_0 \in \bbR$ and $\bz_m = \bv_{m,1} \otimes \bv_{m,2} \otimes \cdots \otimes \bv_{m,m}$ for ${m \geq 1}$.
We remark that this construction does not generally give tensors that can be signatures of paths.
However, they can be represented as linear combinations of signatures, since path signatures span the whole tensor algebra, hence the previous argument about the augmentation carries over. Also, informally, what gives the data-efficiency of inducing tensors is exactly that they are not represented in a basis of signatures.

By linearity of integration and the inner product, the inducing point covariance equals 
\begin{align} \label{eq:gpsig:inducing_cov_n}
    \bbE[f_{\bz}f_{\bz'}] = \sum_{m=0}^M \sigma_m^2 \langle \bz_m, \bz_m^\p \rangle = \sum_{m=0}^M \sigma_m^2 \prod_{k=1}^m\langle \bv_{m,k}, \bv_{m, k}^\p \rangle,
\end{align}
and the cross-covariance between paths and inducing tensors equals
\begin{align}\label{eq:gpsig:cross_cov}
    \bbE[f_{\bx}f_{\bz}] =\sum_{m=0}^M \sigma_m^2 \langle \S_m (\bx), z_m \rangle = \sum_{m=0}^M \int_{\bt \in \Delta_m([0, \hor{\bx}])} \langle \d \bx_{t_1},\bv_{m,1} \rangle \cdots \langle \d \bx_{t_m}, \bv_{m, m} \rangle,
\end{align}
while when working with discrete-time data, using the order-$1$ signature \eqref{eq:gpsig:discrete_sig_cov} it equals
\begin{align} \label{eq:gpsig:discr_cross_cov_n}
    \bbE[f_\bx f_\bz] = \sum_{\bi \in \Delta_m(\len{\bx})} \langle \delta \bx_{i_1}, \bv_{m,1} \rangle  \cdots \langle \delta \bx_{{i_m}},  \bv_{m, m} \rangle.
\end{align}
Below we use the approximation \eqref{eq:gpsig:discr_cross_cov_n} since it makes the recursive algorithms simpler but note that a simple modification exactly computes~\eqref{eq:gpsig:discr_cross_cov_n} for a marginal computational overhead.
\subsection{Algorithms.}
We need to compute the three covariance matrices:~\begin{enumerate*}[label=(\arabic*)] \item $K_{\bZ\bZ}$ of inducing tensors $\bZ$ and inducing tensors $\bZ$, \label{it:tens2tens} \item $K_{\bZ\bX}$ of inducing tensors $\bZ$ and sequences $\bX$, \label{it:tens2stream} \item $K_{\bX\bX}$ of sequences $\bX$ and sequences $\bX$.\label{it:stream2stream} \end{enumerate*}
Using the above tensor representations allows to give vectorized algorithms for~\ref{it:tens2tens} and~\ref{it:tens2stream} in Algorithms~\ref{alg:gpsig:inducing_cov} and \ref{alg:gpsig:cross_cov}, respectively.
For~\ref{it:stream2stream} we use a variation of Algorithm \ref{alg:back:sigkernel_p} which can be found in \cite[App.~D.2]{toth2020bayesian}. We use notation defined in Section~\ref{sec:back:notation}. For $v, v^\p \in V$, $d$ denotes the time to compute $\langle v, v^\p \rangle$, $c$ the memory requirement of $v$. 

\begin{algorithm}[t]
        \begin{footnotesize}
	\caption{Computing the inducing covariances $K_{\bZ\bZ}$}
	\label{alg:gpsig:inducing_cov}
	\begin{algorithmic}[1]
		\STATE {\bfseries Input:} Tensors $\bZ=(\bz_i)_{i=1,\dots,n_\bZ} \subset \prod_{n=0}^m V^{\otimes n}$, scalars $(\sigma^2_0, \sigma^2_1, \dots, \sigma^2_m)$, depth $m \in \bbN$ 
		\STATE Compute $K[i, j, n, k] \gets \langle v^i_{n, k}, v^j_{n, k} \rangle$ for $i, j \in [n_\bZ]$, $n \in [m]$ and $k \in [n]$ \\
		\STATE Initialize $R[i, j] \gets \sigma_0^2$ for $i, j \in [n_{\bZ}]$
		\FOR{$n=1$ {\bfseries to} $m$}
		\STATE Assign $A \gets K[:, :, n, 1]$
		\FOR{$k=2$ {\bfseries to} $n$}
		\STATE Iterate $A \gets K[:, :, n, k] \odot A$
		\ENDFOR
		\STATE Update $R \gets R + \sigma_n^2 \cdot A$
		\ENDFOR
		\STATE {\bfseries Output:} Matrix of inducing covariances $K_{\bZ\bZ} \gets R$
	\end{algorithmic}
        \end{footnotesize}
\end{algorithm}

\begin{algorithm}[t!]
        \begin{footnotesize}
	\caption{Computing the cross-covariances $K_{\bZ\bX}$}
	\label{alg:gpsig:cross_cov}
	\begin{algorithmic}[1]
		\STATE {\bfseries Input:} Tensors $\bZ=(\bz_i)_{i=1,\dots,n_\bZ} \subset \prod_{n=0}^m V^{\otimes n}$, \\ sequences $\bX=(\bx_i)_{i=1,\dots,n_{\bX}} \subset \Seq(\cX)$, scalars $(\sigma^2_0, \sigma^2_1, \dots, \sigma^2_m)$, depth $M \in \bbN$ 
		\STATE Compute $K[i, j, l, m, k] \gets \langle \bv^i_{m, k}, \delta \bx_{j, t_l} \rangle$ for $i \in [n_\bZ]$, $j \in [n_\bX]$, $l \in [L_\bX]$, $m \in [M]$ and $k \in [m]$ \\
		\STATE Initialize $R[i, j] \gets \sigma_0^2$ for $i \in [n_{\bZ}]$, $j \in [n_{\bX}]$
		\FOR{$m=1$ {\bfseries to} $M$}
		\STATE Assign $A \gets K[:, :, n, 1]$
		\FOR{$k=2$ {\bfseries to} $m$}
		\STATE Iterate $A \gets K[:, :, :, m, k] \odot A[:, :, \boxplus+1]$
		\ENDFOR
		\STATE Update $R \gets R + \sigma_m^2 \cdot A[:, :, \Sigma]$
		\ENDFOR
		\STATE {\bfseries Output:} Matrix of cross-covariances $K_{\bZ\bX} \gets R$
	\end{algorithmic}
    \end{footnotesize}
\end{algorithm}

\begin{proposition}\label{eq:gpsig:complexity}

Algorithm~\ref{alg:gpsig:inducing_cov} computes the covariance matrix $K_{\bZ\bZ}$ of $n_\bZ$ inducing points in $O(M^2 \cdot n_\bZ^2 \cdot d )$ steps.
Algorithm~\ref{alg:gpsig:cross_cov} computes the cross-covariance matrix $K_{\bZ\bX}$ in $O(M^2 \cdot n_\bX \cdot n_\bZ \cdot \len{\bX} \cdot d)$ steps.
Additionally to storing the inducing tensors $\bZ$, Algorithm~\ref{alg:gpsig:inducing_cov} requires $O(M^2 \cdot n_\bZ^2)$ memory, Algorithm~\ref{alg:gpsig:cross_cov} requires $O(M^2 \cdot n_\bX \cdot n_\bZ \cdot \len{\bX})$ memory.
\end{proposition}
Proposition~\ref{eq:gpsig:complexity} follows details the scalability and we emphasize:
\begin{enumerate*}[label=(\roman*)]
\item Both algorithms are linear in the maximal sequence length $\len{\bX}$.
\item $M$ is a hyperparameter, and in all our experiments we learnt from the data $M \leq 5$, thus the quadratic complexity in $M$ is negligible.
\item The memory cost of inducing tensors $\bZ$ is much less than for the data $\bX$, which is stored in $O(n_\bX \cdot \len{\bX} \cdot d)$ memory, which is important because the inducing tensors are variational parameters, and not amenable to subsampling, while the learning inputs can be subsampled as noted by~\cite{Hensman2013GaussianPF}.
Especially for \texttt{GP}Us memory cost is decisive and such savings are very important.
\end{enumerate*}

The computation of $K_{\bX\bX}$ detailed in Appendix \cite[App.~D.2]{toth2020bayesian} has time complexity $O((M + d) \cdot n_\bX^2 \cdot \len{\bX}^2)$ and memory of $O(d \cdot n_\bX \cdot \len{\bX} + n_\bX^2 \cdot \len{\bX}^2)$. However, given a factorizing likelihood, one only requires $K_\bX = [k(\bx, \bx)]_{\bx \in \bX}$, which eliminates the quadratic cost in $n_\bX$.
It turns out that this is enough to train on \texttt{GP}Us with reasonable minibatch sizes (e.g. $n_\bX = 50$) on several real world datasets. Finally, note that the \texttt{ELBO} \eqref{eq:gpsig:elbo} requires an additional matrix inversion and multiplication in $O(n_\bZ^2 \cdot n_\bX + n_\bZ^3)$ time, which is not significant in our case.

\section{Experiments} \label{sec:gpsig:experiments}
\paragraph{TS classification}

Using \texttt{GPFlow}~\cite{Matthews2017GPflowAG}, \texttt{Keras}~\cite{Chollet2015Keras}, we implemented three models: \texttt{GPSig}, \texttt{GPSigLSTM}, and \texttt{GPSigGRU}.
All three use the signature covariance with the low-rank inducing tensors of Section~\ref{sec:sparse var tensor}.
Note that in place of the signature kernel with the Euclidean lift \eqref{eq:gpsig:discrete_sig_cov}, we precompose the signature kernel with the \texttt{RBF} kernel as detailed in Section \ref{sec:back:sigkernels} and \ref{sec:back:discrete}, so that it has greater flexibility. All computations carry on mutatis mutandis by replacing inner product evaluations with the \texttt{RBF} kernel. Regarding the implemented models, \texttt{GPSig} is a plain vanilla variational \texttt{GP} classifier.
Previous applications of neural nets to covariance constructions, in particular~\cite{Wilson2016DeepK, al2017learning}, inspired \texttt{GPSigLSTM} and \texttt{GPSigGRU} that include an \texttt{RNN} as a sequence-to-sequence transformation
with $h $ hidden units; see Figures~\ref{fig:gpsig:GPSig} and \ref{fig:gpsig:GPSigRNN} where $\hat \bx$ denotes augmentation with lags and $\kernel_{\hat \bx}$ a static kernel lift. We benchmarked these \texttt{GP} models on $16$ multivariate \texttt{TS} classification datasets, a collection introduced in~\cite{baydogan2015multivarate} that was a semi-standard archive at the time of working on this project in \texttt{TS} classification, e.g. see the citations in \cite[App.~E.4]{toth2020bayesian} for papers that use these datasets. The same datasets are also used in \cite{Fawaz2019DLforTSC} to compare several deep learning architectures for \texttt{TSC}.

As Bayesian baselines we used three \texttt{GP} models: \begin{enumerate*}[label=(\roman*)] \item
\texttt{GPLSTM} and \item \texttt{GPGRU} consist of an \texttt{LSTM} and a \texttt{GRU} network with an \texttt{RBF} kernel on top, in which case the \texttt{RNN}s are used as a sequence-to-vector transformation from $\Seq(\bbR^d)$ to $\bbR^h$; \item \texttt{GPKConv1D} uses the convolutional kernel introduced in \cite{Wilk2017ConvGP} in $1$-dimension (time) \end{enumerate*}. Throughout we used sparse variational inference: for \texttt{GPSigLSTM}, \texttt{GPSigGRU}, \texttt{GPSig}, the inducing tensors detailed in Section~\ref{sec:sparse var tensor} are used; for \texttt{GPLSTM} and \texttt{GPGRU} the inducing points are located in the output space of the \texttt{RNN} layer, $\bbR^h$; for \texttt{GPKConv1D}, the inducing patches of \cite{Wilk2017ConvGP} are used.

We used $n_\bZ = 500$ for all models; further all use a static kernel in one form or another, which we fixed to be the \texttt{RBF} kernel.
The signature kernel was truncated\footnote{For these experiments, the $M=4$ value seemed to give an optimal trade-off between computational complexity and expressiveness of the kernel, see \cite[App.~E.1]{toth2020bayesian} for more details.} at $M=4$, and for \texttt{GPSig} $p=1$ lags were used; the \texttt{GPSigRNN}s did not use lags, as the sequence of hidden states already incorporate lagged information about past observations. The window size in \texttt{GPKConv1D} was set to $w=10$. 
The \texttt{RNN}-architectures were selected independently for all models by grid-search among $6$ variants, that is, the number of hidden units from $\{8, 32, 128\}$ and with or without dropout. For training, early stopping was used with $n = 500$ epochs patience; a learning rate of $\alpha = 1 \times 10^{-3}$; a minibatch size of $50$; as optimizer Adam \cite{kingma2014adam} and Nadam \cite{Dozat2015IncorporatingNM} were employed. Implementations, datasets, and the training and grid-search methodology are respectively detailed in Appendices E.1, E.2, and E.3 in \cite{toth2020bayesian}.

\begin{table}[t]
	\caption{Average ranks of \texttt{GP}s with the 1\textsuperscript{st} and 2\textsuperscript{nd} best in \textbf{bold} and \textit{italicized} for each row.}
	\label{table:gpsig:avg_ranks}
	\begin{center}
		\resizebox{\textwidth}{!}{
			\begin{sc}
				\begin{tabular}{lrrrrrr}
					\toprule
					 & \texttt{GPSigLSTM} & \texttt{GPSigGRU} & \texttt{GPSig} & \texttt{GPLSTM} & \texttt{GPGRU} & \texttt{GPKConv1D} \\
					\midrule
					Mean rank (nlpp, $n_\bX < 300$) & $\mathit{2.80}$ & $2.90$ & $\mathbf{2.20}$ & $4.70$ & $4.00$ & $4.40$ \\
					Mean rank (acc., $n_\bX < 300$) & $\mathit{3.00}$ & $3.10$ & $\mathbf{2.80}$ & $4.25$ & $4.25$ & $3.60$ \\
					Mean rank (nlpp, $n_\bX \geq 300$) & $\mathbf{2.33}$ & $3.33$ & $\mathit{2.83}$ & $4.83$ & $4.33$ & $3.33$ \\
					Mean rank (acc., $n_\bX \geq 300$) & $\mathbf{2.17}$ & $3.50$ & $\mathit{3.00}$ & $4.17$ & $4.33$ & $3.83$ \\
					Mean rank (nlpp, all) & $\mathit{2.63}$ & $3.06$ & $\mathbf{2.44}$ & $4.75$ & $4.13$ & $4.00$ \\
					Mean rank (acc., all) & $\mathbf{2.69}$ & $3.25$ & $\mathit{2.88}$ & $4.22$ & $4.28$ & $3.69$ \\
					\bottomrule
				\end{tabular}
			\end{sc}
  }
	\end{center}
	\end{table}

\begin{figure*}[h!]
	\centering
	\begin{minipage}{0.29\textwidth}
		\centering
		\includegraphics[height=1.85in]{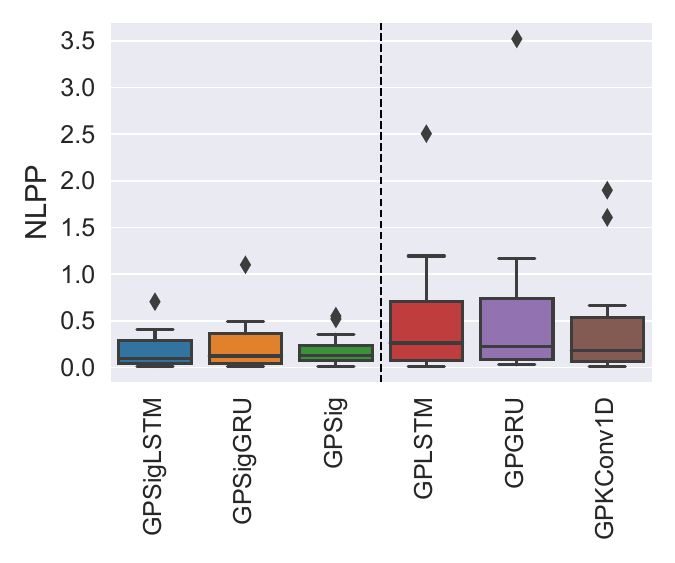}
	\end{minipage}
	\begin{minipage}{0.69\textwidth}
		\centering
		\includegraphics[height=1.85in]{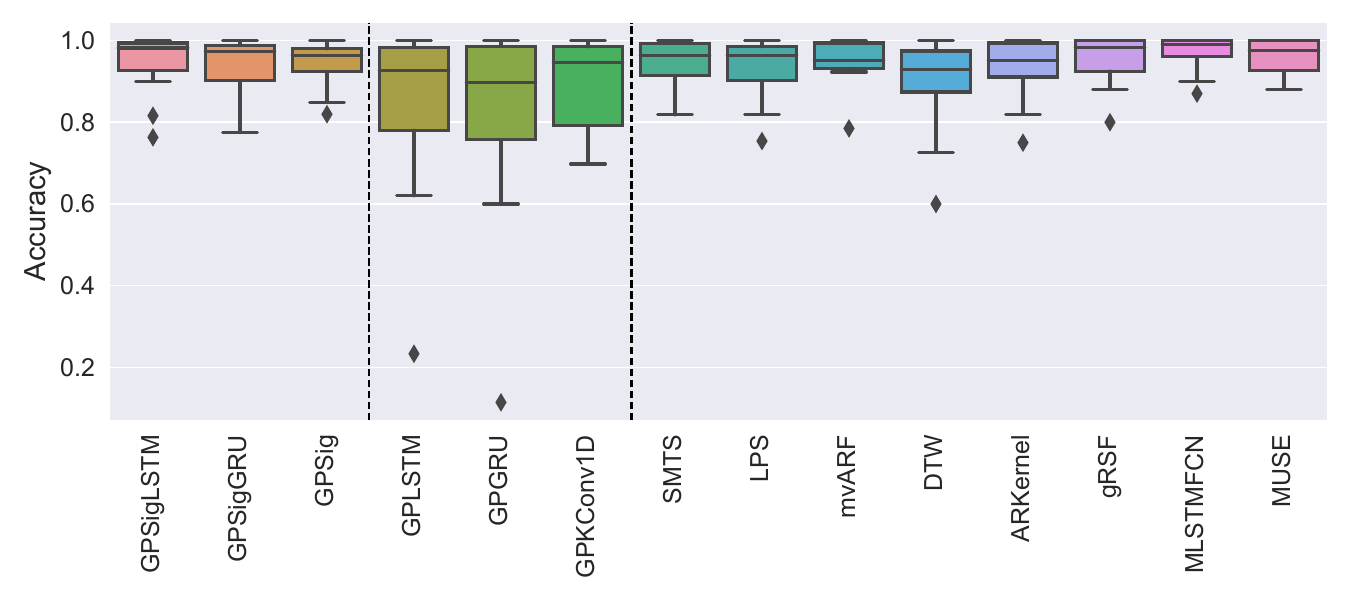}
	\end{minipage}
	\caption{Box-plots of negative log-predictive probabilities (left) and classification accuracies (right) on 16 TSC datasets.}
	\label{fig:gpsig:boxplots}
\end{figure*}

\begin{figure}[h!]
	\begin{minipage}{0.44\textwidth}
		\centering
		\begin{tikzpicture}[scale=0.9, every node/.style={scale=0.9}, shorten >=1pt,draw=black!50, node distance=\layersep]
	
		\node[node, text centered, shading=green] at (0.0,0.5) (x) {$\bx$};
		\node[text centered] at (0.0,1.25) {$\Seq\pars{\bbR^d}$};
		
		\node[node, text centered, shading=blue] at (1.8,0.5) (kx) {$\kernel_{\hat\bx}$};
		\node[text centered] at (1.8, 1.25) {$\Seq(\Hil)$};
		
		\draw[->, color=black] (x) -- (kx);
		
		(4.0,0) rectangle ++(1,1);
		\node[node, text centered, shading=blue] at (3.6,0.5) (sigx) {$\Phi(\kernel_{\hat\bx})$};
		\node[text centered] at (3.6, 1.25) {$\TV{\Hil}$};
		
		\draw[->, color=black] (kx) -- (sigx);
		
		\node[node, text centered, shading=red] at (5.4,0.5) (gp) {$f_{\bx}$};
		\node[text centered] at (5.4, 1.25) {$\GP$};
		
		\draw[->, color=black] (sigx) -- (gp);
		\end{tikzpicture}
		\captionof{figure}{The \texttt{GPSig} model}
		\label{fig:gpsig:GPSig}     
	\end{minipage}
	\begin{minipage}{0.56\textwidth}
		\centering
		\begin{tikzpicture}[scale=0.9, every node/.style={scale=0.9}, shorten >=1pt,draw=black!50, node distance=\layersep]
	
		\node[node, text centered, shading=green] at (0.0,0.5) (x) {$\bx$};
		\node[text centered] at (0.0, 1.25) {$\Seq\pars{\bbR^d}$};
		
		\node[node, text centered, shading=blue] at (1.8,0.5) (rnnx) {$\phi_\theta(\hat\bx)$};
		\node[text centered] at (1.8, 1.25) {$\Seq(\bbR^h)$};
		
		\draw[->, color=black] (x) -- (rnnx);
		
		\node[node, text centered, shading=blue] at (3.6, 0.5) (krnnx) {$\kernel_{\phi_\theta(\hat\bx)}$};
		\node[text centered] at (3.6, 1.25) {$\Seq(\Hil)$};
		
		\draw[->, color=black] (rnnx) -- (krnnx);
		
		\node[node, text centered, shading=blue] at (5.4, 0.5) (sigkrnnx) {$\Phi(\kernel_{\phi_\theta(\hat\bx)})$};
		\node[text centered] at (5.4, 1.25) {$\TV{\cH_k}$};
		
		\draw[->, color=black] (krnnx) -- (sigkrnnx);
		
		\node[node, text centered, shading=red] at (7.2, 0.5) (gp) {$f_{\bx}$};
		\node[text centered] at (7.2, 1.25) {$\GP$};
		
		\draw[->, color=black] (sigkrnnx) -- (gp);
		\end{tikzpicture}
		\captionof{figure}{The \texttt{GPSigRNN} model}
		\label{fig:gpsig:GPSigRNN}     
	\end{minipage}
\end{figure}

\paragraph{Discussion of results}
For \texttt{GP}s, we report accuracies and negative log-predictive probabilities (\texttt{NLPP}), the latter take not only accuracies, but the calibration of probabilities into account as well.
Table \ref{table:gpsig:avg_ranks} shows the average ranks among the \texttt{GP}s.
The full table of \texttt{NLPP}s and accuracies with mean and standard deviation over 5 model trains are reported in Appendix~\ref{app:gpsig:benchmark} in Table~\ref{table:gpsig:full_nlpp_results} and Table~\ref{table:gpsig:full_acc_results}.
As non-Bayesian baselines, we report accuracies of eight frequentist \texttt{TS} classifiers in Table~\ref{table:gpsig:freq_acc_results}. On Figure \ref{fig:gpsig:boxplots}, we visualize the box-plot distributions of \begin{enumerate*}[label=(\roman*)] \item negative log-predictive probabilities of \texttt{GP}s, \item accuracies of both \texttt{GP}s and frequentist methods. \end{enumerate*}

The signature models perform consistently the best in terms of average rankings of both \texttt{NLPP} and accuracy among the \texttt{GP}s. Particularly, they achieve stronger mean performance and a smaller variance across datasets. To explain this, inspecting the results in Tables \ref{table:gpsig:full_nlpp_results}, \ref{table:gpsig:full_acc_results}, we observe that all other \texttt{GP} baselines perform very poorly on some datasets, while the signature based models perform at least moderately well on {all datasets}. We believe this ties in to the universality property of signatures, see Section \ref{subsec:back:props}. The convolutional \texttt{GP}, \texttt{GPKConv1D}, which also has a very small parameter set, performed rather competitively with the deep kernel baselines, even on larger datasets. Comparison among variants of \texttt{GPSig} can be summarized as follows: for smaller datasets ($n_\bX < 300$), \texttt{GPSig} outperforms other variants as it has a very small parameter set; for larger datasets $(n_\bX \geq 300)$, \texttt{GPSigLSTM} performs best which conforms with the intuition that \texttt{RNN}s suffer from small sample sizes.
A related observation is that \texttt{GPLSTM} and \texttt{GPGRU} perform about on par, while \texttt{GPSigLSTM} does much better than \texttt{GPSigGRU}, which suggests that the signature makes explicit use of the additional complexity of the \texttt{LSTM} network.

Compared only in terms of accuracy, \texttt{GPSig} competes with frequentist classifiers: it outperforms the usual \texttt{DTW} baseline and competes with state-of-the art classifiers such as \texttt{MUSE} and \texttt{MLSTMFCN}. Purely based on accuracy, these win overall, but the difference is usually small, hence the extra Bayesian advantages come at a small cost. Obviously \texttt{TS} classification is a vast field and many other models could be considered; e.g.~we did not use recurrent \texttt{GP}s or \texttt{GPSSM}s since
\begin{enumerate*}[label=(\arabic*)]
\item 
they have so-far not been used for \texttt{TS} classification, possibly because there is no sequential nature in the output space,
\item we did not find a \texttt{GPFlow} implementation that would allow to use sequence kernels in the \texttt{GP} transition function. (Implementation of~\cite{ialongo2019overcoming} does currently not allow taking subsequences of past states in the transition function.
An implementation would require much further work, but an interesting project would be to combine our models.) 
\end{enumerate*}

\begin{table}[t]
	\caption{Average ranks of \texttt{GP}s in terms of negative log-predictive probabilities (\texttt{NLPP})  across data domains with the 1\textsuperscript{st} and 2\textsuperscript{nd} best in \textbf{bold} and \textit{italicized} for each row.}
	\label{table:gpsig:domain}
	\begin{center}
			\begin{sc}
				\begin{tabular}{lrrrrrr}
					\toprule
					Domain & \texttt{GPSigLSTM} & \texttt{GPSigGRU} & \texttt{GPSig} & \texttt{GPLSTM} & \texttt{GPGRU} & \texttt{GPKConv1D} \\
					\midrule
Handwriting & $\mathit{2.75}$ & $4.00$ & $\mathbf{1.75}$ & $3.75$ & $5.50$ & $3.25$ \\
Motion & $\mathbf{1.92}$ & $2.42$ & $\mathit{2.25}$ & $5.17$ & $4.08$ & $5.17$ \\
Sensor & $\mathit{2.75}$ & $4.25$ & $\mathbf{2.38}$ & $5.13$ & $3.50$ & $3.00$ \\
Speech & $4.00$ & $\mathbf{1.50}$ & $5.00$ & $4.50$ & $\mathit{2.50}$ & $3.50$ \\
					\bottomrule
				\end{tabular}
			\end{sc}
	\end{center}
	\end{table}

\paragraph{Domain analysis}
Next, we provide a breakdown of the performance of the various \texttt{GP} models across data domains. The datasets can be categorized into 4 different domains: handwriting, motion, sensor, and speech. The average ranks of \texttt{NLPP}s of each \texttt{GP} model for each data domain is provided in Table \ref{table:gpsig:domain}. We observe the following points: \begin{enumerate*}[label=(\arabic*)] \item on handwriting datasets, the vanilla \texttt{GPSig} model provides the best performance with \texttt{GPSigLSTM} in second and \texttt{GPKConv1D} in third while \texttt{GPSigGRU} underperforms; \item on motion datasets, the additional flexibility of the preprocessing \texttt{LSTM} layer shines and \texttt{GPSigLSTM} takes first place while \texttt{GPSig} second and \texttt{GPSigGRU} third; \item on sensor datasets again \texttt{GPSig} performs the best and \texttt{GPSigLSTM} the second best while \texttt{GPSigGRU} severely underperforms; \item somewhat surprisingly, on speech datasets the \texttt{GRU} architecture provides the best performance, and \texttt{GPSigGRU} takes first place while \texttt{GPGRU} second and \texttt{GPKConv1D} third while the other models underperform.    
\end{enumerate*}
Overall, we conclude that the vanilla \texttt{GPSig} model often provides sufficient performance, while augmenting it with a preprocessing \texttt{LSTM} layer can often provide further improvements. However, in some cases, i.e.~as seen on speech data here, the inductive bias of the \texttt{GRU} layer is needed, and the signature kernel can only perform well when the data has been first fed through one such layer.

\paragraph{Inducing tensors vs inducing sequences} Our results rely on the inter-domain approach using tensors to locate inducing points from Section~\ref{sec:sparse var tensor}. 
An alternative is to use sequences for the inducing points, $\bZ \subset \Seq(\cX)$, and controlling their maximal length $\len{\bZ} = \max_{\bz \in \bZ} \len{\bz}$ to be of order $m$, i.e. $\len{\bZ} \sim m$.
We coin this approach \textit{inducing sequences}.
Intuitively, one expects the inducing tensors to be more efficient than inducing sequences, since they make full use of the structure of the signature feature space/covariance.
To test this intuition, we compared the performance of the inducing tensors and inducing sequences subject to both having the same computational complexitiy.
For this experiment, we took the AUSLAN dataset~\cite{Dua2017UCI}, which consists of $n_c = 95$ classes for $n_\bX = 1140$ training examples.
This is a challenging dataset as the inducing variables need to characterize the abundance of classification boundaries.

We used \texttt{GPSig} with the same settings as in the previous experiments.
The hyperparameters of the kernel were a-priori learnt with $n_{\bZ} = 500$ inducing tensors, and we purely investigated how the quality of the approximation changes for both approaches by varying the number of inducing points $n_\bZ$.
For each number of inducing variables, both approaches were trained independently $5$ times for $300$ epochs with random initialization of the inducing variables, for details \cite[App.~E.3]{toth2020bayesian}.
We plot on Figure \ref{fig:gpsig:ind_tens_vs_ind_seq} three metrics: \begin{enumerate*}[label=(\arabic*)] \item the achieved \texttt{ELBO}, \item the achieved accuracy, and \item \texttt{NLPP} on the testing set \end{enumerate*}.
At $n_{\bZ} = 500$ both approaches are close to saturation, but the inducing tensors consistently perform better.
We remark that in practice, an important aspect is also how well the kernel hyperparameters can be recovered, that we did not consider here, and is a tricky question for sparse variational inference in general \cite{bauer2016understanding}.
Although, intuition suggests that the closer the model to saturation is with respect to the inducing points, the more consistent should the optimization be with un-sparsified variational inference.

\begin{figure}[t]
	\centering
	\begin{minipage}{0.49\textwidth}
		\hspace{-40pt}
		\centering
		\includegraphics[width=3.25in]{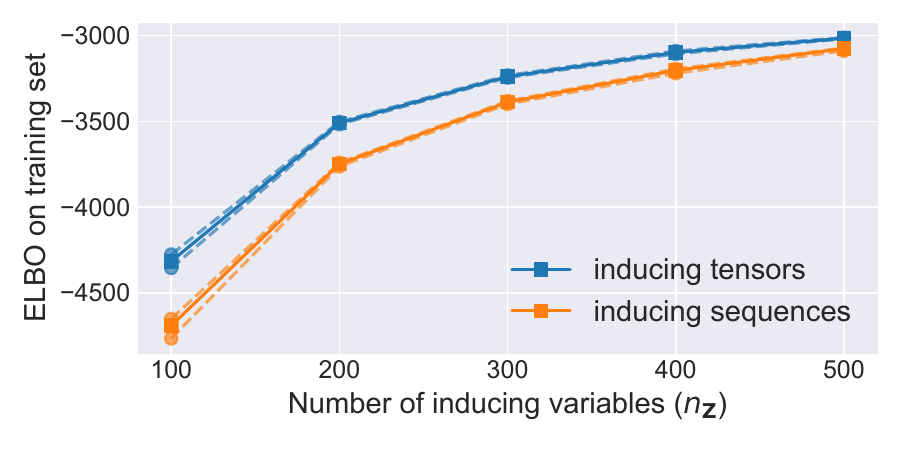}
	\end{minipage}
	\begin{minipage}{0.49\textwidth}
		\hspace{-40pt}
		\centering
		\includegraphics[width=3.25in]{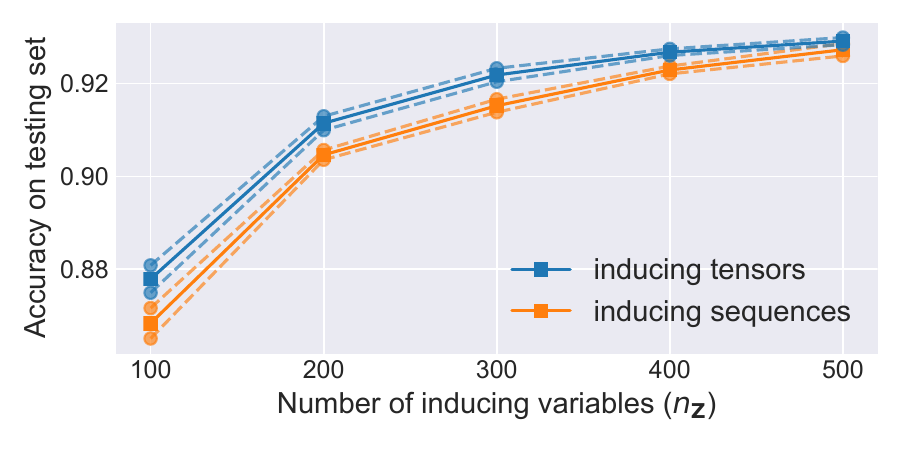}
	\end{minipage}
	\begin{minipage}{0.49\textwidth}
		\hspace{-40pt}
		\centering
		\includegraphics[width=3.25in]{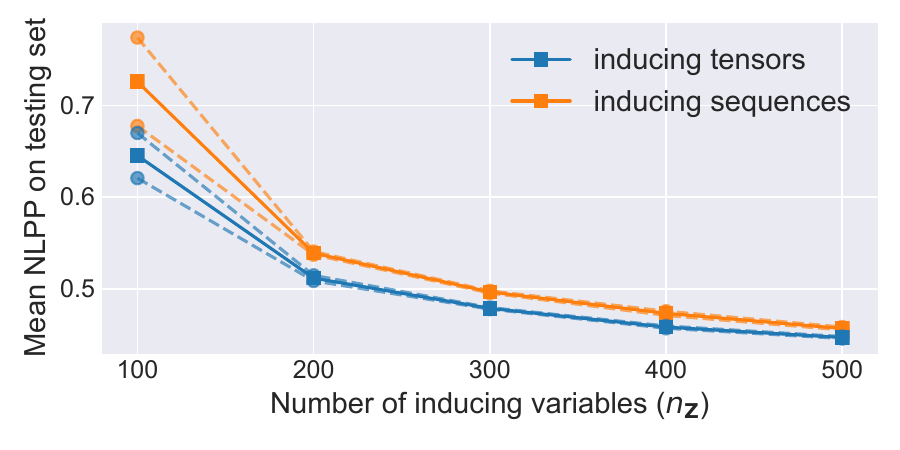}
	\end{minipage}
	\caption{Achieved \texttt{ELBO} (top), accuracy (middle), mean \texttt{NLPP} (bottom) after $300$ epochs of training the variational parameters with random initialization and frozen pre-learnt kernel hyperparameters. Solid is the mean over $5$ independent runs, dashed is the $1$-std region.}
	\label{fig:gpsig:ind_tens_vs_ind_seq}
	\vspace{-10pt}
	\end{figure}

\paragraph{Visualizing inducing tensors}
To gain more intuition, we visualized the feature space for one of the trained models on AUSLAN with $n_\bZ = 500$ inducing tensors.
We used \texttt{UMAP}~\cite{mcinnes2018umap-software} with the semi-metric $(\bx, \by) \mapsto \sqrt{\kernel(\bx, \bx) + \kernel(\by, \by) - 2 \kernel(\bx, \by)}$ for $\bx, \by \in \cX \cup \cX^\p$, see~Figure~\ref{fig:gpsig:tens_umap}.
There are two imminent observations: \begin{enumerate*}[label=(\roman*)] \item in the point cloud corresponding to the data, the classes hardly look linearly separable; \label{umap:point1} \item the tensors, however, seem to live on a completely separate subspace than the data\label{umap:point2}.\end{enumerate*} 
The algorithm achieves $92\%$ accuracy on this set, therefore, point \ref{umap:point1} is likely due to information being lost in the projection. 
However, point \ref{umap:point2} challenges the intuition about classical sparse variational inference, that the inducing points are located mixed-in with the data-points, concentrating close to the classification boundaries \cite{Hensman2015Scalable}.
In general, the mechanism of how inter-domain inducing points represent the information in the data seems to be more complicated than classically.

To explain point \ref{umap:point2}, we remark that this phenomenon is not surprising at all: signature features live in a manifold that is embedded in the linear tensor space $\TV{\Hil}$. 
In general, sparse tensors of the form~\eqref{eq:gpsig:sparse_tens} will {not} be signatures of paths. 
We believe variational inference works because of an interplay of two factors:
Firstly, signatures of finite sequences can be written as finite linear combinations of such sparse tensors. 
Secondly, the prior conditional term used to define $q(f_{\bx} \vert f_{\bZ})) = p(f_{\bx} \vert f_{\bZ})$.
The feature space is {congruent} to the prior \texttt{GP}~\cite{berlinet2011reproducing}, which means that for $\bx \in \Seq(\cX)$, the value of $f_{\bx}$ given $f_{\bZ}$ is not only almost deterministic when $\bx$ is close to any of $\bz \in \bZ$, but when it is close to any linear combinations of elements in $\bZ$.
By the first remark this can always be achieved given a large enough $n_\bZ$.
To sum up, the inducing tensors do not represent signature features individually, but form atomic building blocks such that their linear combinations induce the actual variational posterior at the data-examples.

\begin{figure}[t]
	\centering
	\begin{minipage}{\textwidth}
		\hspace{-40pt}
		\centering
		\includegraphics[width=4.in]{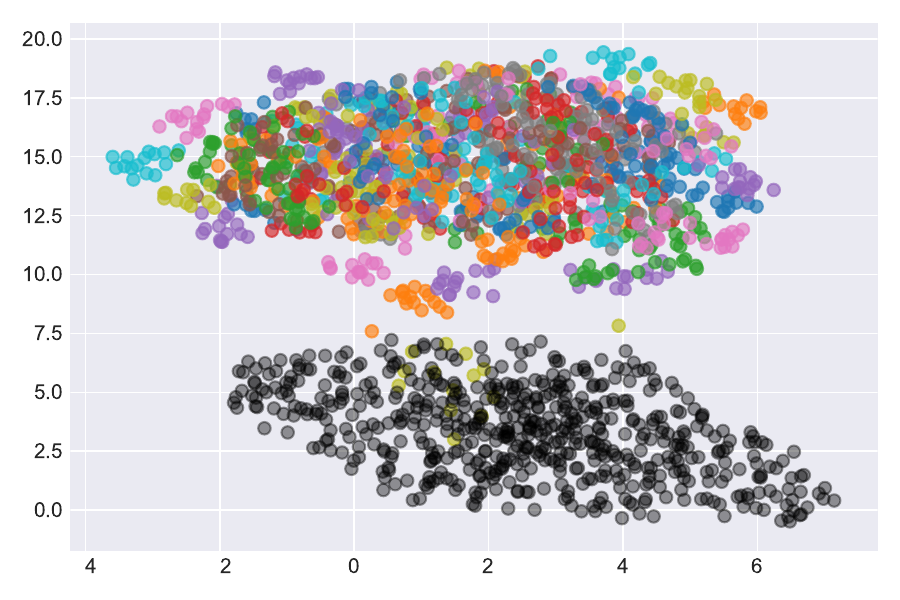}
	\end{minipage} \hspace{-5pt}
	\vspace{-10pt}
	\caption{A \texttt{UMAP} visualization of the features corresponding to data-points (coloured), and inducing tensors (black) in the feature space on the AUSLAN dataset.}
	\label{fig:gpsig:tens_umap}
	\vspace{-10pt}
	\end{figure}

\section{Conclusion}
We used a classical object from stochastic analysis -- signatures -- to define a \texttt{GP} for sequential data.
The \texttt{GP} inherits many of the theoretical guarantees that are known for signature features such as universality and parameterization invariance.
To make it scalable, we develop ``inducing tensors'' that exploit the structure of the feature space, inter-domain inducing points, and variational inference. 
Applied in a plain vanilla variational framework, this yields a classifier, \texttt{GPSig}, that is not only competitive in terms of \texttt{NLPP} with other \texttt{GP} models, but also with state-of-the-art frequentist \texttt{TS} classifiers in terms of accuracy alone. As one of our reviewers remarked, several datasets we consider have a strong signal-to-noise ratio, which makes it worthwhile to point out that even for such datasets, the alternative \texttt{GP} baselines suffer on at least some of them, while the proposed models are consistently able to learn on all datasets. This observation ties in to the universality property, and it suggests that \texttt{GP}s with signatures can be a good starting point when building Bayesian models on time series datasets.

We also demonstrate that signatures can be used as a building block in deep kernels to build larger \texttt{GP} models that leverage the benefits of both, \texttt{RNN}s and signatures. Interestingly, we find that the vanilla \texttt{GPSig} model outperforms the \texttt{GPSigRNN}s for smaller datasets, conforming to the intuition that smaller sample sizes are detrimental for recurrent neural nets. To really get the best of both worlds, one could insert an additional model selection step, that specifies whether a parametric transformation is layer used before feeding the input into the kernel or not. Alternatively, it could also be possible to increase the flexibility of the sequential \texttt{GP} model while staying within a purely nonparametric framework using deep \texttt{GP}s \cite{damianou2013deep} by e.g.~applying a \texttt{GP} layer as observation-wise state-space embedding before the kernel computation. The inference framework of \cite{salimbeni19deep} for deep \texttt{GP}s could also come in handy when moving to datasets with lower signal-to-noise ratios, which can require \texttt{GP} models capable of handling not only epistemic (reducible) uncertainty, but aleatoric (irreducible) uncertainty in the data. It would also be interesting to see if such sequence kernels can be used to improve recurrent \texttt{GP} models \cite{mattos2016recurrent, ialongo2019overcoming} by incorporating sequential information into the \texttt{GP} transition function, that could potentially allow for a more efficient latent state representation.

\begin{subappendices}

 \section{Continuity of \texttt{GP}s with Signature Covariances} \label{app:gpsig:gpsig}
 We specified a covariance function $\kernel$ on the set $\Paths(\cX)$ as inner product of the signature map.
 This guarantees that $(\bx,\by) \mapsto \kernel(\bx,\by) = \langle  \Phi(\bx), \Phi(\by) \rangle$ is a positive definite function, and from the general theory of stochastic processes by the Kolmogorov extension theorem \cite{dudley2002real}, the existence of a centered \texttt{GP} $(f_{\bx})_{\bx \in \Paths(\cX)}$ such that $\bbE[ f_{\bx} f_{\by}] = \kernel(\bx, \by)$.  
 However, this does not guarantee that the sample paths $\bx \mapsto f_{\bx}$ are continuous. 
 Seminal work of Dudley~\cite{dudley1967sizes} showed that such regularity estimates can be derived by bounding the growth of the covering number of the index set of the \texttt{GP} $f$ under the semi-metric
 \begin{align}
  d_\kernel(\bx,\by) = \sqrt{\bbE[ |f_{\bx} - f_{\by}|^2 ]} = \sqrt{\kernel(\bx,\bx) - 2 \kernel(\bx,\by) + \kernel(\by,\by)}. 
 \end{align}
 Already when when the index set is finite dimensional ``nice'' covariance functions can lead to discontinuous \texttt{GP}s, see e.g.~Section 1.4.~in \cite{adler2009random}. 
 Our \texttt{GP} has as index set the space of bounded variation paths $\Paths(\cX)$, which is infinite-dimensional so some caution is needed. 
 However, as we show below we can cover this space by lattice paths and derive covering number estimates.
 
\begin{theorem}\label{thm:covering}
For $V > 0$ and $\epsilon>0$ denote with $N(\epsilon, V)$ the covering number of the set
\begin{align}
    \Paths_L(\cX)= \curls{\bx \in \Paths(\cX) \setgiven \norm{\bx}_\onevar \leq V}
\end{align}
of bounded variation paths of length less or equal than $V$ under the $d_{\sigkernel}$ pseudo-metric.
  Then
  \begin{align}
   \log_2 N(\epsilon,L) \le 2 (d+1) L\frac{\sqrt M}{\epsilon}
  \end{align}
  \end{theorem}
  \begin{proof}
    By definition of the metric $d_{\sigkernel}$
    \begin{align}
    d_{\sigkernel}(x,y) &\equiv \sqrt{\langle \Phi(\bx)-\Phi(\by),
                 \Phi(\bx)-\Phi(\by) \rangle} \\ &= \|\Phi(\bx)-\Phi(\by)\|.
    \end{align}
    By definition $\Phi$ and of the norm $\|\cdot\|$ on $\prod_{m=0}^M \pars{\bbR^d}^{\otimes m}$ this reads 
    \begin{align}\label{eq:gpsig:d_k}
d_{\sigkernel}^2(\bx,\by) &= \sum_{m=1}^M \| \int d\bx^{\otimes m} - \int d\by^{\otimes m}\|^2  \\&\le M \max_{m=1,\ldots,M}\Delta_m^2(\bx,\by)
    \end{align}
   where we denote $\Delta_m(\bx,\by)=\| \int d\bx^{\otimes m} - \int d\by^{\otimes m}\|$.
   Let $\Lattice{s}{L} \subset \Paths_L(\cX)$ be the set of lattice paths starting at $0 \in \bbR^d$ that take steps of size $s$ and that are of total length at most $L$. 
By the results in Section 4 of~\cite{lyons2011inversion}, for every $\bx \in \Paths_L(\cX)$ and every $n \ge 1$ there exists a $\by \in \Lattice{L 2^{-n}}{L}$ such that for every $m \ge 1$, 
\begin{align}\label{eq:gpsig:weijun}
\Delta_m(\bx,\by)\le \frac{d}{2^{n-1}} \frac{4 L^{m-1}}{(m-1)!}. 
\end{align}
Since $\frac{ L^{m-1}}{(m-1)!}\le e^L-1 $ we can apply \eqref{eq:gpsig:weijun} with
$n=n(\epsilon)= 1-\log_2 \frac{{\epsilon} }{d\sqrt{M}4(e^L-1)}$ to get $\Delta_m(\bx,\by) \le \epsilon$.
Hence, there exists a lattice path $\by \in \Lattice{L2^{-n(\epsilon)}}{L}$ such that
\begin{align}
 d_{\sigkernel}(\bx,\by) \le \epsilon.
\end{align}
Further, the set $\Paths(\cX)$ is finite and we can bound it by
\begin{align}
 |\Lattice{L2^{-n(\epsilon)}}{L} | &\le (2^d+1)^{L2^{n(\epsilon)}} \le  2^{(d + 1) L2^{n(\epsilon)}} \\&=   2^{2 (d + 1)L2^{n(\epsilon)-1}} = 2^{2 (d+1)L\frac{\sqrt M}{\epsilon}} 
\end{align}
 where the first inequality follows since a lattice path has at every step $2^d$ directions to choose from and in addition can choose not to make a step.  
The last equality follows from the definition of $n(\epsilon)$.
Since $\bx\in \Paths_L(\cX)$ was chosen arbitrary it follows that $\Paths_L(\cX)$ can be covered by $2^{2(d+1)L\frac{\sqrt M}{\epsilon}}$ balls of radius $\epsilon$ centered at lattice paths.
\end{proof}
Theorem~\ref{thm:covering} combined with Dudley's celebrated entropy estimates gives regularity results for samples of our \texttt{GP}.
In fact, this even yields a modulus of continuity for our \texttt{GP}.
\begin{theorem}
  There exists a centered \texttt{GP} $(f_{\bx})_{\bx \in \cX^L_{paths}}$ that has a covariance $\bbE[f_{\bx} f_{\by}]$ the signature covariance function $k(\bx,\by)=\langle \Phi(\bx), \Phi(\by) \rangle$.
  Moreover, if we denote its modulus of continuity on $\Paths_L(\cX)$ with 
  \begin{align}
   \omega(\delta)= \sup_{\substack{\bx,\by \in \Paths_L(\cX)\\ d_{\sigkernel}(\bx,\by)< \delta}} |f_{\bx} - f_{\by}| 
  \end{align}
  then it holds with probability one that 
  \begin{align}\label{eq:gpsig:modulus}
\limsup_{\delta \rightarrow 0} \frac{\omega(\delta) }{\sqrt {\delta}4\sqrt{ (d+1)L \sqrt{M}}  + c \delta \sqrt{\ln \ln \frac{1}{\delta}} } \le 24
  \end{align}
  where $c>0$ denotes a universal constant.
\end{theorem}
\begin{proof}
The existence of a centered \texttt{GP} $ (\hat f_{\bx})_{\bx}$ with covariance $k$ follows from general results about Gaussian processes. 
The existence of a continuous modification $(f_{\bx})_{\bx \in \cX_{paths,L}}$ of follows from Dudley's theorem if 
 \begin{align}
  \int_0^1 \sqrt{ \log_2 N (\epsilon,L)} d \epsilon  < \infty
 \end{align}
 but by Theorem~\ref{thm:covering} we have \[\int_0^1 \sqrt{ \log_2 N (\epsilon,L)} d \epsilon \le \sqrt {2 (d+1 )L \sqrt{M}} \int_0^1 \frac{1}{\sqrt \epsilon}d \epsilon < \infty.\]
 Dudley's results immediately yield a modulus of continuity in probability.
 By standard arguments this can be strengthened to give an almost sure modulus of continuity. 
 Concretely, we use the formulation given in Theorem 2.7.1 in Chapter 5 of \cite{khoshnevisan2002multiparameter} which guarantees that
 \[
 \limsup_{\delta \rightarrow 0} \frac{\omega(\delta) }{\int_0^\delta \sqrt{ N(\frac{\epsilon}{2},L)} d\epsilon + c \delta \sqrt{\ln \ln \frac{1}{\delta}} } \le 24.
 \]
The bound~\eqref{eq:gpsig:modulus} follows immediately since first term in the denominator equals 
 \begin{align}
   \int_0^\delta \sqrt{ \log_2 N \left(\frac{\epsilon}{2},L\right)} d \epsilon  &= \sqrt {2(d+1)L \sqrt{M}}2 \sqrt{ 2}\sqrt{ \delta } \\&= \sqrt{\delta}4\sqrt{  (d+1)L \sqrt{M}}.
 \end{align}
\end{proof}

\section{Benchmark results} \label{app:gpsig:benchmark}

We report in Table \ref{table:gpsig:full_nlpp_results} and Table \ref{table:gpsig:full_acc_results} the negative log-predictive probabilities and accuracies of the \texttt{GP} models considered in Section \ref{sec:gpsig:experiments}. For each method-dataset pair, 5 models were trained with initialization described in \cite[App.~E.3]{toth2020bayesian}. The variance of the results is therefore due to random initialization and minibatch randomness. The \texttt{RNN} based models used the architectures detailed in \cite[Tab.~4]{toth2020bayesian}. As non-Bayesian baselines, we report the results of recent frequentist \texttt{TS} classification methods from the respective publications, that is, \cite{Cuturi2011AR, baydogan2015learning, Baydogan2015TimeSR, karlsson2016generalized, tuncel2018autoregressive, Schfer2017MUSE, Karim2019LSTMFCN}. Particularly for \texttt{MLSTMFCN}, we report the same results as in \cite{Schfer2017MUSE}.

\begin{table*}[ht]
	\caption{Mean and standard deviation of \texttt{NLPP} on test sets over $5$ independent runs}
	\label{table:gpsig:full_nlpp_results}
	\vskip 0.15in
	\begin{center}
		\begin{small}
			\begin{sc}
   \resizebox{\textwidth}{!}{
				\begin{tabular}{lcccccc}
					\toprule
					Dataset & \texttt{GPSigLSTM} & \texttt{GPSigGRU} & \texttt{GPSig} & \texttt{GPLSTM}  & \texttt{GPGRU} & \texttt{GPKConv1D}\\
					\midrule
                        Arabic Digits & $0.047 \pm 0.030$ & $0.023 \pm 0.006$ & $0.071 \pm 0.021$ & $0.082 \pm 0.022$ & $0.066 \pm 0.010$ & $0.050 \pm 0.003$ \\ 
                        AUSLAN & $0.106 \pm 0.007$ & $0.123 \pm 0.045$ & $0.550 \pm 0.114$ & $0.650 \pm 0.071$ & $0.248 \pm 0.063$ & $1.900 \pm 0.139$ \\ 
                        Character Traj. & $0.031 \pm 0.007$ & $0.258 \pm 0.265$ & $0.108 \pm 0.005$ & $2.506 \pm 1.007$ & $3.523 \pm 0.635$ & $0.409 \pm 0.141$ \\ 
                        CMUsubject16 & $0.088 \pm 0.020$ & $0.040 \pm 0.009$ & $0.089 \pm 0.027$ & $0.270 \pm 0.080$ & $0.089 \pm 0.039$ & $0.255 \pm 0.002$ \\ 
                        DigitShapes & $0.008 \pm 0.001$ & $0.035 \pm 0.051$ & $0.021 \pm 0.001$ & $0.013 \pm 0.002$ & $0.727 \pm 0.569$ & $0.035 \pm 0.003$ \\ 
                        ECG & $0.402 \pm 0.023$ & $0.431 \pm 0.037$ & $0.356 \pm 0.008$ & $0.496 \pm 0.018$ & $0.601 \pm 0.137$ & $0.543 \pm 0.019$ \\ 
                        Jap.~Vowels & $0.080 \pm 0.031$ & $0.053 \pm 0.009$ & $0.069 \pm 0.003$ & $0.061 \pm 0.029$ & $0.052 \pm 0.005$ & $0.067 \pm 0.001$ \\ 
                        Kick vs Punch & $0.301 \pm 0.109$ & $0.493 \pm 0.128$ & $0.224 \pm 0.014$ & $0.696 \pm 0.046$ & $0.674 \pm 0.037$ & $0.662 \pm 0.017$ \\ 
                        LIBRAS & $0.320 \pm 0.045$ & $0.346 \pm 0.091$ & $0.259 \pm 0.021$ & $0.911 \pm 0.056$ & $1.110 \pm 0.248$ & $1.608 \pm 0.311$ \\ 
                        NetFlow & $0.218 \pm 0.009$ & $0.259 \pm 0.078$ & $0.189 \pm 0.014$ & $0.251 \pm 0.041$ & $0.194 \pm 0.011$ & $0.168 \pm 0.081$ \\ 
                        PEMS & $0.704 \pm 0.130$ & $1.100 \pm 0.064$ & $0.520 \pm 0.058$ & $1.194 \pm 0.308$ & $0.784 \pm 0.111$ & $0.537 \pm 0.010$ \\ 
                        PenDigits & $0.289 \pm 0.127$ & $0.399 \pm 0.206$ & $0.146 \pm 0.007$ & $0.185 \pm 0.027$ & $0.187 \pm 0.043$ & $0.181 \pm 0.005$ \\ 
                        Shapes & $0.014 \pm 0.004$ & $0.012 \pm 0.004$ & $0.011 \pm 0.002$ & $0.016 \pm 0.008$ & $0.168 \pm 0.142$ & $0.012 \pm 0.001$ \\ 
                        UWave & $0.113 \pm 0.011$ & $0.121 \pm 0.017$ & $0.140 \pm 0.004$ & $0.745 \pm 0.151$ & $1.168 \pm 1.063$ & $0.189 \pm 0.008$ \\ 
                        Wafer & $0.048 \pm 0.021$ & $0.081 \pm 0.011$ & $0.105 \pm 0.010$ & $0.105 \pm 0.086$ & $0.029 \pm 0.011$ & $0.085 \pm 0.002$ \\ 
                        Walk vs Run & $0.030 \pm 0.008$ & $0.030 \pm 0.008$ & $0.023 \pm 0.007$ & $0.048 \pm 0.040$ & $0.028 \pm 0.000$ & $0.066 \pm 0.001$ \\
                        \midrule
                        Mean & $0.175$ & $0.238$ & $0.180$ & $0.514$ & $0.603$ & $0.423$ \\ 
                        Avg.~rank ($n_\bX < 300$) & $2.800$ & $2.900$ & $2.200$ & $4.700$ & $4.000$ & $4.400$ \\ 
                        Avg.~rank ($n_\bX \geq 300$) & $2.333$ & $3.333$ & $2.833$ & $4.833$ & $4.333$ & $3.333$ \\ 
                        Avg.~rank (all) & $2.625$ & $3.062$ & $2.438$ & $4.750$ & $4.125$ & $4.000$ \\
                    \bottomrule
				\end{tabular}
    }
			\end{sc}
		\end{small}
	\end{center}
\end{table*}

\begin{table*}[h]
	\caption{Mean and standard deviation of accuracies on test sets over $5$ independent runs}
	\label{table:gpsig:full_acc_results}
	\vskip 0.15in
	\begin{center}
		\begin{small}
			\begin{sc}
                    \resizebox{\textwidth}{!}{
				\begin{tabular}{lcccccc}
					\toprule
					Dataset & \texttt{GPSigLSTM} & \texttt{GPSigGRU} & \texttt{GPSig} & \texttt{GPLSTM}  & \texttt{GPGRU} & \texttt{GPKConv1D}\\
					\midrule 
                        Arabic Digits & $0.992 \pm 0.003$ & $0.994 \pm 0.002$ & $0.979 \pm 0.004$ & $0.985 \pm 0.004$ & $0.986 \pm 0.005$ & $0.984 \pm 0.001$ \\ 
                        AUSLAN & $0.983 \pm 0.003$ & $0.978 \pm 0.006$ & $0.925 \pm 0.014$ & $0.880 \pm 0.012$ & $0.949 \pm 0.014$ & $0.784 \pm 0.012$ \\ 
                        Character Traj. & $0.991 \pm 0.003$ & $0.925 \pm 0.078$ & $0.979 \pm 0.002$ & $0.233 \pm 0.331$ & $0.114 \pm 0.050$ & $0.941 \pm 0.013$ \\ 
                        CMUsubject16 & $1.000 \pm 0.000$ & $1.000 \pm 0.000$ & $0.979 \pm 0.017$ & $0.924 \pm 0.051$ & $0.993 \pm 0.014$ & $0.897 \pm 0.000$ \\ 
                        DigitShapes & $1.000 \pm 0.000$ & $0.988 \pm 0.025$ & $1.000 \pm 0.000$ & $1.000 \pm 0.000$ & $0.812 \pm 0.153$ & $1.000 \pm 0.000$ \\ 
                        ECG & $0.816 \pm 0.029$ & $0.832 \pm 0.012$ & $0.848 \pm 0.010$ & $0.782 \pm 0.032$ & $0.734 \pm 0.033$ & $0.760 \pm 0.018$ \\ 
                        Jap.~Vowels & $0.981 \pm 0.005$ & $0.985 \pm 0.004$ & $0.982 \pm 0.005$ & $0.982 \pm 0.004$ & $0.986 \pm 0.005$ & $0.986 \pm 0.002$ \\ 
                        Kick vs Punch & $0.900 \pm 0.063$ & $0.820 \pm 0.098$ & $0.900 \pm 0.000$ & $0.620 \pm 0.075$ & $0.600 \pm 0.110$ & $0.700 \pm 0.089$ \\ 
                        LIBRAS & $0.921 \pm 0.013$ & $0.899 \pm 0.031$ & $0.923 \pm 0.004$ & $0.776 \pm 0.019$ & $0.742 \pm 0.050$ & $0.698 \pm 0.026$ \\ 
                        NetFlow & $0.931 \pm 0.002$ & $0.921 \pm 0.012$ & $0.937 \pm 0.003$ & $0.928 \pm 0.011$ & $0.926 \pm 0.012$ & $0.945 \pm 0.027$ \\ 
                        PEMS & $0.763 \pm 0.016$ & $0.775 \pm 0.019$ & $0.820 \pm 0.014$ & $0.745 \pm 0.044$ & $0.769 \pm 0.020$ & $0.794 \pm 0.008$ \\ 
                        PenDigits & $0.928 \pm 0.030$ & $0.902 \pm 0.048$ & $0.955 \pm 0.002$ & $0.953 \pm 0.008$ & $0.951 \pm 0.008$ & $0.946 \pm 0.001$ \\ 
                        Shapes & $1.000 \pm 0.000$ & $1.000 \pm 0.000$ & $1.000 \pm 0.000$ & $1.000 \pm 0.000$ & $0.867 \pm 0.163$ & $1.000 \pm 0.000$ \\ 
                        UWave & $0.970 \pm 0.004$ & $0.968 \pm 0.006$ & $0.964 \pm 0.001$ & $0.870 \pm 0.029$ & $0.763 \pm 0.225$ & $0.947 \pm 0.002$ \\ 
                        Wafer & $0.988 \pm 0.005$ & $0.978 \pm 0.005$ & $0.965 \pm 0.004$ & $0.966 \pm 0.037$ & $0.994 \pm 0.002$ & $0.984 \pm 0.001$ \\ 
                        Walk vs Run & $1.000 \pm 0.000$ & $1.000 \pm 0.000$ & $1.000 \pm 0.000$ & $1.000 \pm 0.000$ & $1.000 \pm 0.000$ & $1.000 \pm 0.000$ \\
                        \midrule
                        Mean & $0.948$ & $0.935$ & $0.947$ & $0.853$ & $0.824$ & $0.898$ \\ 
                        Avg.~rank ($n_\bX < 300$) & $3.000$ & $3.100$ & $2.800$ & $4.250$ & $4.250$ & $3.600$ \\ 
                        Avg.~rank ($n_\bX \geq 300$) & $2.167$ & $3.500$ & $3.000$ & $4.167$ & $4.333$ & $3.833$ \\ 
                        Avg.~rank (all) & $2.688$ & $3.250$ & $2.875$ & $4.219$ & $4.281$ & $3.688$ \\
					\bottomrule
				\end{tabular}}
			\end{sc}
		\end{small}
	\end{center}
\end{table*}

\begin{table*}[h]
	\caption{Accuracies of frequentist time series classification methods}
	\label{table:gpsig:freq_acc_results}
	\vskip 0.15in
	\begin{center}
		\begin{small}
			\begin{sc}
				\begin{tabular}{lrrrrrrrr}
					\toprule
					Dataset & \texttt{SMTS} & \texttt{LPS} & \texttt{mvARF} & \texttt{DTW} & \texttt{ARKernel} & \texttt{gRSF} & \texttt{MLSTMFCN} & \texttt{MUSE} \\
					\midrule
					Arabic Digits & $0.964$ & $0.971$ & $0.952$ & $0.908$ & $0.988$ & $0.975$ & $0.990$ & $0.992$ \\
					AUSLAN & $0.947$ & $0.754$ & $0.934$ & $ 0.727$ & $0.918$ & $0.955$ & $0.950$ & $0.970$ \\
					Character Traj. & $0.992$ & $0.965$ & $0.928$ & $0.948$ & $0.900$ & $0.994$ & $0.990$ & $0.937$ \\
					CMUsubject16 & $0.997$ & $1.000$ & $1.000$ & $0.930$ & $1.000$ & $1.000$ & $1.000$ & $1.000$ \\
					DigitShapes & $1.000$ & $1.000$ & $1.000$ & $1.000$ & $1.000$ & $1.000$ & $1.000$ & $1.000$ \\
					ECG & $0.818$ & $0.820$ & $0.785$ & $0.790$ & $0.820$ & $0.880$ & $0.870$ & $0.880$ \\
					Jap.~Vowels & $0.969$ & $0.951$ & $0.959$ & $0.962$ & $0.984$ & $0.800$ & $1.000$ & $0.976$ \\
					Kick vs Punch & $0.820$ & $0.900$ & $0.976$ & $0.600$ & $0.927$ & $1.000$ & $0.900$ & $1.000$ \\
					LIBRAS & $0.909$ & $0.903$ & $0.945$ & $0.888$ & $0.952$ & $0.911$ & $0.970$ & $0.894$ \\
					NetFlow & $0.977$ & $0.968$ & NA & $0.976$ & NA & $0.914$ & $0.950$ & $0.961$ \\
					PEMS & $0.896$ & $0.844$ & NA & $0.832$ & $0.750$ & $1.000$ & NA & NA \\
					PenDigits & $0.917$ & $0.908$ & $0.923$ & $0.927$ & $0.952$ & $0.932$ & $0.970$ & $0.912$ \\
					Shapes & $1.000$ & $1.000$ & $1.000$ & $1.000$ & $1.000$ & $1.000$ & $1.000$ & $1.000$   \\
					UWave & $0.941$ & $0.980$ & $0.952$ & $0.916$ & $0.904$ & $0.929$ & $0.970$ & $0.916$ \\
					Wafer & $0.965$ & $0.962$ & $0.931$ & $0.974$ & $0.968$ & $0.992$ & $0.990$ & $0.997$ \\
					Walk vs Run & $1.000$ & $1.000$ & $1.000$ & $1.000$ & $1.000$ & $1.000$ & $1.000$ & $1.000$ \\
					\bottomrule
				\end{tabular} 
			\end{sc}
		\end{small}
	\end{center}
\end{table*}


\end{subappendices}
\endgroup
\begingroup
\chapter{\texttt{Seq2Tens}} \label{ch:s2t}
\section{Introduction}
A central task of learning is to find representations of the underlying data that efficiently and faithfully capture their structure.
In the case of sequential data, one data point consists of a sequence of objects.
This is a rich and non-homogeneous class of data and includes classical uni- or multi-variate time series (sequences of scalars or vectors), video (sequences of images), and text (sequences of letters). 
Particular challenges of sequential data are that each sequence entry can itself be a highly structured object and that data sets typically include sequences of different length which makes naive vectorization troublesome.

\paragraph{Contribution}
Our main result is a generic method that takes a \emph{static feature map} for a class of objects (e.g.~a feature map for vectors, images, or letters) as input and turns this into a feature map for sequences of arbitrary length of such objects (e.g.~a feature map for time series, video, or text). 
We call this feature map for sequences \texttt{Seq2Tens}; among its attractive properties are that it 
\begin{enumerate*}[label=(\roman*)]
\item provides a structured description of sequences; generalizing classical methods for strings, 
\item comes with theoretical guarantees such as universality,
\item can be turned into modular and flexible neural network (\texttt{NN}) layers. 
\end{enumerate*}
The key ingredient to our approach is to embed the feature space of the static feature map into the tensor algebra, and then the product in this algebra is then used to ``stitch together'' the static features of the sequence entries in a structured way.
\paragraph{Outline}
Section \ref{sec:2} formalizes the main ideas of \texttt{Seq2Tens}. 
Section \ref{sec:3} shows how low-rank (\texttt{LR}) constructions combined with \texttt{seq2seq} transforms allows one to efficiently use this rich algebraic structure. 
Section \ref{sec:4} applies the results of Sections~\ref{sec:2} and \ref{sec:3} to build modular and scalable  \texttt{NN} layers for sequential data.
Section \ref{sec:5} demonstrates the flexibility and modularity of this approach on both discriminative and generative benchmarks.
Section \ref{sec: summary} makes connections with previous work and summarizes this article.   
In the appendices we provide mathematical background, extensions, and detailed proofs for our theoretical results.


\section{Capturing order by non-commutative multiplication} \label{sec:2}
We denote the set of sequences of elements in a set $\cX$ by 
\begin{align} \label{eq:s2t:seq1start}
  \Seq(\cX)=\{\bx=(\bx_i)_{i=1,\ldots,L}: \bx_i \in \cX,\, L \ge 1\} 
\end{align}
where $ L\geq 1 $ is some arbitrary length. Notice that in this chapter we use $1$-based indexing for sequences, and for a sequence $\bx=(\bx_1, \dots, \bx_L) \in \Seq(\cX)$, $\len{\bx} = L$. 
Even if $\cX$ itself is a linear space, e.g.~$\cX=\bbR$, $\Seq(\cX)$ is never a linear space since there is no natural addition of two sequences of different length.
\subsection{\texttt{Seq2Tens} in a nutshell}
Given any vector space $ V $ we may construct the tensor algebra $ \T{V} $ over $ V $.
We describe the space $\T{V}$ in detail in Section \ref{sec:back:tensalgs}, but as for now the only thing that is important is that $\T{V}$ is also a vector space that includes $V$, and that it carries a non-commutative product, which is, in a precise sense, ``the most general product'' on $ V $.

The main idea of \texttt{Seq2Tens} is that any ``static feature map'' for elements in $\cX$, $\phi: \cX \to V$, can be used to construct a new feature map $\Phi:\Seq(\cX) \rightarrow \T{V}$ for sequences in $\cX$ by using the algebraic structure of $\T{V}$: 
the non-commutative product on $ \T{V} $ makes it possible to ``stitch together'' the individual features $\phi(\bx_1),\ldots,\phi(\bx_L) \in V \subset \T{V}$ of the sequence $\bx$ in the larger space $\T{V}$ by multiplication in $\T{V}$.
With this we may define the feature map $\Phi(\bx)$ for a sequences $\bx=(\bx_1,\ldots,\bx_L) \in \Seq(\cX)$ as follows 
\begin{enumerate}[label=(\roman*)]
\item\label{itm:lift map}
  lift the map $\phi:\cX \to V$ to a map $\varphi: \cX \to \T{V}$,
	\item \label{itm:lift}
	map $\Seq(\cX) \to \Seq(\T{V})$ by $(\bx_1,\ldots,\bx_L) \mapsto (\varphi(\bx_1),\ldots,\varphi(\bx_L))$, 
	\item \label{itm:mult}
    map $\Seq(\T{V}) \to \T{V}$ by multiplication $(\varphi(\bx_1),\ldots,\varphi(\bx_L)) \mapsto \varphi(\bx_1)\cdots \varphi(\bx_L)$. 
\end{enumerate}
In a more concise form, we define $\Phi$ as
\begin{align}
 \Phi: \Seq(\cX) \to \T{V},\quad \Phi(\bx)=\prod_{i=1}^L \varphi(\bx_i) \label{eq:s2t:mult_varphi}
\end{align}
where $\prod$ denotes multiplication in $\T{V}$.
We refer to the resulting map $\Phi$ as the \texttt{Seq2Tens} map, which stands short for \textit{\textbf{Seq}uences-\textbf{2}-\textbf{Tens}ors}.
Why is this construction a good idea?
First note, that step \ref{itm:lift map} is always possible since $V\subset \T{V}$ and we discuss the simplest such lift before Theorem~\ref{thm:univ} as well as other choices in \cite[App.~B]{toth2021seq2tens}. 
Further, if $\phi$, respectively~$\varphi$, provides a faithful representation of objects in $\cX$, then there is no loss of information in step~\ref{itm:lift}.
Finally, since step \ref{itm:mult} uses ``the most general product'' to multiply $\varphi(\bx_1)\cdots \varphi(\bx_L)$ one expects that $\Phi(\bx) \in \T{V}$ faithfully represents the sequence $\bx$ as an element of $\T{V}$.
Indeed in Theorem~\ref{thm:univ} below we show an even stronger statement, namely that if the static feature map $\phi:\cX \to V$ contains enough non-linearities so that non-linear functions from $\cX $ to $\bbR$ can be approximated as \emph{linear functions} of the static feature map $\phi$, then the above construction extends this property to functions of sequences.
Put differently, \emph{if $\phi$ is a universal feature map for $\cX$, then $\Phi$ is a universal feature map for $\Seq(\cX)$};
that is, any non-linear function $f(\bx)$ of a sequence $\bx$ can be approximated as a linear functional of $\Phi(\bx)$,  $f(\bx) \approx \langle \ell, \Phi(\bx) \rangle$.  
We also emphasize that the domain of $\Phi$ is the space $\Seq(\cX)$ of sequences of \emph{arbitrary} (finite) length.
The remainder of this Section gives details about steps~\ref{itm:lift map},\ref{itm:lift},\ref{itm:mult} for the construction of $\Phi$. 

\subsection{Construction}
\paragraph{Lifting static feature maps}
Step~\ref{itm:lift map} in the construction of $\Phi$ requires turning a given feature map $\phi:\cX \to V$ into a map $\varphi: \cX \to \T{V}$.
Throughout the rest of this article we use the lift  
\begin{align} \label{eq:s2t:oneplus}
  \varphi(\bx) = (1, \phi(\bx), 0, 0 \ldots ) \in \T{V}.
\end{align}
We discuss other choices in \cite[App.~B]{toth2020bayesian}, but attractive properties of the lift \ref{eq:s2t:oneplus} are that
\begin{enumerate*}[label=(\arabic*)]
\item
  the evaluation of $\Phi$ against low-rank tensors becomes a simple recursive formula (Proposition~\ref{prop:rank one},) 
\item 
  it is a generalization of sequence sub-pattern matching as used in string kernels (\cite[App.~B.3]{toth2021seq2tens},  
\item 
  despite its simplicity it performs exceedingly well in practice (Section \ref{sec:4}).
\end{enumerate*}

\paragraph{Extending to sequences of arbitrary length} Steps \ref{itm:lift map} and \ref{itm:lift} in the construction specify how the map $\Phi: \cX \rightarrow \T{V}$ behaves on sequences of length-$1$, that is, single observations. Step \ref{itm:mult} amounts to the requirement that for any two sequences $\bx, \by \in \Seq(V)$, their concatenation defined as $\bz = (\bx_0, \dots, \bx_{\len{\bx}}, \by_0, \dots, \by_\len{\by}) \in \Seq(V)$ can be understood in the feature space as (non-commutative) multiplication of their corresponding features
\begin{align}
	\Phi(\bz) = \Phi(\bx) \cdot \Phi(\by). \label{eq:s2t:conc_product}
\end{align}
In other words, we inductively extend the lift $\varphi$ to sequences of arbitrary length by starting from sequences consisting of a single observation, which is given in \eqref{eq:s2t:mult_varphi}.
Repeatedly applying the definition of the tensor convolution product in \eqref{eq:back:alg_mult} leads to the following explicit formula

\begin{align}\label{eq:s2t:phi formula}
	\Phi_m(\bx) = \sum_{\bi \in \Delta_m(\len{\bx})} \bx_{i_1} \otimes \cdots \otimes \bx_{i_m} \in V^{\otimes m}, \hspace{5pt} \Phi(\bx) = (\Phi_m(\bx))_{m \geq 0},
\end{align}
where $\bx=(\bx_0, \dots, \bx_\len{\bx}) \in \Seq(V)$ and the summation is over non-contiguous subsequences. 

\paragraph{Some intuition: generalized pattern matching}
Our derivation of the feature map $\Phi(\bx) = \in \T{V}$ was guided by general algebraic principles, but \eqref{eq:s2t:phi formula} provides an intuitive interpretation.
It shows that for each $m\ge 1$, the entry $\Phi_m(\bx) \in V^{\otimes m}$ constructs a summary of a long sequence $\bx \in \Seq(V)$ based on subsequences $(\bx_{i_1},\ldots,\bx_{i_m})$ of $\bx$ of length-$m$.
It does this by taking the usual tensor product (see Section \ref{sec:back:tensors}) $\bx_{i_1} \otimes \cdots \otimes \bx_{i_m} \in V^{\otimes m}$ and summing over all possible non-contiguous subsequences.
This is completely analogous to how string kernels provide a structured description of text by looking at non-contiguous substrings of length-$m$ (indeed, Appendix \cite[App.~B.3]{toth2021seq2tens} makes this rigorous).
However, the main difference is that the above construction works for arbitrary sequences and not just sequences of discrete letters.

The algebraic background allows to prove that $\Phi$ is universal, see Theorem \ref{thm:univ} below.

\paragraph{Universality}
A function $ \phi : \cX \to V $ is said to be \emph{universal for }$\cX$ if all continuous functions on $ \cX $ can be approximated as linear functions on the image of $ \phi $.
One of the most powerful features of \texttt{NN}s is their universality~\cite{hornik1991approximation}.
A very attractive property of $\Phi$ is that it preserves universality: if $\phi:\cX \to V$ is universal for $\cX$, then $\Phi: \Seq(\cX) \to \T{V}$ is universal for $\Seq(\cX)$. 
To make this precise, note that $V^{\otimes m}$ is a linear space and therefore any $\ell= (\ell_0,\ell_1,\ldots,\ell_M,0,0,\ldots) \in \T{V}$ consisting of $M$ tensors $\ell_m\in V^{\otimes m}$, yields a linear functional on $\T{V}$; e.g.~if $V=\bbR^d$ and we identify $\ell_m$ in coordinates as $\ell_m=(\ell_m^{i_1,\ldots,i_m})_{i_1,\ldots,i_m \in \{1,\ldots,d\}}$ then
\begin{align}\label{eq:s2t: coordinates functional}
  \langle \ell, \bt \rangle:= \sum_{m=0}^M \langle \ell_m, \bt_m \rangle = \sum_{m = 0}^M \sum_{i_1,\ldots,i_m \in \{1,\ldots,d\}} \ell^{i_1,\ldots,i_m}_m \bt_m^{i_1,\ldots,i_m}.
\end{align}
Thus linear functionals of the feature map $\Phi$, are real-valued functions of sequences.
Theorem \ref{thm:univ} below shows that any continuous function $f:\Seq(\cX)\to \bbR$ can by arbitrary well approximated by a $\ell \in \T{V}$, $f(\bx) \approx \langle \ell, \Phi(\bx) \rangle$.
\begin{theorem}[Universality] \label{thm:univ}
	Let $ \phi : \cX \to V $  be a universal map with a lift that satisfies some mild constraints, then the following map is universal:
	\begin{align}
\Phi:	\mathrm{Seq}(\cX) \to \T{V}, \quad \bx \mapsto \Phi(\bx).
	\end{align}
\end{theorem}
A detailed proof and the precise statement of Theorem \ref{thm:univ} is given in \cite[App.~B]{toth2021seq2tens}.

\section{Approximation by low-rank linear functionals} \label{sec:3}
\paragraph{The combinatorial explosion of tensor coordinates and what to do about it}
The universality of $\Phi$ suggests the following approach to represent a function $f:\Seq(\cX) \to \bbR$ of sequences: First compute $\Phi(\bx)$ and then optimize over $\ell$ (and possibly also the hyperparameters of $\phi$) such that $f(\bx) \approx \langle \ell, \Phi(\bx) \rangle=\sum_{m=0}^M \langle \ell_m, \Phi_m(\bx) \rangle$.   
Unfortunately, tensors suffer from a combinatorial explosion in complexity, i.e.~even just storing $ \Phi_m(\bx) \in  V^{\otimes m} \subset \T{V} $ requires $O(\operatorname{dim}(V)^{m})$ real numbers.
Below we resolve this computational bottleneck as follows: in Proposition \ref{prop:rank one} we show that for a special class of low-rank elements $\ell \in \T{V}$, the functional $\bx \mapsto \langle \ell, \Phi(\bx) \rangle$ can be efficiently computed in both time and memory.
This is somewhat analogous to a kernel trick since it shows that $\langle \ell, \Phi(\bx) \rangle$ can be cheaply computed without explicitly computing the feature map $\Phi(\bx)$. 
However, Theorem~\ref{thm:univ} guarantees universality under no restriction on $\ell$, thus restriction to rank-$1$ functionals limits the class of functions $f(\bx)$ that can be approximated.
Nevertheless, by iterating these ``low-rank functional'' constructions in the form of \texttt{seq2seq} transformations this can be ameliorated.
We give the details below but to gain intuition, we invite the reader to think of this iteration analogous to stacking layers in a neural network: each layer is a relatively simple non-linearity (e.g.~a sigmoid composed with an affine function) but by composing such layers, complicated functions can be efficiently approximated.

\paragraph{Rank-1 functionals are computationally cheap} 
Degree $m=2$ tensors are matrices and low-rank approximations are widely used in practice \cite{Udell19} to address quadratic complexity. 
The definition below generalizes the rank of matrices to tensors of any degree $m$. 
\begin{definition} \label{def:rank}
	The \emph{rank} (also called \emph{\texttt{CP} rank} \cite{carroll1970analysis}) of a degree-$m$ tensor $ \bt_m \in V^{\otimes m} $ is the smallest number $ r \geq 0 $ such that one may write
	\begin{align}
	\bt_m = \sum_{i=0}^r \bv_i^1 \otimes \cdots \otimes \bv_i^m, \quad \bv_i^1, \ldots, \bv_i^m \in V.
	\end{align}
  We say that $\bt = (\bt_m)_{m \ge0} \in \T{V}$ has rank-$1$ (and truncation-$M$) if each $\bt_m \in V^{\otimes m}$ is a rank-$1$ tensor and $\bt_i=0$ for $i>M$.
\end{definition}
The ranks of signature tensors related are investigated in \cite{galuppi2024rank}.
A direct calculation shows that if $\ell$ is of rank-$1$, then $\langle \ell, \Phi(\bx) \rangle$ can be computed efficiently by inner product evaluations.
\begin{proposition}\label{prop:rank one}
  Let $\ell=(\ell_m)_{m \ge 0} \in \T{V}$ be of rank-$1$ and truncation-$M$. Then,
  \begin{align}
    \langle \ell, \Phi(\bx)\rangle &=
    \sum_{m=0}^M \sum_{\bi \in \Delta_m(\len{\bx})} \prod_{k=1}^m \langle \bv_k^m, \phi(\bx_{i_k})\rangle \label{eq:s2t:sum}
  \end{align}
where $\ell_m=\bv_1^m \otimes \dots \otimes \bv_m^m \in V^{\otimes m}$, $\bv_i^m \in V$ and $m=0,\ldots,M$.
\end{proposition}
Note that the inner sum is taken over all non-contiguous subsequences of $\bx$ of length-$m$, analogously to $m$-mers of strings; the proof of Proposition \ref{prop:rank one} is given in \cite[App.~B]{toth2021seq2tens}. While \eqref{eq:s2t:sum} looks expensive, by casting it into a recursive formulation over time, it can be computed in $O(M^2 \cdot \len{\bx} \cdot d)$ time and $O(M^2 \cdot (\len{\bx} + c))$ memory, where $d$ is the inner product evaluation time on $V$, while $c$ is the memory footprint of a $v \in V$. This can further be reduced to $O(M \cdot \len{\bx} \cdot d)$ time and $O(M \cdot (\len{\bx} + c))$ memory by an efficient parameterization of the rank-$1$ element $\ell \in \T{V}$. We give further details and numerical scalability analysis in Appendix \ref{app:s2t:comps}.

\paragraph{Low-rank \texttt{Seq2Tens} maps}
The composition of a linear map $\cL: \T{V} \rightarrow \bbR^{N}$ with $\Phi$ can be computed cheaply in parallel using \eqref{eq:s2t:sum} when $\cL$ is specified through a collection of $N \in \bbN$ rank-$1$ elements $\ell^1,\ldots,\ell^N \in \T{V}$ such that 
\begin{align}
	\tilde\Phi_{\tilde\theta}(\bx_1, \dots, \bx_L) := \cL\circ \Phi(\bx_1,\ldots,\bx_L) = (\langle \ell^j, \Phi(\bx_1,\ldots,\bx_L))_{j=1}^N \in \bbR^{N}.
\end{align}
We call the resulting map $\tilde\Phi_{\tilde\theta}: \Seq(\cX) \rightarrow \bbR^N$ a \textit{\textbf{L}ow-rank \textbf{S}eq\textbf{2T}ens} map of width-$N$ and truncation-$M$, where $M \in \bbN$ is the maximal degree of $\ell^1, \dots, \ell^N$ such that $\ell^j_i = 0$ for $i > M$. The \texttt{LS2T} map is parametrized by \begin{enumerate*}[label=(\arabic*)] \item the component vectors $\bv_{j,m}^k \in V$ of the rank-$1$ elements $\ell^j_m = \bv_{j,m}^1 \otimes \cdots \otimes \bv_{j,m}^{m}$, \item by any parameters $\theta$ that the static feature map $\phi_\theta: \cX \rightarrow V$. We denote these parameters by $\tilde\theta = (\theta, \ell^1, \dots, \ell^N)$ \end{enumerate*}. In addition, by the subsequent composition of $\tilde\Phi_{\tilde \theta}$ with a linear functional $\bbR^{N} \rightarrow \bbR$, we get the following function subspace as hypothesis class
\begin{align}
	\tilde\cH = \big\{\langle \sum_{j=1}^N \alpha_j \ell^j, \Phi(\bx_1, \dots, \bx_L) \rangle \setgiven \alpha_j \in \bbR \big\} \subsetneq \cH = \big\{\langle \ell, \Phi(\bx_1, \dots, \bx_L) \setgiven \ell \in \T{V}\big\}	
\end{align}
Hence, we acquire an intuitive explanation of the (hyper)parameters: the width of the \texttt{LS2T}, $N \in \bbN$ specifies the maximal rank of the low-rank linear functionals of $\Phi$ that the \texttt{LS2T} can represent, while the span of the rank-$1$ elements, $\spn(\ell^1, \dots, \ell^N)$ determine an $N$-dimensional subspace of the dual space of $\T{V}$ consisting of at most rank-$N$ functionals.

Recall now that without rank restrictions on the linear functionals of \texttt{Seq2Tens} features, Theorem~\ref{thm:univ} would guarantee that any real-valued function $f:\Seq(\cX)\to \bbR$ could be approximated by $f(\bx) \approx \langle \ell, \Phi(\bx_1, \dots, \bx_L) \rangle$. As pointed out before, the restriction of the hypothesis class to low-rank linear functionals of $\Phi(\bx_1, \dots, \bx_L$) would limit the class of functions of sequences that can be approximated. To ameliorate this, we use \texttt{LS2T} transforms in a \texttt{seq2seq} fashion that allows us to stack such low-rank functionals, significantly recovering expressiveness.
 
\paragraph{Sequence-to-sequence transforms}
Motivated by the empirical successes of stacked \texttt{RNN}s \cite{graves2013speech, graves2013hybrid, Sutskever2014Seq2Seq}, we build sequence-to-sequence (\texttt{seq2seq}) transformations.
We can use \texttt{LS2T} to build \texttt{seq2seq} transformations in the following way: fix the static map $\phi_{\theta}: \cX \to V$ parametrized by $\theta$ and rank-$1$ elements such that $\tilde\theta = (\theta, \ell^1, \dots, \ell^N)$ and apply the resulting \texttt{LS2T} map $\tilde\Phi_{\tilde \theta}$:
\begin{align} \label{eq:s2t:seq2seq_lls2t}
	\Seq(\cX) \to \Seq(\bbR^N),\quad \bx \mapsto \big(\tilde\Phi_{\tilde\theta}(\bx_1), \tilde\Phi_{\tilde\theta}(\bx_1, \bx_2), \dots, \tilde\Phi_{\tilde\theta}(\bx_1, \dots, \bx_L)\big).
\end{align}
Note that the cost of computing the expanding window \texttt{seq2seq} transform in \eqref{eq:s2t:seq2seq_lls2t} is no more expensive than computing $\tilde\Phi_{\tilde\theta}(\bx_1, \dots, \bx_L)$ itself due to the recursive nature of our algorithms, for further details see Appendix \ref{app:s2t:comps}.

\paragraph{Deep \texttt{seq2seq} transforms}
Inspired by the empirical successes of deep \texttt{RNN}s \cite{graves2013speech, graves2013hybrid, Sutskever2014Seq2Seq}, we iterate the transformation \eqref{eq:s2t:seq2seq_lls2t} $D$-times:
\begin{align}\label{eq:s2t:stacked seq2seq}
  \Seq(\cX) \rightarrow \Seq(\bbR^{N_1}) \rightarrow \Seq(\bbR^{N_2}) \rightarrow \cdots \rightarrow \Seq(\bbR^{N_D}). 
\end{align}
Each of these mappings $\Seq(\bbR^{N_i}) \rightarrow \Seq(\bbR^{N_{i+1}})$ is parametrized by the parameters $\tilde\theta_i$ of a static feature map $\phi_{\theta_i}$ and a linear map $\cL_i$ specified by $N_i$ rank-$1$ elements of $\T{V}$; these parameters are collectively denoted by $\tilde\theta_i= (\theta_i, \ell_i^1, \dots, \ell_i^{N_i})$. 
Evaluating the final sequence in $\Seq(\bbR^{N_D})$ at the last observation-time $t=L$, we get the deep \texttt{LS2T} map with depth-$D$
\begin{align}\label{eq:s2t:SeqFeatures}
  \tilde\Phi_{\tilde \theta_1,\ldots,\tilde \theta_{D}}:\Seq(\cX) \rightarrow \bbR^{n_D}.
\end{align}
Making precise how the stacking of such low-rank \texttt{seq2seq} transformations approximates general functions requires more tools from algebra, and we provide a rigorous quantitative statement in \cite[App.~C]{toth2021seq2tens}. 
Here, we just appeal to the analogy made with adding depth in neural networks mentioned earlier and empirically validate this in our experiments in Section~\ref{sec:4}. 

\section{Building neural networks with \texttt{LS2T} layers} \label{sec:4}
The \texttt{Seq2Tens} map $\Phi$ built from a static feature map $\phi$ is universal if $\phi$ is universal, Theorem~\ref{thm:univ}.
NNs form a flexible class of universal feature maps with strong empirical success for data in $\cX=\bbR^d$, and thus make a natural choice for $\phi$.
Combined with standard deep learning constructions, the framework of Sections~\ref{sec:2} and \ref{sec:3} can build modular and expressive layers for sequences.
\paragraph{Neural \texttt{LS2T} layers}
The simplest choice among many is to use as static feature map $\phi:\cX=\bbR^d \to \bbR^h$ a feedforward network with depth-$P$, $\phi = \phi_{P} \circ \cdots \circ \phi_1$ where $\phi_{j}(\bx) = \sigma(\mathbf{W}_j \bx + \mathbf{b}_j)$ for $\mathbf{W}_j \in \bbR^{h \times d}$, $\mathbf{b}_j \in \bbR^{h}$.
We can then lift this to a map $\varphi: \bbR^d \to \T{\bbR^h}$ as prescribed in~\eqref{eq:s2t:oneplus}.
Hence, the resulting \texttt{LS2T} layer $\bx \mapsto (\tilde\Phi_{\tilde\theta} (\bx_1,\ldots,\bx_i))_{i=1,\ldots,L} $ is a \texttt{seq2seq} transform $\Seq(\bbR^d) \rightarrow \Seq{\bbR^h}$ that is parametrized by $\tilde \theta=(\bW_1, \bb_1, \dots, \bW_P, \bb_P, \ell_1^1, \dots, \ell^{N_1}_1)$.

\paragraph{Bidirectional \texttt{LS2T} layers} The transformation in \eqref{eq:s2t:seq2seq_lls2t} is completely causal in the sense that each step of the output sequence depends on the past.
For generative models, it can behove us to make the output depend on both past and future, see~\cite{graves2013hybrid, baldi1999exploiting,li2018disentangled}.
Similarly to bidirectional \texttt{RNN}s \cite{schuster1997bidirectional, Graves2005BiLSTM}, we may achieve this by defining a bidirectional layer, 
\begin{align}
    \tilde\Phi^{\operatorname{b}}_{(\tilde\theta_1, \tilde\theta_2)}(\mathbf{x}): \Seq(\bbR^d) \rightarrow \Seq(\bbR^{N + N^{\prime}}), \quad \mathbf{x} \mapsto (\tilde\Phi_{\tilde\theta_1}(\bx_1, \ldots, \bx_i), \tilde\Phi_{\tilde\theta_2}(\bx_i, \ldots, \bx_L))_{i=1}^L.
\end{align}
The sequential nature is kept intact by making the distinction between what classifies as past (the first $N$ coordinates) and future (the last $N^\prime$ coordinates) information. This amounts to having a form of precognition in the model, and has been applied in e.g. dynamics generation \cite{li2018disentangled}, machine translation \cite{sundermeyer2014translation}, and speech processing \cite{graves2013hybrid}.

\paragraph{Convolutions and \texttt{LS2T}} We motivate to replace the time-distributed feedforward layers proposed in the paragraph above by temporal convolutions (CNN) instead. Although theory only requires the ``preprocessing'' layer to be a static feature map, we find that it is beneficial to capture some of the sequential information as well, e.g.~using \texttt{CNN}s or \texttt{RNN}s. From a mathematical point of view, \texttt{CNN}s are a straightforward extension since they can be interpreted as time-distributed feedforward layers applied to the input sequence augmented with a $p \in \bbN$ number of its lags for \texttt{CNN} kernel size $p$ (see \cite[App.~D.1]{toth2021seq2tens} for further discussion). 

In the following, we precede our deep \texttt{LS2T} blocks by one or more \texttt{CNN} layers. Intuitively, \texttt{CNN}s and \texttt{LS2T}s are similar in that both transformations operate on subsequences of their input sequence. The main difference between the two lies in that \emph{CNNs operate on contiguous subsequences}, and therefore, capture local, short-range nonlinear interactions between timesteps; \emph{while \texttt{LS2T}s (\eqref{eq:s2t:sum}) use all non-contiguous subsequences}, and hence, learn global, long-range interactions in time. This observation motivates that the inductive biases of the two types of layers (local/global time-interactions) are highly complementary in nature, and we suggest that the improvement in the experiments on the models containing vanilla \texttt{CNN} blocks are due to this.

\section{Experiments}\label{sec:5}
We demonstrate the modularity and flexibility of the above \texttt{LS2T} and its variants by applying it to 
\begin{enumerate*}[label=(\roman*)] 
\item  multivariate time series classification,
\item mortality prediction in healthcare, 
\item  generative modelling of sequential data. 
\end{enumerate*}
In all cases, we take a strong baseline model (\texttt{FCN} and \texttt{GP-VAE}, as detailed below) and upgrade it with \texttt{LS2T} layers. As Thm.~\ref{thm:univ} requires the \texttt{Seq2Tens} layers to be preceded by at least a static feature map, we expect these layers to perform best as an add-on on top of other models, which however can be quite simple, such as a \texttt{CNN}. The additional computation time is negligible (in fact, for \texttt{FCN} it allows to reduce the number of parameters significantly, while retaining performance), but it can yield improvements.
This is remarkable, since the original models are already state-of-the-art on well-established benchmarks. 
\subsection{Multivariate time series classification} \label{subseq:s2t:4_tsc}

As the first task, we consider multivariate time series classification (\texttt{TSC}) on an archive of benchmark datasets collected by \cite{baydogan2015multivarate}. Numerous previous publications report results on this archive, which makes it possible to compare against several well-performing competitor methods from the \texttt{TSC} community. These baselines are detailed in \cite[App.~E.1]{toth2021seq2tens}. This archive was also considered in a recent popular survey paper on \texttt{DL} for \texttt{TSC} \cite{Fawaz2019DLforTSC}, from where we borrow the two best performing models as deep learning baselines: \texttt{FCN} and \texttt{ResNet}. 
The \texttt{FCN} is a fully convolutional network which stacks 3 convolutional layers of kernel sizes $(8, 5, 3)$ and filters $(128, 256, 128)$ followed by a global average pooling (\texttt{GAP}) layer, hence employing global parameter sharing. We refer to this model as \FCN{128}. The \texttt{ResNet} is a residual network stacking 3 \texttt{FCN} blocks of various widths with skip-connections in between \cite{he2016deep} and a final \texttt{GAP} layer.

The \texttt{FCN} is an interesting model to upgrade with \texttt{LS2T} layers, since the \texttt{LS2T} also employs parameter sharing across the sequence length, and as noted previously, convolutions are only able to learn local interactions in time, that in particular makes them ill-suited to picking up on long-range interactions, which is exactly where the \texttt{LS2T} can provide improvements. 
As our models, we consider three simple architectures: \begin{enumerate*}[label=(\roman*)] \item \LStwoTwidth{64}{3} stacks $3$ \texttt{LS2T} layers of order-$2$ and width-$64$; \item \FCNLStwoTwidth{64}{64}{3} precedes the \LStwoTwidth{64}{3} block by an \FCN{64} block; a downsized version of \FCN{128}; \item \FCNLStwoTwidth{128}{64}{3} uses the full \FCN{128} and follows it by a \LStwoTwidth{64}{3} block as before. \end{enumerate*} Also, both \texttt{FCN-LS2T} models employ skip-connections from the input to the \texttt{LS2T} block and from the \texttt{FCN} to the classification layer, allowing for the \texttt{LS2T} to directly see the input, and for the \texttt{FCN} to directly affect the final prediction. These hyperparameters were only subject to hand-tuning on a subset of the datasets, and the values we considered were $H, N \in \{32, 64, 128\}$, $M \in \{2, 3, 4\}$ and $D \in \{1, 2, 3\}$, where $H, N \in \bbN$ is the \texttt{FCN} and \texttt{LS2T} width, resp., while $M \in \bbN$ is the \texttt{LS2T} order and $D \in \bbN$ is the \texttt{LS2T} depth. We also employ techniques such as time-embeddings \cite{Liu2018coordConv}, sequence differencing and batch normalization, see \cite[App.~D.1]{toth2021seq2tens}; \cite[App.~E.1]{toth2021seq2tens} for further details on the experiment and Figure 2 in there for a visualization of the architectures.

\begin{table}
	\caption{Posterior probabilities given by a Bayesian signed-rank test comparison of the proposed methods against the baselines. $\{>\}$, $\{<\}$, $\{=\}$ refer to the respective events that the row method is better, the column method is better, or that they are equivalent.}
	\label{table:s2t:bayes_signed_rank_probs}
	\vspace{-10pt}
	\begin{center}
	\begin{small}
	\begin{sc}
	\resizebox{\textwidth}{!}{
		\begin{tabular}{lccccccccc}
			\toprule
			\multirow{2}{*}{Model} & \multicolumn{3}{c}{\LStwoTwidth{64}{3}} & \multicolumn{3}{c}{\FCNLStwoTwidth{64}{64}{3}} & \multicolumn{3}{c}{\FCNLStwoTwidth{128}{64}{3}} \\
			\cmidrule{2-10}
			& $p(>)$ & $p(=)$ & $p(<)$ & $p(>)$ & $p(=)$ & $p(<)$ & $p(>)$ & $p(=)$ & $p(<)$ \\
			\midrule 
			\texttt{SMTS} \cite{baydogan2015learning} & $0.180$ & $0.000$ & $\mathbf{0.820}$ & $0.010$ & $0.000$ & $\mathbf{0.990}$ & $0.008$ & $0.000$ & $\mathbf{0.992}$ \\
			\texttt{LPS} \cite{Baydogan2015TimeSR} & $0.191$ & $0.002$ & $\mathbf{0.807}$ & $0.012$ & $0.001$ & $\mathbf{0.987}$ & $0.006$ & $0.001$ & $\mathbf{0.993}$ \\
			\texttt{mvARF} \cite{tuncel2018autoregressive} & $0.011$ & $0.140$ & $\mathbf{0.849}$ & $0.000$ & $0.126$ & $\mathbf{0.874}$ & $0.000$ & $0.088$ & $\mathbf{0.912}$ \\
			\texttt{DTW} \cite{Sakoe1978DTW} & $0.033$ & $0.000$ & $\mathbf{0.967}$ & $0.001$ & $0.000$ & $\mathbf{0.999}$ & $0.000$ & $0.000$ & $\mathbf{1.000}$ \\
			\texttt{ARKernel} \cite{Cuturi2011AR} & $0.100$ & $0.097$ & $\mathbf{0.803}$ & $0.000$ & $0.021$ & $\mathbf{0.979}$ & $0.000$ & $0.015$ & $\mathbf{0.985}$ \\
			\texttt{gRSF} \cite{karlsson2016generalized} & $0.481$ & $0.011$ & $\mathbf{0.508}$ & $0.028$ & $0.013$ & $\mathbf{0.960}$ & $0.022$ & $0.013$ & $\mathbf{0.965}$ \\
			\texttt{MUSE} \cite{Schfer2017MUSE} & $0.405$ & $0.128$ & $\mathbf{0.467}$ & $0.001$ & $0.074$ & $\mathbf{0.925}$ & $0.001$ & $0.077$ & $\mathbf{0.922}$ \\
			\texttt{MLSTMFCN} \cite{Karim2019LSTMFCN} & $\mathbf{0.916}$ & $0.043$ & $0.041$ & $0.123$ & $0.071$ & $\mathbf{0.807}$ & $0.055$ & $0.110$ & $\mathbf{0.835}$ \\
			$\text{FCN}_{128}$ \cite{Wang2017TSCfromScratch} & $\mathbf{0.998}$ & $0.002$ & $0.000$ & $0.363$ & $0.186$ & $\mathbf{0.451}$ & $0.169$ & $0.011$ & $\mathbf{0.820}$ \\
			\texttt{ResNet} \cite{Wang2017TSCfromScratch} & $\mathbf{0.998}$ & $0.002$ & $0.001$ & $0.056$ & $0.240$ & $\mathbf{0.704}$ & $0.016$ & $0.048$ & $\mathbf{0.935}$ \\
			\LStwoTwidth{64}{3} & - & - & - & $0.000$ & $0.001$ & $\mathbf{0.999}$ & $0.000$ & $0.001$ & $\mathbf{0.999}$ \\
			\FCNLStwoTwidth{64}{64}{3} & $\mathbf{0.999}$ & $0.001$ & $0.000$ & - & - & - & $0.020$ & $0.387$ & $\mathbf{0.593}$ \\
			\bottomrule
		\end{tabular}}
	\end{sc}
	\end{small}
	\end{center}
\end{table}

\begin{figure}[t]
	\centering
	\begin{minipage}{0.49\textwidth}
		\centering
		\includegraphics[width=2.75in]{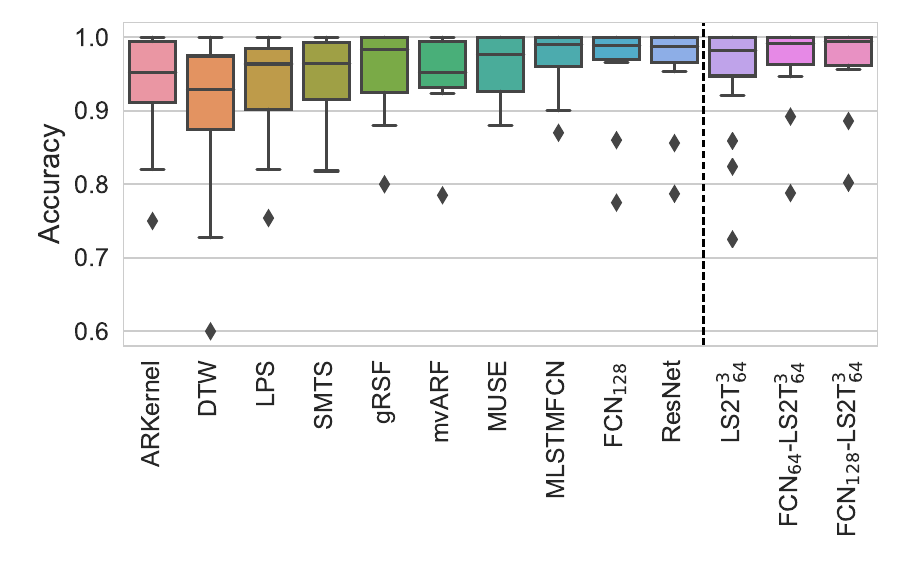}
	\end{minipage}
	\begin{minipage}{0.49\textwidth}
		\centering
		\vspace{-30pt}
		\includegraphics[width=2.75in]{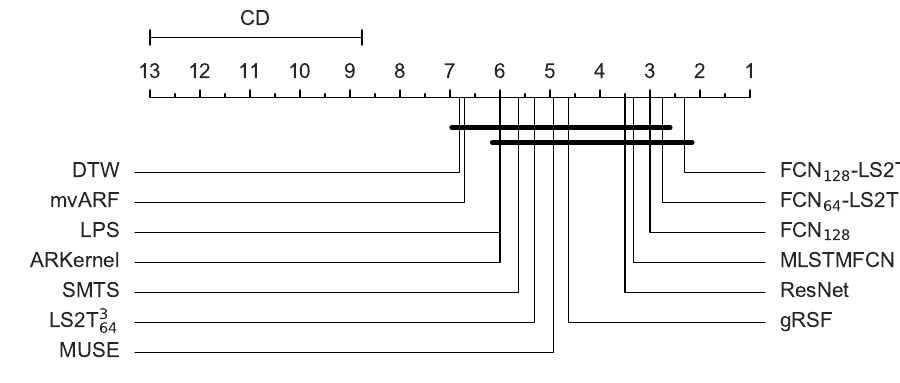}
	\end{minipage}
	\vspace{-10pt}
	\caption{Box-plot of classification accuracies (left) and critical difference diagram (right).}
	\label{fig:s2t:box_and_cd}
\end{figure}

\paragraph{Results} The full table of results is available in Table \ref{table:s2t:classification_accuracies} in Appendix \ref{app:s2t_further_res}, where for the models that we trained ourselves, i.e.~\FCN{128}, \texttt{ResNet}, \LStwoTwidth{64}{3}, \FCNLStwoTwidth{64}{64}{3}, \FCNLStwoTwidth{128}{64}{3}, we report the mean and standard deviation of test set accuracies over 5 model trains. Figure \ref{fig:s2t:box_and_cd} depicts the box-plot distributions of classification accuracies and the corresponding critical difference (\texttt{CD}) diagram. The \texttt{CD} diagram depicts the mean ranks of each method averaged over datasets with a calculated \texttt{CD} region using the Nemenyi test \cite{nemenyi1963distribution} for an alpha value of $\alpha=0.1$. For the Bayesian signed-rank test \cite{benavoli2014bayesian}, we used the implementation from \texttt{\href{https://github.com/janezd/baycomp}{https://github.com/janezd/baycomp}}, and the posterior probabilities are compared in Table \ref{table:s2t:bayes_signed_rank_probs}. The posterior distributions themselves are visualized on Figure 4 \cite[App.~E.1]{toth2021seq2tens}. The region of practical equivalence (\texttt{rope}) was set to $\texttt{rope} = 1 \times 10^{-3}$, that is, two accuracies were practically equivalent if they are at most the given distance from each other.

The results of the signed-rank test indicate that while \LStwoTwidth{3}{64} is only better with high probability ($p \geq 0.8)$ than 5 of the baselines, \FCNLStwoTwidth{64}{3}{64} performs better than all except $\texttt{FCN}_{128}$ and \texttt{ResNet} with high probability ($p \geq 0.8$), while \FCNLStwoTwidth{128}{3}{64} performs better than \emph{all baselines} with high probability ($p \geq 0.8$). 
Note however that \emph{\FCNLStwoTwidth{64}{3}{64} has fewer parameters than $\texttt{FCN}_{128}$ by around $60\%$}, see Table 6 in \cite[App.~E.1]{toth2021seq2tens}, while providing comparable or slightly improved performance. In conclusion, we have verified that \LStwoTwidth{3}{64} already performs well on some classification tasks, and by preceding it with an \texttt{FCN}, its performance is elevated to outperforming all baselines. This observation is not surprising, as Thm. \ref{thm:univ} warrants, the strength of the \texttt{LS2T} layer shines when it is preceded by state-space nonlinearities. Hence, we expect \texttt{LS2T}s to find their best use as a modular building block being part of larger models. We also suggest the improvement on the vanilla \texttt{FCN} composed of \texttt{CNN} blocks is due to the \texttt{LS2T} enhancing the model's ability to pick up long-range interactions, see the discussion at the end of Section \ref{sec:4}.

\begin{table}[t]
	\caption{Average ranks of classifier accuracies across data domains with the best and second best highlighted for each row in \textbf{bold} and \textit{italic}, respectively.}
	\label{table:s2t:domain}
	\begin{center}
		\resizebox{\textwidth}{!}{
		\begin{sc}
		\begin{tabular}{lccccccccccccc}
			\toprule
            Domain & \texttt{ARKernel} & \texttt{DTW} & \texttt{LPS} & \texttt{SMTS} & \texttt{gRSF} & \texttt{mvARF} & \texttt{MUSE} & \texttt{MLSTMFCN} & \FCN{128} & \texttt{ResNet} & \LStwoTwidth{64}{3} & \FCNLStwoTwidth{64}{64}{3} & \FCNLStwoTwidth{128}{64}{3} \\ 
			\midrule
            Handwriting & $5.00$ & $4.75$ & $5.50$ & $4.00$ & $2.75$ & $5.50$ & $5.75$ & $\mathit{2.00}$ & $\mathit{2.00}$ & $2.75$ & $3.50$ & $\mathit{2.00}$ & $\mathbf{1.75}$ \\
Motion & $5.33$ & $7.50$ & $4.67$ & $5.67$ & $4.50$ & $4.67$ & $5.00$ & $3.50$ & $\mathbf{1.83}$ & $2.17$ & $4.17$ & $2.33$ & $\mathit{2.00}$ \\
Sensor & $8.67$ & $6.00$ & $6.50$ & $5.50$ & $4.00$ & $11.50$ & $\mathbf{3.33}$ & $5.00$ & $5.50$ & $6.00$ & $8.25$ & $4.25$ & $\mathit{3.75}$ \\
Speech & $6.00$ & $10.50$ & $10.00$ & $9.00$ & $10.00$ & $10.50$ & $5.50$ & $3.00$ & $3.50$ & $4.00$ & $6.50$ & $\mathit{2.50}$ & $\mathbf{1.50}$ \\

			\bottomrule
		\end{tabular}
	\end{sc}
	}
	\end{center}
\end{table}

\paragraph{Domain analysis}
We provide a breakdown of classifier performance in Table \ref{table:s2t:domain} across four distinct domains: handwriting, motion, sensor, speech. We observe the following \begin{enumerate*}[label=(\arabic*)] \item for handwriting datasets, both variants of \FCNLStwoTwidth{128}{3}{64} provides the best performance with \FCNLStwoTwidth{64}{3}{64}, \FCN{128} and \texttt{MLSTMFCN} coming in close second, and \texttt{ResNet} in third; \item however, on motion datasets \FCN{128} takes first place and \FCNLStwoTwidth{128}{3}{64} in close second and \FCNLStwoTwidth{64}{3}{64} in third place; \item on sensor datasets, surprisingly, the deep learning baselines underperform with \texttt{MUSE} taking first place and \FCNLStwoTwidth{128}{3}{64} coming in second and \FCNLStwoTwidth{64}{3}{64} in third place, while \FCN{128}, \texttt{MLSTMFCN}, and \texttt{ResNet} underperform; \item finally, on speech datasets, the models \FCNLStwoTwidth{128}{3}{64} and \FCNLStwoTwidth{64}{3}{64} come in first and second place respectively with \texttt{MLSTMFCN} coming in third. \end{enumerate*} Overall, we observe that the models \FCNLStwoTwidth{128}{3}{64} and \FCNLStwoTwidth{64}{3}{64} perform well across each data domain, coming in top 3 at all times, while the vanilla deep learning baselines without \texttt{LS2T} layers, \FCN{128}, \texttt{MLSTMFCN} and \texttt{ResNet} only work well on certain data modalities, but underperform on others, such as sensor readings as discussed above.

\subsection{Mortality prediction} \label{subseq:s2t:4_mortality}
We consider the  {\sc Physionet2012} challenge dataset \cite{goldberger2000components} for mortality prediction, which is a case of medical \texttt{TSC} as the task is to predict in-hospital mortality of patients after their admission to the \texttt{ICU}. This is a difficult \texttt{ML} task due to missingness in the data, low signal-to-noise ratio (\texttt{SNR}), and imbalanced class distributions with a prevalence ratio of around $14 \%$. We extend the experiments conducted in \cite{horn2020set}, which we also use as very strong baselines. Under the same experimental setting, we train two models: $\texttt{FCN-LS2T}$ as ours and the \texttt{FCN} as another baseline. For both models, we conduct a random search for all hyperparameters with 20 samples from a pre-specified search space, and the setting with best validation performance is used for model evaluation on the test set over 5 independent model trains, exactly the same way as it was done in \cite{horn2020set}. We preprocess the data using the same method as in \cite[eq.~(9)]{che2018recurrent} and additionally handle static features by tiling them along the time axis and adding them as extra coordinates. We additionally introduce in both models a \texttt{SpatialDropout1D} layer after all \texttt{CNN} and \texttt{LS2T} layers with the same tunable dropout rate to mitigate the low \texttt{SNR} of the dataset.

\begin{table}
	\caption{Comparison of \texttt{FCN}-LS2T and \texttt{FCN} on {\sc Physionet2012} with baselines from \cite{horn2020set}.}
	\label{table:s2t:mort_pred}
	\begin{center}
	\begin{small}
	\begin{sc}
		\begin{tabular}{lcccc}
			\toprule
			Model & Accuracy & \texttt{AUPRC} & \texttt{AUROC} \\
			\midrule
			$\texttt{FCN-LS2T}$ & $\mathbf{84.1 \pm 1.6}$ & $\mathbf{53.9 \pm 0.5}$ & $85.6 \pm 0.5$ \\ 
			\texttt{FCN} & $80.7 \pm 1.7$ & $52.8 \pm 1.3$ & $85.6 \pm 0.2$ \\
			\texttt{GRU-D} & $80.0 \pm 2.9$ & $\mathit{53.7 \pm 0.9}$ & $\mathbf{86.3 \pm 0.3}$ \\
			\texttt{GRU-Simple} & $82.2 \pm 0.2$ & $42.2 \pm 0.6$ & $80.8 \pm 1.1$ \\
			\texttt{IP-Nets} & $79.4 \pm 0.3$ & $51.0 \pm 0.6$ & $\mathit{86.0 \pm 0.2}$ \\
			\texttt{Phased-LSTM} & $76.8 \pm 5.2$ & $38.7 \pm 1.5$ & $79.0 \pm 1.0$ \\
			\texttt{Transformer} & $\mathit{83.7 \pm 3.5}$ & $52.8 \pm 2.2$ & $\mathbf{86.3 \pm 0.8}$ \\
			\texttt{Latent-ODE} & $76.0 \pm 0.1$ & $50.7 \pm 1.7$ & $85.7 \pm 0.6$ \\
			\texttt{SeFT-Attn.} & $75.3 \pm 3.5$ & $52.4 \pm 1.1$ & $85.1 \pm 0.4$ \\
			\bottomrule
		\end{tabular}
	\end{sc}
	\end{small}
	\end{center}
\end{table}

\paragraph{Results} Table \ref{table:s2t:mort_pred} compares the performance of $\texttt{FCN-LS2T}$ with that of \texttt{FCN} and the results from \cite{horn2020set} on 3 metrics: \begin{enumerate*}[label=(\arabic*)] \item {\sc accuracy}, \item area under the precision-recall curve ({\sc \texttt{AUPRC}}), \item area under the \texttt{ROC} curve ({\sc \texttt{AUROC}}). 
\end{enumerate*} We can observe that \emph{FCN-LS2T takes on average first place according to both {\sc Accuracy} and {\sc \texttt{AUPRC}}, outperforming \texttt{FCN} and all \texttt{SOTA} methods}, e.g.~{\texttt{Transformer}} \cite{vaswani2017attention}, {\texttt{GRU-D}} \cite{che2018recurrent}, {\texttt{SeFT}} \cite{horn2020set}, and also being competitive in terms of {\sc \texttt{AUROC}}. This is very promising, and it suggests that \texttt{LS2T} layers might be particularly well-suited to complex and heterogenous datasets, such as medical time series, since the \texttt{FCN-LS2T} models significantly improved accuracy on {\sc ECG} as well, another medical dataset in the previous experiment.

\subsection{Generating sequential data} \label{subseq:s2t:4_gpvae}
Finally, we demonstrate on sequential data imputation for time series and video that \texttt{LS2T}s also provide good representations of sequences in generative models.
\paragraph{The \texttt{GP-VAE} model}
In this experiment, we take as base model the recent \texttt{GP-VAE} \cite{fortuin2019gpvae}, that provides state-of-the-art results for probabilistic sequential data imputation. 
The \texttt{GP-VAE} is essentially based on the \texttt{HI-VAE} \cite{nazabal2018handling} for handling missing data in variational autoencoders (\texttt{VAE}s) \cite{kingma2013auto} adapted to the handling of time series data by the use of a Gaussian process (\texttt{GP}) prior \cite{rasmussen2006} across time in the latent sequence space.
The in-depth description of \texttt{GPVAE} can be found in \cite[App.~E.3]{toth2021seq2tens}.
We extend the experiments conducted in \cite{fortuin2019gpvae}, and we make one simple change to the \texttt{GP-VAE} architecture without changing any other hyperparameters or aspects: we introduce a single bidirectional \texttt{LS2T} layer (\texttt{B-LS2T}) into the encoder network that is used in the amortized representation of the means and covariances of the variational posterior.
The \texttt{B-LS2T} layer is preceded by a time-embedding and differencing block, and succeeded by channel flattening and layer normalization as depicted on Figure 5 in \cite[App.~E.3]{toth2021seq2tens}. The idea behind this experiment is to see if we can improve the performance of a highly complicated model that is composed of many interacting submodels, by an educated introduction of \texttt{LS2T} layers.

\paragraph{Results}
To make the comparison, we ceteris paribus re-ran all experiments the authors originally included in their paper \cite{fortuin2019gpvae}, which are imputation of {\sc Healing MNIST}, {\sc Sprites}, and {\sc Physionet2012}.
The results are in Table \ref{table:s2t:gp_vae_comparison}, which report the same metrics as used in \cite{fortuin2019gpvae}, i.e.~negative log-likelihood (\texttt{NLL}, lower is better), mean squared error (\texttt{MSE}, lower is better) on test sets, and downstream classification performance of a linear classifier (\texttt{AUROC}, higher is better). For all other models beside our \texttt{GP-VAE} (\texttt{B-LS2T}), the results were borrowed from \cite{fortuin2019gpvae}. We observe that simply adding the \texttt{B-LS2T} layer improved the result in almost all cases, except for {\sc Sprites}, where the \texttt{GP-VAE} already achieved a very low \texttt{MSE} score. Additionally, when comparing \texttt{GP-VAE} to \texttt{BRITS} on {\sc Physionet2012}, the authors argue that although the \texttt{BRITS} achieves a higher \texttt{AUROC} score, the \texttt{GP-VAE} should not be disregarded as it fits a generative model to the data that enjoys the usual Bayesian benefits of predicting distributions instead of point predictions. The results display that by simply adding our layer into the architecture, we managed to elevate the performance of \texttt{GP-VAE} to the same level while retaining these same benefits. We believe the reason for the improvement is a tighter amortization gap in the variational approximation \cite{Cremer2018inference} achieved by increasing the expressiveness of the encoder by the \texttt{LS2T} allowing it to pick up on long-range interactions in time. For further discussion, see \cite[App.~E.3]{toth2021seq2tens}.

\begin{table}[t]
	\caption{Performance comparison of \texttt{GP-VAE} (\texttt{B-LS2T}) with the baseline methods}
	\label{table:s2t:gp_vae_comparison}
	\begin{center}
	    \begin{small}
		\begin{sc}
		\makebox[\textwidth][c]{
	    \resizebox{\textwidth}{!}{
		\begin{tabular}{lccccc}
			\toprule
			\multirow{2}{*}{Method} & \multicolumn{3}{c}{HMNIST} & \multicolumn{1}{c}{Sprites} & \multicolumn{1}{c}{Physionet} \\
			\cmidrule{2-6}
			& \texttt{NLL} & \texttt{MSE} & \texttt{AUROC} & \texttt{MSE} & \texttt{AUROC} \\
			\midrule
			    Mean imputation & - & $0.168 \pm 0.000$ & $0.938 \pm 0.000$ & $0.013 \pm 0.000$ & $0.703 \pm 0.000$ \\
			    Forward imputation & - & $0.177 \pm 0.000$ & $0.935 \pm 0.000$ & $0.028 \pm 0.000$ & $0.710 \pm 0.000$ \\
                \texttt{VAE} & $0.599 \pm 0.002$ & $0.232 \pm 0.000$ & $0.922 \pm 0.000$ & $0.028 \pm 0.000$ & $0.677 \pm 0.002$ \\
                \texttt{HI-VAE} & $0.372 \pm 0.008$ & $0.134 \pm 0.003$ & $\mathbf{0.962 \pm 0.00}1$ & $0.007 \pm 0.000$ & $0.686 \pm 0.010$ \\
                \texttt{GP-VAE} & $0.350 \pm 0.007$  & $0.114 \pm 0.002$ & $\mathbf{0.960 \pm 0.002}$ & $\mathbf{0.002 \pm 0.000}$ & $0.730 \pm 0.006$ \\  
                \texttt{GP-VAE} (\texttt{B-LS2T}) & $\mathbf{0.251 \pm 0.008}$ & $\mathbf{0.092 \pm 0.003}$ & $\mathbf{0.962 \pm 0.001}$ & $\mathbf{0.002 \pm 0.000}$ & $\mathbf{0.743 \pm 0.007}$\\
                \texttt{BRITS} & - & - & - & - & $\mathbf{0.742 \pm 0.008}$\\
			\bottomrule
		\end{tabular}}}
	\end{sc}
	\end{small}
	\end{center}
\end{table}
	
\section{Related work and Summary}\label{sec: summary}	
\paragraph{Related Work}
The literature on tensor models in \texttt{ML} is vast. 
Related to our approach we mention pars-pro-toto Tensor Networks \cite{cichocki2016tensor}, that use classical \texttt{LR} decompositions, such as \texttt{CP} \cite{carroll1970analysis}, Tucker \cite{tucker1966some}, tensor trains \cite{oseledets2011tensor} and tensor rings \cite{zhao2019learning};
further, \texttt{CNN}s have been combined with \texttt{LR} tensor techniques \cite{cohen2016expressive,kossaifi2017tensor} and extended to \texttt{RNN}s \cite{khrulkov2019generalized}; Tensor Fusion Networks \cite{zadeh2017tensor} and its \texttt{LR} variants \cite{liu2018efficient,liang2019learning,hou2019deep}; tensor-based gait recognition \cite{tao2007general}.
Our main contribution to this literature is the use of the tensor algebra $\T{V}$ with its convolution product $\cdot$, instead of $V^{\otimes m}$ with the outer product $\otimes$ that is used in the above papers. 
While counter-intuitive to work in a larger space $\T{V}$, the additional algebra structure of $(\T{V},\cdot)$ is the main reason for the nice properties of $\Phi$ (\emph{universality, making sequences of arbitrary length comparable, convergence in the continuous time limit}, see \cite[App.~B]{toth2021seq2tens}, which we believe are in turn the main reason for the \emph{strong benchmark performance}. 
Stacked \texttt{LR} sequence transforms allow to exploit this rich algebraic structure with little computational overhead.
Another related literature are path signatures in \texttt{ML} \cite{lyons2014rough, chevyrev_primer_2016, graham2013sparse, Deepsig, toth2020bayesian}.
These arise as special case of \texttt{Seq2Tens} \cite[App.~B]{toth2021seq2tens} and our main contribution to this literature is that \texttt{Seq2Tens} resolves a well-known computational bottleneck in this literature since it \emph{never needs to compute and store a signature}, instead it \emph{directly and efficiently learns the functional of the signature}.

\paragraph{Summary}
We used a classical non-commutative structure to construct a feature map for sequences of arbitrary length. 
By stacking sequence transforms we turned this into scalable and modular  \texttt{NN} layers for sequence data.
The main novelty is the use of the tensor algebra $\T{V}$ constructed from the static feature space $V$.  
While tensor algebras are classical in mathematics, their use in \texttt{ML} seems novel and underexplored.
We would like to re-emphasize that $(\T{V}, \cdot)$ is not a mysterious abstract space: if you know the outer tensor product $\otimes$ then you can easily switch to the tensor convolution product $\cdot$ by taking sums of outer tensor products, as defined in~\eqref{eq:back:alg_mult}.  
As our experiments show, the benefits of this algebraic structure are not just theoretical but can significantly elevate performance of already strong-performing models. 

\begin{subappendices}

\section{Computational details} \label{app:s2t:comps}

Here we give further information on the implementation of \texttt{LS2T} layers detailed in the main text. For simplicity, we fix the state-space of sequences to be $V  = \bbR^d$ from here onwards.

\subsection{Recursive computations} \label{app:s2t:recursive}
Next, we show how the computation of the maps $\Phi_m$ and $\Phi_{m, \theta}$ can be formulated as a joint recursion over the tensor levels and the sequence itself. 
Since $\Phi_m$ is given by a summation over all noncontiguous length-$m$ subsequences with non-repetitions of a sequence $\bx \in \Seq{\R^d}$, simple reasoning shows that $\Phi_m$ obeys the recursion across $m$ and time for $2 \leq l \leq L$ and $1 \leq m$
\begin{align} \label{eq:s2t:recursive_Seq2Tens}
    \Phi_m(\bx_1, \dots \bx_l) = \Phi_m(\bx_1, \dots \bx_{l-1}) + \Phi_{m-1}(\bx_1, \dots, \bx_{l-1}) \otimes \bx_{l},
\end{align}
with the initial conditions $\Phi_0 \equiv 1$, $\Phi_1(\bx_1) = \bx_1$ and $\Phi_m(\bx_1) = \mathbf{0}$ for $m \geq 2$.

Let $\ell = (\ell_m)_{m \geq 0} \in T(\R^d)$ be a sequence of rank-$1$ tensors with $\ell_m = \bz_{m,1} \otimes \cdots \otimes \bz_{m,m} \in (\R^d)^{\otimes m}$ a rank-$1$ tensor of degree-$m$. Then, $\langle \ell_m, \Phi_m(\bx) \rangle$ may be computed analogously to \eqref{eq:s2t:recursive_Seq2Tens} using 
\begin{align} \label{eq:s2t:ls2t_independent}
    \langle \ell_m, \Phi_m(\bx_1, \dots, \bx_l) \rangle =& \langle \bz_{m, 1} \otimes \cdots \otimes \bz_{m, m}, \Phi_m(\bx_1, \dots, \bx_l) \rangle \\ =& \langle \bz_{m, 1} \otimes \cdots \otimes \bz_{m, m}, \Phi_m(\bx_1, \dots, \bx_{l-1}) \rangle \\
    &+ \langle \bz_{m, 1} \otimes \cdots \otimes \bz_{m, m-1}, \Phi_{m-1}(\bx_1, \dots, \bx_{l-1}) \rangle \langle \mathbf{z}_{m,m}, \bx_{l} \rangle \\
    =& \langle \ell_m, \Phi_m(\bx_1, \dots \bx_{l-1}) \rangle \label{eqline:lm_prev} \\
    &+ \langle \bz_{m, 1} \otimes \cdots \otimes \bz_{m, m-1}, \Phi_{m-1}(\bx_1, \dots, \bx_{l-1}) \rangle \langle \mathbf{z}_{m,m}, \bx_{l} \label{eqline:not_lm_1_prev} \rangle
\end{align}
and the initial conditions can be rewritten as the identities $\langle z_{0, 0}, \Phi_0 \rangle = 1$, $\langle \mathbf{z}_{m,1}, \Phi_1(\bx) \rangle = \langle \mathbf{z}_{m,1}, \bx \rangle$ and $\langle \bz_{m,1} \otimes \cdots \otimes \bz_{m, m}, \Phi_m(\bx_1) \rangle = \mathbf{0}$ for $2 \leq m$.

A slight inefficiency of the previous recursion given in \eqref{eqline:lm_prev}, \eqref{eqline:not_lm_1_prev} is that one generally cannot substitute $\langle \ell_{m-1}, \Phi_m(\bx_1, \dots, \bx_{l-1}) \rangle$ for the $\langle \bz_{m,1} \otimes \dots \otimes \bz_{m, m-1}, \Phi_m(\bx_1, \dots, \bx_{l-1}) \rangle \rangle$ term in \eqref{eqline:not_lm_1_prev}, since $\ell_{m-1} \neq \bz_{m,1} \otimes \cdots \otimes \bz_{m, m-1}$ generally. This means that to construct the degree-$m$ linear functional $\langle \ell_m, \Phi_m(\bx_1, \dots, \bx_l) \rangle$, one has to start from scratch by first constructing the degree-$1$ term $\langle \bz_{m,1}, \Phi_1 \rangle$ first, then the degree-$2$ term $\langle \bz_{m,1} \otimes \bz_{m,2}, \Phi_2 \rangle$, and so forth. This further means in terms of complexities that while \eqref{eq:s2t:recursive_Seq2Tens} has linear complexity in the largest value of $m$, henceforth denoted by $M \in \bbN$, \eqref{eqline:lm_prev}, \eqref{eqline:not_lm_1_prev} has a quadratic complexity in $M$ due to the non-recursiveness of the rank-$1$ tensors $(\ell_m)_{m} = (\bz_{m,1} \otimes \cdots \otimes \bz_{m,m})_m$.

The previous observation indicates that an even more memory and time efficient recursion can be devised by parametrizing the rank-$1$ tensors $(\ell_m)_{m}$ in a recursive way as follows: let $\ell_1 = \bz_1 \in \R^d$ and define $\ell_{m} = \ell_{m-1} \otimes \bz_m \in (\R^d)^{\otimes m}$ for $2 \leq m$, i.e. $\ell_m = \bz_1 \otimes \cdots \otimes \bz_m$ for $\{\bz_1, \dots \bz_m \} \subset \R^d$. This parameterization indeed allows to substitute $\ell_{m-1}$ in \eqref{eqline:not_lm_1_prev} so
\begin{align} \label{eq:s2t:ls2t_recursive}
    \langle \ell_m, \Phi_m(\bx_1, \dots, \bx_l) \rangle = \langle \ell_m, \Phi_m(\bx_1, \dots, \bx_{l-1} \rangle + \langle \ell_{m-1}, \Phi_{m-1}(\bx_1, \dots, \bx_{l-1})\rangle \langle \bz_m, \bx_{l} \rangle,
\end{align}
and hence, due to the recursion across $m$ for both $\ell_m$ and $\Phi_m$, it is now linear in the maximal value of $m$, denoted by $M \in \bbN$. This results in a less flexible, but more efficient LS2T, due to the additional added recursivity constraint on the rank-$1$ weight tensors. We refer to this version as the \textit{recursive variant}, while to the non-recursive construction as the \textit{independent variant}.

Next, we show how the previous computations can be rewritten as a simple \texttt{RNN}-like recursion. For simplicity, we consider the recursive formulation, but the independent variant can also be formulated as such with a larger latent state size. Let $(\ell^j)_{j = 1, \dots, n}$ be $n \in \bbN$ different rank-$1$ recursive weight tensors, i.e. $\ell^j = (\ell^j_m)_{m \geq 0}$, $\ell^j_m = \bz^j_1 \otimes \dots \bz^j_m$ for $\bz^j_m \in \R^d$, $m \geq 0$ and $j = 1, \dots, n$. Also, denote $h_{m,i}^j := \langle \ell_m^j, \Phi_m(\bx_1, \dots, \bx_i) \rangle \in \R$, a scalar corresponding to the output of the $j$th linear functional on the $m$th tensor level for the sequence $(\bx_1, \dots, \bx_i)$. We collect all such functionals for given $m$ and $i$ into $\bh_{m, i} := (h_{m, i}^1, \dots h_{m, i}^n) \in \R^n$, i.e. $\bh_{m, i} = \tilde\Phi_{m, \tilde\theta}(\bx_1, \dots, \bx_i)$. 

Additionally, we collect all weight vectors $\bz_m^j \in \R^d$ for a given $m \in \bbN$ into the matrix $\bZ_m := (\bz_m^1, \dots, \bz_m^n)^\top \in \R^{n \times d}$. Then, we may write the following vectorized version of \eqref{eq:s2t:ls2t_recursive}:
\begin{align}
    \bh_{1, i} &=     \bh_{1, i-1} + \bZ_1 \bx_i, \label{eq:s2t:ls2t_recursive_vectorized_lv1} \\
    \bh_{m, i} &= \bh_{m, i-1} + \bh_{m-1, i-1} \odot \bZ_m \bx_i \quad \text{for } m \geq 2,
\end{align}
with the initial conditions $\bh_{m, 0} = \mathbf{0} \in \R^n$ for all $m \geq 1$, and $\odot$ denoting the Hadamard product.
\newpage
\subsection{Algorithms} \label{app:s2t:algs}

  \begin{algorithm}[t]
        \begin{footnotesize}
	\caption{Computing the \texttt{LS2T} layer with independent tensors across levels}
	\label{alg:s2t:ls2t_independent}
	\begin{algorithmic}[1]
		\STATE {\bfseries Input:} Sequences $(\bx^j)_{j=1,\dots,n_\bx} = (\bx^j_1, \dots, \bx^j_{L})_{j=1,\dots,n_\bx} \subset \Seq(\bbR^d)$, \\
		rank-$1$ tensors $(\ell^k)_{k=1,\dots,n_\ell} = (\bz^k_{m,1} \otimes \cdots \otimes \bz^k_{m,m})^{k=1,\dots,n_\ell}_{m=1,\dots,M} \subset T(\bbR^d)$, \texttt{LS2T} order $M \in \bbN$ 
		\STATE Compute $A[m, i, j, l, k] \gets \langle \bz^j_{m, k}, \bx^i_{l} \rangle$ for $m \in \{1,\dots,M\}$, $i \in \{1,\dots,n_\bx\}$, $j \in \{1,\dots,n_\ell\}$, $l \in \{1, \dots, L\}$ and $k \in \{1, \dots, m\}$ \\
		\FOR{$m=1$ {\bfseries to} $M$}
		\STATE Assign $R \gets A[m, :, :, :, 1]$
		\FOR{$k=2$ {\bfseries to} $m$}
		\STATE Iterate $R \gets A[m, :, :, :, k] \odot R[:, :, \boxplus+1]$ \label{algline:alg1_cumsum1}
		\ENDFOR
		\STATE Save $Y_m \gets \cdot R[:, :, \boxplus]$ \label{algline:alg1_cumsum2}
		\ENDFOR
		\STATE {\bfseries Output:} Sequences $(Y_1, \dots, Y_M)$ each of shape $(n_\bx \times L \times n_\ell)$
	\end{algorithmic}
    \end{footnotesize}
\end{algorithm}

\begin{algorithm}[t]
\begin{footnotesize}
	\caption{Computing the \texttt{LS2T} layer with recursive tensors across levels}
	\label{alg:s2t:ls2t_recursive}
	\begin{algorithmic}[1]
		\STATE {\bfseries Input:} Sequences $(\bx^j)_{j=1,\dots,n_\bx} = (\bx^j_1, \dots, \bx^j_{L})_{j=1,\dots,n_\bx} \subset \Seq(\bbR^d)$, \\
		rank-$1$ tensors $(\ell^k)_{k=1,\dots,n_\ell} = (\bz^k_{1} \otimes \cdots \otimes \bz^k_{m})^{k=1,\dots,n_\ell}_{m=1,\dots,M} \subset T(\bbR^d)$, \texttt{LS2T} order $M \in \bbN$ 
		\STATE Compute $A[m, i, j, l] \gets \langle \bz^j_{m}, \bx^i_{l} \rangle$ for $m \in \{1,\dots,M\}$, $i \in \{1,\dots,n_\bx\}$, $j \in \{1,\dots,n_\ell\}$ and $l \in \{1, \dots, L\}$ \\
		\STATE Assign $R \gets A[1, :, :, :]$
		\STATE Save $Y_1 \gets R[:, :, \boxplus]$ \label{algline:alg2_cumsum1}
		\FOR{$m=2$ {\bfseries to} $M$}
		\STATE Update $R \gets A[m, :, :, :] \odot R[:, :, \boxplus+1]$ \label{algline:alg2_cumsum2}
		\STATE Save $Y_m \gets R[:, :, \boxplus]$
		\ENDFOR
		\STATE {\bfseries Output:} Sequences $(Y_1, \dots, Y_M)$ each of shape $(n_\bx \times L \times n_\ell)$
	\end{algorithmic}
\end{footnotesize}
\end{algorithm}


We have shown previously that one may compute $\Phi_\theta(\bx_1, \dots, \bx_i) = (\Phi_{m, \theta}(\bx_1, \dots, \bx_i))_{m \geq 0}$ recursively in a vectorized way for a given sequence $(\bx_1, \dots, \bx_i) \in \Seq(\bbR^d)$. Now, in Algorithms \ref{alg:s2t:ls2t_independent} and \ref{alg:s2t:ls2t_recursive}, we additionally show how to further vectorize the previous computations across time and the batch. For this purpose, let $(\bx^j)_{j=1, \dots, n_\bX} \subset \Seq(\bbR^d)$ be $n_\bX \in \bbN$ sequences in $\bbR^d$ and $(\ell^k)_{k=1,\dots, n_\ell} \subset T(\bbR^d)$ be $n_\ell$ be rank-$1$ tensors in $T(\bbR^d)$.

 \begin{table}[t]
    \centering
    \begin{small}
    \caption{Forward pass computation time in seconds on a Gefore 2080Ti \texttt{GPU} for varying sequence length $L$, fixed batch size $N=32$, state-space dimension $d=64$ and output dimension $h=64$.}
    \label{table:s2t:fwd_pass}
    \vspace{10pt}
    \resizebox{\textwidth}{!}{
    \begin{tabular}{ccccccccc}
        \toprule
        \multirow{2}{*}{$L$} & \multirow{2}{*}{\texttt{Conv1D}} & \multirow{2}{*}{\texttt{LSTM}} & \multicolumn{3}{c}{\texttt{LS2T}} & \multicolumn{3}{c}{\texttt{LS2T-R}}  \\
        \cmidrule{4-9}
        & & & $M=2$ & $M=6$ & $M=10$ & $M=2$ & $M=6$ & $M=10$ \\
        \midrule
        $32$ & $8.1 \times 10^{-4}$ & $1.2 \times 10^{-1}$ & $1.7 \times 10^{-3}$ & $4.5\times 10^{-3}$ & $9.9\times 10^{-3}$ & $1.7 \times 10^{-3}$ & $2.5 \times 10^{-3}$ & $3.4 \times 10^{-3}$ \\
        $64$ & $8.5 \times 10^{-4}$ & $2.3 \times 10^{-1}$ & $1.8 \times 10^{-3}$ & $4.5 \times 10^{-3}$ & $9.9 \times 10^{-3}$ & $1.8 \times 10^{-3}$ & $2.6 \times 10^{-3}$ & $3.4 \times 10^{-3}$ \\
        $128$ & $9.7 \times 10^{-4}$ & $4.6 \times 10^{-1}$ & $2.1 \times 10^{-3}$ & $4.9 \times 10^{-3}$ & $1.0 \times 10^{-2}$ & $2.1 \times 10^{-3}$ & $2.9 \times 10^{-3}$ & $3.8 \times 10^{-3}$ \\
        $256$ & $1.1 \times 10^{-3}$ & $9.3 \times 10^{-1}$ & $2.4 \times 10^{-3}$ & $5.2 \times 10^{-3}$ & $1.1 \times 10^{-2}$ & $2.4 \times 10^{-3}$ & $3.2 \times 10^{-3}$ & $4.0 \times 10^{-3}$ \\
        $512$ & $1.3 \times 10^{-3}$ & $1.8 \times 10^0$ & $3.2 \times 10^{-3}$ & $6.0 \times 10^{-3}$ & $1.1 \times 10^{-2}$ & $3.0 \times 10^{-3}$ & $4.0 \times 10^{-3}$ & $4.8 \times 10^{-3}$ \\
        $1024$ & $1.9 \times 10^{-3}$ & $3.7 \times 10^0$ & $4.4 \times 10^{-3}$ & $7.1 \times 10^{-3}$ & $1.2 \times 10^{-2}$ & $4.4 \times 10^{-3}$ & $5.1 \times 10^{-3}$ & $6.0 \times 10^{-3}$ \\
        \midrule
    \end{tabular}}
    \end{small}
\end{table}

\subsection{Complexity analysis} \label{app:s2t:complexity}
We give a complexity analysis of Algorithms \ref{alg:s2t:ls2t_independent} and \ref{alg:s2t:ls2t_recursive}. Inspection of Algorithm \ref{alg:s2t:ls2t_independent} says that it has $O(M^2 \cdot n_\bx \cdot L \cdot n_\ell \cdot d)$ complexity in both time and memory with an additional memory cost of storing the $O(M^2 \cdot n_\ell \cdot d)$ number of parameters, the rank-$1$ elements $(\ell^k_m)_m$, which are stored in terms of their components $\bz_{m, j}^k \in \bbR^d$. In contrast, Algorithm \ref{alg:s2t:ls2t_independent} has  a time and memory cost of $O(M \cdot n_\bx \cdot L \cdot n_\ell \cdot d)$, thus linear in $M$, and the recursive rank-$1$ elements are now only an additional $O(M \cdot n_\ell \cdot d)$ number of parameters.

Additionally to the big-O bounds on complexities, another important question is how well the computations can be parallelized, which can have a larger impact on computations when e.g. running on \texttt{GPU}s. Observing the algorithms, we can see that the only operation over the time axis is the cumsum ($\boxplus$) operation in Lines \ref{algline:alg1_cumsum1}, \ref{algline:alg1_cumsum2} (Algorithm \ref{alg:s2t:ls2t_independent}) and Lines \ref{algline:alg2_cumsum1}, \ref{algline:alg2_cumsum2} (Algorithm \ref{alg:s2t:ls2t_recursive}). The cumulative sum operates recursively on the whole time axis, and it can be parallelized using a work efficient scan algorithm to run in $\log$-length time.

To gain further intuition about what kind of performance one can expect for our \texttt{LS2T} layers, we benchmarked the computation time of a forward pass for varying sequence lengths and varying hyperparameters of the model. For comparison, we ran the same experiment with an \texttt{LSTM} layer and a \texttt{Conv1D} layer with a filter size of $32$. The input is a batch of sequences of shape $(n_\bX \times L \times d)$, while the output has shape $(n_\bX \times L \times h)$, where $d \in \bbN$ is the state-space dimension of the input sequences, while $h \in \bbN$ is simply the number of channels or hidden units in the layer. For our layers, we used our own implementation in \texttt{Tensorflow}, while for \texttt{LSTM} and \texttt{Conv1D}, we used the \texttt{Keras} implementation using the \texttt{Tensorflow} backend. 

In Table \ref{table:s2t:fwd_pass}, we report the average computation time of a forward pass over $100$ trials, for fixed batch size $n_\bX = 32$, state-space dimension $d = 64$, output dimension $h = 64$ and varying sequence lengths $L \in \{32, 64, 128, 256, 512, 1024\}$. \texttt{LS2T} and \texttt{LS2T-R} respectively refer to the independent and recursive variants, and $M \in \bbN$ denotes the truncation degree. We can observe that while the \texttt{LSTM} practically scales linearly in $L$, the scaling of \texttt{LS2T} is sublinear for all practical purposes, exhibiting a growth rate that is more close to that of the \texttt{Conv1D} layer, that is fully parallelizable. Specifically, while the \texttt{LSTM} takes $3.7$ seconds to make a forward pass for $L=1024$, all variants of the \texttt{LS2T} layer take less time than that \emph{by a factor of at least a 100}. Additionally, we observe that \texttt{LS2T} exhibits a more aggressive growth rate with respect to the parameter $M$ due to the quadratic complexity in $M$ (although the numbers show only linear growth), while \texttt{LS2T-R} scales very favourably in $M$ as well due to the linear complexity (the results indicate a sublinear growth rate).

\newpage

\section{Further results} \label{app:s2t_further_res}

\begin{table}[h]
	\caption{Classifier accuracies on the multivariate \texttt{TSC} datasets with the best and second best highlighted for each row in \textbf{bold} and \textit{italic}, respectively.}
	\label{table:s2t:classification_accuracies}
	\begin{center}
		\resizebox{\textwidth}{!}{
		\begin{sc}
		\begin{tabular}{lccccccccccccc}
			\toprule
            Dataset & \texttt{ARKernel} & \texttt{DTW} & \texttt{LPS} & \texttt{SMTS} & \texttt{gRSF} & \texttt{mvARF} & \texttt{MUSE} & \texttt{MLSTMFCN} & \FCN{128} & \texttt{ResNet} & \LStwoTwidth{64}{3} & \FCNLStwoTwidth{64}{64}{3} & \FCNLStwoTwidth{128}{64}{3} \\ 
			\midrule
			ArabicDigits & $0.988$ & $0.908$ & $0.971$ & $0.964$ & $0.975$ & $0.952$ & $0.992$ & $0.990$ & $0.995(0.001)$ & $0.995(0.002)$ & $0.979(0.002)$ & $\mathit{0.996}(0.001)$ & $\mathbf{0.997}(0.001)$ \\
			AUSLAN & $0.918$ & $0.727$ & $0.754$ & $0.947$ & $0.955$ & $0.934$ & $0.970$ & $0.950$ & $0.979(0.003)$ & $0.971(0.003)$ & $0.987(0.002)$ & $\mathbf{0.996}(0.001)$ & $\mathit{0.995}(0.001)$ \\
			Char.~Traj. & $0.900$ & $0.948$ & $0.965$ & $0.992$ & $\mathit{0.994}$ & $0.928$ & $0.937$ & $0.990$ & $0.992(0.001)$ & $0.985(0.002)$ & $0.980(0.003)$ & $0.993(0.001)$ & $\mathbf{0.995}(0.000)$ \\
			CMUsubject16 & $\mathbf{1.000}$ & $0.930$ & $\mathbf{1.000}$ & $\mathit{0.997}$ & $\mathbf{1.000}$ & $\mathbf{1.000}$ & $\mathbf{1.000}$ & $\mathbf{1.000}$ & $\mathbf{1.000}(0.000)$ & $\mathbf{1.000}(0.000)$ & $\mathbf{1.000}(0.000)$ & $\mathbf{1.000}(0.000)$ & $\mathbf{1.000}(0.000)$ \\
			DigitShapes & $\mathbf{1.000}$ & $\mathbf{1.000}$ & $\mathbf{1.000}$ & $\mathbf{1.000}$ & $\mathbf{1.000}$ & $\mathbf{1.000}$ & $\mathbf{1.000}$ & $\mathbf{1.000}$ & $\mathbf{1.000}(0.000)$ & $\mathbf{1.000}(0.000)$ & $\mathbf{1.000}(0.000)$ & $\mathbf{1.000}(0.000)$ & $\mathbf{1.000}(0.000)$ \\
			ECG & $0.820$ & $0.790$ & $0.820$ & $0.818$ & $0.880$ & $0.785$ & $0.880$ & $0.870$ & $0.860(0.018)$ & $0.856(0.010)$ & $0.824(0.016)$ & $\mathbf{0.892}(0.015)$ & $\mathit{0.886}(0.014)$ \\
			Jap.~Vowels & $0.984$ & $0.962$ & $0.951$ & $0.969$ & $0.800$ & $0.959$ & $0.976$ & $\mathbf{1.000}$ & $0.990(0.003)$ & $0.989(0.003)$ & $0.984(0.005)$ & $0.991(0.003)$ & $\mathit{0.994}(0.003)$ \\
			Kick vs Punch & $0.927$ & $0.600$ & $0.900$ & $0.820$ & $\mathbf{1.000}$ & $\mathit{0.976}$ & $\mathbf{1.000}$ & $0.900$ & $\mathbf{1.000}(0.000)$ & $\mathbf{1.000}(0.000)$ & $\mathbf{1.000}(0.000)$ & $\mathbf{1.000}(0.000)$ & $\mathbf{1.000}(0.000)$ \\
			LIBRAS & $0.952$ & $0.888$ & $0.903$ & $0.909$ & $0.911$ & $0.945$ & $0.894$ & $\mathbf{0.970}$ & $\mathit{0.966}(0.002)$ & $\mathit{0.966}(0.008)$ & $0.859(0.008)$ & $0.946(0.005)$ & $0.956(0.008)$ \\
			NetFlow & nan & $0.976$ & $0.968$ & $0.977$ & $0.914$ & nan & $0.961$ & $0.950$ & $0.970(0.003)$ & $0.953(0.006)$ & $0.921(0.014)$ & $0.962(0.006)$ & $0.962(0.005)$ \\
			PEMS & $0.750$ & $0.832$ & $0.844$ & $0.896$ & $1.000$ & nan & nan & nan & $0.775(0.019)$ & $0.787(0.008)$ & $0.725(0.013)$ & $0.788(0.025)$ & $0.802(0.017)$ \\
			PenDigits & $0.952$ & $0.927$ & $0.908$ & $0.917$ & $0.932$ & $0.923$ & $0.912$ & $\mathbf{0.970}$ & $\mathit{0.967}(0.002)$ & $0.963(0.001)$ & $0.956(0.002)$ & $0.963(0.003)$ & $0.962(0.002)$ \\
			Shapes & $\mathbf{1.000}$ & $\mathbf{1.000}$ & $\mathbf{1.000}$ & $\mathbf{1.000}$ & $\mathbf{1.000}$ & $\mathbf{1.000}$ & $\mathbf{1.000}$ & $\mathbf{1.000}$ & $\mathbf{1.000}(0.000)$ & $\mathbf{1.000}(0.000)$ & $\mathbf{1.000}(0.000)$ & $\mathbf{1.000}(0.000)$ & $\mathbf{1.000}(0.000)$ \\
			UWave & $0.904$ & $0.916$ & $\mathbf{0.980}$ & $0.941$ & $0.929$ & $0.952$ & $0.916$ & $0.970$ & $\mathit{0.979}(0.001)$ & $0.978(0.001)$ & $0.958(0.001)$ & $0.975(0.002)$ & $0.976(0.001)$ \\
			Wafer & $0.968$ & $0.974$ & $0.962$ & $0.965$ & $\mathit{0.992}$ & $0.931$ & $\mathbf{0.997}$ & $0.990$ & $0.987(0.005)$ & $0.989(0.002)$ & $0.983(0.003)$ & $0.988(0.001)$ & $0.990(0.001)$ \\
			Walk vs Run & $\mathbf{1.000}$ & $\mathbf{1.000}$ & $\mathbf{1.000}$ & $\mathbf{1.000}$ & $\mathbf{1.000}$ & $\mathbf{1.000}$ & $\mathbf{1.000}$ & $\mathbf{1.000}$ & $\mathbf{1.000}(0.000)$ & $\mathbf{1.000}(0.000)$ & $\mathbf{1.000}(0.000)$ & $\mathbf{1.000}(0.000)$ & $\mathbf{1.000}(0.000)$ \\
			\midrule 
			Avg.~acc. & $0.938$ & $0.899$ & $0.933$ & $0.945$ & $0.955$ & $0.949$ & $0.962$ & $\mathbf{0.970}$ & $0.966$ & $0.964$ & $0.947$ & $0.\mathit{968}$ & $\mathbf{0.970}$ \\
			Med.~acc. & $0.952$ & $0.929$ & $0.964$ & $0.964$ & $0.984$ & $0.952$ & $0.976$ & $0.990$ & $0.988$ & $0.987$ & $0.982$ & $\mathit{0.992}$ & $\mathbf{0.994}$ \\
			Sd.~acc. & $0.073$ & $0.111$ & $0.073$ & $0.059$ & $0.058$ & $0.055$ & $0.043$ & $0.039$ & $0.061$ & $0.059$ & $0.079$ & $0.056$ & $0.054$ \\
			Avg.~rank & $6.000$ & $6.812$ & $6.000$ & $5.625$ & $4.625$ & $6.714$ & $4.933$ & $3.333$ & $3.000$ & $3.500$ & $5.312$ & $\mathit{2.750}$ & $\mathbf{2.312}$ \\
			Med.~rank & $6.000$ & $8.000$ & $6.000$ & $6.500$ & $2.500$ & $8.500$ & $4.000$ & $3.000$ & $\mathit{2.500}$ & $3.000$ & $6.500$ & $\mathit{2.500}$ & $\mathbf{2.000}$ \\
			Sd.~rank & $4.071$ & $4.246$ & $4.397$ & $3.810$ & $3.964$ & $4.514$ & $4.044$ & $2.610$ & $2.066$ & $2.251$ & $3.591$ & $1.880$ & $1.493$ \\
			\bottomrule
		\end{tabular}
	\end{sc}
	}
	\end{center}
\end{table}

\end{subappendices}

\endgroup
\begingroup
\chapter{Capturing Graphs by Hypo-Elliptic Diffusions} \label{ch:g2t}
\section{Introduction}

  Obtaining a latent description of the non-Euclidean structure of a graph is central to many applications.
  One common approach is to construct features for each node that represents the local neighborhood; pooling these node features then provides a latent description of the whole graph.
  A classic way to arrive at such node features is by random walks: at the given node one starts a random walk, and extracts a neighbourhood summary from its sample trajectories.
  We revisit this random walk construction and are inspired by two classical mathematical results:
\begin{description}
\item[Hypo-elliptic Laplacian.] In the case of Brownian motion $B=(B_t)_{t \ge 0}$ evolving in $\bbR^n$, the quantity $u(t,x)= \bbE[f(B_t)|B_0=x]$ solves the heat equation $\partial_t u= \Delta u$ on $[0,\infty) \times \bbR^n$, $u(0,x)=f(x)$.
Seminal work of Gaveau \cite{gaveau_principe_1977} shows that if one replaces $f(B_t)$ in the expectation by a functional of the whole trajectory, $F(B_s : s \in [0,t])$, then a path-dependent heat equation can be derived
using the hypo-elliptic Laplacian in place of classic Laplacian.

\item[Tensor Algebras.]
A simple way to capture a sequence -- for us, a sequence of nodes visited by a random walk on a graph -- is to associate with each sequence element an element in an algebra\footnote{An algebra is a vector space where one can multiply elements; e.g.~the set of $n\times n$ matrices with matrix multiplication. This multiplication can be non-commutative; e.g.~$A\cdot B \neq B\cdot A$ for general matrices $A,B$.}  and multiply these algebra elements together. 
If the algebra multiplication is commutative, the sequential structure is lost but if it is non-commutative, this captures the order in the sequence.
In fact, by using the tensor algebra, this can be done faithfully and linear functionals of this algebra correspond to functionals on the space of sequences.
\end{description}

We generalize these ideas from the Euclidean case of $\bbR^d$ to the non-Euclidean case of graphs.
In particular, we construct node features by sampling from a random walk started at the node, but instead of averaging over end points, we average over path-dependent functions.
Informally, instead of asking a random walker that started at a node, \emph{"What do you see now?"} after $k$ steps, we ask \emph{"What have you seen along your way?"}.
The above notions from mathematics about the hypo-elliptic Laplacian and the tensor algebra allow us to formalize this in the form of a generalized graph diffusion equation and we develop algorithms that make this a scalable method.

\paragraph{Related Work}
From the \texttt{ML} literature, \cite{perozzi2014deepwalk, grover2016node2vec} popularized the combination of deep learning architectures to capture random walk histories. 
Such ideas have been incorporated, sometimes implicitly, into graph neural networks (\texttt{GNN})
\cite{scarselli2008graph,bruna2013spectral,schlichtkrull2018modeling,defferrard2016convolutional,hamilton2017inductive,battaglia2016interaction,kipf_semi-supervised_2017} that in turn build on convolutional approaches \cite{lecun1995convolutional,lecun1998gradient,grover2016node2vec}, as well as their combination with attention or message passing \cite{monti2017geometric,velickovic_graph_2018,gilmer2017neural}, and more recent improvements \cite{xu_how_2019,morris2019weisfeiler,maron2019provably,chen2019equivalence} that provide and improve on theoretical guarantees.
Another classic approach are graph kernels, see \cite{borgwardt2020graph} for a recent survey; in particular, the seminal paper \cite{Kondor2002DiffusionKO} explored the connection between diffusion equations and random walkers in a kernel learning context. 
More recently, \cite{chen2020convolutional} proposed sequence kernels to capture the random walk history. Furthermore,~\cite{cochrane2021sk} uses the signature kernel maximum mean discrepancy (\texttt{MMD}) \cite{chevyrev2022signature} as a metric for trees which implicitly relies on the algebra of tensors that we use, and \cite{nikolentzos2020random} aggregates random walk histories to derive a kernel for graphs.
Moreover, the concept of network motifs \cite{Milo2002NetworkMS,schwarze2021motifs} relates to similar ideas that describe a graph by node sequences. 
Further, the Bethe Hessian \cite{Saade2014SpectralCO} has been successfully used in spectral clustering and shares the same goal of capturing path-dependence via "deformed Laplacians", although the mathematical approach is very different to ours.
Directly related to our approach is the topic of learning diffusion models \cite{Klicpera2019DiffusionIG,Chamberlain2021GRANDGN,thorpe_grand_2022,elhag2022graph,beltrami} on graphs. 
While similar ideas on random walks and diffusion for graph learning have been developed by different communities, our proposed method leverages these perspectives by capturing random walk histories through a novel diffusion operation.
 

Our main mathematical influence is the seminal work of Gaveau  \cite{gaveau_principe_1977} from the 1970s 
that shows how Brownian motion can be lifted into a Markov process evolving in a tensor algebra to capture path-dependence.
This leads to a heat equation governed by the hypo-elliptic Laplacian.
These insights had a large influence in \texttt{PDE} theory, see \cite{Rothschild1976HypoellipticDO,hormander1967hypoelliptic}, but it seems that their discrete counterpart on graphs has received no attention despite the well-developed 
literature on random walks on graphs and general non-Euclidean objects, \cite{woess2000random,diaconis1988group, grigoryan2009heat, varopoulos1992analysis}.
A key challenge to go from theory to application is handling the computational complexity.
To do so, we build on ideas from \cite{toth2021seq2tens} to design effective algorithms for the hypo-elliptic graph diffusion.

\paragraph{Contribution and Outline}
We introduce the hypo-elliptic graph Laplacian which allows to effectively capture random walk histories through a generalized diffusion model.
\begin{itemize}
    \item In Section~\ref{sec:g2t:diffusion}, we introduce the hypo-elliptic variants of standard graph matrices such as the adjacency matrix and (normalized) graph Laplacians. These hypo-elliptic variants are formulated in terms of tensor-valued matrices rather than scalar-valued matrices, and can be manipulated using linear algebra in the same manner as the classical setting.
    \item The hypo-elliptic Laplacian leads to a corresponding diffusion model, and in Section~\ref{thm:non_abelian_laplacian}, we show that the solution to this generalized diffusion equation summarizes the microscopic picture of \emph{the entire history of random walks} and not just their location after $k$ steps.
    \item This solution provides a rich description of the local neighbourhood about a node, which can either be used directly as node features or be pooled over the graph to obtain a latent description of the graph. Section~\ref{thm:characterizing_rw_informal} shows that these node features characterize random walks on the graph, and we provide an analogous statement for graphs in~\cite[App.~E]{toth2022capturing}.
    \item One can solve the hypo-elliptic graph diffusion equation with linear algebra, but this is computationally prohibitive and Thm.~\ref{thm:algo} provides an efficient low-rank approximation.
    \item Finally, Section \ref{sec:g2t::experiments} provides experiments and benchmarks. A particular focus is to test the ability of our model to capture long-range interactions between nodes and the robustness of pooling operations, potentially making it more robust to "over-squashing" \cite{alon2021on}.
\end{itemize}

\section{Sequence Features by Non-Commutative Multiplication.}\label{sec:g2t:sequence features}
We recall the algebraic construction of sequence features from Chapter \ref{ch:s2t}.
Let us define the space of finite, but potentially different length sequences in $\bbR^d$ started from $\mathbf{0}$ as
\begin{align}
    \Seq(\bbR^d) = \curls{\bx = (\mathbf{0}, \bx_1, \dots, \bx_L) \setgiven \bx_0, \dots, \bx_L \in \bbR^d, L \in \bbN},
\end{align}
where analogously to Section \ref{ch:back}, we use $0$-based indexing for sequences, and define the effective length of a sequence $\bx = (\mathbf{0}, \bx_1, \dots, \bx_L) \in \Seq(\bbR^d)$ by $\len{\bx} = L$.

Assume we are given an injective map, which we call the \emph{algebra lifting}, $\fm: \bbR^d \to H$ from $\bbR^d$ into an algebra $H$.
We can use this to define a \emph{sequence feature map}\footnote{There are variants of this sequence feature map, which are discussed in~\cite[App.~B]{toth2021seq2tens} and \cite[App.~G]{toth2022capturing}.}
\begin{equation}
\label{eq:g2t:sequence_feature_map}
    \fms: \Seq(\bbR^d) \rightarrow H,\quad\fms(\bx) = \fm(\delta \bx_1) \cdots \fm(\delta \bx_{\len{\bx}} ),
\end{equation}
where $\delta \bx_i = \bx_i - \bx_{i-1}$ for $i\ge 1$ are used to denote the \emph{increments} of a sequence $\bx=(x_0,\ldots,x_k)$. This map associates to any sequence $\bx \in \Seq(\bbR^d)$ an element of the algebra $H$.
If the multiplication in $H$ is commutative, then the map $\fms$ would have no information about the order of increments, i.e. $\fm(\delta_0\bx) \cdots \fm(\delta_{k}\bx) = \fm(\delta_{\pi(0)}\bx) \cdots \fm(\delta_{\pi(k)}\bx)$ for any permutation $\pi$ of $\{0,\ldots,k\}$.
However, if the multiplication in $H$ is "non-commutative enough" we expect $\fms$ to be injective.

\paragraph{A Free Construction}
Many choices for $H$ are possible, but intuitively it makes sense to use the "most general object" for $H$.
The mathematically rigorous approach is to use the \emph{tensor algebra over $\bbR^d$}, see Section \ref{sec:back:tensors}.
Despite this abstract motivation, the algebra $H$ has a concrete form: it is realized as a sequence of tensors in $\bbR^d$ of increasing degree, and is defined by
\begin{align}
\label{eq:g2t:free_algebra}
    H = \{ \bv = (\bv_0, \bv_1,\bv_2,\ldots): \bv_m \in (\bbR^d)^{\otimes m}, \,m \in \bbN,\,  \|\bv\| < \infty\},
\end{align}
where by convention $(\bbR^d)^{\otimes 0}=\bbR$, and we describe the norm $\|\bv\|$ in the paragraph below.
For example, if $\bv=(\bv_m)_{m \ge 0} \in H$, then $\bv_0$ is a scalar, $\bv_1$ is a vector, $\bv_2 \in (\bbR^d)^{\otimes 2}$ is a $d \times d$ matrix, and so on.
The vector space structure of $H$ is given by addition and scalar multiplication as 
\begin{align}
    \bv + \bw = (\bv_m + \bw_m)_{m \ge 0} \in H\quad \text{ and }\quad  \lambda \bv = (\lambda \bv_m)_{m \ge 0} \in H
\end{align}
for $\lambda \in \bbR$,
and the algebra structure is given by
\begin{align}
\label{eq:g2t:algebra_multiplication}
    \bv \cdot \bw = \left( \sum_{i=0}^m \bv_i \otimes \bw_{m-i}\right)_{m \ge 0} \in H.
\end{align}
\paragraph{An Inner Product}\looseness=-1
If $e^1,\ldots,e^d$ is a basis of $\bbR^d$, then every tensor $\bv_m\in (\bbR^d)^{\otimes m}$ can be written as 
\begin{align}
\bv_{m} = \sum_{1\leq i_1,\ldots,i_m \leq d} c_{i_1,\ldots,i_m} e^{i_1} \otimes \cdots \otimes e^{i_m}.
\end{align}
This allows us to define an inner product $\langle \cdot, \cdot \rangle_m$ on $(\bbR^d)^{\otimes m}$ by extending
\begin{equation}
\label{eq:g2t:tensor_inner_product}
\langle e^{i_1}\otimes \cdots \otimes e^{i_m}, e^{j_1}\otimes \cdots \otimes e^{j_m}\rangle_m = \left\{\begin{array}{cl}1 &: i_1=j_1,\ldots,i_m=j_m,\\0 &: \text{otherwise.}\end{array}\right.
\end{equation}
to $(\bbR^d)^{\otimes m}$ by linearity.
This gives us an inner product on $H$,
\begin{align}
\langle \bv, \bw \rangle = \sum_{m \ge 0} \langle \bv_m, \bw_m\rangle_m
\end{align}
such that $H$ is a Hilbert space; in particular we get a norm $\|\bv\| = \sqrt{ \langle \bv,\bv \rangle}$.
To sum up, the space $H$ has a rich structure: it has a vector space structure, it has an algebra structure (a noncommutative product), and it is a Hilbert space (an inner product between elements of $H$ gives a scalar).

\paragraph{Characterizing Random Walks}
From~\eqref{eq:g2t:sequence_feature_map}, we have constructed a map $\fms$ that maps a sequence $\bx \in \Seq(\bbR^d)$ of arbitrary length into the space $H$ (see~\cite[App.~C]{toth2022capturing} for further details). 
Our aim is to apply this to the sequence of node attributes corresponding to random walks on a graph.
Therefore, the expectation of $\fms$ should be able to characterize the distribution of the random walk.
Formally the map $\fms$ is \emph{characteristic} if the map $\mu \mapsto \bbE_{\bx \sim \mu}[\fms(\bx)]$ from the space of probability measures on $\Seq(\bbR^d)$ into $H$ is injective. 
Indeed, if the chosen lifting $\fm$ satisfies some mild conditions this holds for $\fms$; see~\cite[App.~C]{toth2022capturing} and \cite{chevyrev2022signature} for further details. 

\paragraph{Linear Functionals}
The quantity $\bbE_{\bx \sim \mu}[\fms(\bx)]$ characterizes the probability measure $\mu$ but is valued in the infinite-dimensional Hilbert space $H$. 
However, we can instead consider
\begin{align}\label{eq:g2t: linear functionals of expected sig}
 \langle \bell, \bbE_{\bx \sim \mu}[\fms(\bx)]\rangle  \text{ for } \bell = (\bell_0,\bell_1,\bell_2,\ldots, \bell_M,0,\ldots) \in H \text{ and }M \ge 1 
\end{align}
which is equivalent to knowing $\bbE_{\bx \sim \mu}[\fms(\bx)]$; i.e.~the set \eqref{eq:g2t: linear functionals of expected sig} characterizes $\mu$. 
This is analogous to how one can use either the moment generating function of a real-valued random variable or its moments to characterize its distribution; the former is infinite-dimensional (a function), the latter is an infinite sequence of scalars.
We extend a key insight from \cite{toth2021seq2tens} in Section~\ref{sec:g2t:algos}: a linear functional $\langle \bell, \bbE_{\bx \sim \mu}[\fms(\bx)] \rangle$ can be efficiently approximated using low-rank decompositions.

\paragraph{The Tensor Exponential}
While we will continue to keep $\fm$ arbitrary for our main results (see~\cite{toth2021seq2tens} and App.~\ref{app:g2t:variations} for other choices), we will use the \emph{tensor exponential} defined by 
\begin{equation}
\label{eq:g2t:tensor_exponential}
    \exp_\otimes: \bbR^d \rightarrow H, \quad \exp_{\otimes}(x) = \left(\frac{x^{\otimes m}}{m!}\right)_{m\geq 0},
\end{equation}
as the primary example throughout this paper and in the experiments in Section~\ref{sec:g2t::experiments}. With this choice, the induced sequence feature map is the discretized version of a classical object in analysis, called the path signature, see \cite[App.~C]{toth2022capturing} and Section \ref{sec:back:discrete} for further details.


\section{Hypo-Elliptic Diffusions}\label{sec:g2t:diffusion}
Throughout this section, we fix a labelled graph $\cG = (\cV, \cE, \nf)$, where $\cV$ is a set of $n$ nodes $\cV = \{1, \ldots, n\}$, $\cE$ denotes edges and $\nf:\cV \to \bbR^d$ are continuous node attributes\footnote{The labels given by the labelled graph are called \emph{attributes}, while the computed updates are called \emph{features}.} which map each node to an element in the vector space $\bbR^d$.
Two nodes $i, j \in \cV$ are \emph{adjacent} if $(i,j) \in \cE$ is an edge, and we denote this by $i \sim j$. 
The \emph{adjacency matrix} $A$ of a graph is defined by $A_{i,j}  = 1$, whenever $i \sim j$, and $0$ otherwise. 
We denote by $\operatorname{deg}(i)$ the number of neighbours of node $i$.

\paragraph{Random Walks on Graphs}
Let $(B_k)_{k\ge 0}$ be the simple random walk on $\cV$, where the initial node is chosen uniformly. The \emph{transition matrix} of this time-homogeneous Markov chain is 
\begin{align}
P_{i,j} = \bbP(B_k = j| B_{k-1}=i) = \left\{ \begin{array}{cl}\frac{1}{\operatorname{deg}(i)} &: i\sim j\\ 0 &: \text{otherwise}.\end{array}\right.
\end{align}
Denote by $(L_k)_{k \ge 0}$ the random walk lifted to the node attributes in $\bbR^d$, that is
\begin{equation}
\label{eq:g2t:lifted_random_walk}
    L_k = \nf(B_k).
\end{equation}
Recall that the \emph{normalized graph Laplacian} for random walks is defined as $\cL =I-D^{-1}A$, where $D$ is diagonal degree matrix; in particular, the entry-wise definition is 
 \begin{align}
    \label{eq:g2t:laplacian_matrix}
        \cL_{i,j} = \left\{
        \begin{array}{cl}
            -\frac{1}{\operatorname{deg}(i)} &  : i \sim j \\
            1 & : i=j\\
            0 & : \text{otherwise}.
        \end{array}
        \right.
    \end{align}

The discrete graph diffusion equation for $U_k \in \bbR^{n \times d}$ is given by
\begin{align}\label{eq:g2t:heat equation}
    U_k - U_{k-1} = -\cL U_{k-1}, \quad U^{(i)}_0 = \nf(i)
\end{align}
where the initial condition $U_0 \in \bbR^{n \times d}$ is specified by the node attributes.\footnote{The attributes over all nodes are given by an $n \times d$ matrix; in particular $U_k^{(i)}$ is the $i^{\text{th}}$ row of the matrix.}  The probabilistic interpretation of the solution to this diffusion equation is classical and given as
\begin{equation}\label{eq:g2t:classic graph diffusion}
    U_k = \left(\bbE[L_k \given B_0 = i]\right)_{i=1}^n = P^k U_0.
\end{equation}
This allows us to compute the solution $u_k$ using the transition matrix $P = I - \cL$.

\paragraph{Random Walks on Algebras}
We incorporate the history of a random walker by considering
\begin{align}\label{eq:g2t:narw}
    \bbE[\fms(\bL_k) \given B_0 = i]=\bbE[\fm(\delta \bL_1) \cdots \fm(\delta \bL_{\len{\bL}})  \given B_0 = i]
\end{align}
where $\bL_k = (L_0, \ldots, L_k)$.
Since $\fms$ captures the whole history of the random walk $\bL_k$ over node attributes, we expect this expectation to provide a richer summary of the neighborhood of node $i$ than $\bbE[L_k|B_0=i]$. 
The price is however, the computational complexity, since \eqref{eq:g2t:narw} is $H$-valued.
We first show, that analogous to \eqref{eq:g2t:heat equation}, the quantity \eqref{eq:g2t:narw} satisfies a diffusion equation that can be computed with linear algebra. 
To do so, we develop a graph analogue of the hypo-elliptic Laplacian and replace the scalar entries of the matrices with entries from the algebra $H$.

\paragraph{Matrix Rings over Algebras}

We first revisit the adjacency matrix $A \in \bbR^{n \times n}$ and replace it by the \emph{tensor adjacency matrix} $\wA= (\wA)_{i,j} \in \tensalgmr$, that is $\wA$ is a matrix but instead of scalar entries its entries are elements in the algebra $H$.
The matrix $A$ has an entry at $i,j$ if nodes $i$ and $j$ are connected; $\wA$ replaces the $i,j$ entry with an element of $H$ that tells us how node attributes differ,
 \begin{align}
    \label{eq:g2t:exponential_adjacency_matrix}
        \wA_{i,j} = \left\{
        \begin{array}{cl}
            \fm(\nf(j) - \nf(i)) &  : i \sim j \\
            0 & : \text{otherwise}.
        \end{array}
        \right.
    \end{align}
Matrix multiplication works for elements of $\tensalgmr$ by replacing scalar multiplication by multiplication in $H$, that is $(\wB\cdot\wC)_{i,j} = \sum_{k=1}^n \wB_{i,k} \cdot \wC_{k,j}$ for $\wB, \wC \in \tensalgmr$ and $\wB_{i,k} \cdot \wC_{k,j}$ denotes multiplication in $H$ as in eq.~\eqref{eq:g2t:algebra_multiplication}.
For the classical adjacency matrix $A$, the $k$-th power counts the number of length $k$ walks in the graph, so that $(A^k)_{i,j} $ is the number of walks of length $k$ from node $i$ to node $j$.
We can take powers of $\wA$ in the same way as in the classical case, where
\begin{equation}\label{eq:g2t:powers_tensor_adjacency}
    (\wA^k)_{i,j} = \sum_{\bx} \fm(\delta\bx_1) \cdots \fm(\delta\bx_{\len{\bx}})
\end{equation}
where the sum is taken over all length $k$ walks $\bx = (f(i), \ldots f(j))$ from node $i$ to node $j$ (full details are provided in Appendix~\ref{app:g2t:details_diffusion})
Since $\fms(\bx)$ characterizes each walk $\bx$, the entry $\wA^k_{i,j}$ can be interpreted as a summary of all walks which connect nodes $i$ and $j$.

\paragraph{Hypo-elliptic Graph Diffusion}

Similar to the tensor adjacency matrix, we define the \emph{hypo-elliptic graph Laplacian} as the $n \times n$ matrix
\begin{align}
    \wcL = I-D^{-1}\wA \in \tensalgmr,
\end{align}  
where $D$ is the degree matrix embedded into $\tensalgmr$ at tensor degree $0$. The entry-wise definition
\begin{align}
    \label{eq:g2t:tensor_laplacian_matrix}
        \wcL_{i,j} = \left\{
        \begin{array}{cl}
            \frac{-\fm(\nf(j) - \nf(i))}{\operatorname{deg}(i)} &  : i \sim j \\
            1 &: i = j \\
            0 & : \text{otherwise}.
        \end{array}
        \right.
\end{align}
We can now formulate the \emph{hypo-elliptic graph diffusion equation} for $\bv_k \in \tensalgps^n$ as
\begin{equation}
\label{eq:g2t:non_abelian_diffusion_equation}
    \bv_k - \bv_{k-1} = - \wcL \bv_{k-1}, \quad \bv^{(i)}_0 = \fm(\nf(i)).
\end{equation}
Analogous to classical graph diffusion \eqref{eq:g2t:classic graph diffusion}, the hypo-elliptic graph diffusion \eqref{eq:g2t:non_abelian_diffusion_equation} has a probabilistic interpretation in terms of $\bL$ as shown in Thm.~\ref{thm:non_abelian_laplacian} (the proof is given in App.~\ref{app:g2t:details_diffusion}).
\begin{theorem}
\label{thm:non_abelian_laplacian}
    Let $k \in \bbN$, $\bL_k = (L_0, \ldots, L_k)$ be the lifted random walk from~\eqref{eq:g2t:lifted_random_walk}, and $\wP = I - \wcL$ the \emph{tensor adjacency matrix}. The solution to the hypo-elliptic graph diffusion eq.~\eqref{eq:g2t:non_abelian_diffusion_equation} is
    \begin{align}
        \bv_k = \left( \bbE[\fm(\delta_1\bL_k) \cdots \fm(\delta_k \bL_k) | B_0 = i]\right)_{i=1}^n = \wP^k \one_H.
    \end{align}
    Furthermore, if $F \in H^{n \times n}$ is the diagonal matrix with $F_{i,i} = \fm(f(i))$, then
    \begin{align}
        F\bv_k = \left(\bbE[\fms(\bL_k)| B_0 = i]\right)_{i=1}^n.
    \end{align}
\end{theorem}

In the classical diffusion equation, $U_k$ captures the concentration of the random walkers after $k$ time steps over the nodes. In the hypo-elliptic diffusion equation, $\bv_k$ captures summaries of random walk histories after $k$ time steps over the nodes since $\fms(\bL_k)$ summarizes the whole trajectory $\bL_k=(L_0,\ldots,L_k)$ and not only the endpoint $L_k$. 
\paragraph{Node Features and Graph Features}
Theorem \ref{thm:non_abelian_laplacian} can then be used to compute features $\fmn(i) \in \tensalg$ for individual nodes as well as a feature $\fmg(\cG)$ for the entire graph. The former is given by $i$-th component $\bv_k^{(i)}$ of the solution $\bv_k=(\bv_k^{(i)})_{i=1,\ldots,n} \in H^n$ of eq.~\eqref{eq:g2t:non_abelian_diffusion_equation},
\begin{equation}
\label{eq:g2t:nonabelian_node_features}
    \fmn(i) =  \bv_k^{(i)} = \bbE[\fms( \bL_k) \given B_0 = i]= (F\wP^k \bv_0)^{(i)} \in H,
\end{equation}
since the random walk $B$ chooses the starting node $B_0=i$ uniformly at random. The latter can be computed by mean pooling the node features, which also has a probabilistic interpretation
\begin{align} \label{eq:g2t:mean_pooled_features}
     \fmg(\cG) = \frac{1}{n} \sum_{i=1}^n \bv_k^{(i)} = \bbE[\fms(\bL_k)] =  n^{-1} (\one_H^T F\wP^k \bv_0) \in H,
\end{align}
where $\one_H^T = (1_H, \, \ldots, \, 1_H) \in H^n$ is the all-ones vector in $H$ and $1_H$ denotes the unit in $H$.

\paragraph{Characterizing Graphs with Random Walks} 
The graph and node features obtained through the hypo-elliptic diffusion equation are highly descriptive: they characterize the history of the random walk process if one includes the time parameterization, as described in~\cite[App.~C]{toth2022capturing}.

\begin{theorem} 
\label{thm:characterizing_rw_informal}
    Suppose $\Psi$ is the graph feature map from eq.~\eqref{eq:g2t:mean_pooled_features} induced by the tensor exponential algebra lifting including time parameterization. Let $\cG$ and $\cG'$ be two labelled graphs, and $\bL_k = (L_0, \ldots, L_k)$ and $\bL'_k = (L'_0, \ldots, L'_k)$ be the $k$-step lifted random walk as defined in eq.~\eqref{eq:g2t:lifted_random_walk}. Then, $\Psi(\cG) = \Psi(\cG')$ if and only if the distributions of $\bL_k$ and $\bL'_k$ are equal. 
\end{theorem}
It is instructive to contrast this result with the classical diffusion case; the latter only uses the marginal distribution of $L_k$ to capture the graph structure, which at least intuitively has much less expressive power. Indeed, in~\cite[App.~E]{toth2022capturing}, we show that for elementary graphs, this already leads to big differences in expressive power. 
Further, an analogous result holds for the node features, and we prove both results in~\cite[App.~E]{toth2022capturing}.
While we use the tensor exponential in this article, many other choices of $\fms$ are possible and result in graph and node features with such properties: under mild conditions, if the algebra lifting $\fm:\bbR^d \to H$ characterizes measures on $\bbR^d$, the resulting node feature map $\fmn$ characterizes the random walk, see~\cite{toth2021seq2tens}, which in turn implies the above results. Possible variations are discussed in Appendix~\ref{app:g2t:variations}.



\paragraph{General (Hypo-elliptic) Diffusions and Attention}
One can consider more general diffusion operators, such as the normalized Laplacian $\cK$ of a weighted graph.
We define its lifted operator $\wcK \in \tensalg^{n\times n} $ analogous to eq.~\eqref{eq:g2t:tensor_laplacian_matrix}, resulting in a generalization of~\ref{thm:non_abelian_laplacian} with $\wcK$ replacing $\wcL$. 
In the flavour of convolutional \texttt{GNN}s \cite{bronstein2021geometric}, we consider a weighted adjacency matrix $A \in \bbR^{n \times n}$
\begin{align}
    A_{i,j} = \left\{\begin{array}{cl}
        c_{i,j} & : i \sim j \\
        0 & :\text{otherwise},
    \end{array}\right.
\end{align}
for $c_{i,j} > 0$. 
The corresponding normalized Laplacian $\cK$ is given by $\cK = I - D^{-1}A$, where $D$ is a diagonal matrix with $D_{i,i} = \sum_{j \in \cN(i)} c_{i,j}$. 
A common way to learn the coefficients is by introducing parameter sharing across graphs by modelling them as $c_{i,j} = \exp(a(\nf(i), \nf(j)))$ using a local attention mechanism, $a: \bbR^d \times \bbR^d \rightarrow \bbR$ \cite{velickovic_graph_2018}. 
In our implementation, we use additive attention~\cite{bahdanau2015neural} given by $a(\nf(i), \nf(j)) = \texttt{LeakyRelu}_{0.2}(W_s \nf(i) + W_t \nf(j))$, where $W_s, W_t \in \bbR^{1 \times d}$ are linear transformations for the source and target nodes, but different attention mechanisms can also be used; e.g.~scaled dot-product attention \cite{vaswani2017attention}. Then, the corresponding transition matrix $P = D^{-1} A$ is  defined as $P_{ij} = \texttt{softmax}_{k \in \cN(i)}(a(f(i), f(k)))_j$, and the lifted transition matrix as
\begin{align}
  \wP = \left\{\begin{array}{cl}
    P_{i,j} \fm(\nf(j) - \nf(i)) & : i \sim j \\
    0 & :\text{otherwise}.
  \end{array}\right.
\end{align}
The statements of Theorem \ref{thm:non_abelian_laplacian} immediately generalize to this variation by replacing the expectation with respect to a non-uniform random walk. Hence, in this case the use of attention can be interpreted as learning the transition probabilities of a random walk on the graph.



\section{Efficient Algorithms for Deep Learning}\label{sec:g2t:algos}
The previous sections show that the node feature $\fmn(i)$ provides a structured description of the neighborhood of node $i$ and it is instructive to think of a linear functional $\langle \bell, \fmn(i) \rangle$ as answering a specific question about the node neighbourhood, see \cite[App.~E]{toth2022capturing} for examples.
The naive computation of $\langle \bell, \fmn(i) \rangle$ by first computing $\fmn(i)$ and taking the inner product is too expensive, especially when $\bell=(\bell_0,\ldots,\bell_M,0,\ldots)\in H$ for large $M$.
To address this we revisit two observations from \cite{toth2021seq2tens}: first, for a rank-1 functional $\bell \in H$, the computation of $\langle \bell, \fmn(i) \rangle$ is computationally cheap.
Second, restriction to small $M$ limits the expressive power but can be counteracted by composition: any choice of $d$ different functionals $\bell^1,\ldots,\bell^d \in H$ gives a label update 
$f(i) \mapsto (\langle \bell^j, \fmn(i) \rangle)_{j=1,\ldots,d} \in \bbR^d$ for the graph.
Repeating such a label update a few times with low-degree $M$ and rank-1 functionals turns out to be as powerful as computing one update for general functionals with arbitrary high $M$.  
The first observation should not be too surprising given the popularity of low-rank approximations; the second observation is reminiscent to constructing a high-degree polynomial by composing low-degree polynomials\footnote{For example, $1+x+x^2$ composed with $1+2x^2$ yields the degree 4 polynomial $1+(1+2x^2) + (1+2x^2)^2$.} or the width-vs-depth phenomenon in neural nets and we give more details below.

\paragraph{Computing Rank-$1$ Functionals}
First, we focus on a \emph{rank-$1$} linear functional $\bell \in H$ given as
\begin{align}\label{eq:g2t:low-rank functional}
\bell = (\bell_m)_{m \ge 0} \text{ with } \bell_m =u_{M-m+1} \otimes \cdots \otimes u_M \text{ and } \bell_m =0 \text{ for } m > M, 
\end{align}
where $u_m \in \bbR^d$ for $m = 1, \ldots, M$ for a fixed $M \ge 1$. 
Thm~\ref{thm:algo} shows for such $\bell$, the computation of $\langle \bell, \fmnz(i)\rangle$, where $\fmnz(i)$ is the node feature without basepoint, can be done \begin{enumerate*}[label=(\alph*)] \item efficiently by factoring this low-rank structure into the recursive computation, and \item simultaneously for all nodes $i \in \cV$ in parallel \end{enumerate*}. This can then be used to compute rank-$R$ functionals for $R>1$, and for $\langle \bell, \fmn(i)\rangle$; see App.~\ref{app:g2t:low_rank_functionals}, and for a pseudo-code implementation see Algorithm 1 in \cite[App.~F]{toth2022capturing}.

\begin{theorem}\label{thm:algo}
Let $\bell$ be as in \eqref{eq:g2t:low-rank functional} and define $f_{k,m} \in \bbR^{n}$ for $m=1, \dots, M$ as
\begin{align}
f_{1,m} =\frac{1}{m!}\left(P \odot C^{u_{M-m+1}} \odot \cdots \odot C^{u_M}\right) \cdot \one,
\end{align}
where $\one^T = (1, \ldots, 1) \in \bbR^n$ is the all-ones vector; and for $2 \leq k$ and  $1 \leq m \leq M$ recursively as 
  \begin{align} \label{eq:g2t:recursive_low_rank}
   f_{k,m} = P \cdot f_{k-1,m} + \sum_{r=1}^m \frac{1}{r!} \left(P \odot C^{u_{M-m+1}} \odot \cdots \odot C^{u_{M-m+r}}\right) \cdot f_{k-1,m-r},
  \end{align}
  where the matrix $C^u=(C^u_{i,j}) \in \bbR^{n \times n}$ is defined as
\begin{align}
C^u_{i,j}=  \begin{cases}
      \inner{u}{\nf(j) - \nf(i)} &: i \sim j,\\
      0  &: \text{otherwise}.
    \end{cases}
  \end{align} 
  Here $\odot$ denotes element-wise\footnote{For example
    $\begin{bmatrix}
      1& 2 \\ 
      3&4
    \end{bmatrix}
    \odot
    \begin{bmatrix}
      5&6\\7&8 
    \end{bmatrix} =
    \begin{bmatrix}
      5& 12\\
      21&32    \end{bmatrix}.
    $} multiplication, while $\cdot$ denotes matrix multiplication.
  Then, it holds for $i \in \cV$, random walk length $k \in \bbZ_+$, and tensor degree $m = 1, \ldots, M$, that 
  \begin{align}
    f_{k, m}(i) = \langle\bell_m, \fmnz_k(i)\rangle, 
  \end{align}
  where $\fmnz_k(i) = \bbE[\fm(\delta_1\bL_k) \cdots \fm(\delta_k\bL_k) \given B_0 = i]$.
\end{theorem}
 Overall, Eq.~\eqref{eq:g2t:recursive_low_rank} computes $f_{k, m}(i)$ for all $i \in \cV$, $k = 1, \ldots, K$, $m = 1, \ldots, M$ in $O(K \cdot M^2 \cdot N_E + M \cdot N_E \cdot d)$ operations, where $N_E \in \bbN$ denotes the number of edges; see App.~\ref{app:g2t:low_rank_functionals}.
  In particular, one does not need to compute $\fmn(i) \in H$ directly or store large tensors. 


\paragraph{Graph Labelling Layers} 
Fixing $d$ rank 1-functionals $\bell^1,\ldots, \bell^d \in H$ induces a label update $f(i) \mapsto (\langle \bell^i, \fmn(i) \rangle)_{i=1,\ldots,d}\in \bbR^d$.
Theorem \ref{thm:algo} allows us to compute this update in parallel for all nodes in $\cV$. 
Such a label update is similar to hidden layer in a neural network (\texttt{NN}) and we can stack such updates, see Figure~\ref{fig:g2t:architecture} in Appendix~\ref{app:g2t:model_detail}. As in \texttt{NN}, the $d$ functionals in each "graph labelling layer" are optimized by gradient descent.
Finally, note that a rank-$R$ functional is the sum of $R$ rank-1 functionals so we can immediately carry out the same construction with rank-$R$ functionals by adding a linear mixing layer.
To sum up, a graph labelling layer is determined by the random walk length $k$, the maximal tensor degree-$M$, maximal tensor rank-$R$ and the functionals themselves are then found by optimization. 

Using a single layer of low-rank functionals limits the expressiveness but stacking layers allows in practice to approximate general, high-degree $M$ functionals.
Some theoretical results can be found in \cite{toth2021seq2tens}; however, here we simply appeal to the analogy with \texttt{NN} where stacking simple transformations provides a flexible functional class with good inductive bias. 

\section{Experiments} \label{sec:g2t::experiments}
We implemented the above approach and call the resulting model \textbf{G}raph\textbf{2T}ens \textbf{N}etworks since it represents the neighbourhood of a node as a sequence of tensors, which is further pushed through a low-tensor-rank constrained linear mapping, similarly to how neural networks linearly transform their inputs pre-activation. A conceptual difference is that in our case the non-linearity is applied first and the projection secondly, albeit the computation is coupled between these steps. We provide further experiments and ablation studies of our models in App.~\ref{app:g2t:exp_detail}.

\paragraph{Experimental Setup} The aim of our main experiment is to test the following key properties: \begin{enumerate*}[label=(\arabic*)] \item ability to capture long-range interactions between nodes in a graph, \item robustness to pooling operations, hence making it less susceptible to the ``over-squashing'' phenomenon \cite{alon2021on} \end{enumerate*}. 
We do this by following the experiments in \cite{wu2021representing}. In particular, we show that our model is competitive with previous approaches for retaining long-range context in graph-level learning tasks but without computing all pairwise interactions between nodes, thus keeping the influence distribution localized \cite{xu2018representation}. We further give a detailed ablation study to show the robustness of our model to various architectural choices in App.~\ref{app:g2t:ablation}.
As a second experiment, we follow the previous applications of diffusion approaches to graphs that have mostly considered inductive learning tasks, e.g.~on the citation datasets \cite{Chamberlain2021GRANDGN, thorpe_grand_2022, beltrami}. 
Our experimentation on these datasets are available in App.~\ref{app:g2t:further_exp}, where the model performs on par with short-range \texttt{GNN} models, but does not seem to benefit from added long-range information a-priori. However, when labels are dropped in a $k$-hop sanitized way as in \cite{rampavsek2021hierarchical}, the performance decrease is less pronounced.

\paragraph{Datasets} We use two biological graph classification datasets (NCI1 and NCI109), that contain around ${\sim} 4000$ biochemical compounds represented as graphs with ${\sim} 30$ nodes on average~\cite{wale_comparison_2008, PubChem}. The task is to predict whether a compound contains anti-lung-cancer activity. The dataset is split in a ratio of $80\%-10\%-10\%$ for training, validation and testing. Previous work \cite{alon2021on} has found that \texttt{GNN}s that only summarize local structural information can be outperformed by models accounting for global contextual relationships through the use of \textit{fully-adjacent} layers. This was further improved on by \cite{wu2021representing}, where a local neighbourhood encoder consisting of a \texttt{GNN} stack was upgraded with a Transformer submodule \cite{vaswani2017attention} for learning global interactions.

\paragraph{Model Details} We build a \texttt{GNN} architecture primarily motivated by the \texttt{GraphTrans (small)} model from \cite{wu2021representing}, and only fine-tune the pre- and postprocessing layers(s), random walk length, functional degree and optimization settings. In detail, a preprocessing \texttt{MLP} layer with $128$ hidden units is followed by a stack of $4$ \GTN layers each with \texttt{RW} length-$5$, max rank-$128$, max tensor degree-$2$, all equipped with \texttt{JK}-connections \cite{xu2018representation} and a max aggregator. Afterwards, the node features are combined into a graph-level representation using gated attention pooling \cite{li2016gated}. The pooled features are transformed using a final \texttt{MLP} layer with $256$ hidden units, and then fed into a softmax classification layer. The pre- and postprocessing \texttt{ML}P layers employ skip-connections \cite{he2016deep}. Both \texttt{MLP} and \GTN layers are followed by layer normalization \cite{ba2016layer}, where \texttt{GT2N} layers normalize their rank-$1$ functionals independently across different tensor degrees, which corresponds to a particular realization of group normalization \cite{wu2018group}. We randomly drop $10\%$ of the features for all hidden layers during training \cite{srivastava2014dropout}. The attentional variant, \GTAN also randomly drops $10\%$ of its edges and uses $8$ attention heads \cite{velickovic_graph_2018}. Training is performed by minimizing the categorical cross-entropy loss with an $\ell_2$ regularization penalty of $10^{-4}$. For optimization, Adam \cite{kingma2014adam} is used with a batch size of $128$ and an inital learning rate of $10^{-3}$ that is decayed via a cosine annealing schedule \cite{loshchilov2017sgdr} over $200$ epochs. Further intuition about the model and architectural choices are available in App.~\ref{app:g2t:model_detail}.

\paragraph{Baselines} We compare against \begin{enumerate*}[label=(\arabic*)] \item the baseline models reported in \cite{wu2021representing}, \item variations of \texttt{GraphTrans}, \item other recently proposed hierarchical approaches for long-range graph tasks \cite{rampavsek2021hierarchical}. \end{enumerate*} Groups of models in Table \ref{table:g2t:nci_benchmark} are separated by dashed lines if they were reported in separate papers, and the first citation after the name is where the result first appeared. The number of \texttt{GNN} layers in \texttt{HGNet} are not discussed by \cite{rampavsek2021hierarchical}, and we report it as implied by their code. We organize the models into three groups divided by solid lines: \begin{enumerate*}[label=(\alph*)] \item \label{baseline:local} baselines that only apply neighbourhood aggregations, and hierarchical or global pooling schemes,  \item \label{baseline:pairwise} baselines that first employ a local neighbourhood encoder, and afterwards fully densify the graph in some way so all nodes interact with each other \emph{directly}, \item our models that that belong to \ref{baseline:local} \end{enumerate*}.

\begin{table}[t]
  \vspace{-5pt}
  \caption{Comparison of classification accuracies on NCI biological datasets, where we report mean and standard deviation over $10$ random seeds for our models.}
  \label{table:g2t:nci_benchmark}
\resizebox{\textwidth}{!}{
  \centering
  \begin{sc}
  \begin{tabular}{lcccc}
    \toprule
    \textbf{Model} & \textbf{GNN Type} & \textbf{GNN Count} & \textbf{NCI1 (\%)} & \textbf{NCI109 (\%)} \\
    \midrule
    \texttt{Set2Set} \cite{lee2019self, vinyals2016order} & GCN  & $3$ & $68.6 \pm 1.9$ & $69.8 \pm 1.2$ \\
    \texttt{SortPool} \cite{lee2019self,zhang2018end} & GCN & $3$ & $73.8 \pm 1.0$ & $74.0 \pm 1.2$ \\
    \texttt{SAGPool\textsubscript{h}} \cite{lee2019self} & GCN & $3$ & $67.5 \pm 1.1$ & $67.9 \pm 1.4$  \\
    \texttt{SAGPool\textsubscript{g}} \cite{lee2019self} & GCN & $3$ & $74.2 \pm 1.2$ & $74.1 \pm 0.8$ \\
    \middashrule
    \texttt{GIN} \cite{errica2019fair, xu_how_2019} & GIN & $8$ & $80.0 \pm 1.4$ & - \\
    \middashrule
    \texttt{GCN + VN} \cite{ying_hierarchical_2018,gilmer2017neural} & GCN & $2$ & $71.5$ & -\\
    \texttt{HGNet-EdgePool} \cite{ying_hierarchical_2018, schlichtkrull2018modeling} & GCN+RGCN & $3+2$ & $77.1$ & - \\
    \texttt{HGNet-Louvain} \cite{ying_hierarchical_2018, schlichtkrull2018modeling} & GCN+RGCN & $3+2$ & $75.1$ & - \\
    \midrule
    \texttt{GIN + FA} \cite{alon2021on,xu_how_2019} & GIN & $8$ & $81.5 \pm 1.2$ & - \\
    \middashrule 
    \texttt{GraphTrans (small)} \cite{wu2021representing,vaswani2017attention} & GCN & $3$ & $81.3 \pm 1.9$ & $79.2 \pm 2.2$ \\
    \texttt{GraphTrans (large)} \cite{wu2021representing,vaswani2017attention} & GCN & $4$ & $\color{gray} \mathbf{82.6 \pm 1.2}$ & $\color{gray} \mathbf{82.3 \pm 2.6}$ \\
    \midrule
    \textbf{\GTAN} (ours) & \GTAN & $4$ & ${\mathbf{81.9 \pm 1.2}}$ & $78.0 \pm 2.3$ \\
    \textbf{\GTN} (ours) & \GTN & $4$ & $80.7 \pm 2.5$ & ${\mathbf{78.9 \pm 2.5}}$ \\
    \bottomrule
  \end{tabular}
  \end{sc}}
\end{table}

\paragraph{Results} In Table \ref{table:g2t:nci_benchmark}, we report the mean and standard deviation of classification accuracy computed over 10 different seeds. Overall, both our models improve over all baselines in group \ref{baseline:local} on both datasets, maximally by $1.9\%$ on NCI1 and by $4.8\%$ on NCI109. In group \ref{baseline:pairwise}, \GTAN is solely outperformed by \texttt{GraphTrans (large)} on NCI1 by only $0.7\%$. Interestingly, the attention-free variation, \GTN, performs better on NCI109, where it performs very slightly worse than \texttt{GraphTrans (small)}.

\paragraph{Discussion} The previous experiments demonstrate that our approach performs very favourably on long-range reasoning tasks compared to \texttt{GNN}-based alternatives without global pairwise node interactions. Several of the works we compare against have focused on extending \texttt{GNN}s to larger neighbourhoods by specifically designed graph coarsening and pooling operations, and we emphasize two important points: \begin{enumerate*}[label=(\arabic*)] \item our approach can efficiently capture large neighbourhoods without any need for coarsening, \item it already performs well with simple mean-pooling as justified by Theorem \ref{thm:characterizing_rw_informal} and experimentally supported by the ablation studies in~\ref{app:g2t:ablation}. \end{enumerate*}  Although the Transformer-based \texttt{GraphTrans} slightly outperforms our model potentially due to its ability to learn global interactions, it is not entirely clear how much of the global graph structure it is able to infer from interactions of short-range neighbourhood summaries. 
Finally, Transformer models can be bottlenecked by their quadratic complexity in nodes, while our approach only scales with edges, and hence, it can be more favourable for large sparse graphs in terms of computations.

\section{Conclusion} 
\label{sec:g2t:conclusion}
Inspired by classical results from analysis \cite{gaveau_principe_1977}, we introduce the hypo-elliptic graph Laplacian. 
This yields a diffusion equation and also generalizes its classical probabilistic interpretation via random walks but now 
taking history into account. 
In addition to several attractive theoretical guarantees, we provide scalable algorithms. 
Our experiments show that this can lead to largely improved baselines for long-range reasoning tasks. 
A promising future research theme is to develop improvements for the classical Laplacian in this hypo-elliptic context; 
including lazy random walks~\cite{xhonneux_continuous_2020}; nonlinear diffusions~\cite{Chamberlain2021GRANDGN}; and source/sink terms~\cite{thorpe_grand_2022}.
Another theme could be to extend the geometric study \cite{topping22} of over-squashing to this hypo-elliptic point of view which is naturally tied to sub-Riemannian geometry \cite{strichartz1986sub}. 
A limitation in our theoretical results is that for the iterations of low-rank approximations only partial results exist and expanding this is an interesting (algebra-heavy) topic.

\begin{subappendices}

\section{Further Details on Hypo-elliptic Graph Diffusion} \label{app:g2t:details_diffusion}

In Section \ref{sec:g2t:diffusion}, we introduced tensor analogues of classical matrix operators, which were then used to define hypo-elliptic graph diffusion. These tensor-valued matrices allow us to efficiently represent random walk histories on a graph, which can be manipulated using matrix operations. In this section, we discuss further details on the tensor adjacency matrix, provide a proof of Theorem~\ref{thm:non_abelian_laplacian}, and introduce a variation of hypo-elliptic graph diffusion.

As in Section~\ref{sec:g2t:diffusion}, we fix a labelled graph $\cG = (\cV, \cE, \nf)$, where the nodes are denoted by integers, $\cV = \{1, \ldots, n\}$, and $\nf: \cV \rightarrow \bbR^d$ denotes the continuous node attributes. 

\paragraph{Powers of the Tensor Adjacency Matrix} Recall that the powers of the classical adjacency matrix $A$ counts the number of walks between two nodes on the graph. In particular, given $k \in \bbN$, and two nodes $i,j \in \cV$, the result follows as a consequence of the sparsity pattern of $A$,
\begin{align}
\label{eq:g2t:classical_adj_powers}
    (A^k)_{i,j} = \sum_{i_1, \ldots, i_{k-1}=1}^n A_{i,i_1} \cdot A_{i_1, i_2} \cdots A_{i_{k-1},j} = \sum_{i=i_0 \sim \ldots \sim i_k = j} 1.
\end{align}
Note that the product $A_{i,i_1} \cdots A_{i_{k-1}, j} = 0$ unless each pair of consecutive indices are adjacent in the graph, namely $i_{q-1} \sim i_q$ for all $q = 1, \ldots, k$. Applying the same procedure to the tensor adjacency matrix from eq.~\eqref{eq:g2t:exponential_adjacency_matrix}, we obtain a summary of all walks between two nodes, rather than just the number of walks. In particular,

\begin{align}
     (\wA^k)_{i,j} & =\sum_{i=i_0\sim\cdots\sim i_k=j}\wA_{i,i_1}\cdot\wA_{i_1,i_2}\cdots \wA_{i_{k-1},j}  \\
    & = \sum_{i=i_0\sim\cdots\sim i_k=j}\fm(\nf(i_1)-\nf(i)) \cdot \fm(\nf(i_2)-\nf(i_1))\cdots \fm(\nf(j)-\nf(i_{k-1}))\\
     &= \sum_{i=i_0\sim\cdots\sim i_k=j}\fm(\delta_1 \bx) \cdots \fm(\delta_k \bx),
\end{align}
where $\bx = (f(i_0), \ldots, f(i_k))$ denotes the lifted sequence in the vector space $\vs$. Note that this corresponds to the sequence feature map \emph{without} the initial point $\delta_0 \bx$.

\paragraph{Powers of the Tensor Transition Matrix} We now consider powers of the classical transition matrix $P = I - \cL$, where $\cL = I - D^{-1} A$ is the normalized graph Laplacian. The entries of $P^k$ provide length $k$ random walk probabilities; in particular, we have
\begin{align}
    (P^k)_{i,j} = \sum_{i=i_0 \sim \ldots \sim i_k = j} \frac{1}{\deg(i_0) \deg(i_1) \ldots \deg(i_{k-1})} = \bbP[B_k = j | B_0 = i].
\end{align}

The powers of the tensor transition matrix $\wP =I - \wcL$, where $\wcL = I - D^{-1}\wA$ is the hypo-elliptic Laplacian, will be the conditional expectation of the sequence feature map of the random walk process. In particular, 
\begin{align}\label{eq:g2t:powers_tensor_transition}
    (\wP)^k_{i,j} & = \sum_{i=i_0\sim\cdots\sim i_k=j} \fm(\delta_1 \bx) \cdots \fm(\delta_k \bx) \bbP[B_1 = i_1, B_2 = i_2, \ldots, B_k = i_k | B_0 = i]. \\
\end{align}


\paragraph{Proof of Hypo-elliptic Diffusion Theorem} Recall from Section~\ref{sec:g2t:diffusion} that the \emph{hypo-elliptic graph diffusion equation}\footnote{There is a minor typo in~\ref{sec:g2t:diffusion} which does not alter our results; the initial condition should be $\bv_0 = \one_H$ as given here. The statement of~\ref{thm:non_abelian_laplacian} and the resulting node/graph features are adjusted accordingly.} is defined by
\begin{align}
    \bv_k - \bv_{k-1} = -\wcL \bv_{k-1}, \quad \bv_0 = \one_H,
\end{align}
where $\one^T_H = (1_H, \ldots, 1_H) \in H^n$ is the all-ones vector in $H$. Using the above computations for powers of the tensor transition matrix, we can prove Thm.~\ref{thm:non_abelian_laplacian}, which is restated here.
\begin{theorem}
    Let $k \in \bbN$, $\bL_k = (L_0, \ldots, L_k)$ be the lifted random walk from~\eqref{eq:g2t:lifted_random_walk}, and $\wP = I - \wcL$ be the \emph{tensor adjacency matrix}. The solution to the hypo-elliptic graph diffusion eq.~\eqref{eq:g2t:non_abelian_diffusion_equation} is
    \begin{align}
        \bv_k = \left( \bbE[\fm(\delta_1\bL_k) \cdots \fm(\delta_k \bL_k) | B_0 = i]\right)_{i=1}^n = \wP^k \one_H.
    \end{align}
    Furthermore, if $F \in H^{n \times n}$ is the diagonal matrix with $F_{i,i} = \fm(f(i))$, then
    \begin{align}
        F\bv_k = \left(\bbE[\fms(\bL_k)| B_0 = i]\right)_{i=1}^n.
    \end{align}
\end{theorem}


\begin{proof}[Proof of Theorem~\ref{thm:non_abelian_laplacian}]
We begin by proving the first equation. From the definition of hypo-elliptic diffusion, it is straightforward to see that $$\bv_k= (I-\wcL)\bv_{k-1} = \wP\bv_{k-1}  = \wP^k\bv_0 = \wP^k\one_H.$$

We prove coordinate-wise that $\bv_k = \left( \bbE[\fm(\delta_1\bL_k) \cdots \fm(\delta_k \bL_k) | B_0 = i]\right)_{i=1}^n $. Indeed, using the above and eq.~\eqref{eq:g2t:powers_tensor_transition}, the $i$-th coordinate of $\bv_k$ is 
\begin{align}
    \bv_k^{(i)} = (\wP^k\one_H)^{(i)} = \sum_{j=1}^n (\wP^k)_{i,j}  &= \sum_{i=i_0 \sim \ldots \sim i_k} \fm(\delta_1 \bx) \cdots \fm(\delta_k \bx) \bbP[B_1 = i_1,  \ldots, B_k = i_k | B_0 = i] \\
    & = \bbE[\fm(\delta_1\bL_k) \cdots \fm(\delta_k \bL_k) | B_0 = i],
\end{align}
where $\bx = (f(i_0), \ldots, f(i_k))$ is the lifted sequence corresponding to a walk $i_0 \sim \ldots \sim i_k$ on the graph. 
Next, we will also prove the second equation coordinate-wise. Using the above result, we have
\begin{align}
    (F\bv_k)^{(i)} = \fm(f(i)) \bbE[\fm(\delta_1\bL_k) \cdots \fm(\delta_k \bL_k) | B_0 = i] = \bbE[\fms(\bL_k) | B_0 = i],
\end{align}
where we use the fact that $\delta_0\bL_k = L_0 = f(i)$ when we condition $B_0 = i$. 
\end{proof}

\paragraph{Forward Hypo-elliptic Diffusion} 
In the classical setting, we can consider both the forward and backward Kolmogorov equations. Throughout the main text and in the appendix so far, we have been considering the \emph{backward} variants. In this section, we formulate the \emph{forward} analogue of hypo-elliptic diffusion.
In the classical graph setting, this corresponds to the following forward equation for $U_k \in \R^{n \times d}$ given by
\begin{align}\label{eq:g2t:heat equation fwd}
    U^T_k - U^T_{k-1} = - U_{k-1}^T\cL , \quad U^{(i)}_0 = \nf(i)
\end{align}
where the initial condition $U_0 \in \R^{n \times d}$ is specified by the node attributes. Note that because $P = D^{-1} A$ is right stochastic\footnote{Each column sums to one and hence multiplying on the left by a row vector conserves its mass.}, this variation of the graph diffusion equation conserves mass in each coordinate of the node attributes at every time step.

 
Similarly we can formulate the forward hypo-elliptic graph diffusion equation for $\bv_k \in \tensalgps^n$ as
\begin{equation}
\label{eq:g2t:non_abelian_forward_equation}
    \bv^T_k - \bv^T_{k-1} = - \bv^T_{k-1}\wcL , \quad \bv_0 = n^{-1}\one_H.
\end{equation}
The solution of eq.~\ref{eq:g2t:non_abelian_forward_equation} is given below.  

\begin{theorem}
\label{thm:forward_diffusion}
    Let $k \in \bbN$, $\bL_k = (L_0, \ldots, L_k)$ be the lifted random walk from~\eqref{eq:g2t:lifted_random_walk}, and $\wP = I - \wcL$ be the \emph{tensor adjacency matrix}. The solution to the forward hypo-elliptic graph diffusion \eqref{eq:g2t:non_abelian_forward_equation} is
    \begin{align}
        \bv^T_k = \left( \bbP[B_k = i]\bbE[\fm(\delta_1\bL_k) \cdots \fm(\delta_k \bL_k) | B_k = i]\right)_{i=1}^n = \frac{1}{n}\one^T_H\wP^k.
    \end{align}
    Furthermore, if $F \in H^{n \times n}$ is the diagonal matrix with $F_{i,i} = \fm(f(i))$, then
    \begin{align}
        \frac{1}{n}\one^T_H F \wP^k= \left(\bbP[B_k = i] \bbE[\fms(\bL_k)| B_k = i]\right)_{i=1}^n.
    \end{align}
\end{theorem}

\begin{proof}
The proof proceeds in the same way as the backward equation. By definition of the forward hypo-elliptic diffusion, we have
\begin{align}
    \bv_k^T = \bv_0^T \wP^k = \frac{1}{n} \one^T_H \wP^k.
\end{align}

Now, recall that the initial point of the random walk process is chosen uniformly over all nodes; in other words, $\bbP[B_0 = i] = \frac{1}{n}$. Then, we show $\bv^T_k = \left( \bbE[\fm(\delta_1\bL_k) \cdots \fm(\delta_k \bL_k) | B_k = i]\right)_{i=1}^n$ coordinate-wise as
\begin{align}
    \bv_k^{(i)} &= \frac{1}{n} \sum_{j=1}^n (\wP^k)_{j,i}\\
    &= \sum_{i_0 \sim \ldots \sim i_k = i} \fm(\delta_1 \bx) \cdots \fm(\delta_k \bx) \bbP[B_1 = i_1,  \ldots, B_k = i | B_0 = i_0] \bbP[B_0 = i_0] \\
    & = \sum_{i_0 \sim \ldots \sim i_k = i} \fm(\delta_1 \bx) \cdots \fm(\delta_k \bx) \bbP[B_0 = j,  \ldots, B_{k-1} = i_{k-1} | B_k = i] \bbP[B_k = i]\\
    & =  \bbP[B_k = i] \bbE[\fm(\delta_1\bL_k) \cdots \fm(\delta_k \bL_k) | B_k = i],
\end{align}
where $\bx = (f(i_0), \ldots, f(i_k))$ is the lifted sequence corresponding to a walk $i_0 \sim \ldots \sim i_k$ on the graph. We will now prove the second equation. Note that we have
\begin{align}
    (F\wP^k)_{i,j} = \sum_{i=i_0 \sim \ldots \sim i_k = j} \fms(\bx) \bbP[B_1 = i_1, \ldots, B_k = i_k \given B_0 = i].
\end{align}
Then, following the same reasoning as the first equation, we obtain the desired result.
\end{proof}


\paragraph{Weighted Graphs} In the prior discussion, we considered simple random walks in which the walk chooses one of the nodes  adjacent to the current node uniformly at random. We can instead consider more general random walks on graphs, which can be described using a weighted adjacency matrix,
\begin{align}
    A_{i,j} = \left\{ \begin{array}{cl}
        c_{i,j} & : i \sim j \\
        0 &: \text{otherwise}. 
    \end{array}\right. 
\end{align}
In this case, we define the diagonal weighted degree matrix to be $D_{i,i} = \sum_{i \sim j} A_{i,j}$. We now define a weighted random walk $(B_k)_{k \ge 0}$ on the vertices $\cV$ where the transition matrix is given by
\begin{align}
    P_{i,j} \coloneqq D^{-1}A =\bbP(B_k = j \given B_{k-1}=i) = \left\{ \begin{array}{cl}\frac{c_{i,j}}{\sum_{i \sim j'} c_{i,j'}} &: i\sim j\\ 0 &: \text{otherwise}.\end{array}\right.
\end{align}
With these weighted graphs, the powers of the adjacency and transition matrix can be interpreted in a similar manner as the standard case. Powers of the adjacency matrix provide total weights over walks, while the powers of the transition matrix provide the probability of a weighted walk going between two nodes after a specified number of steps. In particular, 
\begin{align}
    (A^k)_{i,j} &= \sum_{i \sim i_1 \sim \ldots \sim i_{k-1} \sim j} c_{i,i_1} \ldots c_{i_{k-1}, j}\\
    (P^k)_{i,j} &= \bbP[B_k = j \given B_0 = i].
\end{align}
The tensor adjacency and tensor transition matrices are defined in the same manner as
\begin{align}
    \wA_{i,j} &\coloneqq \left\{\begin{array}{cl} c_{i,j}\fm(f(j) - f(i)) & : i \sim j \\ 0 & : \text{otherwise} \end{array}\right. \\
    \wP_{i,j} &\coloneqq D^{-1} \wA = \left\{\begin{array}{cl} P_{i,j}\fm(f(j) - f(i)) & : i \sim j \\ 0 & : \text{otherwise} \end{array}\right.
\end{align}
Note that powers of this weighted tensor transition matrix are exactly the same as the unweighted case from eq.~\eqref{eq:g2t:powers_tensor_transition}, and thus the weighted version of both Theorems~\ref{thm:characterizing_rw_informal} and~\ref{thm:forward_diffusion} immediately follow.

\section{Details on the low-rank Algorithm.}
\label{app:g2t:low_rank_functionals}

In this section, we will provide further details and proofs on the low-rank approximation method discussed in Sec.~\ref{sec:g2t:algos}. We can use low-rank tensors to define corresponding functionals of the features obtained via hypo-elliptic diffusion. The following is the definition of \emph{CP-rank}~\cite{carroll1970analysis}.

\begin{definition}
    The \emph{rank} of a level $m$ tensor $\bv_m \in (\bbR^d)^{\otimes m}$ is the smallest number $r \geq 0$ such that
    \begin{align}
        \bv_m = \sum_{i=1}^r v_{i,1} \otimes \ldots \otimes v_{i,m}, \quad v_{i,j} \in \bbR^d.
    \end{align}
    We say that $\bv = (\bv_0, \bv_1, \ldots) \in H$ is \emph{rank $1$} if all $\bv_m$ are rank $1$ tensors. 
\end{definition}

We will now prove~Theorem \ref{thm:algo}, which is stated using the node feature map without \texttt{ZeroStart} (App.~\ref{app:g2t:variations}). In particular, the node feature map is given by
\begin{align}
\label{eq:g2t:fmn_no_zerostart}
    \fmnz_k(i) = \bbE[\fm(\delta_1 \bL_k) \cdots \fm(\delta_k\bL_k) \given B_0 = i] = (\wP^k \one_H)^{(i)} \in H,
\end{align}
where we explicitly specify the walk length $k$ in the subscript. Furthermore, if we also need to specify the tensor degree $m$, we will use two subscripts, where
\begin{align}
    \fmnz_{k,m}(i) \in (\bbR^d)^{\otimes m}
\end{align}
is the level $m$ component of the hypo-elliptic diffusion with a walk length of $k$. Throughout the proof, we omit the $H$ subscript for the all-ones vector, such that $\one^T = (1_H, \ldots, 1_H) \in H^n$, and we denote $\one^T_i = (0, \ldots, 1_H, \ldots, 0) \in H^n$ to be the unit vector in the $i^{\text{th}}$ coordinate. 

\begin{proof}[Proof of Theorem~\ref{thm:algo}]

First, we will show that $f_{1,m}(i) = \langle \bell_m, \fmnz_1(i)\rangle$ for all $m = 1, \ldots, M$. By the definition of hypo-elliptic diffusion, we know that
\begin{align}
    \fmnz_1(i) = \one_i^T \wP \one = \sum_{j=1} \wP_{i,j} = \sum_{i \sim j} \frac{\exp_{\otimes}(\nf(j) - \nf(i))}{d_i}.
\end{align}
By explicitly expressing the level $m$ component, and by factoring out the inner product, we get
\begin{align}
    \left\langle \bell_m, \frac{\exp_{\otimes}(\nf(j) - \nf(i))}{d_i} \right\rangle &= \frac{1}{d_i m!} \prod_{r=M-m+1}^M \langle u_r, \nf(j) - \nf(i)\rangle \\
    & = \frac{1}{m!} (P_{i,j} \cdot C^{u_{M-m+1}}_{i,j} \cdot \ldots \cdot C^{u_M}_{i,j}).
\end{align}
Then by linearity of the inner product, we get $f_{1,m}(i) = \langle \bell_m, \fmnz_1(i)\rangle$.

Next, we continue by induction and suppose that $f_{k-1,m}(i) = \langle \bell_m, \fmnz_{k-1}(i)\rangle$ holds for all $m = 1, \ldots, M$. Starting once again from the definition of hypo-elliptic diffusion, we know that 
\begin{align}
    \fmnz_k(i) = \one_i^T\wP\fmnz_{k-1} = \sum_{i \sim j} \wP_{i,j} \cdot \fmnz_{k-1}(j).
\end{align}
Fix a degree $m$. We explicitly write out the level $m$ component of this equation by expanding $\wP_{i,j}$ and the tensor product as
\begin{align}
    \fmnz_{k,m}(i) &= \sum_{i \sim j} \sum_{r=0}^m \frac{(\nf(j) - \nf(i))^{\otimes r}}{d_i r!} \cdot  \fmnz_{k-1, m- r}(j) \nonumber\\
    & = \sum_{i \sim j} \frac{\fmnz_{k-1, m}(j)}{d_i} + \sum_{r=1}^m \frac{1}{r!} \sum_{i \sim j} \frac{(\nf(j) - \nf(i))^{\otimes r}}{d_i} \cdot  \fmnz_{k-1, m- r}(j) \label{eq:g2t:Phi_k_m_expanded}
\end{align}
Note that the first sum is equivalent to
\begin{align}
    \sum_{i \sim j} \frac{\fmnz_{k-1, m}(j)}{d_i} = \one_i^T P \cdot \fmnz_{k-1,m}.
\end{align}

Applying the linear functional $\bell_m$ and the induction hypothesis to this, we have
\begin{align}
    \left\langle \bell_m, \sum_{i \sim j} \frac{\fmnz_{k-1, m}(j)}{d_i} \right\rangle = \one_i^T P \cdot f_{k-1,m}.
\end{align}
For the second sum in \eqref{eq:g2t:Phi_k_m_expanded}, we factor the inner product and apply the induction hypothesis
\begin{align}
    \Bigg\langle \bell_m , \sum_{i \sim j} \frac{(\nf(j) - \nf(i))^{\otimes r}}{d_i} \cdot  \fmnz_{k-1, m- r}(j) \Bigg\rangle  = \sum_{j=1}^n P_{i,j} \cdot C^{u_{M-m+1}}_{i,j} \cdot \ldots \cdot C^{M-m+r}_{i,j} \cdot f_{k-1, m-r}.
\end{align}
Putting this all together, we get
\begin{align}
    \fmnz_{k,m}(i) = \one_i^T \left( P \cdot f_{k-1, m} + \sum_{r=1}^m \frac{1}{r!} (P \odot C^{u_{M-m+1}} \odot \ldots \odot C^{u_M}) \cdot f_{k-1, m-r}\right) = f_{k,m}(i).
\end{align}
\end{proof}

\paragraph{Computational Complexity}
We will now consider the computational complexity of our algorithms. We begin by noting that the naive approach of computing $\Phi_k(i) = (F\wP^k\one_H)^{(i)}$ has the computational complexity of matrix multiplication; though this counts tensor operations, which itself requires $O(d^m)$ scalar multiplications at tensor degree $m$. This is computationally too expensive for practical applications.

Next, we consider the complexity of the recursive low-rank algorithm from Thm.~\ref{thm:algo}, where the primary computational advantage is that we only perform \emph{scalar} operations rather than \emph{tensor} operations. We consider the recursive step from eq.~\eqref{eq:g2t:recursive_low_rank}, reproduced here for $m = M$,
\begin{align}
   f_{k,M} = P \cdot f_{k-1,M} + \sum_{r=1}^M \frac{1}{r!} \left(P \odot C^{u_{1}} \odot \cdots \odot C^{u_{r}}\right) \cdot f_{k-1,M-r}.
\end{align}
Given a graph $\cG = (\cV, \cE, f)$ with $n = |\cV|$ nodes and $E = |\cE|$ edges, both $P$ and $C^u$ are sparse $n \times n$ (scalar) matrices with $O(E)$ nonzero entries, and $f_{k,m}$ is an $n$-dimensional column vector. Recall that $\odot$ denotes element-wise multiplication, and thus both the sparse matrix-matrix multiplication and the sparse matrix-vector multiplication have complexity $O(E)$. Furthermore, the entry-wise products $C^{u_1} \odot \cdots \odot C^{u_r}$ differ by only one factor between $r=m$ and $r=m+1$, and thus, computing $f_{k,M}$ assuming all lower $f_{k', m'}$ have been computed has complexity $O(M E)$. Taking into account the two recursive parameters results in a complexity of $O(kM^2 E)$. Note that this is the complexity to compute features for \emph{all nodes}. Once the $f_{k,m}$ are computed, the complexity of adding the zero starting point is $O(M)$.

\section{Variations and Hyperparameters of Hypo-Elliptic Diffusion}
\label{app:g2t:variations}

In this appendix, we summarize possible variations of the sequence feature map, leading to different hypo-elliptic diffusion features. The choice of variation is learned during training, and we also summarize the hyperparameters used for our features. While the theoretical results on characterizing random walks, such as Thm.~\ref{thm:characterizing_rw_informal}, depend on specific choices of the sequence feature map, there exist analogous results for these variations, which can characterize random walks up to certain equivalences. Furthermore, the computation of these variations can be performed in the same way: through tensorized linear algebra for exact solutions, and through an analogous low-rank method (as in~Thm.~\ref{thm:algo}) for approximate solutions. 

We fix an algebra lifting $\fm: \bbR^d \rightarrow H$ and let $\bx = (x_0, \ldots, x_k) \in \Seq(\bbR^d)$. The simplest sequence feature map to define simply multiplies the terms in the sequence together as
\begin{align} \label{eq:g2t:nodiff}
    \fms(\bx) = \fm(x_0) \cdots \fm(x_k).
\end{align}
Note that this is \emph{not} the sequence feature map used in the main text. We will now discuss several variations of this map, where $\fms_{\inc, \zs}$ is the one primarily used in the main text and $\fms_{\inc, \zs, \tp}$ is used to characterize random walks in Theorem~\ref{thm:characterizing_rw_informal} and \cite[App.~E]{toth2022capturing}.

\begin{description}
    \item [Increments ($\texttt{Diff}$).] Rather than directly multiplying terms in the sequence together, we can instead multiply the \emph{increments} as
    \begin{align}
    \fms_\inc(\bx) = \fm(\delta_1 \bx) \cdots \fm(\delta_k\bx),
    \end{align}
    where $\delta_i \bx = x_i - x_{i-1}$ for $i \geq 1$. In both cases, the sequence feature map is the path signature of a continuous piecewise-linear path when we set $\fm = \exp_\otimes$, as discussed in~\cite[App.~C]{toth2022capturing}, and it is instructive to use this perspective to understand the effect of increments. If we use increments, the path corresponding to the sequence is
    \begin{align}
    \bX_\inc(t) = x_{i} + (t-i)(x_{i+1}-x_i) \text{ for }t \in \left[{i},{i+1}\right),
    \end{align}
    while if we do not use increments, the path corresponding to the sequence is
    \begin{align}
    \bX(t) = \sum_{j=0}^{i - 1} x_j + (t-i)x_i \text{ for }t \in \left[{i},{i+1}\right).
    \end{align}
    Thus, when we use increments the sequence $\bx$ corresponds to vertices of the path $\bX_\inc$; if we do not, it corresponds to the vectors between vertices of the path $\bX$. In practice, this variation corresponds to taking 1\textsubscript{st} differences of $\bx$ before using eq.~\eqref{eq:g2t:nodiff}. 
    
    \item [Zero starting point ($\texttt{ZeroStart}$).] The sequence feature map with increments, $\fms_\inc$, as defined above is \emph{translation-invariant}, meaning $\fms_\inc(\bx+a) = \fms_\inc(\bx)$, where $\bx+a = (x_0 +a, x_1 +a, \ldots, x_k +a)$ for some $a \in \bbR^d$. In order to remove translation invariance, we can start each sequence at the origin $0 \in \bbR^d$ by pre-appending a $0$ to each sequence. A concise way to define the resulting \emph{zero started} sequence feature map is
    \begin{align}
    \fms_{\inc, \zs}(\bx) = \fm(\delta_0 \bx) \cdots \fm(\delta_k\bx),
    \end{align}
    where we define $\delta_0 \bx = x_0$. This is the sequence feature map defined in eq.~\eqref{eq:g2t:sequence_feature_map}. Note that this does not change the sequence feature map if we do not use increments.
    
    \item [Time parameterization.] When we relate sequences to piecewise linear paths as described in~\cite[App.~C]{toth2022capturing}, we can use the fact that the path signature is invariant under reparameterization, or more generally, tree-like equivalence~\cite{hambly2010uniqueness}. In terms of discrete sequences, this includes invariance with respect to $0$ elements in the sequence (without increments), and repeated elements in the sequence (with increments). In order to remove this invariance, we can include \emph{time parameterization} by setting
    \begin{align}
        \fms_{-, \tp}(\bx) = \fms_{-}(\bar{\bx}),
    \end{align}
    where $\bar{\bx} = (\bar x_0, \ldots, \bar x_k) \in \Seq(\bbR^{d+1})$, with $\bar x_i = (i, x_i) \in \bbR^{d+1}$. This is a simple form of positional encoding, but other encodings are also possible, e.g.~sinusoidal waves \cite{vaswani2017attention}.
    
    \item [Algebra lifting ($\texttt{AlgOpt}$).] Throughout this article, we have used the tensor exponential as the algebra lifting. However, we can also scale each level of the lifting independently, and keep these as hyperparameters to optimize. In particular, for a sequence $\bc = (c_0, c_1, \ldots) \in \bbR^\bbN$, we define $\fm^\bc: \bbR^d \rightarrow H$ to be
    \begin{align}
        \fm^\bc(x) = \left(c_m x^{\otimes m}\right)_{m=0}^\infty,
    \end{align}
    where $1 / m!$ is used as initialization for $c_m$ and learned along with the other parameters.
\end{description}

The choice of which variant of the sequence feature map to use depends on which invariance properties are important for the specific problem. In practice, the choice can be learned during the training, which is done in our experiments in Section~\ref{sec:g2t::experiments}. Furthermore, the features obtained through the low-rank hypo-elliptic diffusion depend on three hyperparameters:
\begin{enumerate*}[label=(\arabic*)]
    \item the length of random walks;
    \item the number of low-rank functionals;
    \item the maximal tensor degree;
    \item the number of iterations (layers).
\end{enumerate*}
Note that the first three hyperparameters can also potentially vary across different iterations. 

\section{Experiments} \label{app:g2t:exp}

We have implemented the low-rank algorithm for our layers given in Theorem \ref{thm:algo} using \texttt{Tensorflow}, \texttt{Keras} and \texttt{Spektral} \cite{grattarola2021graph}. Code is available at \url{https://github.com/tgcsaba/graph2tens}. All experiments were run on one of 3 computing clusters that were overall equipped with 5 NVIDIA Geforce 2080Ti, 2 Quadro GP100, 2 A100 GPUs. The largest GPU memory allocation any one of the experiments required was around ${\sim}5$GB. For the experiments ran using different random seeds, the seed was used to control the randomness in the \begin{enumerate*}[label=(\arabic*)] \item data splitting process, \item parameter initialization, \item optimization procedure. \end{enumerate*} For an experiment with overall $n_{\text{runs}}$ number of runs, the used seeds were $\{0, \dots, n_{\text{runs}}-1\}$.
\subsection{Model details.} \label{app:g2t:model_detail}
\begin{figure}[t]
    \centering
    \includegraphics[width=\textwidth]{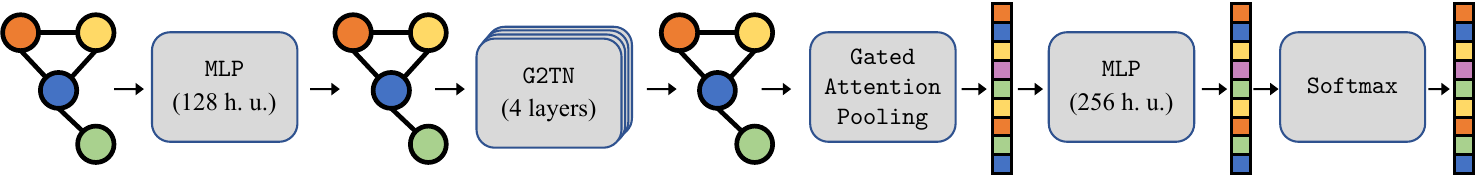}
    \caption{Visualization of the architecture used for NCI1 and NCI109 described in Section \ref{sec:g2t::experiments}.}
    \label{fig:g2t:architecture}
\end{figure}

The architecture described in Section \ref{sec:g2t::experiments} is visualized on Figure \ref{fig:g2t:architecture}, where for \GTAN the \GTN stack is ceteris paribus replaced by the attentional variation. To further the discussion, here we provide more details on the model implementation.

\paragraph{Initialization} In our \GTN and \GTAN layers, we linearly project tensor-valued features using linear functionals that use the same low-rank parameterization as the \emph{recursive} variation in \cite[App.~D.2]{toth2021seq2tens}. There, the authors also propose in App.~D.5 an initialization heuristic for setting their variances, that we have employed with a centered uniform distribution.

\paragraph{Regularization} We also apply $\ell_2$ regularization for both experiments in Sections \ref{sec:g2t::experiments} and \ref{app:g2t:exp_detail}. Given a max.~tensor degree $M \geq 1$ and a max.~tensor rank $R \geq 1$ representing the layer width, there are overall $RM$ rank-$1$ linear functionals of the form
\begin{align} \label{eq:g2t:lin_func}
    \bell^r_m = u^r_{M-m+1} \otimes \cdots \otimes u^r_{M} \quad \text{for } m=1,\dots M \text{ and } r=1, \dots, R.
\end{align}
Since in practice these are represented using only the component vectors $u^r_{i} \in \bbR^d$, a naive approach would compute the $2$-norm of each component vector $u^r_i$ to penalize the loss function. However, we found this approach to underperform compared to the following. Although our algorithm represents the functionals using a rank-$1$ decomposition and computes the projection of each tensor-valued node feature without explicitly building high-dimensional tensors, conceptually we still have tensors ($\bell_m^r$) acting on tensors ($\fmn(i)$ for $i \in \cV$), and hence the tensor norm, $\norm{\bell_m^r}_2$, should be used for regularization. Fortunately, this can also be computed efficiently:
\begin{align}
    \norm{\bell_m^r}_2 = \norm{u_{M-m+1}^r \otimes \cdots \otimes u_M^r}_2 = \norm{u_{M-m+1}^r}_2 \cdots \norm{u_M^r}_2.
\end{align}
Further, as is common, we replace the $\ell_2$ norm with its squared value, sum over all functionals in \eqref{eq:g2t:lin_func}, and multiply by the regularization parameter $\lambda > 0$ so the final penalty is
\begin{align}
    \texttt{L2Penalty} = \lambda \sum_{r=1}^R \sum_{m=1}^M \norm{u_{M-m+1}^r}_2^2 \cdots \norm{u_M^r}_2^2.
\end{align}

\section{Details on experiments.} \label{app:g2t:exp_detail}
Here we extend the experimental results demonstrated in Section \ref{sec:g2t::experiments}.

\begin{table}[t]
\caption{Accuracies computed over 5 seeds of \GTAN ablated by ceteris paribus changing a single model setting.}
\label{table:g2t:ablation_g2tan}
\resizebox{\textwidth}{!}{ 
  \centering
  \begin{sc}
  \begin{tabular}{lccccccc}
    \toprule
    \textbf{Dataset} & \texttt{NoDiff} & \texttt{NoZeroStart} & \texttt{NoAlgOpt} & \texttt{NoJK} & \texttt{NoSkip} & \texttt{NoNorm} & \texttt{AvgPool} \\
    \midrule
    NCI1 & $80.1 \pm 0.7$ &  $79.5 \pm 1.8$ & $81.6 \pm 1.6$ & $82.1 \pm 1.8$ & $81.8 \pm 0.9$ & $81.6 \pm 1.5$ & $82.4 \pm 0.9$ \\
    NCI109 & $78.2 \pm 1.2$ & $77.7 \pm 1.8$ & $77.5 \pm 1.3$ & $77.6 \pm 1.2$ & $78.3 \pm 1.3$ & $79.8 \pm 1.4$ & $77.6 \pm 1.3$ \\
    \bottomrule
  \end{tabular}
\end{sc}
}
\end{table}

\begin{table}[t]
    \caption{Accuracies computed over 5 seeds of \GTN ablated by ceteris paribus changing a single model setting.}
    \label{table:g2t:ablation_g2tn}
      \centering
      \begin{sc}
      \resizebox{\textwidth}{!}{
      \begin{tabular}{lccccccc}
        \toprule
        \textbf{Dataset} & \texttt{NoDiff} & \texttt{NoZeroStart} & \texttt{NoAlgOpt} & \texttt{NoJK} & \texttt{NoSkip} & \texttt{NoNorm} & \texttt{AvgPool} \\
        \midrule
        NCI1 & $79.1 \pm 1.5$ &  $80.7 \pm 0.6$ & $78.9 \pm 1.3$ & $79.4 \pm 1.9$ & $81.0 \pm 1.3$ & $79.5 \pm 1.1$ & $80.3 \pm 1.7$ \\
        NCI109 & $77.8 \pm 1.2$ & $79.5 \pm 1.5$ & $76.5 \pm 1.7$ & $78.1 \pm 1.9$ & $77.0 \pm 1.6$ & $77.4 \pm 2.7$ & $77.6 \pm 1.8$ \\
        \bottomrule
      \end{tabular}
      }
    \end{sc}
\end{table}

\subsection{Ablation Studies} 
\label{app:g2t:ablation}
Further to using attention, we give ideas for variations on our models in Appendix \ref{app:g2t:variations} which can be summarized briefly as: \begin{enumerate*}[label=(\roman*)] \item using increments of node attributes (\texttt{Diff}), \item preprending a zero point to sequences (\texttt{ZeroStart}), \item  optimizing over the algebra embedding (\texttt{AlgOpt}) \end{enumerate*}, all of which are built into our main models. Further, the previous architectural choices aimed at incorporating several commonly used tricks for training \texttt{GNN}s. We investigate the effect of the previous variations and ablate the architectural ``tricks'' by measuring the performance change resulting from ceteris paribus removing it. Table~\ref{table:g2t:ablation_g2tan} shows the result for \GTAN. To summarize the main observations, the model is robust to all the architectural changes, removing the layer norm even improves on NCI109. Importantly, replacing the attention pooling with mean pooling does not significantly affect the performance, but actually slightly improves on NCI1. Regarding variations, $\texttt{AlgOpt}$ slightly improves on both datasets, while removing $\texttt{Diff}$ and $\texttt{ZeroStart}$ significantly degrades the accuracy on NCI1. The latter means that translations of node attributes are important, not just their distances.

 We give the analogous ablation study for the \GTN model in Table \ref{table:g2t:ablation_g2tn}, and compare the derived conclusions between the attentional and attention-free versions. First, we discuss the layer variations. Similarly to \GTAN, \texttt{NoDiff} slightly decreases the accuracy. However the conclusions regarding \texttt{NoZeroStart} and \texttt{NoAlgOpt} are different. In this case, removing \texttt{ZeroStart} actually improves the performance, while on \GTAN the opposite was true. An interpretation of this phenomenon is that only the attention mapping that is used to learn the random walks required information about translations of the layer's features, and not the tensor-features themselves. Another difference is that \texttt{NoAlgOpt} degrades the accuracy more significantly for \GTN. A possible explanation is that since \GTAN layers are more flexible thanks to their use of attention, they rely less on being able to learn the algebraic lift, while as \GTN layers are more rigid in their random walk definition, and benefit more from the added flexibility of \texttt{AlgOpt}. Additionally, it seems that \GTN is more sensitive to the various architectural options, and removing any of them, i.e.~\texttt{NoJK}, \texttt{NoSkip}, \texttt{NoNorm} or \texttt{AvgPool}, degrades the accuracy by $1\%$ or more on at least one dataset. Intuitively, it seems that overall the \GTAN model is more robust to the various architectural ``tricks'', and more adaptable due to its ability to learn the random walk.

\subsection{Further Experiments} \label{app:g2t:further_exp}
\paragraph{Citation datasets}
Additionally to transductive learning on the biological datasets, we have carried out inductive learning tasks on some of the common citation datasets, i.e.~Cora, Citeseer \cite{sen2008collective} and Pubmed \cite{namata2012query}. We follow \cite{shchur2018pitfalls} in carrying out the experiment, and use the largest connected component for each dataset with $20$ training examples per class, $30$ validation examples per class, and the rest of the data used for testing.  The hyperparameters of our models and optimization procedure were based on the settings of the GAT model in \cite[Table 4]{shchur2018pitfalls}, which we have slightly fine-tuned on Cora and used for the other datasets. In particular, a single layer of \GTN or \GTAN is used with $64$ functionals, max.~tensor degree $2$ and random walk length $5$. The dropout rate was tuned to $0.9$, while the attentional dropout was set to $0.3$ in \GTAN. Optimization is carried out with Adam \cite{kingma2014adam}, a fixed learning rate of $0.01$ and $\ell_2$ regularization strength $0.01$. Training is stopped once the validation loss does not improve for $50$ epochs, and restored to the best checkpoint. For both of our models, \texttt{NoDiff} is used that we found to improve on the results as opposed to using increments of node features. The dropout rate had to be tuned as high as $0.9$, which suggests very strong overfitting, hence the additional complexity of \texttt{AlgOpt} was also contrabeneficial, and \texttt{NoAlgOpt} was used. As such, each model employs a single hidden layer, which is followed by layer normalization that was found to perform slightly better than other normalizations, e.g.~graph-level normalization.
\begin{table}[t]
\caption{Accuracies of our models on the citation datasets computed over 100 seeds compared with the 4 consistently best performing baselines from \cite{shchur2018pitfalls}.}
\label{table:g2t:cit_results}
\centering
\resizebox{\textwidth}{!}{
\begin{sc}
    \begin{tabular}{lcccccc}
    \toprule
    \textbf{Dataset} & \textbf{GCN} & \textbf{GAT} & \textbf{MoNet} & \textbf{GraphSage (mean)} & \textbf{\GTN (ours)} & \textbf{\GTAN (ours)} \\
    \midrule
    Cora & $81.5 \pm 1.3$ & $81.8 \pm 1.3$ & $81.3 \pm 1.3$ & $79.2 \pm 7.7$ & $\mathbf{82.6 \pm 1.0}$ & $82.0 \pm 1.1$ \\
    Citeseer & $\mathbf{71.9 \pm 1.9}$ & $71.1 \pm 1.9$ & $71.2 \pm 2.0$ & $71.6 \pm 1.9$ & $69.4 \pm 1.0$ & $68.2 \pm 1.3$ \\
    Pubmed & $77.8 \pm 2.9$ & $78.7 \pm 2.3$ & $78.6 \pm 2.3$ & $77.4 \pm 2.2$ & $\mathbf{78.8 \pm 1.9}$ & $78.0 \pm 1.9$ \\
    \bottomrule
    \end{tabular}
    \end{sc}
}
\end{table}

\begin{table}[t]
    \caption{Accuracies of our models on $k$-hop sanitized citation datasets computed over 100 seeds compared with the 5 consistently best performing models from \cite{rampavsek2021hierarchical}.}
    \label{table:g2t:results_cit_khop}
    \centering
    \resizebox{\textwidth}{!}{
    \begin{sc}
    
        \begin{tabular}{llccccccc}
        \toprule
        \textbf{k} & \textbf{Dataset} & \textbf{GCN} & \textbf{GAT} & \textbf{g-U-net} & \textbf{HGNet-EP} & \textbf{HGNet-L} & \textbf{\GTN (ours)} & \textbf{\GTAN (ours)} \\
        \midrule 
        \multirow{3}{*}{$1$} & Cora & $76.7$ & $78.5$ & $78.1$ & $77.2$ & $77.1$ & $\mathbf{81.8 \pm 1.3}$ & $80.9 \pm 1.4$ \\
        & Citeseer & $64.2$ & $66.4$ & $63.0$ & $64.3$ & $64.10$ & $\mathbf{68.1 \pm 1.3}$ & $66.6 \pm 1.4$ \\
        & Pubmed & $75.8$ & $75.9$ & $75.8$ & $77.0$ & $76.3$ & $\mathbf{78.7 \pm 1.9}$ & $77.5 \pm 2.0$ \\
        \middashrule
        \multirow{3}{*}{$2$} & Cora & $72.0$ & $73.4$ & $74.4$ & $74.0$ & $75.4$ & $\mathbf{79.4 \pm 2.8}$ & $78.3 \pm 2.9$ \\
        & Citeseer & $58.3$ & $59.4$ & $57.3$ & $57.8$ & $59.9$ & $\mathbf{63.6 \pm 1.8}$ & $62.2 \pm 1.9$ \\
        & Pubmed & $72.1$ & $73.1$ & $72.4$ & $72.9$ & $75.1$ & $\mathbf{77.1 \pm 1.7}$ & $76.3 \pm 1.6$ \\
        \bottomrule
        \end{tabular}
    \end{sc}
    }
    \vspace{-10pt}
    \end{table}

The results of our models over 100 runs are reported in Table \ref{table:g2t:cit_results} compared with the 4 consistently best performing baselines on these datasets from \cite{shchur2018pitfalls}, i.e.~GCN \cite{kipf_semi-supervised_2017}, GAT \cite{velickovic_graph_2018}, MoNet \cite{monti2017geometric}, and GraphSage \cite{hamilton2017inductive} with a mean aggregator. Firstly on Cora, both \GTN and \GTAN outperform the baselines with a more significant improvement for \GTN. For CiteSeer, our models are left somewhat behind compared to the baselines in terms of accuracy. Finally, they are again competitive on Pubmed, where \GTN takes the top score with a very slight lead. Two consistent observations are: \begin{enumerate*}[label=(\arabic*)] \item \GTN and \GTAN have a lower variance than all baselines, \item \GTN consistently outperforms \GTAN. \end{enumerate*} The latter may be attributed to the observation that due to the severe overfitting on these datasets, the additional complexity of the attention mechanism in \GTAN is unhelpful for generalization.

\paragraph{$K$-hop Sanitized Splits} Recent work \cite{rampavsek2021hierarchical} has demonstrated that it is possible to make the previously considered citation datasets more suitable for testing the ability of a model to learn long-range information by dropping node labels in a structured way. Concretely, they use a label resampling strategy to guarantee that if a node is selected for a data split, none of its $k$-th degree neighbours are included in any splits, i.e.~training, validation nor testing, allowing to reduce the effect of short-range ``label imprinting''. In practice, we select a maximal independent set from the graph with respect to the $k$-th power of the adjacency matrix with self-loops, and repeat the previous experiment with the same data splitting method, model choice and training settings as before. The experiment seed also controls the random maximal independent set that is selected.  

The results of our models trained on the citation datasets sanitized this way are available in Table \ref{table:g2t:results_cit_khop} computed over 100 seeds for $k=1,2$. As baseline results, we compare against the 5 best performing models from \cite{rampavsek2021hierarchical}, where various \texttt{GNN} depths were also considered for each model, and we use the \emph{best} reported result for each of the baselines. Overall, we can observe that all baseline models exhibit a very sharp drop in performance as $k$ is increased, while for \GTN and \GTAN, the performance decrease is not nearly as pronounced. For both $k=1$ and $k=2$, our models perform better than the baselines on all datasets. This is explained by the fact that due to using random walks length of $5$, the models can efficiently pick up on information outside of the sanitized neighbourhoods. As previously, \GTN performs better than \GTAN as the additional flexibility of the attention layer does not lead to improved generalization performance when severe overfitting is present. This experiment demonstrates that the proposed models efficiently pick up on long-range information within larger neighbourhoods, and it suggests that on inductive learning tasks they should be more robust to sparse labeling rates compared to common short-range \texttt{GNN} models.

\subsection{Computation Time}
\label{app:g2t:computation_time}

Table \ref{table:g2t:comp_time_graph} shows the computation time of one forward pass of a \texttt{G2TN} layer with various graphs. This empirically demonstrates that our layer does not depend on the number of nodes, and only depends linearly on the number of edges as expected from the theoretical complexity.

Table \ref{table:g2t:comp_time_modelparam} shows the computation time of one forward pass of a \texttt{G2TN} layer with a fixed graph and various model parameters. Empirically, our layer scales roughly linearly in both the walk length $k$ and the maximum tensor level $M$. Linear complexity in walk length is expected, but the linearity in the maximum tensor level $M$ is due to parallelization of the algorithm presented in Thm.~\ref{thm:algo}. This is discussed in the pseudocode section of \cite[App.~F]{toth2022capturing}.

\begin{table}[t]
    \caption{Computation time of one forward pass of one \texttt{G2TN} layer in milliseconds on a Nvidia GeForce 2080TI GPU for a graph with varying nodes ($N$) and edges ($E$), while node feature dimension is fixed at $128$. }
    \label{table:g2t:comp_time_graph}
    \begin{adjustbox}{center}
      \centering
      \begin{sc}
      \begin{tabular}{cccccc}
        \toprule
        & $E=5000$ & $E=10000$ & $E=20000$ & $E=40000$ & $E=80000$ \\
        \midrule
        $N=500$ & $37$ & $61$ & $98$ & $151$ & $293$ \\
        $N=1000$ & $37$ & $62$ & $104$ & $152$ & $290$\\
        $N=2000$ & $39$ & $63$ & $89$ & $160$ & $293$\\
        $N=4000$ & $39$ & $63$ & $109$ & $160$ & $294$\\
        $N=8000$ & $42$ & $65$ & $108$ & $156$ & $295$\\
        \bottomrule
      \end{tabular}
      \end{sc}
    \end{adjustbox}
\end{table}

\begin{table}[t]
    \caption{Computation time of one forward pass of one \texttt{G2TN} layer in milliseconds on a Nvidia GeForce 2080TI GPU for a fixed graph with $N=2000$ nodes and $E=10000$ edges. The walk length ($k$) and the maximum level $M$ of the model are varied. }
    \label{table:g2t:comp_time_modelparam}
    \begin{adjustbox}{center}
      \centering
      \begin{sc}
      \begin{tabular}{cccccc}
        \toprule
        & $k=2$ & $k=4$ & $k=6$ & $k=8$ & $k=10$ \\
        \midrule
        $M=1$ & $14$ & $26$ & $45$ & $59$ & $70$\\
        $M=2$ & $20$ & $40$ & $72$ & $81$ & $101$\\
        $M=3$ & $26$ & $55$ & $95$ & $115$ & $163$\\
        $M=4$ & $34$ & $73$ & $101$ & $150$ & $188$\\
        $M=5$ & $42$ & $92$ & $139$ & $187$ & $237$\\
        \bottomrule
      \end{tabular}
      \end{sc}
    \end{adjustbox}
\end{table}

\subsection{Parameter Count}
In both the NCI datasets, the \texttt{G2TN} and \texttt{G2T(A)N} have 505k and 519k trainable parameters respectively. Thus, our models are comparable in size to the GraphTrans, which is reported to have 500k trainable parameters~\cite{wu2021representing}.

\subsection{Summary of Dataset Parameters}
We provide a summary of dataset statistics and the model parameters in the experiments. Here, $M$ denotes the maximum tensor degree, $k$ denotes the walk length, and $R$ denotes the maximum rank of the low-rank functions. Concretely, $R$ is the number of node features used within the network, and is analogous to the width of a neural network. For the computation of the diameter of a disconnected graph, we take the largest diameter of any connected component.  
\begin{table}[h]
    \caption{Dataset statistics and \texttt{G2TN} model parameters used.}
    \label{table:g2t:dataset_stats}
      \centering
      \resizebox{\textwidth}{!}{
      \begin{sc}
      \begin{tabular}{cccccccccc}
        \toprule
        & graphs & total nodes & total edges &  avg. diam. & classes & layers & $M$ &  $k$ &  $R$ \\
        \midrule
        NCI1& 4110 & 122k & 132k   & 13.3 &2 & 4 & 2 & 5 & 128\\
        NCI109& 4127 & 122k & 132k  & 13.1 &2 & 4 & 2 & 5 & 128\\
        CORA & 1 & 2485 & 5069  & 19 & 7 &4 & 2 & 5 & 64\\
        Citeseer & 1 & 2110 & 3668  & 28& 6 & 4 & 2 & 5 & 64\\
        Pubmed & 1 &  19k & 44k & 18 &3 & 4 & 2 & 5 & 64\\
        \bottomrule
      \end{tabular}
      \end{sc}
    }
\end{table}



\end{subappendices}

\endgroup
\begingroup
\chapter{Random Fourier Signature Features} \label{ch:rfsf}

\section{Introduction} \label{sec:rfsf:introduction}
Machine learning has successfully been applied to tasks that require learning from complex and structured data types on non-Euclidean domains. Feature engineering on such domains is often tackled by exploiting the geometric structure and symmetries existing within the data \cite{bronstein2021geometric}.
Learning from sequential data (such as video, text, audio, time series, health data, etc.) is a classic, but an ongoing challenge due to the following properties:
\begin{itemize}
\item \emph{Non-Euclidean data.} The data domain is nonlinear since there is no obvious and natural way of adding sequences of different length. 

\item \emph{Time-space patterns.} 
  Statistically significant patterns can be distributed over time and space, that is, capturing the order structure in which ``events'' arise is crucial. 

\item \emph{Time-warping invariance.} The meaning of many sequences is often invariant to reparameterization also frequently called time-warping, at least to an extent; e.g.~a sentence spoken quicker or slower contains (essentially) the same information. 

\item \emph{Discretization and irregular sampling.}
  Sequences often arise by sampling along an irregularly spaced grid of an underlying continuous time process. A general methodology should be robust as the sampling gets finer, sequences approximate paths (continuous-time limit), or as the discretization grid varies between sequences.

\item \emph{Scalability.} Sequence datasets can quickly become massive, so the computational complexity should grow subquadratically, in terms of all of the state-space dimension, and the length and number of sequences. 
\end{itemize}
The signature kernel $\sigkernel$ is the state-of-the-art kernel for sequential data \cite{toth2020bayesian,salvi2021signature,lemercier2021scaling} that addresses the first 4 of the above questions and can rely on the modular and powerful framework of kernel learning \cite{scholkopf2002learning}.
Its construction is motivated by classic ideas from stochastic analysis that give a structured description of a sequence by developing it into a series of tensors.
We refer to \cite{lee2023signature} for a recent overview of its various constructions and applications. In the real-world, various phenomena are  well-modelled by systems of differential equations. The path signature arises naturally in the context of controlled differential equations. The role of the signature here is to provide a basis for the effects of a driving signal on systems of controlled differential equations. In essence, it captures the interactions of a controlling signal with a nonlinear system. This explains the widespread applicability of signatures to various problems across the sciences \cite{lyons2007differential}. There is also geometric intuition behind signatures, see Section 1.2.4 in \cite{chevyrev_primer_2016}.
\paragraph{Features vs Kernel/Primal vs Dual}
Kernel learning circumvents the costly evaluation of a high- or infinite-dimensional feature map by replacing it with the computation of a Gram matrix which contains as entries the inner products of features between all pairs of data points.
This can be very powerful since the inner product evaluation can often be done cheaply by the celebrated "kernel trick", even for infinite-dimensional feature spaces, but the price is that now the computational cost is quadratic in the number of samples, and downstream algorithms further often incur a cubic cost usually in the form of a matrix inversion. On the other hand, when finite-dimensional features can used for learning, the primal formulation of a learning algorithm can perform training and inference in a cost that is linear with respect to the sample size assuming that the feature dimension is fixed. This motivates the investigation of finite-dimensional approximations to kernels that mimic their expressiveness at a lower computational cost. It is an interesting question how the feature dimension should scale with the dataset size to maintain a given (optimal) learning performance in downstream tasks,
which is investigated for instance by \cite{rudi17generalization,carratino18learning,sun2018but,li2019towards,sriperumbudur22approximate,lanthaler2023error}.

\paragraph{Computational Cost of the Signature kernel}
In the context of the signature kernel, one data point is itself a whole sequence. 
Hence, given a data set $\bX$ consisting of $N \in \bbZ_+$ sequences where each sequence $\bx=(\bx_1,\ldots,\bx_{\len{\bx}})$ is of maximal length $\len{\bx} \leq \ell \in \bbZ_+$ and has sequence entries $\bx_i$ in a state-space of dimension $d$, then the existing algorithms to evaluate the Gram matrix of the $\sigkernel$ scale quadratically, i.e.~as $O(N^2 \ell^2 d)$, both in sequence length $\ell$ and number of sequences $N$.
So far this has only been addressed by subsampling (either directly the sequence elements to reduce the length or by column subsampling via the Nyström approach \cite{williams2000nystrom}), which can lead to crude approximations and performance degradation on large-scale datasets.

\paragraph{Contribution}
Random Fourier Features (\RFF) \cite{rahimi2007random} is a classic technique to enjoy both the benefits of the primal and dual approach. 
Here, a low-dimensional and random feature map is constructed, which although does not approximate the feature map of a translation-invariant kernel, its inner product is with high probability close to the kernel itself.
The main contribution of this article is to carry out such a construction for the signature kernel.
Concretely, we construct a random feature map on the domain of sequences called Random Fourier Signature Features (\RFSF), such that its inner product is a random kernel $\rffsigkernel$ for sequences that is both 
\begin{enumerate*}[label=(\roman*)]
    \item an unbiased estimator for $\sigkernel$, and 
    \item has analogous probabilistic approximation guarantees to the classic \RFF{} kernel.
\end{enumerate*}
The challenge is that a direct application of the classic \RFF{} technique is not feasible since this relies on Bochner's theorem which does not apply since the sequence domain is not even a linear space and the feature domain is non-Abelian, which makes the use of (generalizations of \cite{fukumizu2008characteristic}) Bochner's theorem difficult due to the lack of sufficiently explicit representations.
We tackle this challenge by combining the algebraic structure of signatures with probabilistic concentration arguments; a careful analysis of the error propagation yields uniform concentration guarantees similar to the \RFF{} on $\bbR^d$.
Then, we introduce dimensionality reduction techniques for random tensors further approximating $\rffsigkernel$ to define the extremely scalable variants $\rffsigkernelDP$ and $\rffsigkernelTRP$ called \RFSFD{} and \RFSFT{} saving considerable amounts of computation time and memory by low-dimensional projection of the feature set of the \RFSF{}.

Hence, analogously to the classic \RFF{} construction, the random kernels $\rffsigkernel, \rffsigkernelDP, \rffsigkernelTRP$ simultaneously enjoy the expressivity of an infinite-dimensional feature space as well as linear complexity in sequence length.
This overcomes the arguably biggest drawback of the signature kernel, which is the quadratic complexity in sample size and sequence length; the price for reducing the complexities by an order is that this approximation only holds with high probability. As in the case of the classic \RFF, our experiments show that this is in general a very attractive tradeoff. Concretely, we demonstrate in the experiments that the proposed random features \begin{enumerate*}[label=(\arabic*)] \item provide comparable performance on moderate sized datasets to full-rank (quadratic time) signature kernels, \item outperform other random feature approaches for time series on both moderate- and large-scale datasets, \item allow scaling to datasets of a million time series\end{enumerate*}.

\paragraph{Related Work}
The signature kernel has found many applications; for example, it is used in ABC-Bayes \cite{dyer2023approximate}, economic scenario validation \cite{andres2023signaturebased}, amortised likelihood estimation \cite{dyer2022amortised}, the analysis of RNNs \cite{fermanian2021framing}, analysis of trajectories in Lie groups \cite{lee2020path}, metrics for generative modelling \cite{buehler2021generating,kidger2021neural}, or dynamic analysis of topological structures \cite{giusti2023signatures}. 
For a general overview see \cite{lee2023signature}.
All of these applications can benefit from a faster computation of the signature kernel with theoretical guarantees.
Previous approaches address the quadratic complexity of the signature kernel only by subsampling in one form or another:
\cite{kiraly2019kernels} combine a structured Nystr{\"o}m type-low-rank approximation to reduce complexity in dimension of samples and sequence length, \cite{toth2020bayesian} combine this with inducing point and variational methods, \cite{salvi2021signature} uses sequence-subsampling, \cite{lemercier2021scaling} use diagonal approximations to Gram matrices in a variational setting.
Related to this work is also the random nonlinear projections in \cite{lyons2017sketching}; further, \cite{morrill2021generalised} combine linear dimension projection in a general pipeline and \cite{cuchiero2021discrete} use signatures in reservoir computing. 
Directly relevant for this work is recent progress on tensorized random projections \cite{sun2021tensor,rakhshan2020tensorized}.
Random Fourier Features \cite{rahimi2007random,rahimi2008weighted} are well-understood theoretically \cite{sutherland2015error, sriperumbudur2015optimal,sriperumbudur22approximate,liao2020random,avron2017random, szabo2019kernel, chamakh2020orlicz, ullah2018streaming,chamakh2020orlicz}. 
In particular, its generalization properties are studied in e.g.~\cite{bach2013sharp,li2019towards, sun2018but, lanthaler2023error}, where it is shown that the feature dimension need only scale sublinearly in the dataset size for supervised learning, and a similar result also holds for kernel principal component analysis \cite{sriperumbudur22approximate}. 
Several variations have been proposed over the years \cite{le2013fastfood, feng2015random, choromanski2016recycling, yu2016orthogonal, choromanski2017unreasonable, choromanski2018geometry, choromanski2022hybrid}, even finding applications in deep learning \cite{tancik2020fourier}. Alternative random feature approaches for polynomial and Gaussian kernels based on tensor sketching have been proposed in e.g.~\cite{wacker2022complex, wacker2022improved, wacker2022local}. Gaussian sketching has also been applied in the RKHS for kernel approximation \cite{kpotufe2020gaussian}.  For a survey, the reader is referred to \cite{chamakh2020orlicz,liu2021random}.

\paragraph{Outline} Section \ref{sec:rfsf:prereq} provides background on the prerequisites of our work: Random Fourier Features, and Signature Features/Kernels. Section \ref{sec:rfsf:RFSF} contains our proposed methods with theoretical results; it introduces Random Fourier Signature Features (\RFSF) $\rffsig[\leq M]$,
\RFSF{} kernels $\rffsigkernel[\leq M]$ (where $M \in \bbZ_+$ is the truncation level introduced later), and most importantly their theoretical guarantees. Theorem \ref{thm:main} quantifies the approximation $\sigkernel[\leq M](\bx,\by) \approx \rffsigkernel[\leq M](\bx,\by)$ uniformly. Then, we discuss additional variants: the \RFSFD{} kernel $\rffsigkernelDP[\leq M]$ and the \RFSFT{} kernel $\rffsigkernelTRP[\leq M]$, which build on the previous construction using dimensionality reduction with corresponding concentration results in Theorems \ref{thm:main2} and \ref{thm:main3}. Section \ref{sec:rfsf:experiments} compares the performance of the proposed scalable signature kernels against popular approaches on \SVM{} multivariate time series classification, which demonstrates that the proposed kernel not only significantly improves the computational complexity of the signature kernel, it also provides comparable performance, and in some cases even improvements in accuracy as well. Hence, we take the best of both worlds: linear batch, sequence, and state-space dimension complexities, while approximately enjoying the expressivity of an infinite-dimensional \RKHS{} with high probability.

\section{Background} \label{sec:rfsf:prereq}

\subsection{Notation} \label{sec:rfsf:notation}
We denote the real numbers by $\bbR$, natural numbers by $\bbN = \curls{0, 1, 2, \dots}$, positive integers by $\bbZ_+ = \curls{1, 2, 3, \dots}$, the range of positive integers from $1$ to $n \in \bbZ_+$ by $\bracks{n} = \curls{1, 2, \dots, n}$. Given $a, b \in \bbR$, we denote their maximum by $a \vee b = \max(a, b)$ and their minimum by $a \wedge b = \min(a, b)$. We define the collection of all ordered $m$-tuples with non-repeating entries by 
\begin{align} \label{eq:rfsf:delta_m_def}
    \Delta_m(n) = \curls{1 \leq i_1 < i_2 < \cdots < i_m \leq n \setgiven i_1, i_2, \dots, i_n \in [n]}.
\end{align}

In general, $\cX$ refers to a subset of the input domain, where the various objects are defined, generally taken to be a subset $\bbR^d$ unless otherwise stated. For a vector $\bx \in \bbR^d$, we denote its $\ell_p$ norm by $\norm{\bx}_p = \pars{\sum_{i=1}^d \abs{x_i}^p}^{1/p}$. For a matrix $\bA \in \bbR^{d \times e}$, we denote the spectral and the Frobenius norm
by $\norm{\bA}_2 = \sup_{\norm{\bx}_2 = 1} \norm{\bA \bx}_2$ and $\norm{\bA}_F = \pars{\sum_{i=1}^e \norm{\bA \be_i}_2^2}^{1/2}$, where $\{\be_1, \dots, \be_e\}$ is the canonical basis of $\bbR^e$. The transpose of a matrix $\b A$ is denoted by $\b A^\top$. For a differentiable $f: \bbR^d \to \bbR$, we denote its gradient at $\bx \in \bbR^d$ by $\nabla f(\bx) = \pars{\nicefrac{\partial f(\bx)}{\partial x_i}}_{i=1}^d$, and its collection of partial derivatives with respect to $\bs = (x_{i_1}, \dots, x_{i_k})$ by $\partial_\bs f(\bx) = \pars{\nicefrac{\partial f(\bx)}{\partial x_{i_j}}}_{j=1}^k$.

$\Seq(\cX)$ refers to sequences of finite, but unbounded length with values in the set $\cX$:
\begin{align}
    \Seq(\cX) = \{\bx=(\bx_0,\ldots,\bx_L) \setgiven \bx_i \in \cX, L \in \bbZ_+   \}.
\end{align}
We denote the effective length of a sequence $\bx = (\bx_0, \dots, \bx_L) \in \Seq(\cX)$ by $\len{\bx} = L$, and define the 1\textsuperscript{st}-order forward differencing operator as $\delta \bx_i = \bx_{i+1} - \bx_i$. We define the $1$-variation functional of a sequence $\bx \in \Seq(\cX)$ as $\norm{\bx}_{\onevar} = \sum_{i=1}^{\ell_{\bx} - 1} \norm{\delta \bx_i}_2$ as a measure of sequence complexity.

\subsection{Signature Features for Sequential Data}
A challenge in machine learning when constructing feature maps for datasets of sequences is that the sequence length can vary from instance to instance; the space of sequences $\Seq(\cX) = \curls{(\bx_i)_{i=1}^\ell \setgiven \bx_0, \dots, \bx_\ell \in \cX \spc \text{and} \spc \ell \in \bbZ_+}$ includes sequences of various lengths, and they should all get mapped to the same feature space, while preserving the information about the ordering of elements. A concatenation property of path signatures called Chen's identity \cite[Thm.~2.9]{lyons2007differential}
turns concatentation into multiplication provides a principled approach to construct features for sequences. Below we recall the construction from Chapter \ref{ch:s2t} of the discrete-time order-$1$ signature features originally introduced in Section~\ref{sec:back:discrete}.

The key idea is to define the discrete-time signature of $1$-step increments, and then glue features together by algebra multiplication to guarantee the Chen identity holds by construction.
Now assume we are given a static feature map $\kernelfeatures: \cX \to \Hil$ into some Hilbert space $\Hil$. Our task is to construct from this feature map for elements of $\cX$, a feature map for sequences of arbitrary length in $\cX$.
A natural first step is to apply the feature map $\kernelfeatures$ elementwise to a sequence $\bx \in \Seq(\cX)$ to lift it into the feature space $\Hil$ of $\kernelfeatures$, $\kernelfeatures(\bx) = \pars{\kernelfeatures(\bx_i)}_{i=0}^{\len{\bx}} \in \Seq(\Hil)$. The challenge is now to construct a feature map for sequences in $\Hil$.
Simple aggregation of the individual features fails; e.g.~summation of the individual features $\kernelfeatures(\bx_i)$ would lose order information, vectorization $\pars{\kernelfeatures(\bx_0), \ldots, \kernelfeatures(\bx_{\len{\bx}})} \in \Hil^{\len{\bx}}$ makes sequences of different length not comparable.
It turns out that multiplication is well-suited for this task in a suitable algebra.

Fortunately, there is a natural way to embed any Hilbert space $\Hil$ into a larger Hilbert space $\HilT$ that is also a non-commutative algebra, see Sec.~\ref{sec:back:tensalgs}.
First, we take the 1\textsuperscript{st} differences,
\begin{align} \label{eq:rfsf:seq_diff}
\bx \mapsto \delta \kernelfeatures(\bx) = \pars{\kernelfeatures(\bx_{i}) - \kernelfeatures(\bx_{i-1})}_{i=1}^{\len{\bx}} \in \Hil^{\len{\bx}},\quad \text{where} \spc \bx \in \Seq(\cX)
\end{align}
since it is more natural to keep track of changes rather than absolute values.
Then we identify $\Hil$ as a subset of $\HilT$. 
The simplest choice given the above construction of $\HilT$ is
\begin{align} \label{eq:rfsf:embed_1px}
\iota: \bh \mapsto (1, \bh, \mathbf{0}, \mathbf{0}, \ldots) \in \HilT \quad \text{where} \spc \bh \in \Hil.
\end{align}

A direct calculation shows that composing the maps \eqref{eq:rfsf:seq_diff}, \eqref{eq:rfsf:embed_1px}, and multiplying the individual entries in $\HilT$ results in a sequence summary using all non-contiguous subsequences, since in each multiplication step a sequence entry is either selected once or not at all.
This gives rise to the discretized signatures $\signature: \Seq(\cX) \to \HilT$ for $\bx \in \Seq(\cX)$ with $\len{\bx} \geq 2$:
\begin{align} \label{eq:rfsf:sig_feat}
    \signature(\bx) = \prod_{i=1}^{\len{\bx}} \iota(\delta\kernelfeatures(\bx_i)) = \pars{\sum_{\bi \in \Delta_m(\len{\bx})} \delta \kernelfeatures(\bx_{i_1}) \otimes \cdots \otimes \delta \kernelfeatures(\bx_{i_m})}_{m \geq 0},
\end{align}
where $\Delta_m: \bbZ_+ \to \bbZ_+^m$
is as defined in \eqref{eq:rfsf:delta_m_def} and $\bi = (i_1, \dots, i_m)$. Thus, the sequence feature is itself a sequence, however, now a sequence of tensors indexed by their degree $m \in \bbN$ in contrast to being indexed by the time index $i \in [\len{\bx}]$.
These sequence features are invariant to a natural transformation of time series called time-warping, but can also be made sensitive to it by including time as an extra coordinate with the mapping $\bx = (\bx_i)_{i=1}^{\len{\bx}} \mapsto (t_i, \bx_i)_{i=1}^{\len{\bx}}$. It also possesses similar approximation properties to path signatures in Section~\eqref{subsec:back:props}, i.e.~uniform approximation of functions of sequences on compact sets; see Appendices A and B in \cite{toth2021seq2tens}.

Despite the abstract derivation, the resulting feature map $\signature$ is---in principle---explicitly computable when $\Hil=\bbR^d$; see \cite{kidger2021signatory} for details.
However, when the static feature map $\kernelfeatures$ is high- or infinite-dimensional, this is not feasible and we discuss a kernel trick further below.

\begin{remark} \label{remark:truncation}
We used the map $\iota$, as defined in \eqref{eq:rfsf:embed_1px}, to embed $\Hil$ into $\HilT$.
Other choices are possible, for example one could use the embedding $\hat\iota: \Hil \to \HilT$ for $\bh \in \Hil$
\begin{align}\label{eq:rfsf:embed_exp}
  \hat\iota(\bh) = \pars{1, \bh, \frac{\bh^{\otimes 2}}{2!}, \frac{\bh^{\otimes 3}}{3!},\ldots} {\in\HilT}.
\end{align}
This embedding is actually the classical choice in mathematics, but different choices of the embedding lead to, besides potential improvements in benchmarks, mildly different computational complexities and interesting algebraic questions  \cite{diehl2023generalized,toth2021seq2tens, toth2022capturing}.
\end{remark}

Finally, it can be useful to only consider the first $M \in \bbZ_+$ tensors in the series $\signature(\bx)$ analogously to using the first $M$ moments in classic polynomial regression to avoid overfitting. Hence, we define the $M$-truncated signature features for $M \in \bbZ_+$ as
\begin{align}\label{eq:rfsf:sig trunc}
 \signature[\leq M](\bx) = \pars{1,\signature[1](\bx),\ldots,\signature[M](\bx), \mathbf{0}, \mathbf{0}, \dots } \quad \text{for} \spc \bx \in \Seq(\cX),
\end{align}
where $\signature[m](\bx)$ is the projection of $\signature(\bx)$ onto $\Hil^{\otimes m}$.
In practice, we regard $M \in \bbZ_+$, and the choice of the embedding as hyperparameters to optimize. From now on, we will focus on the order-$1$ case, but we note that all algorithms can be generalized fully to the higher-order case.

\paragraph{Signature Kernels}
The signature is a powerful feature set for nonlinear regression on paths and sequences. 
A computational bottleneck associated with it is the dimensionality of the feature space $\HilT$. As we are dealing with tensors, for $\Hil$ finite-dimensional $\signature[m](\bx)$ is a tensor of degree-$m$ which has $\pars{\dim \Hil}^{m}$ coordinates that need to be computed. This can quickly become computationally expensive. For infinite-dimensional $\Hil$, e.g.~when $\Hil$ is a reproducing kernel Hilbert space (\RKHS), which is one of the most interesting settings due to the modelling flexibility, it is infeasible to directly compute $\signature$. 
In Section \ref{sec:back:discrete}, we discuss a kernel trick allows to compute the inner product of discretized signature features \eqref{eq:rfsf:sig_feat} up to a given degree $M \in \bbZ_+$ using dynamic programming, even when $\Hil$ is infinite-dimensional.
As mentioned before, here we focus on discrete-time, and our starting point is the approach of order-$1$ discretized signature kernel, which is often referred to as a non-geometric approximation to the signature kernel \cite{diehl2023generalized}. For the continuous-time treatment of signature kernels, see Section \ref{sec:back:sigkernels}.

Above we described a generic way to turn a static feature map $\kernelfeatures: \cX \to \Hil$ into a feature map $\signature[\leq M](\bx)$ for sequences, see \eqref{eq:rfsf:sig_feat}. The signature kernel is a powerful formalism that allows to transform any static kernel on $\cX$ into a kernel for sequences that evolve in $\cX$. Let $\kernel: \cX \times \cX \to \bbR$ be a continuous and bounded kernel, and from now on, let $\Hil$ denote its \RKHS, and $\kernelfeatures(\bx) = \kernel_\bx \equiv \kernel(\bx, \cdot)$ the associated reproducing kernel lift for $\bx \in \cX$. We define the $M$-truncated (discretized) signature kernel $\sigkernel[\leq M]: \Seq(\cX) \times \Seq(\cX) \to \bbR$ for $M \in \bbZ_+$ as
\begin{align}
    \sigkernel[\leq M](\bx, \by)
    &=
    \inner{\signature[\leq M](\bx)}{\signature[\leq M](\by)}_{\HilT}
    =
    \sum_{m=0}^M \inner{\signature[m](\bx)}{\signature[m](\by)}_{\Hil^{\otimes m}}
    \\
    &=
    \sum_{m=0}^M \sigkernel[m](\bx, \by)
    =
    \sum_{m=0}^M \sum_{\substack{\bi \in \Delta_m(\len{\bx})\\\bj \in \Delta_m(\len{\by})}} \delta_{i_1, j_1} \kernel(\bx_{i_1}, \by_{j_1}) \cdots \delta_{i_m, j_m} \kernel(\bx_{i_m}, \by_{j_m}), \label{eq:rfsf:sigkernel_def}
\end{align}
where we defined the level-$m$ (discretized) signature kernel $\sigkernel[m]: \Seq(\cX) \times \Seq(\cX) \to \bbR$ for $m \in [M]$ as $\sigkernel[m](\bx, \by) = \inner{\signature[m](\bx)}{\signature[m](\by)}_{\Hil^{\otimes m}}$, and $\delta_{i,j}$ denotes 2\textsuperscript{nd}-order cross-differencing such that $\delta_{i, j} \kernel (\bx_i, \by_j) = \kernel(\bx_{i+1}, \by_{j+1}) - \kernel(\bx_{i+1}, \by_{j}) - \kernel(\bx_{i}, \by_{j+1}) + \kernel(\bx_{i}, \by_{j})$ for $i \in [\len{\bx}]$ and $j \in [\len{\by}]$.

Although \eqref{eq:rfsf:sigkernel_def} looks expensive to compute, \cite{kiraly2019kernels} applies dynamic programming to efficiently compute $\sigkernel[\leq M]$ using a recursive algorithm. Importantly, \eqref{eq:rfsf:sigkernel_def} avoids computing tensors, and only depends on the entry-wise evaluations of the static kernel $\kernel(\bx_i, \by_j)$. Indeed, this leads to a computational cost of $O((M + d) \len{\bx} \len{\by})$, which is feasible for sequences evolving in high-dimensional state-spaces, but only with moderate sequence length. Note that the same bottleneck applies to PDE-based approaches. In part, the aim of this article is to alleviate this quadratic cost in sequence length, while approximately enjoying the modelling capability of working within an infinite-dimensional \RKHS. 

\subsection{Random Fourier Features}

Kernel methods allow to implicitly use an infinite-dimensional feature map $\kernelfeatures:\cX \to \Hil$ by evaluation of the inner product $\kernel(\bx,\by)= \langle \kernelfeatures(\bx), \kernelfeatures(\by) \rangle_\Hil$, when $\cH$ is a Hilbert space. This inner product can often be evaluated without direct computation of $\kernelfeatures(\bx)$ and $\kernelfeatures(\by)$ via the kernel trick. Although this makes them a powerful tool due to the resulting flexibility, the price of this flexibility is a trade-off in complexity with respect to the number of samples $N \in \bbZ_+$. Disregarding the price of evaluating the kernel $\kernel(\bx, \by)$ momentarily, kernel methods require the computation of a Gram matrix with $O(N^2)$ entries, that further incurs an $O(N^3)$ computational cost by most downstream algorithms, such as \texttt{KRR} \cite{shawe2004kernel}, \texttt{GP} \cite{rasmussen2006}, and \texttt{SVM} \cite{scholkopf2002learning}. Several techniques reduce this complexity, and the focal point of this article is the Random Fourier Feature (\RFF) technique of \cite{rahimi2007random, rahimi2008uniform, rahimi2008weighted}, which can be applied to any continuous, bounded, translation-invariant kernel on $\bbR^{d}$.\footnote{A kernel is called translation-invariant if $\kernel(\bx, \by) = \kernel(\bx+\bz, \by+\bz)$ for any $\bx, \by, \bz \in \bbR^{d}$.} Throughout, we write with some abuse of notation $\kernel(\bx-\by) \equiv \kernel(\bx, \by)$. 

A corollary of Bochner's theorem \cite{rudin2017fourier} is that any continuous, bounded, and translation-invariant kernel $\kernel: \bbR^{d} \times \bbR^{d} \to \bbR$ can be represented as the Fourier transform of a non-negative finite measure $\Lambda$ called the spectral measure associated to $\kernel$, i.e. for $\bx, \by \in \cX$
\begin{align} \label{eq:rfsf:bochner}
    \kernel(\bx - \by) = \int_{\bbR^d} \exp(i \bw^\top(\bx - \by)) \d\Lambda(\bw).
\end{align}
We may, without loss of generality, assume that $\Lambda$ is a probability measure such that $\Lambda(\bbR^{d}) = 0$, which amounts to working with the kernel $\kernel(\bx - \by) / \kernel(\mathbf{0})$. \cite{rahimi2007random} proposed to draw $\dimRFF \in \bbZ_+$ $\iid$ 
samples from $\Lambda$, $\bw_1,\ldots,\bw_{\dimRFF} \stackrel{\iid}{\sim} \Lambda$, to define the random feature map for $\bx \in \cX$ by
\begin{align} \label{eq:rfsf:rff_def}
    \rff: \cX \to \HilRFF = \bbR^{2\dimRFF}, \quad \rff(\bx) = \frac{1}{\sqrt{\tilde d}}\pars{\cos\pars{\bW^\top \bx}, \sin{\pars{\bW^\top \bx}}},     
\end{align}
where $\bW = \pars{\bw_i}_{i=1}^{\dimRFF} \in \bbR^{d \times \dimRFF}$. Then, the corresponding random kernel is defined for $\bx, \by \in \cX$ as 
\begin{align} \label{eq:rfsf:rffkernel_def}
    \rffkernel: \cX \times \cX \to \bbR, \quad \rffkernel(\bx, \by) = \inner{\rff(\bx)}{\rff(\by)}_{\HilRFF} = \frac{1}{\dimRFF} \sum_{i=1}^{\dimRFF} \cos\pars{\bw_i^\top(\bx - \by)} 
\end{align}
to provide a probabilistic approximation to $\kernel$. Indeed, it is a straightforward exercise to check that $\kernel(\bx, \by) = \expe{\rffkernel(\bx, \by)} \approx\rffkernel(\bx, \by)$.
This approximation converges exponentially fast in $\tilde d$ and uniformly over compact subsets of $\bbR^d$ as proven in \cite[Claim~1]{rahimi2007random}. This bound was later tightened and extended to the derivatives of the kernel in the series of works \cite{sriperumbudur2015optimal,szabo2019kernel, chamakh2020orlicz}, and we provide an adapted version under Theorem~\ref{thm:rff_derivative_approx} in the supplement. 

\section{Methodology} \label{sec:rfsf:RFSF}
The goal of this section is to build random features for sequences, that enjoy the benefit of linear sequence length and low-dimensional feature complexity with theoretical guarantees that the corresponding inner product is close to the $M$-truncated (discretized) signature kernel $\sigkernel[\leq M]$ with high probability. We construct these random features in a two step process: firstly, we reduce the feature space from infinite to finite (but high) dimensionality through a careful construction using random Fourier features (\RFF s), and in the second step we apply further dimensionality reduction to reduce the complexity to an even lower dimensional space in order to aid in scalability. Although we present this construction as distinct steps, the steps are coupled during the computation, and the features can be computed directly without the initial step.

\subsection{Random Fourier Signature Features} In Section~\ref{sec:rfsf:prereq}, we recalled the \RFF{} construction, which associates to a continuous, bounded, translation-invariant kernel $\kernel: \cX \times \cX \to \bbR$ on $\cX$ a spectral measure $\Lambda$, and approximates $\kernel$ by drawing samples from $\Lambda$ to define the random features $\rff: \cX \to \HilRFF$ \eqref{eq:rfsf:rff_def}, and the random kernel $\rffkernel: \cX \times \cX \to \bbR$ \eqref{eq:rfsf:rffkernel_def}. Afterwards, we presented a generic way to turn any such static features $\rff: \cX \to \HilRFF$ for elements of $\cX$ into sequence features for sequences that evolve in $\cX$ via $\signature[\leq M]: \Seq(\cX) \to \HilT$. Applying this construction with the \RFF{} as feature map on $\cX$ would already result in a random feature map for sequences, i.e.~a map from $\Seq(\cX)$ into $\HilRFFT$.
Taking the inner product in $\HilRFFT$ of this new random feature map for sequences would, however, only yield a biased estimator for the truncated signature kernel $\sigkernel[\leq M]$.
We correct for this bias by revisiting our previous construction, and build an unbiased approximation to $\sigkernel[\leq M]$ using independent \RFF{} copies in each tensor multiplication step. Then, we show in Theorem \ref{thm:main} that this random estimator comes with good probabilistic guarantees.

The probabilistic construction procedure is outlined in the following definition.

\begin{definition} \label{def:rffsig_def}
Let $\bW^{(1)}, \dots, \bW^{(M)} \stackrel{\iid}{\sim} \Lambda^{\dimRFF}$ be $\iid$ random matrices sampled from $\Lambda^{\dimRFF}$ for \RFF{} dimension $\dimRFF \in \bbZ_+$, and define the independent \RFF{} maps $\rff_m: \cX \to \HilRFF$ as in \eqref{eq:rfsf:rff_def}, i.e.~$\rff_m(\bx) = \frac{1}{\sqrt{\dimRFF}}\pars{\cos({{\bW^{(m)}}^\top} \bx), \sin({\bW^{(m)}}^\top \bx)}$ for $m \in \bracks{M}$ and $\bx \in \cX$. The $M$-truncated Random Fourier Signature Feature (\RFSF) map $\rffsig[\leq M]: \Seq(\cX) \to \HilRFFT$ from sequences in $\cX$ into the free algebra over $\HilRFF$ is defined for truncation level $M \in \bbZ_+$ and $\bx \in \Seq(\cX)$ as
\begin{align}
\rffsig[\leq M](\bx) = \pars{\sum_{\bi \in \Delta_m(\len{\bx})} \delta \rff_1(\bx_{i_1}) \otimes \cdots \otimes \delta \rff_m(\bx_{i_m})}_{m=0}^M.
\label{eq:rfsf:rffsigdef}
\end{align}
Further, the \RFSF{} kernel $\rffsigkernel[\leq M]: \Seq(\cX) \times \Seq(\cX) \to \bbR$ can be computed for $\bx, \by \in \Seq(\cX)$ as
\begin{align}
    \rffsigkernel[\leq M](\bx,\by) &= \inner{\rffsig[\leq M](\bx)}{\rffsig[\leq M](\by)}_{\HilRFFT}
    =
    \sum_{m=0}^M \inner{\rffsig[m](\bx)}{\rffsig[m](\by)}_{\HilRFF^{\otimes m}}
    \\
    &=
    \sum_{m=0}^M \rffsigkernel[m](\bx, \by)
    =
    \sum_{m=0}^M \sum_{\substack{\bi \in \Delta_m(\len{\bx})\\\bj \in \Delta_m(\len{\by})}} \delta_{i_1, j_1} \tilde\kernel_1(\bx_{i_1}, \by_{j_1}) \cdots \delta_{i_m, j_m} \tilde\kernel_m(\bx_{i_m}, \by_{j_m}), \label{eq:rfsf:rffsigkernel_def}
\end{align}
where we defined the level-$m$ \RFSF{} kernel $\rffsigkernel[m]: \Seq(\cX) \times \Seq(\cX) \to \bbR$ for $m \in \bbN$ as $\rffsigkernel[m](\bx, \by) = \inner{\rffsig[m](\bx)}{\rffsig[m](\by)}_{\HilRFF^{\otimes m}}$ with the convention that $\rffsigkernel[0] \equiv 1$, and $\rffkernel_1, \dots, \rffkernel_M: \cX \times \cX \to \bbR$ are independent \RFF{} kernels defined as in \eqref{eq:rfsf:rffkernel_def} with the random weights $\bW^{(1)}, \dots, \bW^{(M)} \in \bbR^{d \times \dimRFF}$.
\end{definition}

Since the feature map $\rffsig[\leq M]$ can be directly evaluated in the feature space recursively, $\rffsigkernel[\leq M]$ has linear complexity in the sequence length. However, it requires computing high-dimensional tensors, where the degree-$m$ component $\rffsig[m](\bx) \in \HilRFF^{\otimes m}$ has $(\dim \HilRFF)^m = (2\dimRFF)^m$ coordinates, making it infeasible for large $m, \dimRFF \in \bbZ_+$. Remark \ref{remark:alg_RFSF} discusses the computational complexity in detail. Further, note that the kernel can be evaluated by means of a kernel trick exactly analogously to the evaluation of \eqref{eq:rfsf:sigkernel_def}, but in this case there are no computational gains compared to the infinite-dimensional signature kernel $\sigkernel[\leq M](\bx, \by)$.

Next, we provide a theoretical analysis to show that the random kernel $\rffsigkernel[\leq M](\bx, \by)$ converges to the ground truth signature kernel $\sigkernel[\leq M](\bx, \by)$ exponentially fast and uniformly over compact state-spaces $\cX \subseteq \bbR^d$, generalizing the result \cite[Claim~2]{rahimi2007random} to this non-Euclidean domain of sequences.
Throughout the analysis, we need certain regularity properties of $\Lambda$ in order to invoke quantitative versions of the law of large numbers, i.e.~properties such as boundedness, existence of the moment-generating function, moment-boundedness, or belonging to certain Orlicz spaces of random variables. Boundedness of the spectral measure is too restrictive an assumption, since a continuous, bounded, translation-invariant kernel $\kernel: \cX \times \cX \to \bbR$ is characteristic if and only if the support of its spectral measure is $\bbR^d$, see \cite[Prop.~8]{sriperumbudur2010relation}. Hence, we instead work with the assumption that its moments are well-controllable, i.e.~the tails of the distribution are not ``too heavy''. Specifically, we assume the Bernstein moment condition that
\begin{align} \label{eq:rfsf:w_cond_main}
    \bbE_{\bw \sim \Lambda}\bracks{w_i^{2m}} \leq \frac{m! S^2 R^{m-2}}{2} \quad \text{for all} \spc i \in [d]
\end{align}
for some $S, R > 0$. We show in the Supplementary Material under Lemmas \ref{lem:alpha_bernstein_cond} and \ref{lem:alpha_exp_norm_to_bernstein}, in a more general context, that this is equivalent to $\Lambda$ being a sub-Gaussian probability measure; see e.g.~\cite[Sec~2.3]{boucheron2013concentration} and \cite[Sec.~2.5]{vershynin2018high} about sub-Gaussianity. This of course includes the spectral measure of the Gaussian kernel defined for bandwidth $\sigma > 0$ and $\bx, \by \in \cX$ $\kernel(\bx, \by) = \exp\pars{-\nicefrac{\norm{\bx - \by}_2^2}{2\sigma^2}}$, which has a Gaussian spectral distribution $\bw \sim \cN\pars{0, \nicefrac{1}{\sigma^2} \b I_d}$, and therefore calculation gives $\bbE_{w \sim \cN\pars{0, \nicefrac{1}{\sigma^2}}}\bracks{w^{2m}} = \frac{2^m \Gamma\pars{m + \frac{1}{2}})}{ \sigma^{2m}\sqrt{\pi}} < \frac{m!}{2} \pars{\frac{2\sqrt{2}}{\sigma^2 \sqrt[4]{\pi}}}^2\pars{\frac{2}{\sigma^2}}^{m-2}$,
since $\Gamma\pars{m + \nicefrac{1}{2}} < \Gamma(m + 1) = m!$. Hence $\Lambda$ satisfies condition \eqref{eq:rfsf:w_cond_main} with $S, R$ as given here.
Now we state our approximation theorem regarding $\rffsigkernel[m]$, which quantifies that it is a (sub-)exponentially good estimator of $\sigkernel[m]$ with high probability and uniformly.

\begin{theorem} \label{thm:main}
    Let $\kernel: \bbR^d \times \bbR^d \to \bbR$ be a continuous, bounded, translation-invariant
    kernel with spectral measure $\Lambda$, which satisfies \eqref{eq:rfsf:w_cond_main}.
    Let $\cX \subset \bbR^d$ be compact and convex with diameter $\abs{\cX}$, $\cX_\Delta = \{\bx - \by : \bx, \by \in \cX \}$. Then, the following quantities are finite: $\sigma_\Lambda^2 = \bbE_{\bw \sim \Lambda}\bracks{\norm{\bw}_2^2}$, $L = \norm{\bbE_{\bw \sim \Lambda}\bracks{\bw \bw^\top}}_2^{1/2}$, $E_{i,j} = \bbE_{{\bw \sim \Lambda}}\bracks{\abs{w_i w_j} \norm{\bw}_2}$ and $D_{i, j} = \sup_{\bz \in \cX_\Delta} \norm{\nabla \bracks{\frac{\partial^2\kernel(\bz)}{\partial z_i \partial z_j}}}_2$ for $i, j \in [d]$. Further, for any max.~sequence $1$-var $V>0$, and signature level $m \in \mathbb{Z}_+$, for $\epsilon > 0$
    \begin{align}
        \bbP & \bracks{\sup_{\substack{\bx, \by \in \Seq(\cX) \\ \norm{\bx}_\onevar, \norm{\by}_\onevar \leq V}} \abs{\sigkernel[m](\bx, \by) - \rffsigkernel[m](\bx, \by)} \geq \epsilon } \le
        \\
        &\leq
        m
        \begin{cases}
        \pars{C_{d, \cX} \pars{\frac{\beta_{d, m, V}}{\epsilon}}^\frac{d}{d+1} + d}
        \exp\pars{-\frac{\dimRFF}{2(d+1)(S^2 + R)} \pars{\frac{\epsilon}{\beta_{d, m, V}}}^{2}} \,&\text{ for }\, \epsilon < \beta_{d, m, V}  \\
        \pars{C_{d, \cX} \pars{\frac{\beta_{d, m, V}}{\epsilon}}^{\frac{d}{(d+1)m}} + d}
        \exp\pars{-\frac{\dimRFF}{2(d+1)(S^2 + R)} \pars{\frac{\epsilon}{\beta_{d, m,v}}}^{\frac{1}{m}} } \,&\text{ for }\, \epsilon \geq \beta_{d, m, V},
        \end{cases}
        \label{eq:rfsf:mainthm_bound}
    \end{align}
    where $C_{d, \cX} = 2^\frac{1}{d+1} 16 \abs{\cX}^\frac{d}{d+1} \sum_{i,j=1}^d (D_{i,j} + E_{i,j})^\frac{d}{d+1}$ and $\beta_{d, m, V} = m \pars{2 V^{2} \pars{L^2 \vee 1} \pars{\sigma_\Lambda^2 \vee d}}^m$.
\end{theorem}

The proof is provided in the supplement under Theorem \ref{thm:rfsf_approx}. The result shows that the random kernel $\rffsigkernel[m]$ approximates the signature kernel $\sigkernel[m]$ uniformly over subsets of $\Seq(\cX)$ of sequences $\bx \in \Seq(\cX)$ with maximal $1$-variation $V$, $\norm{\bx}_\onevar \leq V$, assuming that the state-space $\cX \subset \bbR^d$ is a convex and compact domain.
The error bound is analogous to the classic \RFF{} bounds, in the sense that the tail probability decreases exponentially fast as a function of the \RFF{} dimension $\dimRFF$. The functional form of the bound is inherited from Theorem \ref{thm:rff_derivative_approx}, which provides an analogous result for the derivatives of \RFF. This link follows from Lemma \ref{lem:RFSF_approx}, which connects the concentration of the \RFSF{} kernel to the second derivatives of \RFF.

The main difference from the classic case, i.e.~\cite[Claim~1]{rahimi2007random} and Theorem \ref{thm:rff_derivative_approx}, is the appearance of $\beta_{d,m,V}$ which controls a regime change in the tail behaviour. Concretely, for $\epsilon < \beta_{d, m, V}$ \eqref{eq:rfsf:mainthm_bound} has a polynomial plus a sub-Gaussian tail, while for $\epsilon > \beta_{d, m, V}$ has a $\pars{\nicefrac{1}{m}}$-subexponential tail.
This is not surprising as the inner summand in \eqref{eq:rfsf:rffsigkernel_def} is the $m$-fold tensor product of $m$ independent \RFF{} kernels, which makes the tail heavier exactly by an exponent of $\nicefrac{1}{m}$.
The constant itself, $\beta_{d,m,V}$, depends on \begin{enumerate*}[label=(\roman*)] \item the maximal sequence 1-variation $V$, which measures a notion of time-warping invariant sequence complexity;
\item the Lipschitz constant of the kernel $L$ (see Examples \ref{example:2ndmoment_lip} and \ref{example:rff_lip}); \item the trace of the second moment of $\Lambda$, $\sigma_\Lambda^2 = \bbE_{\bw \sim \Lambda}\bracks{\norm{\bw}_2^2}$; \item the state-space dimension $d$; \item and the signature level $m$ itself. \end{enumerate*} 

\begin{remark} \label{remark:alg_RFSF}
    Algorithm \ref{alg:rfsf:rfsf} demonstrates the computation of the \RFSF{} map $\rffsig[\leq M]$ given a dataset of sequences $\bX = (\bx_i)_{i=1}^{N} \subset \Seq(\cX)$. Upon inspection, we can deduce that the algorithm has a computational complexity of $O\pars{N \ell (M d \dimRFF^M)}$. Importantly, it is linear in $\ell$, the sequence length, although scales polynomially in the \RFF{} sample size $\dimRFF^M$, and exponentially in the truncation-$M$.
\end{remark}

\subsection{Dimensionality Reduction: Diagonal Projection} Previously, we introduced a featurized approximation $\rffsigkernel[\leq M]$ to the signature kernel $\sigkernel[\leq M]$, called the \RFSF{} kernel, which reduces the computation from the infinite-dimensional \RKHS{} to a finite-dimensional feature space using random tensors. Although this makes the computation in the feature space viable of the \RFSF{} map $\rffsig[\leq M]$, it is still tensor-valued, which incurs a computational cost of $O(\dimRFF + \dimRFF^2 + \cdots + \dimRFF^m)$ in the \RFF{} dimension $\dimRFF \in \bbZ_+$. Now, we take another step towards scalability and apply further dimensionality reduction. By examining the structure of these tensors, we introduce a diagonally projected variant called \RFSFD{} that considerably reduces their sizes. We emphasize that the above \RFSF{} construction is crucial: it approximates the inner product in an infinite-dimensional space, and now we further approximate it in an even lower dimensional space. The benefit is that one does not have to go through the computation of the initial \RFSF{} map, but only the selected degrees of freedom have to be computed.

As a first observation, we notice that the computation of \eqref{eq:rfsf:rffsigkernel_def} can be reformulated, due to \eqref{eq:rfsf:rffkernel_def} and linearity of the differencing operator, in the following way:
\begin{align}
    \rffsigkernel[m](\bx, \by)
    =
    \frac{1}{\dimRFF^m} \sum_{q_1, \dots, q_m = 1}^{\dimRFF}  \sum_{\substack{\bi \in \Delta_m(\len{\bx}) \\ \bj \in \Delta_m(\len{\by})}} \prod_{p=1}^m \delta_{i_p, j_p} \cos\pars{{\bw_{q_p}^{(p)}}^\top (\bx_{i_p} - \by_{j_p})} \label{eq:rfsf:rffsigkernel_explicit}
\end{align}
by spelling out the definition of the \RFF{} kernel, where $\bw_{1}^{(1)}, \dots, \bw_{\dimRFF}^{(m)} \stackrel{\iid}{\sim} \Lambda$, such that $\bW^{(p)} = \pars{\bw_1^{(p)}, \dots, \bw_{\dimRFF}^{(p)}} \in \bbR^{d \times \dimRFF}$ as defined in Def.~\ref{def:rffsig_def}. Now, we may observe that there is a dependency structure among the samples being averaged in \eqref{eq:rfsf:rffsigkernel_explicit}, since the outer summation is over the Cartesian product $(q_1, \dots, q_m) \in [\dimRFF]^{\times m}$, which suggests that we might be able to drastically reduce the degrees of freedom by restricting this summation to only go over an independent set of samples.
One way to do this is to restrict to multi-indices of the form $\cI = \curls{(q, \dots, q) \in [\dimRFF]^{\times m} \given q \in [\dimRFF]}$, i.e.~we diagonally project the index set, motivating the name of the approach stated in the following definition.
\begin{definition} \label{def:rffsigdp}
    Let $\bw^{(1)}_1, \dots, \bw^{(M)}_{\dimRFF} \stackrel{\iid}{\sim} \Lambda$ for $\dimRFF \in \bbZ_+$, and define $\hat\kernelfeatures_{m,q}: \cX \to \hat\Hil = \bbR^2$ with sample size $\hat d = 1$ for $q \in [\dimRFF]$ and $m \in [M]$, such that $\hat\kernelfeatures_{m,q}(\bx) = \pars{\cos({\bw^{(m)}_q}^\top \bx), \sin({\bw^{(m)}_q}^\top \bx)}$ for $\bx \in \cX$. The $M$-truncated Diagonally Projected Random Fourier Signature Feature (\RFSFD) map $\rffsigDP[\leq M]: \Seq(\cX) \to \HilRFFTDP = \bigoplus_{m = 0}^M \pars{{\hat\Hil}^{\otimes m}}^{\dimRFF}$ is defined for truncation $M \in \bbZ_+$ and $\bx \in \Seq(\cX)$ as
    \begin{align}
        \rffsigDP[\leq M](\bx) = \frac{1}{\sqrt{\dimRFF}} \pars{\pars{\sum_{\bi \in \Delta_m(\len{\bx})} \delta \hat\kernelfeatures_{1,q}(\bx_{i_1}) \otimes \cdots \otimes \delta \hat\kernelfeatures_{m,q}(\bx_{i_m})}_{q=1}^{\dimRFF}}_{m=0}^M.
    \end{align}
Then, the \RFSFD{} kernel can be directly computed for $\bx, \by \in \Seq(\cX)$ via
\begin{align}
    \rffsigkernelDP[\leq M](\bx, \by)
    &=
    \inner{\rffsigDP[\leq M](\bx)}{\rffsigDP[\leq M](\by)}_{\HilRFFTDP} 
    =
    \sum_{m=0}^M \inner{\rffsigDP[m](\bx)}{\rffsigDP[m](\by)}_{\pars{{\hat\Hil}^{\otimes m}}^{\dimRFF}}
    \\
    &=
    \sum_{m=0}^M \rffsigkernelDP[m](\bx, \by)
    =
    \frac{1}{\dimRFF} \sum_{m=0}^M \sum_{q=1}^{\dimRFF}  \mathrlap{\sum_{\substack{\bi \in \Delta_m(\ell_{\bx}-1)\\\bj \in \Delta_m(\ell_{\by}-1)}}} \hspace{38pt} \delta_{i_1, j_1} \hat\kernel_{1,q}(\bx_{i_1}, \by_{j_1}) \cdots \delta_{i_m, j_m} \hat\kernel_{m,q}(\bx_{i_m}, \by_{j_m}), \label{eq:rfsf:rffsigdpkernel_def}
\end{align}
    where we defined the level-$m$ \RFSFD{} kernel $\rffsigkernelDP[m]: \Seq(\cX) \times \Seq(\cX) \to \bbR$ for $m \in \bbN$ and $\bx, \by \in \Seq(\cX)$ as $\rffsigkernel(\bx, \by) = \inner{\rffsigDP[m](\bx)}{\rffsigDP[m](\by)}_{\pars{{\hat\Hil}^{\otimes m}}^{\dimRFF}}$ with the convention that $\rffsigkernelDP[0] \equiv 1$, and $\hat\kernel_{m,q}: \cX \times \cX \to \bbR$ are independent \RFF{} kernels with sample size $\hat d = 1$ defined for $\bx, \by \in \cX$ as $\hat\kernel_{m,q}(\bx, \by) = \inner{\hat\kernelfeatures_{m,q}(\bx)}{\hat\kernelfeatures_{m,q}(\by)}_{\hat\Hil}$ with the random weights $\bw^{(m)}_q \in \bbR^d$ for $q \in [\dimRFF], m \in [M]$.
\end{definition}

Note that by the definition of the \RFF{} kernels in \eqref{eq:rfsf:rffsigdpkernel_def}, we may substitute that $\hat\kernel_{p,q}(\bx, \by) = \cos({\bw^{(p)}_q}^\top(\bx - \by))$ for $\bx, \by \in \cX$, so \eqref{eq:rfsf:rffsigdpkernel_def} is equivalently written for $\bx, \by \in \Seq(\cX)$ as
\begin{align}
    \rffsigkernelDP[m](\bx, \by) = \frac{1}{\dimRFF} \sum_{q=1}^{\dimRFF} \mathrlap{\sum_{\substack{\bi \in \Delta_m(\ell_{\bx}-1)\\\bj \in \Delta_m(\ell_{\by}-1)}}} \hspace{40pt} \delta_{i_1, j_1} \cos({\bw^{(1)}_q}^\top(\bx_{i_1} - \by_{j_1})) \cdots \delta_{i_m, j_m} \cos({\bw^{(m)}_q}^\top(\bx_{i_m} - \by_{j_m})), 
\end{align}
which is what we set out to do in the above paragraph; that is, restrict the outer summation onto the diagonal projection of the index set.
Another way to look at Definition \ref{def:rffsigdp} is that the \RFSFD{} kernel in \eqref{eq:rfsf:rffsigdpkernel_def} is constructed by defining $\dimRFF$ independent \RFSF{} kernels, each with internal \RFF{} sample size $\hat d = 1$, and then taking their average; the concatenation of their corresponding features are then the features of the \RFSFD{} map. Note that for \RFF{} sample size $1$, each \RFF{} map has dimension $2$, i.e.~$\hat\Hil = \bbR^2$, and hence, the corresponding \RFSF{} kernels have dimension $1 + 2 + \cdots + 2^M = (2^{M+1} - 1)$, which by concatenation results in the overall dimensionality of the \RFSFD{} kernel being $\dim \HilRFFTTRP = \dimRFF\pars{2^{M+1} - 1}$. This relates to the computational complexity of the \RFSFD{} map; for details see Remark \ref{remark:alg_rfsf_dp}.

Next, we state our concentration result regarding the level-$m$ \RFSFD{} kernel $\rffsigkernelDP[m](\bx, \by)$.
\begin{theorem} \label{thm:main2}
    Let $\kernel: \bbR^d \times \bbR^d \to \bbR$ be a continuous, bounded, translation-invariant kernel with spectral measure $\Lambda$, which satisfies \eqref{eq:rfsf:w_cond_main}. Then, for level $m \in \bbZ_+$, $\bx, \by \in \Seq(\cX)$, and $\epsilon > 0$ 
    \begin{align} \label{eq:rfsf:mainthm2_bound}
        \bbP\bracks{\abs{\rffsigkernelDP[m](\bx, \by) - \sigkernel[m](\bx, \by)} \geq \epsilon}
        \leq
        2\exp\pars{-\frac{1}{4}\min\curls{
        \pars{\frac{\sqrt{\dimRFF} \epsilon}{2C_{d, m, \bx, \by}}}^2, 
        \pars{\frac{\dimRFF \epsilon}{\sqrt{8}C_{d, m, \bx, \by}}}^{\frac{1}{m}}
        }},
     \end{align}
    where $L = \norm{\bbE_{\bw \sim \Lambda}\bracks{\bw \bw^\top}}$ is the Lipschitz constant of $\kernel$, and $C_{d, m, \bx, \by} > 0$ is bounded by
    \begin{align}
    C_{d, m, \bx, \by}
    \leq
    \sqrt{8} e^4 (2\pi)^{1/4} e^{1/24} (4e^3\norm{\bx}_\onevar \norm{\by}_\onevar /m)^m \pars{\pars{2d\max(S, R)}^m + \pars{L^2/\ln 2}^m}.
    \end{align}
\end{theorem}
The proof is provided in the supplement under Theorem \ref{thm:rfsf_dp_approx}. The result shows that the \RFSFD{} kernel converges for any two sequences $\bx, \by \in \Seq(\cX)$ with a $\pars{\nicefrac{1}{m}}$-subexponential convergence rate with respect to the sample size $\dimRFF \in \bbZ_+$. 
Similarly to Theorem \ref{thm:main}, the bound has a phase transition, where for small values of $\epsilon$, it has a sub-Gaussian tail, while for larger values, it has a $\pars{\nicefrac{1}{m}}$-subexponential tail. A crucial difference from the previous bound is that now the phase transition happens at $\epsilon^\star = C_{d, m, \bx, \by} 2^\frac{2m-3/2}{2m-1} \dimRFF^{\frac{1-m}{2m-1}}$, which depends on the sample size $\dimRFF$. This means that for fixed value of $\epsilon > 0$, the phase transition always happens eventually as $\dimRFF$ gets large enough, hence the convergence rate with respect to $\dimRFF$ is $\pars{\nicefrac{1}{m}}$-subexponential regardless of the value of $\epsilon$. The slightly reduced rate of convergence compared to the \RFSF{} kernel in Theorem \ref{thm:main} is to be expected, since the sample size of the \RFSFD{} kernel is analogously reduced by an exponent of $\pars{\nicefrac{1}{m}}$ with respect to $\dimRFF$ in comparison.
The constant $C_{d,m,\bx, \by}$, similarly to \eqref{eq:rfsf:mainthm_bound}, depends on \begin{enumerate*}[label=(\roman*)] \item the 1-variation of sequences $\norm{\bx}_\onevar, \norm{\by}_\onevar$ that measure the complexity of the sequences; \item $L > 0$, the Lipschitz constant of the kernel $\kernel$ (see Examples \ref{example:2ndmoment_lip},  \ref{example:rff_lip}); \item the moment bound parameters $S, R > 0$ from condition \eqref{eq:rfsf:w_cond_main}; \item the state-space dimension $d$;  and \item the signature level $m$. \end{enumerate*}

\begin{remark} \label{remark:alg_rfsf_dp}
    Algorithm \ref{alg:rfsf:rsfsdp} demonstrates the computation of the \RFSFD{} map $\rffsigDP[\leq M]$ given a dataset of sequences $\bX = (\bx_i)_{i=1}^N \subset \Seq(\cX)$. Upon counting the operations, we deduce that the algorithm has a computational complexity $O\pars{N \ell \dimRFF (M d + 2^{M})}$. Crucially, it is linear in both $\ell$, the maximal sequence length, and $\dimRFF$, the sample size of the random kernel.
\end{remark}
\subsection{Dimensionality Reduction: Tensor Random Projection} \label{sec:rfsf:trp} Previously, we built the \RFSFD{} map by subsampling an independent set from the samples that constitute \RFSF{} kernel. Here, we propose an alternative dimensionality reduction technique that starts again from the \RFSF{} map, and uses random projections to project this generally high-dimensional tensor onto a lower dimension. Random projections are a classic technique in data science for reducing the data dimension, while preserving its important structural properties. They are built upon the celebrated Johnson-Lindenstrauss lemma \cite{johnson1986extensions}, which states that a set of points in a high-dimensional space can be embedded into a space of much lower dimension, while approximately preserving their 
geometry. Exploiting this property, we construct a tensor random projected (\TRP) variant of our random kernel called \RFSFT{}, such that the computation is coupled between the \RFSF{} and \TRP{} maps, similarly to a kernel trick.

Tensorized random projections \cite{sun2021tensor, rakhshan2020tensorized} construct random projections for tensors with concise parameterization that respects their tensorial nature. Given tensors $\bs, \bt \in \pars{\bbR^d}^{\otimes m}$ for $m \in \bbZ_+$, the \TRP{} map with \CP{} (\texttt{CANDECOMP/PARAFAC} \cite{kolda2009tensor})  rank-$1$ is built via a random functional 
$\pr: \pars{\bbR^d}^{\otimes m} \to \bbR$ such that $\pr(\bs) = \inner{\bp_1 \otimes \cdots \otimes \bp_m}{\bs}_{\pars{\bbR^d}^{\otimes m}}$, where $\bp_1, \dots, \bp_m \stackrel{\iid}{\sim} \cN(\b 0, \b I_d)$ are $d$-dimensional component vectors sampled from a standard normal distribution. Then, the inner product can be estimated as $\pr(\bs) \pr(\bt) \approx \expe{\pr(\bs) \pr(\bt)} = \inner{\bs}{\bt}_{\pars{\bbR^d}^{\otimes m}}$. Variance reduction is achieved by stacking $n \in \bbZ_+$ such random projections, each with $\iid$ component vectors $\bp_1^{(1)}, \dots, \bp_m^{(n)} \stackrel{\iid}{\sim} \cN(\b 0, \b I_d)$. Hence, the \TRP{} operator is defined as
\begin{align} \label{eq:rfsf:trp_def}
    \TRP: \pars{\bbR^d}^{\otimes m} \to \bbR^n, \quad  \TRP(\bs) = \frac{1}{\sqrt{n}}\pars{\inner{\bp_1^{(i)} \otimes \cdots \otimes \bp_m^{(i)}}{\bs}}_{i=1}^n.
\end{align}
On the one hand, this allows to represent the random projection map onto $\bbR^n$ using only $O(n m d)$ parameters as opposed to the $O(n d^m)$ parameters in a densely parametrized random projection; and on the other, it allows for downstream computations to exploit the low-rank structure of the operator, as we shall do so in the definition stated below.
\begin{definition} \label{def:rffsigtrp_def}
Let $\bW^{(1)}, \dots, \bW^{(M)} \stackrel{\iid}{\sim} \Lambda^{\dimRFF}$ be $\iid$ random matrices sampled from $\Lambda^{\dimRFF}$ for \RFF{} dimension $\dimRFF \in \bbZ_+$, define the independent \RFF{} maps $\rff_m: \cX \to \HilRFF$ as in \eqref{eq:rfsf:rff_def}, i.e.~$\rff_m(\bx) = \nicefrac{1}{\sqrt{\dimRFF}}\pars{\cos({{\bW^{(m)}}^\top} \bx), \sin({\bW^{(m)}}^\top \bx)}$ for $m \in \bracks{M}$ and $\bx \in \cX$, and let $\bP^{(1)}, \dots, \bP^{(M)} \stackrel{\iid}{\sim} \cN^{\dimRFF}\pars{\b 0, \b I_{2 \dimRFF}}$
be random matrices with $\iid$ standard normal entries. The $M$-truncated Tensor Random Projected Random Fourier Signature Feature (\RFSFT) map $\rffsigTRP[\leq M] \Seq(\cX) \to \HilRFFTTRP = \bbR^{M \dimRFF}$ is defined for truncation level $M \in \bbZ_+$ and $\bx \in \Seq(\cX)$ as
\begin{align}
    \rffsigTRP[\leq M](\bx) &=
    \frac{1}{\sqrt{\dimRFF}} \pars{\pars{\sum_{\bi \in \Delta_m(\len{\bx})} \inner{\bp^{(1)}_q}{\delta \rff_1(\bx_{i_1})} \cdots \inner{\bp^{(m)}_q}{\delta \rff_m(\bx_{i_m})}}_{q=1}^{\dimRFF}}_{m=0}^M
    \\
    &=
    \frac{1}{\sqrt{\dimRFF}} \pars{\sum_{\bi \in \Delta_m(\len{\bx})} \pars{{\bP^{(1)}}^\top \delta \rff_1(\bx_{i_1})} \odot \cdots \odot \pars{{\bP^{(m)}}^\top \delta \rff_m(\bx_{i_m})}}_{m=0}^M,
\label{eq:rfsf:rffsigtrpdef}
\end{align}
where $\bP^{(m)} = \pars{\bp_q^{(m)}}_{q=1}^{\dimRFF} \in \bbR^{2\dimRFF \times \dimRFF}$, and $\odot$ denotes the Hadamard product\footnote{The Hadamard product stands for component-wise multiplication of the vectors $\bx, \by \in \bbR^n$, $\bx \odot \by = \pars{x_i y_i}_{i=1}^n$.}.
The \RFSFT{} kernel $\rffsigkernelTRP[\leq M]: \Seq(\cX) \times \Seq(\cX) \to \bbR$ can then be directly computed for sequences $\bx, \by \in \Seq(\cX)$ by
\begin{align}
    \rffsigkernelTRP[\leq M](\bx,\by) &= \inner{\rffsigTRP[\leq M](\bx)}{\rffsigTRP[\leq M](\by)}_{\HilRFFT}
    =
    \sum_{m=0}^M \inner{\rffsigTRP[m](\bx)}{\rffsigTRP[m](\by)}_{\HilRFF^{\otimes m}}
    \\
    &=
    \sum_{m=0}^M \rffsigkernelTRP[m](\bx, \by)
    =
    \frac{1}{\dimRFF} \sum_{m=0}^M \sum_{q=1}^{\dimRFF} \sum_{\substack{\bi \in \Delta_m(\len{\bx})\\\bj \in \Delta_m(\len{\by})}} \prod_{p=1}^m \inner{\bp^{(p)}_q}{\delta \rff_p(\bx_{i_p})} \inner{\bp^{(p)}_q}{\delta \rff_p(\by_{j_p})}, \label{eq:rfsf:rffsigtrpkernel_def}
\end{align}
where we defined the level-$m$ \RFSFT{} kernel $\rffsigkernelTRP[m]: \Seq(\cX) \times \Seq(\cX) \to \bbR$ for $m \leq M$ as $\rffsigkernelTRP[m](\bx, \by) = \inner{\rffsigTRP[m](\bx)}{\rffsigTRP[m](\by)}_{\HilRFF^{\otimes m}}$ with the convention that $\rffsigkernelTRP[0] \equiv 1$.
\end{definition}

We remark that \eqref{eq:rfsf:rffsigtrpdef} is equivalent to the \TRP{} operator \eqref{eq:rfsf:trp_def} applied to the \RFSF{} map \eqref{eq:rfsf:rffsigdef} by exploiting bilinearity of the inner product, and using that it factorizes over the tensor components, as described in \eqref{eq:back:inner_tensor}. Then, the unbiasedness of \eqref{eq:rfsf:rffsigtrpkernel_def} follows from the fact that the \TRP{} operator is an isometry under expectation, which is applied to the \RFSF{} tensor $\rffsig[m]$, therefore $\rffsigkernelTRP[m]$ kernel is conditionally an unbiased estimator of $\rffsigkernel[m]$ given the \RFSF{} weights $\bW^{(1)}, \dots, \bW^{(m)} \in \bbR^{d \times \dimRFF}$.
By the tower rule for expectations, $\rffsigkernelTRP[m]$ is an unbiased estimator of $\sigkernel[m]$. The approximation quality is then governed by two factors: \begin{enumerate*}[label=(\roman*)] \item \label{it:trp1} how well the \TRP{} projected kernel $\rffsigkernelTRP[m]$ approximates $\rffsigkernel[m]$; \item \label{it:trp2} the quality of the approximation of $\rffsigkernel[m]$ with respect to $\sigkernel[m]$. \end{enumerate*} Note that \ref{it:trp2} has already been discussed in Theorem \ref{thm:main}. Here, we state the following theoretical result which quantifies \ref{it:trp1}. Combining these two results by means of triangle inequality and union bounding quantifies that $\rffsigkernelTRP[m]$ is a good estimator of $\sigkernel[m]$.

\begin{theorem} \label{thm:main3}
    Let $\kernel: \bbR^d \times \bbR^d \to \bbR$ be a continuous, bounded, translation-invariant kernel with spectral measure $\Lambda$, which satisfies \eqref{eq:rfsf:w_cond_main}. Then, the following bound holds for \RFSFT{} kernel for signature level $m \in \bbZ_+$ sequences $\bx, \by \in \Seq(\cX)$ and $\epsilon > 0$
    \begin{align} \label{eq:rfsf:mainthm3_bound}
        \bbP\bracks{\abs{\rffsigkernelTRP[m](\bx, \by) - \rffsigkernel[m](\bx, \by)} \geq \epsilon}
        \leq
        C_{d, \Lambda}
        \exp\pars{- \pars{\frac{m^2 \dimRFF^{\frac{1}{2m}} \epsilon^{\frac{1}{m}}}{2\sqrt{2}e^3 R \norm{\bx}_\onevar \norm{\by}_\onevar}}^\frac{1}{2}},
    \end{align}
    where the absolute constant is defined as $C_{d, \Lambda} = 2\pars{1 + \frac{S}{2R} + \frac{S^2}{4R^2}}^d$.
\end{theorem}

The proof is given in the supplement under Theorem \ref{thm:rfsf_trp_approx} utilizing the hypercontractivity of Gaussian polynomials \cite{janson1997gaussian} that is used to quantify the concentration of the \TRP{} estimator. The concentration of the \RFSFT{} kernel is then governed by Theorems \ref{thm:main} and \ref{thm:main3} combined. Together, they show that for smaller values of $\epsilon$ (i.e.~the regime change as discussed below Theorem \ref{thm:main}), the probability has a polynomial plus a sub-Gaussian tail, while for large $\epsilon$, it has a $\pars{\frac{1}{2m}}$-subexponential tail due to \eqref{eq:rfsf:mainthm3_bound}, and the dominant convergence rate with respect to $\dimRFF$ is $\pars{\frac{1}{4m}}$-subexponential. This means that in terms of convergence, \RFSFT{} is the slowest among the 3 variations introduced so far. However, it is also the most efficient in terms of overall dimension, hence downstream computational complexity as well, since $\HilRFFTTRP = \bbR^{M \dimRFF}$.
Remark \ref{remark:alg_rfsf_trp} discusses the computational complexity in detail.

\begin{remark} \label{remark:alg_rfsf_trp}
    Algorithm \ref{alg:rfsf:rsfstrp} demonstrates the computation of the \RFSFT{} map $\rffsigTRP[\leq M]$ given a dataset of sequences $\bX = (\bx_i)_{i=1}^N \subset \Seq(\cX)$. Counting the operations, here we can deduce that the algorithm has an $O\pars{M  N  \ell  \dimRFF  (d + \dimRFF)}$ computational complexity. This variation is also linear in $\ell$, the maximal sequence length, although it is quadratic in $\dimRFF$.
\end{remark}

\subsection{Numerical evaluation}
\begin{figure}
\centering
\includegraphics[width=0.6\textwidth]{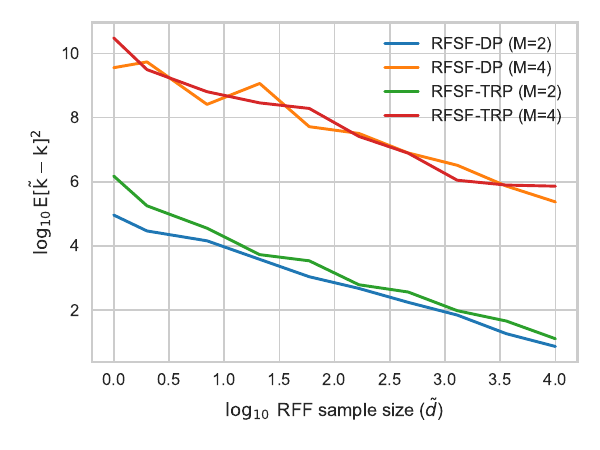}
\caption{Approximation error of random kernels against \RFF{} size on $\log$-$\log$ plot.}
\label{fig:approx}
\end{figure}
Here, we numerically evaluate the approximation error of the proposed scalable kernels, that is, \RFSFD{} and \RFSFT{}. We do not include \RFSF{} since its dimensionality shows polynomial explosion in the base sample size $\dimRFF$ due to its tensor-based representation, which makes its computation infeasible for reasonable values of $\dimRFF$. We generate $d$-dimensional synthetic time series of length-$\ell$ using a \texttt{VAR}(1) process $\tilde\bx \in \Seq(\cX)$, such that $\tilde\bx_0 = \b 0$ and $\tilde\bx_{t+1} = \nicefrac{1}{\sqrt{d}} A \tilde\bx_t + \b\epsilon_t$, where $A \sim \cN^{d \times d}(0, 1)$ and $\b\epsilon_t \sim \cN(\b 0, \sigma^2 \b I_d)$. Then, we compute the normalized version $\bx \in \Seq(\cX)$ of $\tilde\bx$, which is rescaled to have $1$-variation $V > 0$, i.e. $\bx_t = V \tilde\bx_t / \norm{\tilde\bx}_{\onevar}$. We set $d = 10$, $\ell = 100$, $\sigma = 0.1$ and $V = 100$. We compute the squared deviation between the groundtruth signature kernel and the randomized approximations for two randomly sampled time series in this way. This process is repeated for $100$ randomly sampled time series and $100$ times resampled random kernel evaluations, giving rise to overall $10000$ evaluations for each value of $\dimRFF$. Figure \ref{fig:approx} shows the average approximation error plotted against values of $\dimRFF$ on a $\log$-$\log$ plot. We can observe that both \RFSFD{} and \RFSFT{} have approximately the same error curves for a given value of truncation level $M$, and the steepness appears to be the same across different levels of $M$. This means that \RFSFT{} is slightly more efficient in terms of dimensionality, since its dimension is $M\dimRFF$ as opposed to $2^{M+1} \dimRFF$ in \RFSFD{}. We also observe that both curves are close to being linear, which indicates that the approximation error scales approximately as $O(\dimRFF^{-\alpha})$ for some value of $\alpha > 0$.


\section{Experiments}\label{sec:rfsf:experiments}

\subsection{Time series classification}
We perform multivariate time series classification to investigate the performance of the scalable \RFSF{} variants compared to the full-rank signature kernel and other quadratic time baseline kernels, and further, to demonstrate the scalability to large-scale datasets, where the quadratic sample complexity becomes prohibitive.
We use support vector machine (\SVM) \cite{steinwart2008support} classification for classifying multivariate time series on datasets of various sizes. For quadratic time kernels, the dual \SVM{} formulation is used, while for kernels with feature representations, we use the primal formulation that has linear complexity in the size of the dataset $n \in \bbZ_+$ aiding in scalability to truly large-scale datasets. For each considered kernel/feature, we use a \texttt{GPU}-based implementation provided in the \KS{} library\footnote{\href{https://github.com/tgcsaba/KSig}{https://github.com/tgcsaba/KSig}} \cite{toth2025ksig}. For large-scale experiments with the featurized kernels, linear \SVM{} implementation is used from the \texttt{cuML} library \cite{raschka2020machine}, while the dual \SVM{} on moderate-scale datasets uses the \texttt{sklearn} library \cite{scikit-learn}. For multi-class problems, we use the one-vs-one classification strategy. This study is also the largest scale comparison of signature kernels to date which extends the datasets considered in \cite{salvi2021signature}.
The hardware used was 2 computer clusters equipped with overall 8 NVIDIA 3080 Ti \texttt{GPU}s.
\begin{table}
 \caption{Computational complexities of kernels in our experiments; $N \in \bbZ_+$ is the number of time series, $\ell \in \bbZ_+$ is their length, $d \in \bbZ_+$ is their state-space dimension, $M \in \bbZ_+$ is the signature truncation level, $\dimRFF \in \bbZ_+$ is the \texttt{RF} dimension, $W \in \bbZ_+$ is the warping length in \texttt{RWS}.}
 \label{table:rfsf:complexity}
 \resizebox{\textwidth}{!}{
\begin{tabular}{cccccccc}
\toprule
 \texttt{RFSF-DP} & \texttt{RFSF-TRP} & \texttt{KSig} & \texttt{KSigPDE} & \texttt{RWS} & \texttt{GAK} & \texttt{RBF} & \texttt{RFF}
 \\
 \midrule
 $O\pars{N\ell\tilde d\pars{Md + 2^M}}$ & $O\pars{N\ell M\tilde d\pars{d + \tilde d}}$ & $O\pars{N^2 \ell^2 \pars{M + d}}$ & $O\pars{N^2 \ell^2 d}$ & $O\pars{N\ell W d}$ & $O\pars{N^2 \ell^2 d}$ & $O\pars{N^2 \ell d}$ & $O\pars{N \ell d\dimRFF}$
 \\
\bottomrule
\end{tabular}
}
\end{table}

\subsection{Methods}
We compare the proposed variants \RFSFD{} and \RFSFT{} to the baselines described here: 
\begin{enumerate*}[label=(\arabic*)] \item the $M$-truncated Signature Kernel \cite{kiraly2019kernels} \KS{} formulated via the kernel trick, and is a quadratic time baseline; \item the Signature-PDE Kernel \cite{salvi2021signature} \KSP{}, which uses the 2\textsuperscript{nd}-order PDE solver and also has quadratic complexity; \item the Global Alignment Kernel \cite{cuturi2011fast} \GAK{}, one of the most popular sequence kernels to day and can can be related to the signature kernel, see \cite[Sec.~5]{kiraly2019kernels}; \item Random Warping Series \cite{wu2018random} \RWS{}, which produces features by \texttt{DTW} alignments between the input and randomly sampled time series; \item the \RBF{} kernel, which treats the whole time series as a vector of length $\bbR^{\ell d}$, \item Random Fourier Reatures \cite{rahimi2007random} \RFF{}, which also treats the time series as a long vector. We excluded \RFSF{} from the comparison, as it is unfeasible to compute it with reasonable sample sizes $\dimRFF$ due to the polynomial explosion of dimensions in its tensor-based representation. The complexities are compared in Table \ref{table:rfsf:complexity}.
\end{enumerate*}

\subsection{Hyperparameter selection}
For each dataset-kernel, we perform cross-validation to select the optimal hyperparameters that are optimized over the Cartesian product of the following options. For each method that requires a static kernel, we use the \RBF{} kernel with bandwidth hyperparameter $\sigma > 0$. This is specified in terms of a rescaled median heuristic, i.e.~
\begin{align} \label{eq:rfsf:alpha_select}
    \sigma = \alpha \med \curls{\norm{\bx_i - \bx_j^\prime}_2 / 2 \given i \in [\len{\bx}], j \in [\ell_{\bx^\prime}], \bx, \bx^\prime \in \bX}, \quad \text{for} \spc \alpha > 0,
\end{align}
where $\alpha$ is chosen from $\alpha \in \{10^{-3}, \dots, 10^3\}$ on a logarithmic grid with $19$ steps. For each kernel that is not normalized by default (i.e.~the \GAK{} and \RBF{} kernels are normalized, the former is because without normalization it blows up) , we select whether to normalize to unit norm in feature space via $\kernel(\bx, \by) \mapsto \kernel(\bx, \by) / \sqrt{\kernel(\bx, \bx) \kernel(\by, \by)}$. The \SVM{} hyperparameter $C > 0$ is selected from $C \in \{10^0, 10^1, 10^2, 10^3, 10^4\}$. Further, motivated by previous work that investigates the effect of path augmentations in the context of signature methods \cite{morrill2021generalised}, we chose 3 augmentations to cross-validate over. First is parameterization encoding, which gives the classifier the ability to remove the warping invariance of a given sequence kernel, adding the time index as an additional coordinate, i.e. for each time series in the dataset $\bx \in \bX$, we augment it via $\bx = (\bx_i)_{i=1}^{\len{\bx}} \mapsto (\beta i / \len{\bx}, \bx_i)_{i=1}^{\len{\bx}}$, where $\beta > 0$ is the parameterization intensity chosen from $\beta \in \{10^0, 10^1, 10^2, 10^3, 10^4\}$. The second augmentation is the basepoint encoding, the role of which is to remove the translation invariance of signature features. Note that when the static base kernel is chosen to be a nonlinear kernel other than the Euclidean inner product, the signature kernel is not completely translation-invariant due to the state-space nonlinearities, but it is close being that by the $L$-Lipschitz property in Lemma \ref{example:2ndmoment_lip} valid for of the static kernels considered in this work. The basepoint encoding adds an initial $\mathbf{0}$ step at the beginning of each time series, i.e.~for $\bx \in \bX$, $\bx=(\bx_0, \dots, \bx_{\len{\bx}}) \mapsto (\mathbf{0}, \bx_0, \dots, \bx_{\len{\bx}})$. The third augmentation is the lead-lag map, which is defined as $\bx = (\bx_0, \dots, \bx_{\len{\bx}}) \mapsto \pars{(\bx_0, \bx_0), (\bx_1, \bx_0), (\bx_1, \bx_1), \dots, (\bx_{\len{\bx}}, \bx_{\len{\bx}}), (\bx_{\len{\bx}}, \bx_{\len{\bx}})}$. For the truncation-based signature kernels, we select the truncation level $M \in \bbZ_+$ from $M \in \{2, 3, 4, 5\}$. For \RWS{}, we select the warping length from $W \in \{10, 20, \dots, 100\}$ as suggested by the authors. This makes \RWS{} the most expensive feature-based kernel, and so as to fit within the same resource limitations, we omit cross-validating over the path augmentations. We select the standard deviation $\sigma > 0$ of the warping series from the same grid as $\alpha$ in \eqref{eq:rfsf:alpha_select}. For all \texttt{RF} approaches, we set the \texttt{RF} dimension $\tilde d \in \bbZ_+$, so the overall dimension is $1000$. Note that for \RFSFD{} and \RFSFT{} this is respectively $2^{M+1} \dimRFF$ and $M \dimRFF$, where $\dimRFF$ is the base \RFF{} sample size; for \RWS{} it is the number of warping series $\dimRFF$; while for \RFF{} it is twice the number of samples $2 \dimRFF$.

\subsection{Datasets}
\paragraph{\texttt{UEA} Archive}
The \texttt{UEA} archive \cite{dau2019ucr} is a collection of overall 30 datasets for benchmarking classifiers on multivariate time series classification problems containing both binary and multi-class tasks.
The data modality ranges from various sources e.g.~human activity recognition, motion classification, ECG classification, EEG/MEG classification, audio spectra recognition, and others. The sizes of the datasets in terms of number of time series range from moderate ($\leq 1000$ examples) to large ($\leq 30000$), and includes various lengths between $8$ and $18000$. A summary of the dataset characteristics can be found in Table~2 in \cite{dau2019ucr}. Pre-specified train-test splits are provided for each dataset, which we follow. We evaluate all considered kernels on the moderate datasets ($\leq 1000$ time series), but because the non-feature-based become very expensive computationally beyond these sizes, we only evaluate feature-based approaches on medium and large datasets ($\geq 1000$ time series). Each featurized approach is trained and evaluated $5$ times on each dataset in order to account for the randomness in the hyperparameter selection procedure and evaluation.

\begin{table}[!t]
\begin{minipage}{\textwidth}
 \centering
 \captionof{table}{Comparison of \SVM{} test accuracies on moderate multivariate time series classification datasets. For each row, the best result is highlighted in \textbf{bold}, and the second best in \textit{italic}.}
 \label{table:rfsf:results}
 \resizebox{\textwidth}{!}{
 \begin{sc}
\begin{tabular}{lcccccccc}
\toprule
 & \texttt{RFSF-DP} & \texttt{RFSF-TRP} & \texttt{KSig} & \texttt{KSigPDE} & \texttt{RWS} & \texttt{GAK} & \texttt{RBF} & \texttt{RFF}
\\
\midrule
ArticularyWordRecognition & $0.984$ & $0.981$ & $\mathbf{0.990}$ & $0.983$ & $\mathit{0.987}$ & $0.977$ & $0.977$ & $0.978$
\\ 
AtrialFibrillation & $0.373$ & $0.320$ & $\mathit{0.400}$ & $0.333$ & $\mathbf{0.427}$ & $0.333$ & $0.267$ & $0.373$
\\ 
BasicMotions & $\mathbf{1.000}$ & $\mathbf{1.000}$ & $\mathbf{1.000}$ & $\mathbf{1.000}$ & $\mathit{0.995}$ & $\mathbf{1.000}$ & $0.975$ & $0.860$
\\ 
Cricket & $0.964$ & $0.964$ & $0.958$ & $\mathit{0.972}$ & $\mathbf{0.978}$ & $0.944$ & $0.917$ & $0.886$
\\ 
DuckDuckGeese & $0.636$ & $\mathit{0.664}$ & $\mathbf{0.700}$ & $0.480$ & $0.492$ & $0.500$ & $0.420$ & $0.372$
\\ 
ERing & $0.921$ & $0.936$ & $0.841$ & $\mathit{0.941}$ & $\mathbf{0.945}$ & $0.926$ & $0.937$ & $0.915$
\\ 
EigenWorms & $\mathit{0.817}$ & $\mathbf{0.837}$ & $0.809$ & $0.794$ & $0.623$ & $0.511$ & $0.496$ & $0.443$
\\ 
Epilepsy & $\mathbf{0.949}$ & $\mathit{0.942}$ & $\mathbf{0.949}$ & $0.891$ & $0.925$ & $0.870$ & $0.891$ & $0.777$
\\ 
EthanolConcentration & $0.457$ & $0.439$ & $\mathbf{0.479}$ & $\mathit{0.460}$ & $0.284$ & $0.361$ & $0.346$ & $0.325$
\\ 
FingerMovements & $0.608$ & $0.624$ & $\mathbf{0.640}$ & $\mathit{0.630}$ & $0.612$ & $0.500$ & $0.620$ & $0.570$
\\ 
HandMovementDirection & $\mathit{0.573}$ & $0.568$ & $\mathbf{0.595}$ & $0.527$ & $0.403$ & $\mathbf{0.595}$ & $0.541$ & $0.454$
\\ 
Handwriting & $0.434$ & $0.400$ & $0.479$ & $0.409$ & $\mathbf{0.591}$ & $\mathit{0.481}$ & $0.307$ & $0.249$
\\ 
Heartbeat & $0.717$ & $0.712$ & $0.712$ & $\mathbf{0.722}$ & $0.714$ & $0.717$ & $0.717$ & $\mathit{0.721}$
\\ 
JapaneseVowels & $0.978$ & $0.978$ & $\mathbf{0.986}$ & $\mathbf{0.986}$ & $0.955$ & $\mathit{0.981}$ & $\mathit{0.981}$ & $0.979$
\\ 
Libras & $0.898$ & $\mathbf{0.928}$ & $\mathit{0.922}$ & $0.894$ & $0.837$ & $0.767$ & $0.800$ & $0.800$
\\ 
MotorImagery & $\mathit{0.516}$ & $\mathbf{0.526}$ & $0.500$ & $0.500$ & $0.508$ & $0.470$ & $0.500$ & $0.482$
\\ 
NATOPS & $0.906$ & $0.908$ & $0.922$ & $\mathbf{0.928}$ & $\mathit{0.924}$ & $0.922$ & $0.917$ & $0.900$
\\ 
PEMS-SF & $0.800$ & $0.808$ & $0.827$ & $\mathit{0.838}$ & $0.701$ & $\mathbf{0.855}$ & $\mathbf{0.855}$ & $0.770$
\\ 
RacketSports & $0.874$ & $0.861$ & $\mathbf{0.921}$ & $\mathit{0.908}$ & $0.878$ & $0.849$ & $0.809$ & $0.755$
\\ 
SelfRegulationSCP1 & $0.868$ & $0.856$ & $\mathit{0.904}$ & $\mathit{0.904}$ & $0.829$ & $\mathbf{0.915}$ & $0.898$ & $0.885$
\\ 
SelfRegulationSCP2 & $0.489$ & $0.510$ & $\mathit{0.539}$ & $\mathbf{0.544}$ & $0.481$ & $0.511$ & $0.439$ & $0.492$
\\ 
StandWalkJump & $0.387$ & $0.333$ & $\mathit{0.400}$ & $\mathit{0.400}$ & $0.347$ & $0.267$ & $\mathbf{0.533}$ & $0.267$
\\ 
UWaveGestureLibrary & $0.882$ & $0.881$ & $\mathbf{0.912}$ & $0.866$ & $\mathit{0.897}$ & $0.887$ & $0.766$ & $0.846$
\\ 
\midrule
Avg.~acc. & $\mathit{0.740}$ & $0.738$ & $\mathbf{0.756}$ & $0.735$ & $0.710$ & $0.702$ & $0.692$ & $0.656$
\\ 
Avg.~rank & $3.609$ & $3.739$ & $\textbf{2.348}$ & $\mathit{2.957}$ & $3.957$ & $4.174$ & $4.913$ & $5.913$
\\
\bottomrule
\end{tabular}
\end{sc}
}
\end{minipage}
\vspace{10pt}
\resizebox{\textwidth}{!}{
\begin{minipage}{0.51\textwidth}
    \vspace{40pt}
        \centering
        \includegraphics[width=\textwidth]{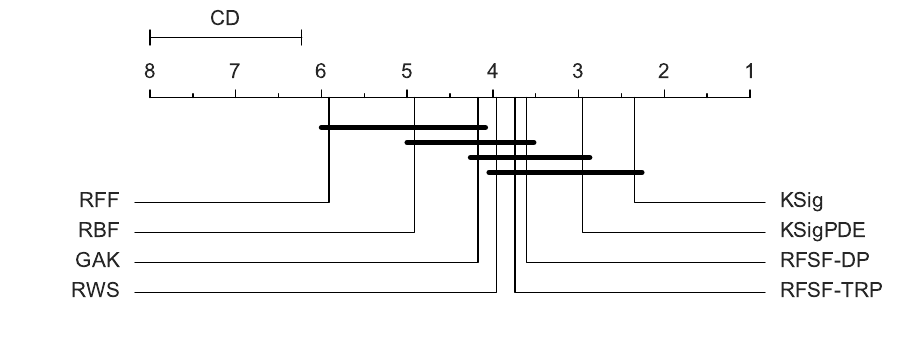}
        \captionof{figure}{Critical difference diagram comparison on moderate datasets of considered approaches using two-tailed Nemenyi test \cite{demvsar2006statistical}.}
        \label{fig:cd_diagram}
\end{minipage}
\hspace{1.5pt}
\begin{minipage}{0.49\textwidth}
\vspace{10pt}
\captionof{table}{Comparison of accuracies on large-scale datasets of random features.}
\label{table:rfsf:large_scale_results}
\resizebox{1.\textwidth}{!}{
\begin{sc}
\begin{tabular}{lcccc}
\toprule
  & \RFSFD{} & \RFSFT{} & \texttt{RWS} & \texttt{RFF}
\\ 
\midrule
CharacterTrajectories & $\mathit{0.990}$ & $\mathit{0.990}$ & $\mathbf{0.991}$ & $0.989$
\\ 
FaceDetection & $\mathit{0.653}$ & $\mathbf{0.656}$ & $0.642$ & $0.572$
\\ 
InsectWingbeat & $\mathit{0.436}$ & $\mathbf{0.459}$ & $0.227$ & $0.341$
\\ 
LSST & $0.589$ & $\mathit{0.624}$ & $\mathbf{0.631}$ & $0.423$
\\ 
PenDigits & $\mathit{0.983}$ & $0.982$ & $\mathbf{0.989}$ & $0.980$
\\ 
PhonemeSpectra & $\mathit{0.204}$ & $\mathit{0.204}$ & $\mathbf{0.205}$ & $0.083$
\\ 
SITS1M & $\mathbf{0.745}$ & $\mathit{0.740}$ & $0.610$ & $0.718$
\\ 
SpokenArabicDigits & $\mathbf{0.981}$ & $\mathit{0.980}$ & $\mathbf{0.981}$ & $0.964$
\\ 
fNIRS2MW & $\mathbf{0.659}$ & $\mathit{0.658}$ & $0.621$ & $0.642$
\\
\midrule
Avg.~acc. & $\mathit{0.693}$ & $\mathbf{0.699}$ & $0.655$ & $0.635$
\\ 
Avg.~rank & $\mathbf{1.778}$ & $\mathit{1.889}$ & $2.222$ & $3.333$
\\
\bottomrule
\end{tabular}
\end{sc}
}
\end{minipage}
}
\end{table}

\paragraph{Mental Workload Intensity Classification}
We evaluate featurized approaches on a large-scale brain-activity recording data set called \texttt{fNIRS2MW}.\footnote{\href{https://github.com/tufts-ml/fNIRS-mental-workload-classifiers}{https://github.com/tufts-ml/fNIRS-mental-workload-classifiers}} This dataset contains brain activity recordings collected from overall $68$ participants during a $30$-$60$ minute experimental session, where they were asked to carry out tasks of varying intensity. The collected time series are sliced into $30$ second segments using a sliding window, and each segment is labelled with an intensity level ($0$-$3$), giving rise to overall $\sim 100000$ segments, which we split in a ratio of $80-20$ for training and testing. We convert the task into a binary classification problem by assigning a label whether the task is low ($0$ or $1$) or high ($2$ or $3$) intensity.

\paragraph{Satellite Image Classification} As a massive scale task, we use a satellite imagery dataset\footnote{ \href{https://cloudstor.aarnet.edu.au/plus/index.php/s/pRLVtQyNhxDdCoM}{https://cloudstor.aarnet.edu.au/plus/index.php/s/pRLVtQyNhxDdCoM}} of $N = 10^6$ time series. Each length $\ell = 46$ time series corresponds to a vegetation index calculated from remote sensing data, and the task is to classify land cover types \cite{petitjean2012satellite} by mapping vegetation profiles to various types of crops and forested areas corresponding to $24$ classes. We split the dataset in a ratio of $90$-$10$ for training and testing.

\subsection{Results} 
Table \ref{table:rfsf:results} compares test accuracies on moderate size multivariate time series classification datasets with $N \leq 1000$ from the \texttt{UEA} archive. \KS{} provides state-of-the-art performance among all sequence kernels with taking the highest aggregate score in terms of all of average accuracy, average rank, and number of first places. Our proposed random feature variants \RFSFD{} and \RFSFT{} provide comparable performance on most of the datasets in terms of accuracy, and they are only outperformed by \KS{} and \KSP{} with respect to average accuracy and rank. Interestingly, \RFSFT{} has more first place rankings, but \RFSFD{} performs slightly better on average. This shows that on datasets of these sizes, using either of \RFSFD{} and \RFSFT{} does not sacrifice model performance - even leading to improvements in some cases, potentially due to the implicit regularization effect of restricting to a finite-dimensional feature space - and it can already provide speedups. We visualize the critical difference diagram comparison of all considered approaches in Figure \ref{fig:cd_diagram}.

Table \ref{table:rfsf:large_scale_results} demonstrates the performance of scalable approaches, i.e.~\RFSFD{}, \RFSFT{}, \RWS{} and \RFF{} on the remaining \texttt{UEA} datasets ($N \geq 1000$), the dataset \texttt{fNIRS2MW} ($N=10^5$), and the satellite dataset \texttt{SITS1M} ($N = 10^6$). We find it infeasible to perform full cross-validation for quadratic time kernels on these datasets due to expensive kernel computations and downstream cost of dual \SVM{}. The results show that both variants \RFSFD{} and \RFSFT{} perform significantly better on average with respect to accuracy and rank then both \RWS{} and \RFF{}. Note when \RWS{} takes first place, it only improves over our approach marginally, however, when it underperforms, it often does so severely. This is not surprising as both \RFSFD{} and \RFSFT{} approximate the signature kernel, which is a universal kernel on time series; it is theoretically capable of learning from any kind of time series data as supported by its best overall performance above.

\begin{table}[!t]
 \centering
 \captionof{table}{Average ranks of \SVM{} classifier accuracies per data domain. For each row, the best result is highlighted in \textbf{bold}, and the second best in \textit{italic}.}
 \label{table:rfsf:domain}
 \begin{sc}
\begin{tabular}{lcccccccc}
\toprule
 Domain & \texttt{RFSF-DP} & \texttt{RFSF-TRP} & \texttt{KSig} & \texttt{KSigPDE} & \texttt{RWS} & \texttt{GAK} & \texttt{RBF} & \texttt{RFF} \\
\midrule
Audio & $3.333$ & $3.667$ & $\mathbf{2.333}$ & $\mathit{2.667}$ & $4.667$ & $3.000$ & $4.000$ & $4.333$  \\
ECG & $3.000$ & $5.000$ & $\mathbf{2.000}$ & $3.000$ & $\mathit{2.500}$ & $5.000$ & $3.500$ & $4.500$  \\
EEG/MEG & $4.200$ & $3.400$ & $\mathbf{2.000}$ & $\mathit{2.800}$ & $5.800$ & $3.800$ & $4.600$ & $5.400$  \\
HAR & $3.556$ & $3.556$ & $\mathit{2.667}$ & $3.000$ & $\mathbf{2.222}$ & $4.111$ & $5.333$ & $6.667$  \\
Motion & $\mathit{2.500}$ & $3.000$ & $\mathbf{2.000}$ & $4.000$ & $3.500$ & $6.500$ & $7.000$ & $7.000$  \\
Other & $4.000$ & $4.000$ & $\mathbf{2.000}$ & $\mathbf{2.000}$ & $7.500$ & $\mathit{3.000}$ & $3.500$ & $6.500$  \\

\bottomrule
\end{tabular}
\end{sc}
\end{table}

\paragraph{Domain analysis} Next, we provide a domain analysis on the moderate sized \texttt{TSC} datasets in Table \ref{table:rfsf:domain} split across 6 data domains \cite{dau2019ucr}: audio spectra classification (Audio), electrocardiogram (ECG) recordings, electroencephalography and magnetoencephalography (EEG/MEG), human activitiy recognition (HAR) based on accelerometer or gyroscope data, motion recognition (Motion) using other data sources, and other data modalities. We make the following observations: \begin{enumerate*}[label=(\arabic*)]
\item on Audio, \texttt{KSig} and \texttt{KSigPDE} come in first and second respectively while \texttt{GAK} is in third place; \item on ECG, \texttt{KSig} is first and \texttt{RWS} is second while \texttt{RFSF-DP} and \texttt{KSigPDE} are tied for third place, \texttt{RFSF-TRP} underperforms on this modality; \item on EEG/MEG, \texttt{KSig} and \texttt{KSigPDE} are first and second again, while \texttt{RFSF-TRP} is third, and \texttt{RFSF-DP} somewhat underperforms; \item for HAR datasets, \texttt{RWS} is first, \texttt{KSig} is second, and \texttt{KSigPDE} is third, while \texttt{RFSF-DP} and \texttt{RFSF-TRP} are tied for fourth place; \item on Motion datasets \texttt{KSig} is first, \texttt{RFSF-DP} is second, and \texttt{RFSF-TRP} is third; \item on Other datasets, \texttt{KSig} and \texttt{KSigPDE} are tied for first place, while \texttt{GAK} comes in second, and \texttt{RBF} in third.
\end{enumerate*} 
Overall, we observe that \texttt{KSig} takes undisputed first place, while \texttt{KSigPDE} often comes in close second. The performance of \texttt{RFSF-TRP} and \texttt{RFSF-DP} vary across domains, one often outperforming the other indicating that the way they approximate the signature kernel captures patterns in the data differently. Note that \texttt{RWS} outperforms the \texttt{RFSF} variants on some data modalities, however, it also often underperforms severely while at least one \texttt{RFSF} variation always performs at least moderately well on each data domain.

\section{Conclusion}
We constructed a random kernel $\rffsigkernel[\leq M]$ for sequences that benefits from
\begin{enumerate*}[label=(\roman*)]
  \item lifting the original sequence to an infinite-dimensional \RKHS{} $\Hil$,
  \item linear complexity in sequence length,
  \item being with high probability close to the signature kernel $\sigkernel$.
\end{enumerate*}
Thereby it combines the strength of the signature kernel $\sigkernel$ which is to implicitly use the iterated integrals of a sequence that has an infinite-dimensional \RKHS{} $\Hil$ as state-space with the strength of (unkernelized) signature features $\signature$ that only require linear time complexity.
Our main theoretical result extends the theoretical guarantees for translation-invariant kernels on linear spaces to the signature kernel defined on the nonlinear domain $\Seq(\cX)$; however, the proofs differ from the classic case and require to analyse the error propagation in tensor space.
A second step is more straightforward, and combines this approach with random projections in finite-dimensions for tensors to reduce the complexity in memory further.
The advantages and disadvantages of the resulting approach are analogous to the classic \RFF{} technique on $\bbR^d$, namely a reduction of computational complexity by an order for the price of an approximation that only holds with high probability.
As in the classic \RFF{} case, our experiments indicate that this is in general a favourable tradeoff.

In the future, it would be interesting both theoretically and empirically to replace the vanilla Monte Carlo integration in the \RFF{} construction by block-orthogonal random matrices as done in \cite{yu2016orthogonal}.
Further, our random features can also be used to define an unbiased appoximation to the inner product of expected signatures, which has found usecases, among many, in nonparametric hypothesis testing and market regime detection \cite{chevyrev2022signature, horvath2023non}, training of generative models \cite{ni2021sig, issa2023non}, and graph representation learning \cite{toth2022capturing}.  


\begin{subappendices}
\section{Concentration of measure} \label{apx:prob}

\subsection{Classic inequalities} The following inequalities are classic, and their proofs are in analysis textbooks, e.g.~\cite{dudley2002real,rudin1976principles}. Firstly, Jensen's inequality is useful for convex (concave) expectations.
\begin{lemma}[Jensen's inequality] \label{lem:jensen}
    Let $X$ be an integrable random variable, and $f: \bbR \to \bbR$ a convex function, such that $f\pars{X}$ is also integrable. Then, the following inequality holds:
    \begin{align}
        f\pars{\expe{X}} \leq \expe{f\pars{X}}.
    \end{align}
\end{lemma}

Hölder's inequality is a generalization of the Cauchy-Schwarz inequality to $L^p$ spaces.

\begin{lemma}[Hölder's inequality] \label{lem:holder}
    Let $p, q \geq 1$
    such that $\frac{1}{p} + \frac{1}{q} = 1$. Let $X$ and $Y$ respectively be $L^p$ and $L^q$ integrable random variables, i.e.~$\expe{\abs{X}^p} < \infty$ and $\expe{\abs{Y}^q} < \infty$. Then, $XY$ is integrable, and it holds that
    \begin{align}
        \expe{\abs{XY}} \leq \bbE^{1/p}\bracks{\abs{X}^p} \bbE^{1/q}\bracks{\abs{Y}^q}.
    \end{align}
\end{lemma}

Although not inherently a probabilistic inequality, Young's inequality can be used to decouple products of random variables.

\begin{lemma}[Young's inequality]
    Let $p,q > 0$ with $\frac{1}{p} + \frac{1}{q} = 1$. Then, for every $a,b \geq 0$
    \begin{align}
        ab \leq \frac{a^p}{p} + \frac{b^q}{q}.
    \end{align}
\end{lemma}

\begin{lemma}[Reverse Young's inequality] \label{lem:reverse}
    Let $p,q > 0$ such that $\frac{1}{p} - \frac{1}{q} = 1$. Then, for every $a \geq 0$ and $b > 0$, it holds that
    \begin{align}
        ab \geq \frac{a^p}{p} - \frac{b^{-q}}{q}.
    \end{align}
\end{lemma}
\begin{proof} Apply Young's inequality with $p^\p = \frac{1}{p}$ and $q^\p = \frac{q}{p}$ to $a^\p = (ab)^p$ and $b^\p = b^{-p}$.
\end{proof}

\subsection{Subexponential concentration}

Next, we state a variation on the well-known Bernstein inequality, which holds for  random variables in the subexponential class. The condition \eqref{eq:rfsf:bernstein_cond_twot} on a random variable $X$, in this case, is formulated as a moment-growth bound, called a Bernstein moment condition. We show in Lemmas \ref{lem:alpha_bernstein_cond} and \ref{lem:alpha_exp_norm_to_bernstein} (in a more general setting) that \eqref{eq:rfsf:bernstein_cond_twot} is equivalent to the random variable being subexponential. Further, note that the condition itself is given in terms of non-centered random variables, while the statement of the theorem is about their centered counterparts. For similar results, see \cite[Sec.~2.2.2]{van1996weak}.
\begin{theorem}[Bernstein inequality - one-tailed] \label{thm:bernstein_onetail}
Let $X_1, \dots, X_n$ be independent random variables satisfying the following moment-growth condition for some $S, R > 0$,
    \begin{align} \label{eq:rfsf:bernstein_cond_onet}
        \expe{X_i^k} \leq \frac{k! S^2 R^{k-2}}{2} \quad \text{for all} \spc k \geq 2.
    \end{align}
Then, it holds for $\tilde X_i = X_i - \expe{X_i}$ that
\begin{align} \label{eq:rfsf:bernstein_ineq_onet}
    \prob{\sum_{i=1}^n \tilde X_i \geq t} \leq \exp\pars{\frac{-t^2}{2(nS^2 + Rt)}}.
\end{align}
\end{theorem}
\begin{proof}
    We have for $\lambda > 0$
    \begin{align}
        \prob{\sum_{i=1}^n \tilde X_i \geq t}
        & \stackrel{(a)}{\leq}
        \exp(-\lambda t) \expe{\exp\pars{\lambda \sum_{i=1}^n \tilde X_i}}
        \\&\stackrel{(b)}{=}
        \exp\pars{-\lambda t - \lambda \sum_{i=1}^n \expe{X_i}} \prod_{i=1}^n \expe{\exp(\lambda X_i)},
        \label{eq:rfsf:chernoff}
    \end{align}
    where (a) holds for any $\lambda  >0$ by the  Chernoff bound applied to $\sum_{i=1}^n \tilde X_i$, and in (b) we used the independence of $X_i$-s. Bounding the moment-generating function of $X_i$, one gets
    \begin{align}
        \expe{\exp(\lambda X_i)}
        &\stackrel{\text{(d)}}{=}
        1 + \lambda \expe{X_i} + \sum_{k=2}^\infty \frac{\lambda^k}{k!} \expe{X_i^k}
        \stackrel{\text{(e)}}{\leq}
        1 + \lambda \expe{X_i} + \frac{S^2 \lambda^2}{2} \sum_{k=0}^\infty \lambda^k R^k
        \\
        &\stackrel{\text{(f)}}{=} 1 + \lambda \expe{X_i} + \frac{S^2 \lambda^2}{2(1- \lambda R)}
         \stackrel{\text{(g)}}{\leq}
        \exp\pars{\lambda \expe{X_i} + \frac{S^2\lambda^2}{2(1 - \lambda R)}}, \label{eq:rfsf:mgf_bound}
    \end{align}
    where (d) is due to the Taylor expansion of the exponential function, (e) implied by \eqref{eq:rfsf:bernstein_cond_onet}, (f) holds by the sum of geometric series for any $\lambda < 1/R$, (g) follows from the inequality $1+x \leq \exp(x)$ for all $x \in \bbR$. Using this to bound \eqref{eq:rfsf:chernoff},
    \begin{align}
        \prob{\sum_{i=1}^n \tilde X_i \geq t}
        &\leq
        \exp\pars{-\lambda t - \lambda \sum_{i=1}^n \expe{X_i}} \prod_{i=1}^n \exp\pars{\lambda \expe{X_i} + \frac{S^2\lambda^2}{2(1 - \lambda R)}}
        \\
        &=
        \exp\pars{-\lambda t + \frac{n S^2\lambda^2}{2(1 - \lambda R)}}. \label{eq:rfsf:bernstein_ineq_lambda}
    \end{align}
    Finally, choosing $\lambda = t / (nS^2 + Rt)$ gives the statement.
\end{proof}

\begin{corollary}[Bernstein inequality - two-tailed] \label{thm:bernstein_twotail}
Let $X_1, \dots, X_n$ be independent random variables satisfying the following moment-growth condition for some $S, R > 0$,
    \begin{align} \label{eq:rfsf:bernstein_cond_twot}
        \expe{\abs{X_i}^k} \leq \frac{k! S^2 R^{k-2}}{2} \quad \text{for all} \spc k \geq 2.
    \end{align}
Then, it holds for $\tilde X_i = X_i - \expe{X_i}$ that
\begin{align} \label{eq:rfsf:bernstein_ineq_twot}
    \prob{\abs{\sum_{i=1}^n \tilde X_i} \geq t} \leq 2  \exp\pars{\frac{-t^2}{2(nS^2 + Rt)}}.
\end{align}
\end{corollary}
\begin{proof}
Applying Theorem \ref{thm:bernstein_onetail} to $\sum_{i=1}^n \tilde X_i$ and $-\sum_{i=1}^n \tilde X_i$, and combining the two bounds gives the two-tailed result.
\end{proof}

\subsection{$\alpha$-exponential concentration}

In this section, we introduce a specific class of Orlicz norms for heavy-tailed random variables, which generalizes subexponential distributions.
\begin{definition}[$\alpha$-exponential Orlicz norm] \label{def:alpha_subexp_norm}
    Let $\alpha > 0$ and define the function
    \begin{align}
    \Psi_\alpha: \bbR \to \bbR, \quad \Psi_\alpha(x) = \exp\pars{x^\alpha} - 1 \quad \text{for all} \spc x \in \bbR.
    \end{align}
    The $\alpha$-exponential Orlicz (quasi-)norm of a random variable $X$ is defined as
    \begin{align}
        \norm{X}_{\Psi_\alpha} = \inf\curls{t > 0 : \expe{\Psi_\alpha\pars{\frac{\abs{X}}{t}}} \leq 1},
    \end{align}
    adhering to the standard convention that $\inf \emptyset = \infty$.
\end{definition}
If a random variable $X$ satisfies $\norm{X}_{\Psi_\alpha} < \infty$, it is either called an $\alpha$-(sub)exponential random variable (if $\alpha < 1$) \cite{sambale2022notes,gotze2021concentration,chamakh2020orlicz}, or sub-Weibull of order-$\alpha$ \cite{kuchibhotla2022moving,zhang2021concentration}.

An alternative characterization of the $\alpha$-exponential norm is the following tail-bound.

\begin{remark}[Tail bound] \label{remark:orlicz_tail} Let $\alpha > 0$ and $X$ be a random variable such that $\norm{X}_{\Psi_\alpha} < \infty$. Then,
    \begin{align}
        \prob{\abs{X} \geq \epsilon} \leq 2\exp\pars{- \frac{\epsilon^\alpha}{\norm{X}_{\Psi_\alpha}^\alpha}}.
    \end{align}
\end{remark}
\begin{proof}
    Due to Definition \ref{def:alpha_subexp_norm}, $\expe{\exp\pars{\frac{\abs{X}^\alpha}{\norm{X}^\alpha_{\Psi_\alpha}}}} \leq 2$, hence by Markov's inequality
    \begin{align}
        \bbP\bracks{\abs{X} \geq \epsilon}
        =
        \prob{\frac{\abs{X}^\alpha}{\norm{X}_{\Psi_\alpha}^\alpha}
        \geq
        \frac{\epsilon^\alpha}{\norm{X}_{\Psi_\alpha}^\alpha}}
        \leq
        e^{-\frac{\epsilon^\alpha}{\norm{X}_{\Psi_\alpha}^\alpha}} \expe{\exp\pars{\frac{\abs{X}^\alpha}{\norm{X}_{\Psi_\alpha}^\alpha)}}}
        \leq
        2 e^{-\frac{\epsilon^\alpha}{\norm{X}_{\Psi_\alpha}^\alpha}}.
    \end{align}
\end{proof}

Note that although $\norm{\cdot}_{\Psi_\alpha}$ is often referred to as a norm, it does not satisfy the triangle inequality for $\alpha < 1$, although we can still relate the norm of the sum to the sum of the norms.
\begin{lemma}[Generalized triangle inequality for Orlicz norm] \label{lem:alpha_exp_triangle}
    It holds for any random variables $X, Y$ and $\alpha > 0$ that
    \begin{align}
        \norm{X + Y}_{\Psi_\alpha} \leq C_\alpha \pars{\norm{X}_{\Psi_\alpha} + \norm{Y}_{\Psi_\alpha}},
    \end{align}
    where $C_\alpha = 2^{1/\alpha}$ if $\alpha < 1$ and $1$ otherwise.
\end{lemma}
\begin{proof}
    See Lemma~A.3. in \cite{gotze2021concentration}.
\end{proof}

A useful property of $\alpha$-exponential norms is that they satisfy a Hölder-type inequality.
\begin{lemma}[Hölder inequality for Orlicz norm] \label{lem:alpha_exp_holder}
    It holds for any random variables $X_1, \dots, X_k$ and $\alpha_1, \dots, \alpha_k > 0$ that
    \begin{align}
        \norm{\prod\nolimits_{i=1}^k X_i}_{\Psi_\alpha} \leq \prod\nolimits_{i=1}^k \norm{X_i}_{\Psi_{\alpha_i}},
    \end{align}
    where $\alpha = \pars{\sum_{i=1}^k \alpha_i^{-1}}^{-1}$.
\end{lemma}
\begin{proof}
    See Lemma~A.1. in \cite{gotze2021concentration}.
\end{proof}

Next, we show in the following two lemmas that a random variable $X$ is $\alpha$-exponential if and only if $\abs{X}^\alpha$ satisfies a Bernstein moment-growth condition.

\begin{lemma}[Bernstein condition implies Orlicz norm bound] \label{lem:alpha_bernstein_cond}
    Let $\alpha > 0$ and $X$ be a random variable that satisfies for $S, R > 0$ that
    \begin{align} \label{eq:rfsf:alpha_bernstein_cond}
        \expe{\abs{X}^{k \alpha}} \leq \frac{k! S^2 R^{k-2}}{2} \quad \text{for all} \spc k \geq 2.
    \end{align}
    Then, it holds that
    \begin{align}
        \norm{X}_{\Psi_\alpha} \leq 2\pars{S \vee R }^{1/\alpha}.
    \end{align}
\end{lemma}
\begin{proof}
    Firstly, due to Jensen's inequality (Lemma \ref{lem:jensen}) and \eqref{eq:rfsf:alpha_bernstein_cond}, we have
    \begin{align}
        \expe{\abs{X}^\alpha} \leq \bbE^{1/2}\bracks{\abs{X}^{2\alpha}} \leq S. \label{eq:rfsf:1st_alpha_moment}
    \end{align}
    We proceed similarly to the proof of Theorem \ref{thm:bernstein_onetail}. For $t > 0$, we have
    \begin{align}
        \expe{\exp{\frac{\abs{X}^\alpha}{t^\alpha}}}
        &\stackrel{\text{(a)}}{=}
        1 + \frac{\expe{\abs{X}^\alpha}}{t^\alpha} + \sum_{m=2}^\infty \frac{\expe{\abs{X}^{m \alpha}}}{t^{m \alpha} m!}
        \stackrel{\text{(b)}}{\leq}
        1 + \frac{S}{t^\alpha} + \frac{S^2}{2t^{2\alpha}} \sum_{m=0}^\infty \pars{\frac{R}{t^{\alpha}}}^m
        \\
        &\stackrel{\text{(c)}}{=}
        \underbrace{1 + \frac{S}{t^\alpha} + \frac{S^2}{2t^{\cancel{2}\alpha}} \frac{\cancel{t^\alpha}}{t^\alpha - R}}_{f(t)},
    \end{align}
    where (a) is due to the Taylor expansion of the exponential function, (b) is due to \eqref{eq:rfsf:alpha_bernstein_cond} and \eqref{eq:rfsf:1st_alpha_moment}, (c) is the sum of a geometric series for $R < t^\alpha$. Defining $f(t) = 1 + \frac{S}{t^\alpha}\pars{1 + \frac{S}{2(t^\alpha - R)}}$ and solving for $f(t) \leq 2$ leads to the quadratic inequality
    \begin{align}
        0 \leq t^{2\alpha} - (S + R)t^\alpha + SR - \frac{S^2}{2},
    \end{align}
    which has roots $t^\alpha_{1,2} = \frac{1}{2}\pars{S + R \pm \sqrt{\pars{S - R}^2 + 2S^2}}$. We discard the left branch, which violates the condition $R < t^\alpha$ since
    \begin{align}
        t_2^\alpha = \frac{1}{2}\pars{S + R - \sqrt{(S - R)^2 + 2S^2}} \leq \frac{1}{2}\pars{S + R - \abs{S - R}} = S \wedge R \leq R.
    \end{align}
    Now, as the inequality $\sqrt{x^2 + y^2} \leq \abs{x} + \abs{y}$ holds for all $x, y \in \bbR$, we get
    \begin{align}
        t_1^\alpha &\leq \frac{1}{2}\pars{S + R + \sqrt{(S-R)^2 + 2S^2}} \leq \frac{1}{2}\pars{S + R + \abs{S-R} + \sqrt{2}S}
        \\
        &= S \vee R + \frac{\sqrt{2}}{2}S \leq 2 (S \vee R)
    \end{align}
    and hence choosing $t \geq \pars{2 (S \vee R)}^{1/\alpha} \geq t_1$ implies $f(t) \leq 2$, and we are done.
\end{proof}

The other direction is proven in the following lemma.

\begin{lemma}[Finite Orlicz norm implies Bernstein condition]  \label{lem:alpha_exp_norm_to_bernstein}
    Let $\alpha > 0$ and $X$ be a random variable such that $\norm{X}_{\Psi_\alpha} < \infty$. Then,
    \begin{align}
        \expe{\abs{X}^{k \alpha}} \leq \frac{k! S^2 R^{k-2}}{2} \quad \text{for all } k \geq 2,
    \end{align}
    where  $S = \sqrt{2} \norm{X}_{\Psi_\alpha}^\alpha$ and $R = \norm{X}_{\Psi_\alpha}^\alpha$.
\end{lemma}
\begin{proof}
    We have that
    \begin{align}
        \expe{\abs{X}^{k \alpha}} 
        & \stackrel{\text{(a)}}{=}
        k! \norm{X}_{\Psi_\alpha}^{k\alpha} \expe{\frac{1}{k!} \pars{\frac{\abs{X}}{\norm{X}_{\Psi_\alpha}}} ^{k \alpha}}
        \stackrel{\text{(b)}}{\leq}
        k! \norm{X}_{\Psi_\alpha}^{k \alpha} \expe{\exp\pars{\frac{\abs{X}^\alpha}{\norm{X}^\alpha_{\Psi_\alpha}}} - 1}
        \\
        &\stackrel{\text{(c)}}{\leq}
        k! \norm{X}_{\Psi_\alpha}^{k \alpha}
        \stackrel{\text{(d)}}{=}
        \frac{k! \pars{\sqrt{2}\norm{X}_{\Psi_\alpha}^{\alpha}}^2 \pars{\norm{X}_{\Psi_\alpha}^\alpha}^{k-2}}{2},
    \end{align}
    where (a) is simply multiplying and dividing by the same values, (b) is the inequality $1 + x^k / k! \leq e^x$ for $x > 0$, (c) follows from Definition \ref{def:alpha_subexp_norm}, (d) is reorganizing terms to the required form.
\end{proof}

We adapt the following concentration inequality from \cite{kuchibhotla2022moving} for $\alpha$-exponential summation.

\begin{theorem}[Concentration inequality for $\alpha$-subexponential summation] \label{thm:alpha_subexp_concentration}
Let $\alpha \in (0, 1)$ and $X_1, \dots, X_n$ be independent, centered random variables with $\norm{X_i}_{\Psi_\alpha} \leq M_\alpha$ for all $i \in [n]$. Then, it holds for $t> 0$ that
\begin{align} \label{eq:rfsf:alpha_ineq1}
    \prob{\abs{\sum_{i=1}^n X_i} \geq  C_\alpha \pars{2 \sqrt{nt} + \sqrt{2} 4^{1/\alpha} t^{1/\alpha}}} \leq 2e^{-t},
\end{align}
where $C_\alpha = \sqrt{8} e^4 (2\pi)^{1/4} e^{1/24} (2e/\alpha)^{1/\alpha} M_\alpha$. Alternatively, it holds for $\epsilon > 0$ that
\begin{align} \label{eq:rfsf:alpha_ineq2}
    \prob{\abs{\sum_{i=1}^n X_i} \geq \epsilon} \leq 2\exp\pars{-\frac{1}{4}\min\curls{\pars{\frac{\epsilon}{2\sqrt{n}C_\alpha}}^2, \pars{\frac{\epsilon}{\sqrt{8}C_\alpha }}^{\alpha}}}.
\end{align}
\end{theorem}
\begin{proof}
    The inequality \eqref{eq:rfsf:alpha_ineq1} follows directly from \cite[Theorem~3.1]{kuchibhotla2022moving}.

    To show \eqref{eq:rfsf:alpha_ineq2}, let $g_1(t) = 2C_\alpha \sqrt{nt}$ and $g_2(t) = \sqrt{2}C_\alpha4^{1/\alpha} t^{1/\alpha}$. Now, for $t > 0$
    \begin{align}
        \prob{\abs{\sum_{i=1}^n X_i} \geq 2\pars{g_1(t) \vee g_2(t)}} \leq \prob{\abs{\sum_{i=1}^n X_i} \geq g_1(t) + g_2(t)} \leq 2e^{-t},
    \end{align}
    which is equivalently written as
    \begin{align}
        \prob{\abs{\sum_{i=1}^n X} \geq \epsilon} \leq 2\exp\pars{- \pars{g^{-1}_1(\epsilon/2) \wedge g^{-1}_2(\epsilon/2)}}.
    \end{align}
\end{proof}

\subsection{Hypercontractivity}
Here we provide an alternative approach for the concentration of heavy-tailed random variables, specifically for polynomials of Gaussian random variables. The following lemma states a moment bound for such Gaussian chaoses, considered in e.g.~\cite{janson1997gaussian,boucheron2013concentration}.

\begin{lemma}[Moment bound for Gaussian polynomial] \label{lem:hyper_moments}
    Consider a degree-$p$ polynomial $f(X) = f(X_1, \dots, X_n)$ of independent centered Gaussian random variables, $X_1, \dots X_n \stackrel{\iid}{\sim} \cN(0, 1)$. Then, for all $k \geq 2$
    \begin{align}
        \bbE^{1/k}\bracks{\abs{f(X)}^k} \leq (k-1)^{p/2} \bbE^{1/2}\bracks{\abs{f(X)}^2}.
    \end{align}
\end{lemma}

\begin{theorem}[Concentration inequality for Gaussian polynomial] \label{thm:hyper_concentration} Consider a degree-$p$ polynomial $f(X) = f(X_1, \dots, X_n)$ of independent centered Gaussian random variables, $X_1, \dots, X_n \stackrel{\iid}{\sim} \cN(0, 1)$. Then, for $p \geq 2$ and $\epsilon > 0$
\begin{align}
    \prob{\abs{f(X) - \expe{f(X)}} \geq \epsilon} \leq 2\exp\pars{-\frac{\epsilon^{2/p}}{2\sqrt{2}e\bbV^{1/p}\bracks{f(X)}}},
\end{align}
where $\bbV\bracks{\cdot}$ denotes the variance of a random variable.
\end{theorem}
\begin{proof}
    Without loss of generality, we may assume that $\expe{f(X)} = 0$ and $\bbV\bracks{f(X)} = 1$.
    Since Lemma \ref{lem:hyper_moments} holds, we have for $p, k \geq 2$
    \begin{align}
        \expe{\abs{f(X)}^{2k/p}} &\stackrel{\text{(a)}}{\leq} \bbE^{2/p}\bracks{\abs{f(X)}^k} 
        \stackrel{\text{(b)}}{\leq} (k-1)^{k} \leq k^{k} \stackrel{\text{(c)}}{\leq} k! e^k
        \leq \frac{k! (\sqrt{2}e)^2 e^{k-2}}{2},
    \end{align}
    where (a) holds due to Jensen inequality since $(\cdot)^{2/p}$ is concave, (b) is Lemma \ref{lem:hyper_moments}, and (c) is due to Stirling's approximation. Hence, $\abs{f(X)}^{2/p}$ satisfies a Bernstein condition with $S = \sqrt{2}e$ and $R = e$. Therefore, by Lemma \ref{lem:alpha_bernstein_cond}, we have $\norm{f(X)}_{\Psi_{2/p}} \leq (2\sqrt{2}e)^{p/2}$.

    Then, by Remark \ref{remark:orlicz_tail}, it holds that
    \begin{align}
        \prob{\abs{f(X)} \geq \epsilon} \leq 2 \exp\pars{- \frac{\epsilon^{2/p}}{\norm{f(X)}_{\Psi_{2/p}}^{2/p}}} \leq 2 \exp\pars{- \frac{\epsilon^{2/p}}{2\sqrt{2}e}}.
    \end{align}
\end{proof}

\section{Random Fourier Features} \label{apx:RFF}

\cite{rahimi2007random} provides a uniform bound for the \RFF{} error over a compact and convex domain $\cX \subset \bbR^d$ with diameter $\abs{\cX}$ for some absolute constant $C > 0$,
\begin{align} \label{eq:rfsf:rff_conv_prob}
    \prob{\sup_{\bx, \by \in \cX}\abs{\rffkernel(\bx, \by) - \kernel(\bx, \by)} \geq \epsilon} \leq C \pars{\frac{\sigma_{{\Lambda}} \abs{\cX}}{\epsilon}}^2 \exp\pars{-\frac{-{\dimRFF}\epsilon^2}{4(d+2)}},
\end{align}
where $\sigma_\Lambda^2 = \bbE_{\bw \sim {\Lambda}}\bracks{\norm{\bw}^2}$, and \cite{sutherland2015error} shows that $C \leq 66$. Equation \eqref{eq:rfsf:rff_conv_prob} implies \cite[Sec.~2]{sriperumbudur2015optimal} 
\begin{align}
    \sup_{\bx, \by \in \cX} \abs{\rffkernel(\bx, \by) - \kernel(\bx, \by)} = O_p\pars{\abs{\cX}\sqrt{\dimRFF^{-1}\log \dimRFF}},
\end{align}
where the $X_n = O_p(a_n)$ notation for $a_n > 0$ refers to $\lim_{C \to 0} \limsup_{n \to \infty} \prob{\frac{\abs{X_n}}{a_n} > C} = 0$.

The converse result\footnote{\label{footnote:exp-improved-guarantee}These error guarantees can also be improved exponentially both for the approximation of kernel values \cite{sriperumbudur2015optimal} and kernel derivatives \cite{szabo2019kernel,chamakh2020orlicz} in terms of the size of the domain where it holds, i.e.~$\abs{\cX}$.} for \RFF{} derivatives was shown in \cite[Thm.~5]{sriperumbudur2015optimal}. The idea of the proof is to cover the input domain by an $\epsilon$-net, and control the approximation error on the centers, while simultaneously controlling the Lipschitz constant of the error to get the bound to hold uniformly. We provide an adapted version with two main differences: \begin{enumerate*}[label=(\arabic*)] \item using the Bernstein inequality from Corollary \ref{thm:bernstein_twotail}, where the Bernstein condition is given in terms of non-centered random variables, and \item using the covering numbers of \cite{cucker2002mathematical}.\end{enumerate*} We will use this theorem in proving Theorem \ref{thm:main} for controlling the approximation error of the derivatives of \RFF s.

Given $\bp \in \bbN^d$, we denote $\abs{\bp} = p_1 + \ldots + p_d$, for a function $f: \bbR^d \to \bbR$ the $\bp$-th partial derivative by $\partial^\bp f(\bz) = \frac{\partial^{\abs{\bp}} f(\bz)}{\partial^{p_1} z_1 \ldots \partial^{p_d} z_d}$, for a vector $\bw \in \bbR^d$ the $\bp$-th power by $\bw^\bp = w_1^{p_1}\cdots w_d^{p_d}$.

\begin{theorem}[Concentration inequality for \RFF{} kernel derivatives] \label{thm:rff_derivative_approx}
    Let $\bp, \bq \in \bbN^d$ and $\kernel: \bbR^d \times \bbR^d \to \bbR$ be a continuous, bounded, translation-invariant kernel such that $\bz \mapsto \nabla \bracks{\partial^{\bp, \bq} \kernel(\bz)}$ is continuous. Let $\cX \subset \bbR^d$ be a compact and convex domain with diameter $\abs{\cX}$, and denote by $D_{\bp, \bq, \cX} = \sup_{\bz \in \cX_{\Delta}} \norm{\nabla \bracks{\partial^{\bp, \bq} \kernel(\bz)}}_{2}$, where $\cX_{\Delta} = \curls{\bx - \by \given \bx, \by \in \cX}$.
    Let $\Lambda$ be the spectral measure of $\kernel$ satisfying that $E_{\bp, \bq} = \bbE_{\bw \sim \Lambda}\bracks{\abs{\bw^{\bp + \bq}} \norm{\bw}_2} < \infty$, and the Bernstein moment condition for some $S, R > 0$,
    \begin{align} \label{eq:rfsf:moment_growth}
        \bbE_{\bw \sim \Lambda}\bracks{\abs{\bw^{\bp + \bq}}^k} \leq \frac{k! S^2 R^{k-2}}{2} \quad \text{for all} \spc k \geq 2.
    \end{align}
    Then, for $C_{\bp, \bq, \cX} = \abs{\cX} (D_{\bp, \bq, \cX} + E_{\bp, \bq})$ and $\epsilon > 0$, 
    \begin{align} \label{eq:rfsf:rff_derivative_ineq}
        \bbP&\bracks{
            \sup_{\bx, \by \in \cX}\abs{\partial^{\bp, \bq}\rffkernel(\bx, \by) - \partial^{\bp, \bq}\kernel(\bx, \by)}
            \geq
            \epsilon}
        \leq
        16\pars{\frac{C_{\bp,\bq,\cX}}{\epsilon}}^{\frac{d}{d+1}}\exp\pars{\frac{-\dimRFF \epsilon^2}{4(d+1)(2S^2 + R\epsilon)}}.
    \end{align}
\end{theorem}
\begin{proof} We adapt the proof of \cite{rahimi2007random,sriperumbudur2015optimal}. Note that as $\cX$ is compact, so is $\cX_\Delta$, and it can be covered by an $\epsilon$-net of at most $T= (4\abs{\cX}/r)^d$ balls of radius $r > 0$ \cite[Prop.~5]{cucker2002mathematical} with centers $\bz_1, \dots, \bz_T \in \cX_\Delta$. Since for all $\bz \in \cX_\Delta$ there exists $i \in \{1, \dots, T\}$ such that $\norm{\bz-\bz_i}_2 \leq r$, it holds for $f(\bz) = \partial^{\bp,\bq}\rffkernel(\bz) - \partial^{\bp, \bq}\kernel(\bz)$ and $L_f = \sup_{\bs \in \cX_\Delta} \norm{\nabla f(\bs)}_2$ that
    \begin{align} \label{eq:rfsf:delta_error}
        \abs{f(\bz) - f(\bz_i)} \stackrel{\text{(a)}}{\leq} \sup_{\bs \in \cX_\Delta} \norm{\nabla f(\bs)}_2 \norm{\bz - \bz_i}_2 \stackrel{\text{(b)}}{\leq} \sup_{\bs \in \cX_\Delta} \norm{\nabla f(\bs)}_2 r = L_f r,
    \end{align}
    where (a) is due to the mean-value theorem followed by Cauchy-Schwarz inequality, and (b) is since $\norm{\bz - \bz_i} \leq r$. Now, by triangle inequality, it holds for any $\bz \in \cX_\Delta$ that 
    \begin{align} \label{eq:rfsf:grad_error}
    \norm{\nabla f(\bz)}_2 \leq \norm{\nabla \bracks{\partial^{\bp, \bq}\rffkernel(\bz)}}_2 +\norm{\nabla \bracks{\partial^{\bp, \bq}\kernel(\bz)}}_2 \leq \norm{\nabla \bracks{\partial^{\bp, \bq}\rffkernel(\bz)}}_2 + D_{\bp, \bq, \cX}.
    \end{align}
    Differentiating $\rffkernel(\bz)$ as defined in \eqref{eq:rfsf:rffkernel_def}, we get by the chain rule and triangle inequality that
    \begin{align} \label{eq:rfsf:grad_rff}
        \norm{\nabla \bracks{\partial^{\bp, \bq}\rffkernel(\bz)}}_2 = \frac{1}{\dimRFF}\norm{\sum_{j=1}^{\dimRFF} \bw_j \bw_j^{\bp + \bq} \cos^{(\abs{\bp + \bq})}(\bw_j^\top \bz)}_2 \leq \frac{1}{\dimRFF} \sum_{j=1}^{\dimRFF} \abs{\bw_j^{\bp + \bq}} \norm{\bw_j}_2,
    \end{align}
    where $\cos^{(n)}$ denotes the $n$-th derivative of $\cos$ for $n \in \bbN$.
    The main idea of the proof is then
    \begin{align}
        \bigcap_{i=1}^T \{\abs{f(\bz_i)} < \epsilon/2\} \bigcap \{L_f < \epsilon/2r\} \subseteq \{\abs{f(\bz)} < \epsilon \given \forall \bz \in \cX_\Delta\},
    \end{align}
    since $\abs{f(\bz)} \leq \abs{f(\bz) - f(\bz_i)} + \abs{f(\bz_i)}$ and \eqref{eq:rfsf:delta_error} holds. Therefore, taking the complement and bounding the union, we get our governing inequality for the uniform error over $\cX_\Delta$:
    \begin{align}
        \prob{\sup_{\bz \in \cX_\Delta} \abs{f(\bz)} \geq \epsilon} \leq \sum_{i=1}^T\prob{\abs{f(\bz_i)} \geq \epsilon/2} + \prob{L_f \geq \epsilon/2r}. \label{eq:rfsf:error_decomp}
    \end{align}
    Now, we need to bound all probabilities on the RHS. First, we deal with $L_f$, 
    \begin{align}
        \prob{L_f \geq \epsilon/2r} 
        &\stackrel{\text{(c)}}{\leq}
        \expe{L_f} \frac{2r}{\epsilon} \stackrel{\text{(d)}}{\leq} \pars{D_{\bp,\bq,\cX} + \frac{1}{\dimRFF}\sum_{j=1}^{\dimRFF}\bbE_{\bw_j \sim \Lambda}\bracks{\abs{\bw_j^{\bp + \bq}} \norm{\bw_j}_2}}\frac{2r}{\epsilon}
        \\
        &\leq
        \pars{D_{\bp,\bq, \cX} + E_{\bp, \bq}} \frac{2r}{\epsilon},
        \label{eq:rfsf:lipschitz_bound}
    \end{align}
    where (c) is Markov's inequality, (d) is \eqref{eq:rfsf:grad_error} and \eqref{eq:rfsf:grad_rff}.
    To deal with the centers in \eqref{eq:rfsf:error_decomp}, note $\partial^{\bp, \bq} \rffkernel(\bz)$ can be written as a sample average of $\dimRFF$ $\iid$ terms as per \eqref{eq:rfsf:rffkernel_def} since
    \begin{align}
        \partial^{\bp, \bq} \rffkernel(\bz) = \partial^{\bp, \bq} \pars{\frac{1}{\dimRFF} \sum_{i=1}^{\dimRFF} \cos(\bw_j^\top \bz)} =  \frac{1}{\dimRFF} \sum_{i=1}^{\dimRFF} \partial^{\bp, \bq} \cos(\bw_j^\top \bz)
    \end{align}
    so that the Bernstein inequality (Cor.~\ref{thm:bernstein_twotail}) is applicable. For $j = 1, \dots, \dimRFF$, we have
    \begin{align}
    \bbE_{\bw_j \sim \Lambda}{\abs{\partial^{\bp, \bq} \cos\pars{\bw_j^\top \bz}}^m}
    =
    \bbE_{\bw_j \sim \Lambda}{\abs{\bw_j^{\bp + \bq} \cos^{(\abs{\bp + \bq})}(\bw_j^\top \bz)}^m}
    \leq
    \bbE_{\bw_j \sim \Lambda}{\abs{\bw_j^{\bp + \bq}}^m}
    \leq
    \frac{k!S^2R^{k-2}}{2},
    \end{align}
    and that $f(\bz) = \partial^{\bp, \bq} \rffkernel(\bz) - \partial^{\bp, \bq} \kernel(\bz) = \partial^{\bp, \bq} \rffkernel(\bz) - \expe{\partial^{\bp, \bq} \rffkernel(\bz)}$ by the dominated convergence theorem.
    Hence, we may call the Bernstein inequality (Cor.~\ref{thm:bernstein_twotail}) to control $f(\bz_i)$ so that
    \begin{align} \label{eq:rfsf:center_bound}
        \prob{\abs{f(\bz_i)} \geq \epsilon / 2}
        =
        \prob{\abs{\partial^{\bp, \bq} \rffkernel(\bz) - \partial^{\bp, \bq} \kernel(\bz)} \geq \epsilon / 2}
        \leq
        2 \exp\pars{\frac{-\dimRFF\epsilon^2}{4(2S^2 + R\epsilon)}}.
    \end{align}
    Combining the bounds for $\abs{f(\bz_i)}$ \eqref{eq:rfsf:center_bound}, and for $L_f$ \eqref{eq:rfsf:lipschitz_bound}, into \eqref{eq:rfsf:error_decomp} yields
    \begin{align}
        \prob{\sup_{\bz \in \cM_\Delta} \abs{f(\bz)} \geq \epsilon} \leq 2 \pars{\frac{4 \abs{\cX}}{r}}^d\exp\pars{\frac{-\dimRFF \epsilon^2}{4(2S^2 + R\epsilon)}} + \pars{D_{\bp, \bq, \cX} + E_{\bp, \bq}} \frac{2r}{\epsilon},
    \end{align}
    which has the form $g(r) = \tau_1 r^{-d} + \tau_2 r$, that is minimized by choosing $r^\star = (d\tau_1 / \tau_2)^{\frac{1}{d+1}}$. This choice sets it to the form $g(r^\star) = \tau_1^{\frac{1}{d+1}}\tau_2^{\frac{d}{d+1}} \pars{d^{\frac{1}{d+1}} + d^{\frac{-d}{d+1}}}$, so that by substituting back in
    \begin{align}
        \prob{\sup_{\bz \in \cM_\Delta} \abs{f(\bz)} \geq \epsilon}
        &\leq
        F_d 2^\frac{3d+1}{d+1} \pars{\frac{\abs{\cX}(D_{\bp, \bq, \cX} + E_{\bp, \bq})}{\epsilon}}^\frac{d}{d+1} \exp\pars{\frac{-\dimRFF \epsilon^2}{4(d+1)(2S^2 + R\epsilon)}}. \label{eq:rfsf:rff_proof_final_bound}
    \end{align}
    Finally, we note that for $d \geq1$, $F_d = d^\frac{1}{d+1} + d^\frac{-d}{d+1} \leq 2$ and $2^\frac{3d+1}{d+1} \leq 8$.
\end{proof}

\section{Bounds on Signature Kernels} \label{apx:sig_bounds}
We first set the ground for proving our main theorems by introducing the notion of $L$-Lipschitz kernels to help control distance distortions in the feature space. This subclass of kernels will be useful for us in relating the $1$-variation of sequences in feature space to that in the input space.
After this, we will prove various smaller lemmas and supplementary results for signature kernels, which will lead up to Lemma \ref{lem:RFSF_approx}, which is our main tool for proving Theorem \ref{thm:main}, and it relates the concentration of our \RFSF{} kernel $\rffsigkernel$ to the second derivatives of the \RFF{} kernel $\rffkernel$. Then, the proof of Theorem \ref{thm:main} quantifying the concentration of the \RFSF{} kernel, $\rffsigkernel$, will follow from putting together Lemma \ref{lem:RFSF_approx} with Theorem \ref{thm:rff_derivative_approx}, and we will also make use of the Bernstein inequality from Theorem \ref{thm:bernstein_onetail}. The proof of Theorem \ref{thm:main2} for the RFSF-DP kernel, $\rffsigkernelDP$, will follow by combining the results of this section with $\alpha$-exponential concentration, in particular, Theorem \ref{thm:alpha_subexp_concentration}. Finally, Theorem \ref{thm:main3} for the RFSF-TRP kernel, $\rffsigkernelTRP$, will be proven using lemmas from this section, and the hypercontractivity result from Theorem \ref{thm:hyper_concentration}.

\subsection{Distance bounds in the RKHS}
\begin{definition}[Lipschitz kernel]\label{def:lipschitz kernel} Let $(\cX, d)$ be a metric space. We call a kernel $\kernel: \cX \times \cX \to \bbR$ with RKHS $\Hil$ an $L$-Lipschitz kernel over $\cX$ for some $L > 0$ if it holds for all $\bx, \by \in \cX$ that
    \begin{align}
        \norm{\kernel_\bx - \kernel_\by}_{\Hil} = \sqrt{\kernel(\bx, \bx) + \kernel(\by, \by) - 2 \kernel(\bx, \by)}  \leq L d(\bx, \by),
    \end{align}
    where $\kernel_\bx= \kernel(\bx,\cdot), \kernel_\by= \kernel(\by,\cdot) \in \Hil$.
\end{definition}

\begin{example}[Finite 2\textsuperscript{nd} spectral moment implies Lipschitz]\label{example:2ndmoment_lip}
    Let $\kernel: \bbR^d \times \bbR^d \to \bbR$ be a continuous, bounded and translation-invariant kernel with RKHS $\Hil$ and spectral measure $\Lambda$, such that $\sigma_\Lambda^2 = \bbE_{\bw \sim \Lambda}\bracks{\norm{\bw}^2_2} < \infty$. Then, it holds that $\kernel$ is $\norm{\bbE_{\bw \sim \Lambda}\bracks{\bw \bw^\top}}_2^{1/2}$-Lipschitz, so for any $\b x,\b y\in\bbR^d$ one has
    \begin{align}
        \norm{\kernel_\bx - \kernel_\by}_\Hil \leq \norm{\bbE_{\bw \sim \Lambda}\bracks{\bw\bw^\top}}_2^{1/2}\norm{\bx - \by}_2.
    \end{align}
\end{example}
\begin{proof}
    We have for $\bx, \by \in \bbR^d$ that
    \begin{align}
        \norm{\kernel_\bx - \kernel_\by}_\Hil^2 &\stackrel{\text{(a)}}{=} \kernel(\bx, \bx) + \kernel(\by, \by) - 2 \kernel(\bx, \by) \stackrel{\text{(b)}}{=} \int_{\bbR^d} 2 - 2\exp\pars{i \bw^\top(\bx - \by)} \d \Lambda(\bw)
        \\
        &\stackrel{\text{(c)}}{=} \int_{\bbR^d} 2 - 2\cos\pars{\bw^\top(\bx - \by)} \d \Lambda(\bw)
        \stackrel{\text{(d)}}{\leq} \int_{\bbR^d} \pars{\bw^\top(\bx - \by)}^2 \d \Lambda(\bw)
        \\
        &\stackrel{\text{(e)}}{=} (\bx - \by)^\top \bbE_{\bw \sim \Lambda}\bracks{\bw \bw^\top} (\bx - \by) \stackrel{\text{(f)}}{\leq} \norm{\bbE_{\bw \sim \Lambda}\bracks{\bw \bw^\top}}_2 \norm{\bx - \by}_2^2,
    \end{align}
    where (a) holds due to the reproducing property, (b) is Bochner's theorem, (c) is because the imaginary part of the integral evaluates to $0$ as the kernel is real-valued, (d) is due to the inequality $1 - t^2/2 \leq \cos(t)$ for all $t \in \bbR$, (e) is because $\pars{\bw^\top(\bx-\by)}^2 = (\bx-\by)^\top \bw \bw^\top (\bx - \by)$, and (f) is Cauchy-Schwarz inequality combined with the definition of the spectral norm.
\end{proof}

\begin{example}[Random Lipschitz bound for \RFF{}] \label{example:rff_lip}
    Let $\rffkernel:\bbR^d \times \bbR^d \to \bbR$ be an \RFF{} kernel defined as in \eqref{eq:rfsf:rff_def} corresponding to some spectral measure $\Lambda$, and let $\HilRFF$ denote its feature space corresponding to the \RFF{} map $\rff: \bbR^d \to \HilRFF$ defined as in \eqref{eq:rfsf:rff_def}, so that given $\bw_1, \dots, \bw_{\dimRFF} \stackrel{\iid}{\sim} \Lambda$, we have for $\bx, \by \in \bbR^d$ that
    \begin{align}
        \rffkernel(\bx, \by) = \frac{1}{\dimRFF} \sum_{j=1}^{\dimRFF} \cos\pars{\bw_j^\top(\bx - \by)}.
    \end{align}
    Let $\bW = (\bw_1, \dots, \bw_{\dimRFF}) \in \bbR^{d \times \dimRFF}$ be the random matrix with column vectors $\bw_1, \dots, \bw_{\dimRFF} \stackrel{\iid}{\sim} \Lambda$. Then, $\rffkernel$ is $\pars{\frac{\norm{\bW}_2}{\sqrt{\dimRFF}}}$-Lipschitz, so that for any $\bx, \by \in \bbR^d$, we have the inequality
    \begin{align}
        \norm{\rff(\bx) - \rff(\by)}_{\HilRFF} \leq \frac{\norm{\bW}_2}{\sqrt{\dimRFF}} \norm{\bx - \by}_2.
    \end{align} 
\end{example}
\begin{proof}
    The proof follows analogously to that of Example \ref{example:2ndmoment_lip}. Let $\bx, \by \in \bbR^d$, then
    \begin{align}
        \norm{\rff(\bx) - \rff(\by)}_{\HilRFF}^2 &\stackrel{\text{(a)}}{=} 2 - \frac{2}{\dimRFF} \sum_{i=1}^{\dimRFF} \cos(\bw_i^\top (\bx-\by)) \stackrel{\text{(b)}}{\leq} \frac{1}{\dimRFF} \sum_{i=1}^{\dimRFF} \pars{\bw_i^\top (\bx - \by)}^2 \\
        &\stackrel{\text{(c)}}{=} \frac{1}{\dimRFF} (\bx - \by)^\top \bW \bW^\top (\bx - \by) \stackrel{\text{(d)}}{\leq} \frac{1}{\dimRFF} \norm{\bW}^2_2 \norm{\bx - \by}^2_2,
    \end{align}
    where (a) is due to the cosine identity $\cos(a-b) = \cos(a)\cos(b) + \sin(a)\sin(b)$ for all $a,b \in \bbR$, (b) is due to the inequality $1 - t^2/2 \leq \cos(t)$ for all $t \in \bbR$, (c) is because $\pars{\bw_i^\top(\bx-\by)}^2 = (\bx-\by)^\top \bw_i \bw_i^\top (\bx - \by)$ and $\bW \bW^\top = \sum_{i=1}^{\dimRFF} \bw_i \bw_i^\top$, (d) is due to the Cauchy-Schwarz inequality combined with the definition of the spectral norm.
\end{proof}
\paragraph{Bounds for the Signature Kernel}
A well-known property of signature features that they decay factorially fast with respect to the tensor level $m \in \bbN$.
\begin{lemma}[Norm bound for signature features] \label{lem:1var_sig_norm}
    Let $L > 0$ and $\kernel: \cX \times \cX \to \bbR$ be an $L$-Lipschitz kernel with RKHS $\Hil$. Then, we have for the level-$m$ signature features $\signature[m](\bx)$ of the sequence $\bx  \in \Seq(\cX)$ that
    \begin{align}
        \norm{\signature[m](\bx)}_{\Hil^{\otimes m}}
        \leq
        \frac{\pars{L \norm{\bx}_{\onevar}}^m}{m!}.
    \end{align}
\end{lemma}
\begin{proof}
    We have
    \begin{align}
        \norm{\signature[m](\bx)}_{\Hil^{\otimes m}} 
        &=
        \norm{\sum_{\bi \in \Delta_m(\len{\bx})} \delta \kernel_{\bx_{i_1}} \otimes \cdots \otimes \delta\kernel_{\bx_{i_m}}}_{\cH^{\otimes m}}
        \\
        &\stackrel{\text{(a)}}{\leq}
        \sum_{\bi \in \Delta_m(\len{\bx})} \norm{\delta\kernel_{\bx_{i_1}}}_\cH \cdots \norm{\delta \kernel_{\bx_{i_m}}}_\cH
        \\
        &\stackrel{\text{(b)}}{\leq} L^m \sum_{\bi \in \Delta_m(\len{\bx})} \norm{\delta\bx_{i_1}}_2 \cdots \norm{\delta \bx_{i_m}}_2
        \\
        &\stackrel{\text{(c)}}{\leq}
        \frac{\pars{L \sum_{i=1}^{\len{\bx}} \norm{\delta \bx_i}_2}^m}{m!}
        \stackrel{\text{(d)}}{=} \frac{\pars{L \norm{\bx}_\onevar}^m}{m!}, \label{line:signorm2}
    \end{align}
    where (a) follows from triangle inequality followed by factorizing the tensor norm, (b) is the $L$-Lipschitzness property, (c) from completing the multinomial expansion and normalizing by the number of permutations (note that $\Delta_m(\len{\bx})$ contains a single permutation of all such multi-indices that have nonrepeating entries), and (d) is the definition of sequence $1$-variation.
\end{proof}

The following bound for the signature kernel is a direct consequence of the previous lemma. 
\begin{corollary}[Upper bound for signature kernel] \label{lem:ksig_bound}
Let $\kernel: \cX \times \cX \to \bbR$ be an $L$-Lipschitz kernel, and $\sigkernel[m]: \Seq(\cX) \times \Seq(\cX) \to \bbR$ the level-$m$ signature kernel built from $\kernel$. Then, we have the following bound for $\bx, \by \in \Seq(\cX)$
    \begin{align}
        \abs{\sigkernel[m](\bx, \by)}
        \leq \frac{\pars{L^2 \norm{\bx}_\onevar \norm{\by}_\onevar}^m}{(m!)^2}.
    \end{align}
\end{corollary}
\begin{proof}
    Note that without a kernel trick, $\sigkernel[m]$ is written for $\bx, \by \in \Seq(\cX)$ as the inner product
    \begin{align}
        \abs{\sigkernel[m](\bx, \by)}
        &=
        \abs{\inner{\sum_{\bi \in \Delta_m(\len{\bx})} \delta \kernel_{\bx_{i_1}} \otimes \cdots \otimes \delta \kernel_{\bx_{i_m}}}{\sum_{\bj \in \Delta_m(\len{\by})} \delta \kernel_{\by_{j_1}} \otimes \cdots \otimes \delta \kernel_{\by_{j_m}}}_{\Hil^{\otimes m}}}
        \\
        &\stackrel{\text{(a)}}{\leq} \norm{\sum_{\bi \in \Delta_m(\len{\bx})} \delta \kernel_{\bx_{i_1}} \otimes \cdots \otimes \delta \kernel_{\bx_{i_m}}}_{\Hil^{\otimes m}}\norm{\sum_{\bj \in \Delta_m(\len{\by})} \delta \kernel_{\by_{j_1}} \otimes \cdots \otimes \delta \kernel_{\by_{j_m}}}_{\Hil^{\otimes m}}
        \\
        &\stackrel{\text{(b)}}{\leq} \frac{\pars{L^2 \norm{\bx}_\onevar \norm{\by}_\onevar}^m}{(m!)^2},
    \end{align}
    where (a) follows from the Cauchy-Schwarz inequality, and (b) is implied by Lemma \ref{lem:1var_sig_norm}.
\end{proof}

A similar upper bound to Lemma \ref{lem:1var_sig_norm} also holds for the \RFSF{} kernel $\rffsigkernel[m]$, that now depends on the norms of the random matrices $\bW^{(1)}, \dots, \bW^{(m)}$, hence is itself random.

\begin{lemma}[Random norm bound for \RFSF{}] \label{lem:1var_rffsig_norm}
    Let $\rffsig[m]: \Seq(\cX) \to \HilRFFT$ be the level-$m$ \RFSF{} map defined as in \eqref{eq:rfsf:rffsigdef} built from some spectral measure $\Lambda$. Then, we have for $\bx \in \Seq(\cX)$ that
    \begin{align}
        \norm{\rffsig[m](\bx)}_{{\HilRFF}^{\otimes m}} \leq \frac{\norm{\bW^{(1)}}_2 \cdots \norm{\bW^{(m)}}_2}{m!} \pars{\frac{\norm{\bx}_\onevar}{\sqrt{\dimRFF}}}^m,
    \end{align}
    where $\bW^{(1)}, \dots, \bW^{(m)} \stackrel{\iid}{\sim} \Lambda^{\dimRFF}$ are random matrices sampled from $\Lambda^{\dimRFF}$.
\end{lemma}
\begin{proof}
    The proof follows analogously to Lemma \ref{lem:1var_sig_norm}. We have for $\bx \in \Seq(\cX)$ that
    \begin{align}
        \norm{\rffsig[m](\bx)}_{{\HilRFF}^{\otimes m}}
        &=
        \norm{\sum_{\bi \in \Delta_m(\len{\bx})} \delta \rff_1(\bx_{i_1}) \otimes \cdots \otimes \delta \rff_m(\bx_{i_m})}_{\HilRFF^{\otimes m}}
        \\
        &\stackrel{\text{(a)}}{\leq}
        \sum_{\bi \in \Delta_m(\len{\bx})} \norm{\delta \rff_1(\bx_{i_1})}_{\HilRFF} \cdots \norm{\delta \rff_m(\bx_{i_m})}_{\HilRFF}
        \\
        &\stackrel{\text{(b)}}{\leq}
        \frac{\norm{\bW^{(1)}}_2 \cdots \norm{\bW^{(m)}}_2}{\dimRFF^{m/2}} \sum_{\bi \in \Delta_m(\len{\bx})} \norm{\delta \bx_{i_1}}_2 \cdots \norm{\delta \bx_{i_m}}_2
        \\
        &\stackrel{\text{(c)}}{\leq}
        \frac{\norm{\bW^{(1)}}_2 \cdots \norm{\bW^{(m)}}_2}{\dimRFF^{m/2}} \frac{\norm{\bx}_{\onevar}^m}{m!},
    \end{align}
    where (a) is the triangle inequality and factorization of tensor norm, (b) is using the Lipschitzness of \RFF s from Example \ref{example:rff_lip}, (c) is the same as steps (c)-(d) in Lemma \ref{lem:1var_sig_norm}.
\end{proof}
Then, the following is again an application of the Cauchy-Schwarz inequality.
\begin{corollary}[Random upper bound for \RFSF{} kernel]
\label{lem:krffsig_bound}
    Let $\rffsigkernel[m]: \Seq(\cX) \times \Seq(\cX) \to \bbR$ be the level-$m$ \RFSF{} kernel defined as in \eqref{eq:rfsf:rffsigkernel_def} built from some spectral measure $\Lambda$.  Then, for all $\bx, \by \in \Seq(\cX)$
    \begin{align}
        \abs{\rffsigkernel[m](\bx, \by)}
        &\leq
        \frac{\norm{\bW^{(1)}}_2^2 \cdots \norm{\bW^{(m)}}_2^2}{(m!)^2}  \pars{\frac{\norm{\bx}_\onevar \norm{\by}_\onevar}{\dimRFF}}^m
    \end{align}
    where $\bW^{(1)}, \dots, \bW^{(m)} \stackrel{\iid}{\sim} \Lambda^{\dimRFF}$ are random matrices sampled from $\Lambda^{\dimRFF}$.
\end{corollary}

The following lemma does not concern signatures, but will be useful to us later in the proof of Lemma \ref{lem:RFSF_approx} by providing a mean-value theorem for the cross-differencing operator $\delta_{i,j}$.
\begin{lemma}[2\textsuperscript{nd} order mean-value theorem] \label{lem:second_mvt}
Let $f: \bbR^d \times \bbR^d \to \bbR$ be a twice differentiable function, and $\cX \subset \bbR^d$ be a convex and compact set. Then, we have for any $\bu, \bv, \bx, \by \in \cX$ that
\begin{align}
    f(\bx, \by) - f(\bx, \bv) - f(\bu, \by) + f(\bu, \bv) \leq \sup_{\bs, \bt \in \cX} \norm{\partial^2_{\bs, \bt} f(\bs, \bt)}_2 \norm{\bx - \bu}_2 \norm{\by - \bv}_2,
\end{align}
where $\partial^2_{\bs, \bt} f(\bs, \bt) = \pars{\frac{\partial^2 f(\bs, \bt)}{\partial s_i \partial t_j}}_{i,j=1}^d$ refers to the submatrix of the Hessian of cross-derivatives.
\begin{proof}
    Keeping $\bv, \by \in \cX$ as fixed, we may define $g: \bbR^d \to \bbR$ as $g(\cdot) = f(\cdot, \by) - f(\cdot, \bv)$, so that the expression above can be written as $g(\bx) - g(\bu)$, and by the convexity of $\cX$, we may apply the mean-value theorem to find that $\exists s \in (0, 1)$ such that
    \begin{align}
        g(\bx) - g(\bu) = \inner{\nabla g(s \bx + (1-s) \bu)}{ \bx - \bu} \leq \sup_{\bs \in \cX}\norm{\nabla g(\bs)}_2 \norm{\bx - \bu}_2, \label{eq:rfsf:mvt1}
    \end{align}
    where $\nabla g(\bs) = \pars{\frac{\partial g(\bs)}{\partial s_i}}_{i=1}^d$ denotes the gradient of $g$, while the second inequality follows from the Cauchy-Schwarz inequality and the compactness of $\cX$ (so that the $\sup$ exists). Also note that $\nabla g(\bs) = \partial_\bs f(\bs, \by) - \partial_\bs f(\bs, \bv)$, so defining $h: \bbR^d \to \bbR^d$ as $h(\cdot) = \partial_\bs f (\bs, \cdot)$ and applying the vector-valued mean-value inequality \cite[Thm.~9.19]{rudin1976principles} to $h$ gives that $\exists t \in (0, 1)$ such that
    \begin{align}
        \norm{h(\by) - h(\bv)}_2 \leq \norm{\Jac_h(t\by + (1-t)\bv)}_2 \norm{\by - \bv}_2 \leq \sup_{\bt \in \cX} \norm{\Jac_h(\bt)}_2 \norm{\by - \bv}_2, \label{eq:rfsf:mvt2}
    \end{align}
    where $\Jac_h(\bt) = \pars{\frac{\partial h_i (\bt)}{\partial t_j}}_{i,j=1}^d$ refers to the Jacobian of $h$. Putting the inequalities \eqref{eq:rfsf:mvt1} and \eqref{eq:rfsf:mvt2} together and substituting back the function $f$ gives the desired result.
\end{proof}
\end{lemma}

This is our final lemma in our exposition of supplementary results about the (random) signature kernels, and it will be our main tool for proving Theorem \ref{thm:main}.
\begin{lemma}[Uniform upper bound for deviation of \RFSF{} kernel] \label{lem:RFSF_approx}
    Let $\cX \subset \bbR^d$ be a convex and compact set, $\kernel: \bbR^d \times \bbR^d \to \bbR$ a continuous, bounded, translation-invariant $L$-Lipschitz kernel and $\rffkernel: \bbR^d \times \bbR^d \to \bbR$ the corresponding \RFF{} kernel.
    Then, the level-$m$ ($m\in \bbN$) signature and \RFSF{} kernels are uniformly close for $V>0$ by
    \begin{align}
        &\sup_{\substack{\bx, \by \in \Seq(\cX) \\ \norm{\bx}_\onevar, \norm{\by}_\onevar \leq V}} \abs{\rffsigkernel[m](\bx, \by) - \sigkernel[m](\bx, \by)}
        \\
        &\leq V^{2m} \sum_{k=1}^m \frac{L^{2(m-k)}}{\dimRFF^{k-1}((k-1)!)^2} \norm{\bW^{(1)}}_2^2 \cdots \norm{\bW^{(k-1)}}_2^2 \sup_{\bs, \bt \in \cX} \norm{\partial^2_{\bs,\bt} \rffkernel_{k}(\bs, \bt) - \partial^2_{\bs,\bt} \kernel(\bs, \bt)}_2, 
    \end{align}
    where $\rffkernel_{1}, \ldots, \rffkernel_{m}$ are independent \RFF{} kernels with weights $\bW^{(1)}, \dots, \bW^{(m)} \stackrel{\iid}{\sim} \Lambda^{\dimRFF}$, and $\partial^2_{\bs, \bt} f(\bs, \bt) = \pars{\frac{\partial^2 f(\bs, \bt)}{\partial s_i \partial t_j}}_{i,j=1}^d$ for a twice-differentiable function $f:\bbR^d \times \bbR^d \to \bbR$.
\end{lemma}
\begin{proof}
First of all, by Lemma \ref{lem:ksig_bound} and Lemma \ref{lem:krffsig_bound}, the supremum exists.
In the following, given a sequence $\bx \in \Seq(\cX)$, we denote its $0:l$ slice for some $l \in [\len{\bx}]$ by $\bx_{:l} = (\bx_0, \dots, \bx_l)$.
Then, it holds for any $m \geq 1$ recursively for the signature kernel that
\begin{align} \label{eq:rfsf:sigkernel_rec}
    \sigkernel[m](\bx, \by) = \sum_{k=1}^{\len{\bx}} \sum_{l=1}^{\len{\by}} \sigkernel[m-1](\bx_{0:k-1}, \by_{0:l-1}) \delta_{k, l} \kernel(\bx_k, \by_l),
\end{align}
and analogously for the \RFSF{} kernel that
\begin{align} \label{eq:rfsf:rffsigkernel_rec}
    \rffsigkernel[m](\bx, \by) = \sum_{k=1}^{\len{\bx}} \sum_{l=1}^{\len{\by}} \rffsigkernel[m-1](\bx_{0:k-1}, \by_{0:l-1}) \delta_{k,l} \rffkernel_{m}(\bx_k, \by_l).
\end{align}
Combining these recursions together, we have for the uniform error that
\begin{align}
    \epsilon_m
    =& \sup_{\substack{\bx, \by \in \Seq(\cX)\\ \norm{\bx}_\onevar, \norm{\by}_\onevar \leq V}} \abs{\rffsigkernel[m](\bx, \by) - \sigkernel[m](\bx, \by)}
    \\
    =& \sup_{\substack{\bx, \by \in \Seq(\cX)\\ \norm{\bx}_\onevar, \norm{\by}_\onevar \leq V}} \abs{\sum_{k=1}^{\len{\bx}} \sum_{l=1}^{\len{\by}} \rffsigkernel[m-1](\bx_{0:k-1}, \by_{0:l-1}) \delta_{k,l}\rffkernel_{m}(\bx_k, \by_l) - \sigkernel[m-1](\bx_{0:k-1}, \by_{0:l-1}) \delta_{k,l}\kernel(\bx_k, \by_l)}
    \\
    \stackrel{\text{(a)}}{\leq}& \sup_{\substack{\bx, \by \in \Seq(\cX)\\ \norm{\bx}_\onevar, \norm{\by}_\onevar \leq V}} \abs{\sum_{k=1}^{\len{\bx}} \sum_{l=1}^{\len{\by}} \rffsigkernel[m-1](\bx_{0:k-1}, \by_{0:l-1}) (\delta_{k,l}\rffkernel_{m}(\bx_k, \by_l) - \delta_{k,l}\kernel(\bx_k, \by_l))}
    \\
    &+ \sup_{\substack{\bx, \by \in \Seq(\cX)\\ \norm{\bx}_\onevar, \norm{\by}_\onevar \leq V}} \abs{\sum_{k=1}^{\len{\bx}} \sum_{l=1}^{\len{\by}} (\rffsigkernel[m-1](\bx_{0:k-1}, \by_{0:l-1}) - \rffsigkernel[m-1](\bx_{0:k-1}, \by_{0:l-1})) \delta_{k,l}\kernel(\bx_k, \by_l)}
    \\
    \stackrel{\text{(b)}}{\leq}& \sup_{\substack{\bx, \by \in \Seq(\cX)\\ \norm{\bx}_\onevar, \norm{\by}_\onevar \leq V}} \sum_{k=1}^{\len{\bx}} \sum_{l=1}^{\len{\by}} \abs{\rffsigkernel[m-1](\bx_{0:k-1}, \by_{0:l-1})}\abs{\delta_{k,l}\rffkernel_{m}(\bx_k, \by_l) - \delta_{k,l}\kernel(\bx_k, \by_l)}
    \\
    &+ \sup_{\substack{\bx, \by \in \Seq(\cX)\\ \norm{\bx}_\onevar, \norm{\by}_\onevar \leq V}} \sum_{k=1}^{\len{\bx}} \sum_{l=1}^{\len{\by}} \abs{\rffsigkernel[m-1](\bx_{0:k-1}, \by_{0:l-1}) - \sigkernel[m-1](\bx_{0:k-1}, \by_{0:l-1})}\abs{\delta_{k,l}\kernel(\bx_k, \by_l)}
    \\
    \stackrel{\text{(c)}}{\leq}& \underbrace{\sup_{\substack{\bx, \by \in \Seq(\cX)\\ \norm{\bx}_\onevar, \norm{\by}_\onevar \leq V}} \abs{\rffsigkernel[m-1](\bx, \by)}}_{\text{(i)}}  \underbrace{\sup_{\substack{\bx, \by \in \Seq(\cX)\\ \norm{\bx}_\onevar, \norm{\by}_\onevar \leq V}} \sum_{k=1}^{\len{\bx}} \sum_{l=1}^{\len{\by}} \abs{\delta_{k,l}\rffkernel_{m}(\bx_k, \by_l) - \delta_{k,l}\kernel(\bx_k, \by_l)}}_{\text{(ii)}}
    \\
    &+ \underbrace{\sup_{\substack{\bx, \by \in \Seq(\cX)\\ \norm{\bx}_\onevar, \norm{\by}_\onevar \leq V}} \abs{\rffsigkernel[m-1](\bx, \by) - \sigkernel[m-1](\bx, \by)}}_{\text{(iii)}} \underbrace{\sup_{\substack{\bx, \by \in \Seq(\cX)\\ \norm{\bx}_\onevar, \norm{\by}_\onevar \leq V}} \sum_{k=1}^{\len{\bx}} \sum_{l=1}^{\len{\by}} \abs{\delta_{k,l}\kernel(\bx_k, \by_l)}}_{\text{(iv)}}, \label{eq:rfsf:term_splitting}
\end{align}
where (a) follows from adding and subtracting the cross-terms and applying triangle inequality, (b) follows from applying triangle inequality over the summations, (c) follows from noting that if $\norm{\bx}_\onevar, \norm{\by}_\onevar \leq V$ then so is $\norm{\bx_{0:k-1}}_\onevar, \norm{\by_{0:l-1}}_\onevar \leq V$ for $k \in [\len{\bx}]$ and $l \in [\len{\by}]$, and thus justifiably pulling out the supremums.

Now, we deal with terms (i)--(iv) individually. For (i), we have Corollary \ref{lem:krffsig_bound}, so
\begin{align}
    \sup_{\substack{\bx, \by \in \Seq(\cX)\\ \norm{\bx}_\onevar, \norm{\by}_\onevar \leq V}} \abs{\rffsigkernel[m-1](\bx, \by)} \leq \frac{\norm{\bW^{(1)}}_2^2 \cdots \norm{\bW^{(m-1)}}_2^2}{((m-1)!)^2} \pars{\frac{V^2}{\dimRFF}}^{m-1}. \label{eq:rfsf:term1}
\end{align}

To deal with (ii), we can apply Lemma \ref{lem:second_mvt} with $f = \rffkernel_{m} - \kernel$ to get
\begin{align}
    &\sup_{\substack{\bx, \by \in \Seq(\cX)\\ \norm{\bx}_\onevar, \norm{\by}_\onevar \leq V}} \sum_{k=1}^{\len{\bx}} \sum_{l=1}^{\len{\by}} \abs{\delta_{k,l} \rffkernel_{m}(\bx_k, \by_l) - \delta_{k,l} \kernel(\bx_k, \by_l)}
    \\
    &\leq \sup_{\bs, \bt \in \cX} \norm{\partial^2_{\bs,\bt} \rffkernel_{m}(\bs, \bt) - \partial^2_{\bs, \bt} \kernel(\bs, \bt)}_2 \sup_{\substack{\bx, \by \in \Seq(\cX)\\ \norm{\bx}_\onevar, \norm{\by}_\onevar \leq V}} \sum_{k=1}^{\len{\bx}} \sum_{l=1}^{\len{\by}} \norm{\delta \bx_k}_2 \norm{\delta \by_l}_2
    \\
    &\leq V^2 \sup_{\bs, \bt \in \cX} \norm{\partial^2_{\bs, \bt} \rffkernel_{m}(\bs, \bt) - \partial^2_{\bs,\bt} \kernel(\bs, \bt)}_2. \label{eq:rfsf:term2}
\end{align}
For (iii), we note that it is simply $\epsilon_{m-1}$.
Finally, we can write (iv) as an inner product and apply Cauchy-Schwarz and $L$-Lipschitzness of $\kernel$ so that
\begin{align}
    &\sup_{\substack{\bx, \by \in \Seq(\cX)\\ \norm{\bx}_\onevar, \norm{\by}_\onevar \leq V}} \sum_{k=1}^{\len{\bx}} \sum_{l=1}^{\len{\by}} \abs{\delta_{k,l} \kernel(\bx_k, \by_l)}
    = \sup_{\substack{\bx, \by \in \Seq(\cX)\\ \norm{\bx}_\onevar, \norm{\by}_\onevar \leq V}} \sum_{k=1}^{\len{\bx}} \sum_{l=1}^{\len{\by}} \abs{\inner{\delta \kernel_{\bx_k}} {\delta \kernel_{\by_l}}}
    \\
    &\leq \sup_{\substack{\bx, \by \in \Seq(\cX)\\ \norm{\bx}_\onevar, \norm{\by}_\onevar \leq V}} \sum_{k=1}^{\len{\bx}} \sum_{l=1}^{\len{\by}} \norm{\delta \kernel_{\bx_k}}_{\Hil}\norm{\delta \kernel_{\by_l}}_{\Hil} \leq L^2 V^2. \label{eq:rfsf:term3}
\end{align}
Putting equations \eqref{eq:rfsf:term1}, \eqref{eq:rfsf:term2}, \eqref{eq:rfsf:term3} together in \eqref{eq:rfsf:term_splitting}, we get that
\begin{align}
    \epsilon_m
    =& \sup_{\substack{\bx, \by \in \Seq(\cX)\\ \norm{\bx}_\onevar, \norm{\by}_\onevar \leq V}} \abs{\rffsigkernel[m](\bx, \by) - \sigkernel[m](\bx, \by)}
    \\
    \leq& \frac{V^{2m}}{\dimRFF^{m-1} ((m-1)!)^2} \sup_{\bs, \bt \in \cX} \norm{\partial^2_{\bs,\bt} \rffkernel_{m}(\bs, \bt) - \partial^2_{\bs,\bt} \kernel(\bs, \bt)}_2 \prod_{p=1}^{m-1} \norm{\bW^{(p)}}_2^2 + L^2 V^2 \epsilon_{m-1}, \label{eq:rfsf:terms_rec}
\end{align}
which gives us a recursion for estimating $\epsilon_m$. The initial step, $m=1$, can be estimated by
\begin{align}
    \epsilon_1
    &=
    \sup_{\substack{\bx, \by \in \Seq(\cX)\\ \norm{\bx}_\onevar, \norm{\by}_\onevar \leq V}} \abs{\rffsigkernel[1](\bx, \by) - \sigkernel[1](\bx, \by)}
    \\
    &=
    \sup_{\substack{\bx, \by \in \Seq(\cX)\\ \norm{\bx}_\onevar, \norm{\by}_\onevar \leq V}} \abs{\sum_{k=1}^{\len{\bx}} \sum_{l=1}^{\len{\by}} \delta_{k, l} \rffkernel_{1}(\bx_k, \by_l) - \delta_{k,l} \kernel(\bx_k, \by_l)}
    \\
    &=
    \sup_{\substack{\bx, \by \in \Seq(\cX)\\ \norm{\bx}_\onevar, \norm{\by}_\onevar \leq V}} \sum_{k=1}^{\len{\bx}} \sum_{l=0}^{\len{\by}} \abs{\delta_{k,l} \rffkernel_{1}(\bx_k, \by_l) - \delta_{k,l} \kernel(\bx_k, \by_l)}
    \\
    &\leq V^2 \sup_{\bs, \bt \in \cX} \norm{\partial^2_{\bs,\bt} \rffkernel_{1}(\bs, \bt) - \partial^2_{\bs, \bt} \kernel(\bs, \bt)}_2, \label{eq:rfsf:terms_init}
\end{align}
which is actually analogous to \eqref{eq:rfsf:terms_rec} since $\epsilon_0 = 0$. Now, we may unroll the recursion \eqref{eq:rfsf:terms_rec} with the initial condition \eqref{eq:rfsf:terms_init}, and we get
\begin{align}
    \epsilon_m \leq V^{2m} \sum_{k=1}^m \frac{L^{2(m-k)}}{\dimRFF^{k-1} ((k-1)!)^2} \sup_{\bs, \bt \in \cX} \norm{\partial^2_{\bs,\bt} \rffkernel_{k}(\bs, \bt) - \partial^2_{\bs,\bt} \kernel(\bs, \bt)}_2 \prod_{p=1}^{k-1} \norm{\bW^{(p)}}_2^2.
\end{align}
\end{proof}

\subsection{Proofs of main concentration results} Here, we provide proofs of the main concentration results, i.e.~Theorems \ref{thm:main}, \ref{thm:main2}, \ref{thm:main3}, respectively under Theorems \ref{thm:rfsf_approx}, \ref{thm:rfsf_dp_approx}, \ref{thm:rfsf_trp_approx}.
\begin{theorem}[Concentration inequality for \RFSF{} kernel] \label{thm:rfsf_approx}
    Let $\cX \subset \bbR^d$ be a compact and convex set with diameter $\abs{\cX}$, and $\cX_\Delta = \{\bx - \by : \bx, \by \in \cX \}$. Let $\kernel: \bbR^d \times \bbR^d \to \bbR$ be a continuous, bounded, translation-invariant kernel with spectral measure $\Lambda$, which satisfies for some $S, R > 0$ that
    \begin{align} \label{eq:rfsf:rfsf_approx_cond}
        \bbE_{\bw \sim \Lambda}\bracks{\abs{w_i}^{2k}} \leq \frac{k! S^2 R^{k-2}}{2} \quad \text{for all} \spc i \in [d] \spc \text{and} \spc k \geq 2.
    \end{align}

    Then, the following quantities are finite: $\sigma_\Lambda^2 = \bbE_{\bw \sim \Lambda}\bracks{\norm{\bw}_2^2}$, $L = \norm{\bbE_{\bw \sim \Lambda}\bracks{\bw \bw^\top}}_2^{1/2}$, $E_{i,j} = \bbE_{{\bw \sim \Lambda}}\bracks{\abs{w_i w_j} \norm{\bw}_2}$ and $D_{i,j} = \sup_{\bz \in \cX_\Delta} \norm{\nabla \bracks{\frac{\partial^2\kernel(\bz)}{\partial z_i \partial z_j}}}_2$ for $i,j \in [d]$. Further, for any maximal sequence $1$-variation $V>0$, and signature level $m \in \mathbb{Z}_+$, it holds for the level-$m$ \RFSF{} kernel $\rffsigkernel[m]: \Seq(\cX) \times \Seq(\cX) \to \bbR$ defined as in \eqref{eq:rfsf:rffsigkernel_def} and the signature kernel $\sigkernel[m]: \Seq(\cX) \times \Seq(\cX) \to \bbR$ defined as in \eqref{eq:rfsf:sigkernel_def} for $\epsilon > 0$ that
    \begin{align} \label{eq:rfsf:rfsf_approx_bound}
        \bbP & \bracks{\sup_{\substack{\bx, \by \in \Seq(\cX) \\ \norm{\bx}_\onevar, \norm{\by}_\onevar \leq V}} \abs{\sigkernel[m](\bx, \by) - \rffsigkernel[m](\bx, \by)} \geq \epsilon} \le
        \\
        &\leq
        m
        \begin{cases}
        \pars{C_{d, \cX} \pars{\frac{\beta_{d, m, V}}{\epsilon}}^\frac{d}{d+1} + d}
        \exp\pars{- \frac{\dimRFF}{2(d+1)(S^2 + R)} \pars{\frac{\epsilon}{\beta_{d, m, V}}}^{2}} \quad
        &\text{for} \spc \epsilon < \beta_{d, m, V}  \\
        \pars{C_{d, \cX} \pars{\frac{\beta_{d, m, V}}{\epsilon}}^{\frac{d}{(d+1)m}} + d}
        \exp\pars{- \frac{\dimRFF}{2(d+1)(S^2 + R)} \pars{\frac{\epsilon}{\beta_{d, m, V}}}^\frac{1}{m}} \quad
        &\text{for} \spc \epsilon \geq \beta_{d, m, V},
        \end{cases}
    \end{align}
    where $C_{d, \cX} = 2^\frac{1}{d+1} 16 \abs{\cX}^\frac{d}{d+1}\sum_{i,j=1}^d \pars{D_{i,j} + E_{i, j}}^\frac{d}{d+1}$ and $\beta_{d, m, V} = m \pars{2 V^{2} \pars{L^2 \vee 1} \pars{\sigma_\Lambda^2 \vee d}}^m$.
\end{theorem}
\begin{proof}
    \emph{Finite quantities.}
    To start off with, due to \eqref{eq:rfsf:rfsf_approx_cond} with $m=2$ and Jensen's inequality (Lemma \ref{lem:jensen}), we get
    \begin{align}
        \expe{\abs{w_i}^2} \leq \bbE^{1/2}\bracks{\abs{w_i}^4} < \infty \quad \text{for all} \spc i \in \bracks{d}.
    \end{align}
    Hence, by linearity of the expectation $\sigma_\Lambda^2 = \expe{\norm{\bw}^2} = \sum_{i=1}^d \expe{w_i^2} < \infty$. Next, due to Hölder's inequality (Lemma \ref{lem:holder}), it holds that
    \begin{align}
    \expe{w_i w_j} \leq \expe{\abs{w_i w_j}} \leq \bbE^{1/2}\bracks{\abs{w_i}^2}\bbE^{1/2}\bracks{\abs{w_j}^2} < \infty \quad \text{for all} \spc i,j \in \bracks{d},
    \end{align}
    therefore $L < \infty$. Further, applying Hölder's inequality twice followed by Jensen's inequality,
    \begin{align}
    \expe{\abs{w_i w_j w_k}} &\leq \bbE^{2/3}\bracks{\abs{w_i w_j}^{3/2}} \bbE^{1/3}\bracks{\abs{w_k}^3} \leq \bbE^{1/3}\bracks{\abs{w_i}^3} \bbE^{1/3}\bracks{\abs{w_j}^3} \bbE^{1/3}\bracks{\abs{w_k}^3} \\
    &\leq \bbE^{1/6}\bracks{\abs{w_i}^6} \bbE^{1/6}\bracks{\abs{w_j}^6} \bbE^{1/6}\bracks{\abs{w_k}^6} < \infty \quad \text{for all} \spc i, j, k \in \bracks{d}, \label{eq:rfsf:ijk_moment}
    \end{align} which is finite due to \eqref{eq:rfsf:rfsf_approx_cond} with $m=3$. Now, because of the $\ell^1$-$\ell^2$ norm inequality,
    \begin{align}
        \expe{\abs{w_i w_j} \norm{\bw}_2} \leq \expe{\abs{w_i w_j} \norm{\bw}_1} = \sum_{k=1}^d \expe{\abs{w_i w_j w_k}} < \infty,
    \end{align}
    hence $\bar E < \infty$. Next, as per \cite[Thm.~1.2.1.(iii)]{sasvari2013multivariate}, as \eqref{eq:rfsf:ijk_moment} holds for all $i,j,k \in \bracks{d}$, $\kernel$ is $3$-times continuously differentiable, which combined with the compactness of $\cX$, hence that of $\cX_\Delta$, gives $\sup_{\bz \in \cX_\Delta} \abs{\frac{\partial^3 \kernel(\bz)}{\partial z_i \partial z_j \partial z_k}} < \infty$. Finally, from the $\ell^1$-$\ell^2$ inequality again, we get that
    \begin{align}
    \sup_{\bz \in \cX_\Delta} \norm{\nabla \bracks{\frac{\partial^2 \kernel(\bz)}{\partial z_i \partial z_j}}}_2 \leq \sup_{\bz \in \cX_\Delta} \norm{\nabla \bracks{\frac{\partial^2 \kernel(\bz)}{\partial z_i \partial z_j}}}_1
    \leq \sum_{k=1}^d \sup_{\bz \in \cX_\Delta} \abs{\frac{\partial^3 \kernel(\bz)}{\partial z_i \partial z_j \partial z_k}} < \infty,
    \end{align}
    which shows the finiteness of $\bar D$. This finishes showing that the stated quantities are finite.
    
    \emph{Splitting the bound.}
    To start proving our main inequality, first note that as per Example \ref{example:2ndmoment_lip}, $\kernel$ is $L$-Lipschitz (see Def.~\ref{def:lipschitz kernel}). Hence, Lemma \ref{lem:RFSF_approx} yields that
    \begin{align}
        \sup_{\substack{\bx, \by \in \Seq(\cX) \\ \norm{\bx}_\onevar, \norm{\by}_\onevar \leq V}} &\abs{\rffsigkernel[m](\bx, \by) - \rffsigkernel[m](\bx, \by)}
        \\
        &\leq V^{2m} \sum_{k=1}^m \frac{L^{2(m-k)}}{\dimRFF^{k-1}((k-1)!)^2} \sup_{\bs, \bt \in \cX} \norm{\partial^2_{\bs,\bt} \rffkernel_{k}(\bs, \bt) - \partial^2_{\bs,\bt} \kernel(\bs, \bt)}_2 \prod_{p=1}^{k-1} \norm{\bW^{(p)}}_2^2. \label{eq:rfsf:rfsf_bound_rhs}
    \end{align}
    We bound the summand in the previous line in probability for each $k \in  \bracks{m}$.
    For brevity, denote $\alpha_{m, k} = \frac{V^{2m} L^{2(m-k)}}{((k-1)!)^2}$, and first consider the case $k \geq 2$, so that we have
\begin{align}
    P_k(\epsilon) &= \prob{\frac{\alpha_{m, k}}{\dimRFF^{k-1}} \norm{\bW^{(1)}}_2^2 \cdots \norm{\bW^{(k-1)}}_2^2 \sup_{\bs, \bt \in \cX} \norm{\partial_{\bs, \bt} \rffkernel_{k}(\bs, \bt) - \partial_{\bs, \bt} \kernel(\bs, \bt)}_2 \geq \epsilon}
    \\
    &\stackrel{\text{(a)}}{\leq}
    \prob{\alpha_{m, k} \pars{\frac{\norm{\bW^{(1)}}_2^2 + \ldots + \norm{\bW^{(k-1)}}_2^2}{\dimRFF (k-1)}}^{k-1} \sup_{\bs, \bt \in \cX} \norm{\partial_{\bs, \bt} \rffkernel_{k}(\bs, \bt) - \partial_{\bs, \bt} \kernel(\bs, \bt)}_2 \geq \epsilon}
    \\
    &\stackrel{\text{(b)}}{=}
    \prob{\underbrace{\frac{\norm{\bW^{(1)}}_2^2 + \ldots + \norm{\bW^{(k-1)}}_2^2}{\dimRFF (k-1)}}_{(A_k)} \underbrace{\sup_{\bs, \bt \in \cX} \norm{\partial_{\bs, \bt} \rffkernel_{k}(\bs, \bt) - \partial_{\bs, \bt} \kernel(\bs, \bt)}_2^{\frac{1}{k-1}}}_{(B_k)} \geq \pars{\frac{\epsilon}{\alpha_{m, k}}}^{\frac{1}{k-1}}},
\end{align}
where in (a) we used the arithmetic-geometric mean inequality, and in (b) we divided both sides by $\alpha_{m, k}$ and took the $(k-1)$th root.
Further, setting $t = \pars{\frac{\epsilon}{\alpha_{m, k}}}^{\frac{1}{k-1}}$, we have for $\gamma > 0$
\begin{align}
    P_k(\epsilon)
    &\leq &
    \prob{A_k \cdot B_k \ge t}
    \stackrel{\text{(c)}}{\leq}
    \prob{\pars{A_k - \gamma} B_k \geq \frac{t}{2}}
     +
    \prob{B_k \geq \frac{t}{2\gamma}}
    \\
    &\stackrel{\text{(d)}}{\leq} &
    \inf_{\tau > 0} \curls{
    \prob{A_k - \gamma \geq \frac{\tau}{2}}
    +
    \prob{B_k \geq \frac{t}{\tau}}}
     +
    \prob{B_k \geq \frac{t}{2\gamma}}, \label{eq:rfsf:union bound}
\end{align}
where in (c) we added and subtracted $\gamma B_k$ and applied a union bound, while in (d) we combined a union bound with the relation $\{XY \geq \epsilon\} \subseteq \{X \geq \tau\} \bigcup \{Y \geq \epsilon / \tau\}$ which holds for any $\tau > 0$.

Our aim is now to obtain good probabilistic bounds on $A_k$ and $B_k$ to use in \eqref{eq:rfsf:union bound} with the specific choice of $\gamma = \expe{A_k} = \sigma_\Lambda^2$.

\emph{Bounding $A_k$.}
By the inequality between the spectral and Frobenius norms, we have
\begin{align}
    \sum_{p=1}^{k-1} \norm{\bW^{(p)}}_2^2 \leq \sum_{p=1}^{k-1} \norm{\bW^{(p)}}_F^2 = \sum_{i=1}^d \sum_{p=1}^{k-1} \sum_{j=1}^{\dimRFF} \pars{w_{i, j}^{(p)}}^2.
\end{align}
Now note that $w_{i,j}^{(p)}$ are $\iid$ copies of the $i$\textsuperscript{th} marginal of $\Lambda$ for all $j \in \bracks{\dimRFF}$ and $p \in \bracks{k-1}$, so that via \eqref{eq:rfsf:rfsf_approx_cond} $\pars{w_{i,j}^{(p)}}^2$ satisfies the Bernstein moment condition
\begin{align}
    \expe{\pars{w_{i,j}^{(p)}}^{2k}} \leq \frac{k! S^2 R^{k-2}}{2} \quad \text{for all} \spc i \in [d], j \in [\dimRFF], p \in [m], k \geq 2.
\end{align}
Hence, we may apply the Bernstein inequality from Theorem \ref{thm:bernstein_onetail} so that for $i \in \bracks{d}$
\begin{align}
    \prob{\frac{1}{\dimRFF(k-1)} \sum_{p=1}^{k-1}\sum_{j=1}^{\dimRFF} \pars{w_{i, j}^{(p)}}^2 - \sigma_i^2 \geq \epsilon}
    \leq
    \exp\pars{\frac{-\dimRFF(k-1) \epsilon^2}{2(S^2 + R\epsilon)}}, \label{eq:rfsf:w_coord_prob_bound}
\end{align}
where  $\sigma_i^2 = \bbE_{\bw \sim \Lambda}\bracks{w_{i}^2}$.
Combining these bounds for all $i \in \bracks{d}$ and denoting $\sigma_\Lambda^2 = \sum_{i=1}^d \sigma_i^2$,
\begin{align}
    \prob{A_k - \sigma_\Lambda^2\geq \epsilon}
    &=
    \prob{\frac{1}{\dimRFF(k-1)}\sum_{i=1}^d \sum_{p=1}^{k-1} \sum_{j=1}^{\dimRFF} \pars{w_{i, j}^{(p)}}^2 - \sigma_\Lambda^2 \geq \epsilon}
    \\
    &\leq
    \sum_{i=1}^d \prob{\frac{1}{\dimRFF(k-1)}\sum_{p=1}^{k-1} \sum_{j=1}^{\dimRFF} \pars{w_{i, j}^{(p)}}^2 - \sigma_i^2 \geq \frac{\epsilon}{d}}
    \\
    &\leq
    d \exp\pars{\frac{-\dimRFF(k-1) \pars{\frac{\epsilon}{d}}^2}{2\pars{S^2 + R\frac{\epsilon}{d}}}}. \label{eq:rfsf:w_full_prob_bound}
\end{align}
Hence, we have the required probabilistic bound for the term in \eqref{eq:rfsf:union bound} containing $A_k$.

\emph{Bounding $B_k$.}
    One can bound the spectral norm by the max norm so that
    \begin{align}
        B_k^{k-1} = \sup_{\bs, \bt \in \cX} \norm{\partial_{\bs, \bt} \rffkernel_{k}(\bs, \bt) - \partial_{\bs,\bt} \kernel(\bs, \bt)}_2
        &\leq
        \sup_{\bs, \bt \in \cX} \norm{\partial_{\bs, \bt} \rffkernel_{k}(\bs, \bt) - \partial_{\bs,\bt} \kernel(\bs, \bt)}_{\max}
        \\
        &= \max_{i,j=1,\dots, d} \sup_{\bs, \bt \in \cX} \abs{\frac{\partial^2 \rffkernel_{k}(\bs, \bt)}{\partial s_i \partial t_j} - \frac{\partial^2 \kernel(\bs, \bt)}{\partial s_i \partial t_j}} \label{eq:rfsf:rff_derivs_max}.
    \end{align}
    Let $i, j \in \bracks{d}$ and denote $E_{i,j} = \bbE_{\bw \sim \Lambda}\bracks{\abs{w_i w_j} \norm{\bw}_2}$ and $D_{i, j} = \sup_{\bz \in \cX_\Delta} \norm{\nabla \bracks{\partial^{\be_i, \be_j} \kernel(\bz)}}_2$, which are finite as previously shown. Due to Hölder's inequality (Lemma \ref{lem:holder}) and \eqref{eq:rfsf:rfsf_approx_cond},
    \begin{align}
        \expe{\abs{w_i w_j}^m} \leq \bbE^{1/2}\bracks{w_i^{2m}} \bbE^{1/2}\bracks{w_j^{2m}} < \infty.
    \end{align}
    Recall that $\kernel$ is $3$-times continuously differentiable, so that the conditions required by Theorem \ref{thm:rff_derivative_approx} are satisfied, that we now call to our aid in controlling the \RFF{} kernel derivatives,
    \begin{align}
        \prob{\sup_{\bs, \bt \in \cX} \abs{\frac{\partial^2 \rffkernel_{k}(\bs, \bt)}{\partial s_i \partial t_j} - \frac{\partial^2 \kernel(\bs, \bt)}{\partial s_i \partial t_j}}
        \geq \epsilon}
        \leq 16 C^\prime_{d, \cX, i, j} \epsilon^{-\frac{d}{d+1}} \exp\pars{\frac{-\dimRFF \epsilon^2}{4(d+1)(2S^2 + R\epsilon)}},
    \end{align}
    where we defined $C^\prime_{d, \cX, i, j} = \pars{\abs{\cX} (D_{i, j} + E_{i, j})}^\frac{d}{d+1}$. 
    Hence, noting that the max satisfies the relation $\{\max_i \xi_i \geq \epsilon\} = \bigcup_i \{\xi_i \geq \epsilon\}$
    and union bounding \eqref{eq:rfsf:rff_derivs_max} in probability, we get that
    \begin{align}
        \prob{B_k^{k-1} \geq \epsilon}
        &\leq
        \sum_{i,j=1}^d
        \prob{\sup_{\bs, \bt \in \cX} \abs{\frac{\partial^2 \rffkernel_{k}(\bs, \bt)}{\partial s_i \partial t_j} - \frac{\partial^2 \kernel(\bs, \bt)}{\partial s_i \partial t_j}} \geq \epsilon}
        \\
        &\leq
        16 C^\prime_{d, \cX} \epsilon^{-\frac{d}{d+1}} \exp\pars{\frac{-\dimRFF \epsilon^2}{4(d+1)(2S^2 + R\epsilon)}}, \label{eq:rfsf:rff_deriv_prob_bound}
    \end{align}
    where we denote $C^\prime_{d, \cX} = \sum_{i,j=1}^d C^\prime_{d, \cX, i, j} = \abs{\cX}^\frac{d}{d+1} \sum_{i,j=1}^d \pars{D_{i,j} + E_{i,j}}^\frac{d}{d+1}$.

\emph{Putting it together.} Now that we  have our bounds for $A_k$ and $B_k$, we put everything together, that is, plug the bounds \eqref{eq:rfsf:w_full_prob_bound} and \eqref{eq:rfsf:rff_deriv_prob_bound} into \eqref{eq:rfsf:union bound}, so that we get
{\small\begin{align}
    P_k(\epsilon)
    \leq&
    \inf_{\tau > 0} \curls{
    d \exp\pars{\frac{-\dimRFF(k-1) \pars{\frac{\tau}{2d}}^2}{2\pars{S^2 + R\frac{\tau}{2d}}}}
    +
    16 C^\prime_{d, \cX} \pars{\frac{\tau}{t}}^\frac{d(k-1)}{d+1} \exp\pars{\frac{-\dimRFF \pars{\frac{t}{\tau}}^{2(k-1)}}{4(d+1) \pars{2S^2 + R\pars{\frac{t}{\tau}}^{k-1}}}}
    }
    \\
    & +
    16 C^\prime_{d, \cX} \pars{\frac{2\sigma_\Lambda^2}{t}}^\frac{d(k-1)}{d+1}  \exp\pars{\frac{-\dimRFF \pars{\frac{t}{2\sigma_\Lambda^2}}^{2(k-1)}}{4(d+1) \pars{2 S^2 + R\pars{\frac{t}{2\sigma_\Lambda^2}}^{k-1}}}}
    \\
    \stackrel{\text{(e)}}{\leq} &
    d \exp\pars{\frac{-\dimRFF(k-1) \pars{\frac{t^{\frac{k-1}{k}}}{2d}}^2}{2\pars{S^2 + R\frac{t^\frac{k-1}{k}}{2d}}}}
    +
    16 C^\prime_{d, \cX} \pars{\frac{1}{t}}^\frac{d(k-1)}{(d+1)k} \exp\pars{\frac{-\dimRFF t^\frac{2(k-1)}{k}}{4(d+1) \pars{2S^2 + Rt^\frac{k-1}{k}}}}
    \\
    & +
    16 C^\prime_{d, \cX} \pars{\frac{2\sigma_\Lambda^2}{t}}^\frac{d(k-1)}{d+1}  \exp\pars{\frac{-\dimRFF \pars{\frac{t}{2\sigma_\Lambda^2}}^{2(k-1)}}{4(d+1) \pars{2 S^2 + R\pars{\frac{t}{2\sigma_\Lambda^2}}^{k-1}}}}
    \\
    \stackrel{\text{(f)}}{=} &
    d \exp\pars{\frac{-\dimRFF(k-1) \pars{\frac{\pars{\epsilon/\alpha_{m,k}}^\frac{1}{k}}{2d}}^2}{2\pars{S^2 + R\frac{\pars{\epsilon/\alpha_{m, k}}^\frac{1}{k}}{2d}}}}
    +
    16 C^\prime_{d, \cX} \pars{\frac{\alpha_{m, k}}{\epsilon}}^\frac{d}{(d+1)k} \exp\pars{\frac{-\dimRFF \pars{\frac{\epsilon}{\alpha_{m,k}}}^\frac{2}{k}}{4(d+1) \pars{2S^2 + R\pars{\frac{\epsilon}{\alpha_{m,k}}}^\frac{1}{k}}}}
    \\
    & +
    16 C^\prime_{d, \cX} \pars{\frac{\alpha_{m,k}\pars{2\sigma_\Lambda^2}^{k-1}}{\epsilon}}^\frac{d}{d+1} \exp\pars{\frac{-\dimRFF \pars{\frac{\epsilon}{\alpha_{m,k} \pars{2\sigma_\Lambda^2}^{k-1}}}^2}{4(d+1) \pars{2 S^2 + R\frac{\epsilon}{\alpha_{m,k} \pars{2\sigma_\Lambda^2}^{k-1}}}}},
\end{align}}
where (a) follows from substituting \eqref{eq:rfsf:rff_deriv_prob_bound} and \eqref{eq:rfsf:w_full_prob_bound} into \eqref{eq:rfsf:union bound} with the choice of $\gamma = \sigma_\Lambda^2$, (b) from choosing $\tau = t^\frac{k-1}{k}$, and (c) from putting back $t = (\epsilon/\alpha_{m, k})^\frac{1}{k-1}$.

Note that the previous applies for all $k \geq 2$. For $k=1$, we have by \eqref{eq:rfsf:rff_deriv_prob_bound} that
\begin{align}
    P_1(\epsilon) &= \prob{\sup_{\bs, \bt \in \cX} \norm{\partial_{\bs, \bt} \rffkernel_{1}(\bs, \bt) - \partial_{\bs, \bt} \kernel(\bs, \bt)}_2 \geq \frac{\epsilon}{\alpha_{m, 1}}}
    \\
    &\leq
    16 C^\prime_{d, \cX} \pars{\frac{\alpha_{m, 1}}{\epsilon}}^\frac{d}{d+1} \exp\pars{\frac{-\dimRFF \pars{\frac{\epsilon}{\alpha_{m, 1}}}^2}{4(d+1)\pars{2S^2 + R\frac{\epsilon }{\alpha_{m, 1}}}}}.
\end{align}
\emph{Combining and simplifying.} We can now combine the bounds for $P_1, \dots, P_m$ into \eqref{eq:rfsf:rfsf_bound_rhs},
{\small\begin{align}
    \bbP & \bracks{\sup_{\substack{\bx, \by \in \Seq(\cX)\\\norm{\bx}_\onevar,\norm{\by}_\onevar \leq V}} \abs{\rffsigkernel[m](\bx, \by) - \sigkernel[m](\bx, \by)} \geq \epsilon} \leq \sum_{k=1}^m P_k\pars{\frac{\epsilon}{m}}
    \\
    \stackrel{\text{(g)}}{\leq} &
    2^\frac{1}{d+1} 8 C^\prime_{d, \cX}\sum_{k=1}^{m} \pars{\frac{2^k {\sigma_\Lambda^2}^{k-1} m \alpha_{m, k}}{\epsilon}}^\frac{d}{d+1}
    \exp\pars{- \frac{\dimRFF}{2(d+1)} \cdot \frac{\pars{\frac{\epsilon}{2^k \sigma_\Lambda^{2(k-1)} m\alpha_{m, k}}}^2}{S^2 + R \frac{\epsilon}{2^k \sigma_\Lambda^{2(k-1)} m\alpha_{m, k} }}}
    \\
    & +
    2^\frac{1}{d+1} 8 C^\prime_{d, \cX} \sum_{k=2}^m \pars{\frac{2^k m\alpha_{m, k}}{\epsilon}}^\frac{d}{(d+1)k}
    \exp\pars{- \frac{\dimRFF}{2(d+1)} \cdot \frac{\pars{\frac{\epsilon}{2^k m\alpha_{m,k}}}^\frac{2}{k}}{ S^2 + R\pars{\frac{\epsilon}{2^k m\alpha_{m,k}}}^\frac{1}{k}}}
    \\
    & +
    d \sum_{k=2}^M
    \exp\pars{-\frac{\dimRFF(k-1)}{2} \frac{\pars{\frac{\epsilon}{2^k d^k m\alpha_{m,k}}}^\frac{2}{k}}{S^2 + R \pars{\frac{\epsilon}{2^k d^k m\alpha_{m, k}}}^\frac{1}{k}}}
    \\
    \stackrel{\text{(h)}}{\leq} &
    2^\frac{1}{d+1} 8 C^\prime_{d, \cX} \sum_{k=1}^{m} \pars{\frac{2^k \pars{\sigma_\Lambda^2 \vee d}^k m \alpha_{m, k}}{\epsilon}}^\frac{d}{d+1}
    \exp\pars{- \frac{\dimRFF}{2(d+1)} \cdot \frac{\pars{\frac{\epsilon}{2^k \pars{\sigma_\Lambda^2 \vee d}^k m\alpha_{m, k}}}^2}{S^2 + R \frac{\epsilon}{2^k \pars{\sigma_\Lambda^2 \vee d}^k m\alpha_{m, k} }}}
    \\
    & +
    2^\frac{1}{d+1} 8 C^\prime_{d, \cX} \sum_{k=1}^m \pars{\frac{2^k \pars{\sigma_\Lambda^2 \vee d}^k m\alpha_{m, k}}{\epsilon}}^\frac{d}{(d+1)k}
    \exp\pars{- \frac{\dimRFF}{2(d+1)} \cdot \frac{\pars{\frac{\epsilon}{2^k \pars{\sigma_\Lambda^2 \vee d}^k m\alpha_{m,k}}}^\frac{2}{k}}{ S^2 + R\pars{\frac{\epsilon}{2^k \pars{\sigma_\Lambda^2 \vee d}^k m\alpha_{m,k}}}^\frac{1}{k}}}
    \\
    & +
    d \sum_{k=1}^m
    \exp\pars{-\frac{\dimRFF}{2(d+1)} \frac{\pars{\frac{\epsilon}{2^k \pars{\sigma_\Lambda^2 \vee d}^k m\alpha_{m,k}}}^\frac{2}{k}}{S^2 + R \pars{\frac{\epsilon}{2^k \pars{\sigma_\Lambda^2 \vee d}^k m\alpha_{m, k}}}^\frac{1}{k}}}
    \\
    \stackrel{\text{(f)}}{\leq} &
    2^\frac{1}{d+1} 8 C^\prime_{d, \cX} \sum_{k=1}^{m} \pars{\frac{\beta_{d, m, V}}{\epsilon}}^\frac{d}{d+1}
    \exp\pars{- \frac{\dimRFF}{2(d+1)} \frac{\pars{\frac{\epsilon}{\beta_{d, m, V}}}^2}{S^2 + R \frac{\epsilon}{\beta{d, m, V}}}}
    \\
    & +
    2^\frac{1}{d+1} 8 C^\prime_{d, \cX} \sum_{k=1}^m \pars{\frac{\beta_{d, m, V}}{\epsilon}}^\frac{d}{(d+1)k}
    \exp\pars{- \frac{\dimRFF}{2(d+1)} \frac{\pars{\frac{\epsilon}{\beta_{d, m, V}}}^\frac{2}{k}}{S^2 + R\pars{\frac{\epsilon}{\beta_{d, m, V}}}^\frac{1}{k}}}
    \\
    & +
    d \sum_{k=1}^m
    \exp\pars{-\frac{\dimRFF}{2(d+1)} \frac{\pars{\frac{\epsilon}{\beta_{d, m, V}}}^\frac{2}{k}}{S^2 + R\pars{\frac{\epsilon}{\beta_{d, m, V}}}^\frac{1}{k}}},
\end{align}}
where (g) follows from rearranging the expressions from (f), while (h) from unifying the coefficients and that $1 \leq d \leq \max(\sigma_\Lambda^2, d)$ and $f(x) = x^2 / (a + bx)$ is monotonically increasing in $x$ on the positive half-line for $a,b > 0$, while (f) from $\alpha_{m, k} = V^{2m} L^{2(m-k)} / ((k-1)!)^2 \leq (VL)^{2m}$, using that $f(x)$ is increasing, and defining $\beta_{d, m, V} = m\pars{2 V^{2} (L^2 \vee 1) (\sigma_\Lambda^2 \vee d)}^m$.

\emph{Conclusion.}
Finally, we split the bound into two cases: the first case is if the error is big, i.e. $\epsilon \geq \beta_{d, m, V} = m\pars{2 V^{2} (L^2 \vee 1) (\sigma_\Lambda^2 \vee d)}^m$, when we decrease all the exponents to $\frac{1}{m}$,
\begin{align}
    \bbP & \bracks{\sup_{\substack{\bx, \by \in \Seq(\cX)\\\norm{\bx}_\onevar, \norm{\by}_\onevar \leq V}} \abs{\rffsigkernel[m](\bx, \by) - \sigkernel[m](\bx, \by)} \geq \epsilon}
    \\
    \leq &
    m \cdot \pars{2^\frac{1}{d+1} 16 C^\prime_{d, \cX} \pars{\frac{\beta_{d, m, V}}{\epsilon}}^{\frac{d}{(d+1)m}} + d}
    \exp\pars{-\frac{\dimRFF}{2(d+1)} \frac{\pars{\frac{\epsilon}{\beta_{d, m, V}}}^\frac{2}{m}}
    {S^2 + R \pars{\frac{\epsilon}{\beta_{d, m, V}}}^{\frac{1}{m}}}}
\end{align}
The other when the error is small, i.e. $\epsilon < \beta_{d, m, V}$, when we increase all the exponents to $1$
\begin{align}
    \bbP & \bracks{\sup_{\substack{\bx, \by \in \Seq(\cX)\\\norm{\bx}_\onevar, \norm{\by}_\onevar \leq V}} \abs{\rffsigkernel[m](\bx, \by) - \sigkernel(\bx, \by)} \geq \epsilon}
    \\
    \leq &
    m \cdot \pars{2^\frac{1}{d+1} 16 C^\prime_{d, \cX} \pars{\frac{\beta_{d, m, V}}{\epsilon}}^\frac{d}{d+1} + d}
    \exp\pars{-\frac{\dimRFF}{2(d+1)} \frac{\pars{\frac{\epsilon}{\beta_{d, m, V}}}^{2}}
    {S^2 + R \pars{\frac{\epsilon}{\beta_{d, m, V}}}}}.
\end{align}
The claimed estimate follows by denoting $C_{d, \cX} = 2^\frac{1}{d+1}16 C^\prime_{d, \cX}$ and simplifying.
\end{proof}

Next, we prove Theorem \ref{thm:main2} to show an approximation bound for the \RFSFD{} kernel.

\begin{theorem}[Concentration inequality for \RFSFD{} kernel] \label{thm:rfsf_dp_approx}
    Let $\kernel: \bbR^d \times \bbR^d \to \bbR$ be a continuous, bounded, translation-invariant kernel with spectral measure $\Lambda$, which satisfies for some $S, R > 0$ that
    \begin{align} \label{eq:rfsf:rfsf_dp_approx_cond}
        \bbE_{\bw \sim \Lambda}\bracks{\abs{w_i}^{2k}} \leq \frac{k! S^2 R^{k-2}}{2} \quad \text{for all} \spc i \in [d] \spc \text{and} \spc k \geq 2.
    \end{align}

    Then, for signature level $m \in \mathbb{Z}_+$ and $\bx, \by \in \Seq(\cX)$, it holds for $\epsilon > 0$ that:
    \begin{align}
    \bbP\bracks{\abs{\rffsigkernelDP[m](\bx, \by) - \sigkernel[m](\bx, \by)} \geq \epsilon}
        \leq
        2\exp\pars{-\frac{1}{4}\min
        \curls{\begin{array}{c}
        \pars{\frac{\sqrt{\dimRFF} \epsilon}{2C_{d, m}\norm{\bx}_\onevar^m \norm{\by}_\onevar^m}}^2,
        \\
        \pars{\frac{\dimRFF \epsilon}{\sqrt{8}C_{d, m} \norm{\bx}_\onevar^m \norm{\by}_\onevar^m}}^{\frac{1}{m}}
        \end{array}}
        },
    \end{align}
    where the absolute constant $C_{d, m} > 0$ satisfies that
    \begin{align}
    C_{d, m} \leq
    \sqrt{8} e^4 (2\pi)^{1/4} e^{1/24} (4e^3/m)^m \pars{\pars{2d\max(S, R)}^m + \pars{L^2/\ln 2}^m}.
    \end{align}
\end{theorem}
\begin{proof}
    Let $\rffsigkernelhat[m]^{(1)}, \dots, \rffsigkernelhat[m]^{(\dimRFF)}$ be independent copies of the \RFSF{} kernel, each with internal \RFF{} sample size $\hat d = 1$, such that $\rffsigkernelDP[m] = \frac{1}{\dimRFF} \sum_{k=1}^{\dimRFF} \rffsigkernelhat[m]^{(k)}$.
    Our goal is to call Theorem \ref{thm:alpha_subexp_concentration} with $\alpha = \frac{1}{m}$, and therefore, we compute an upper bound on the $\Psi_{1/m}$-norm of $\rffsigkernel[m]^{(k)}(\bx, \by) - \sigkernel[m](\bx, \by)$ for all $k \in [\dimRFF]$; for a definition of the $\alpha$-exponential norm, see Definition~\ref{def:alpha_subexp_norm}.

    By Lemma \ref{lem:1var_rffsig_norm}, it holds for any $\bx, \by \in \Seq(\cX)$ and $k \in [m]$ that
    \begin{align}
        \abs{\rffsigkernelhat[m]^{(k)}(\bx, \by)} \leq \frac{\pars{\norm{\bx}_\onevar \norm{\by}_\onevar}^m}{(m!)^2} \norm{\bw_k^{(1)}}_2^2 \cdots \norm{\bw_k^{(m)}}_2^2,
    \end{align}
    where $\bw_k^{(1)}, \dots, \bw_k^{(m)} \stackrel{\iid}{\sim} \Lambda$ are the random weights that parametrize $\rffsigkernelhat[m]^{(k)}$ for all $k \in [\dimRFF]$.
    Now, calling Lemma \ref{lem:alpha_bernstein_cond} with $\alpha=1$ yields that, due to \eqref{eq:rfsf:rfsf_dp_approx_cond}, the following holds:
    \begin{align}
        \norm{{w^{(p)}_{k, i}}^2}_{\Psi_1} \leq 2\max(S, R) \quad \text{for all} \spc i \in [d], k \in [\dimRFF], p \in [m].
    \end{align}
    Note that for $\alpha = 1$, $\norm{\cdot}_{\Psi_\alpha}$ satisfies the triangle inequality (see Lemma \ref{lem:alpha_exp_triangle}), and hence
    \begin{align} \label{eq:rfsf:bw_alpha_norm_bound}
        \norm{\norm{\bw_k^{(p)}}_2^2}_{\Psi_1} \leq \sum_{i=1}^d \norm{{w_{k,i}^{(p)}}^2}_{\Psi_1} \leq 2d \max(S, R) \quad \text{for all} \spc k \in [\dimRFF], p \in [m].
    \end{align}
    As $\norm{\cdot}_{\Psi_\alpha}$ is positive homogenous and satisfies a Hölder-type inequality (see Lemma \ref{lem:alpha_exp_holder}):
    \begin{align}
        \norm{\rffsigkernelhat[m]^{(k)}(\bx, \by)}_{\Psi_{1/m}} &\stackrel{\text{(a)}}{\leq} \frac{\pars{\norm{\bx}_\onevar \norm{\by}_\onevar}^m}{(m!)^2} \norm{\norm{\bw_k^{(1)}}_2^2}_{\Psi_1} \cdots \norm{\norm{\bw_k^{(m)}}_2^2}_{\Psi_1}
        \\
        &\stackrel{\text{(b)}}{\leq} \frac{\pars{2d\norm{\bx}_\onevar \norm{\by}_\onevar \max(S, R)}^m}{(m!)^2} \quad\text{for all}\spc k \in [\dimRFF], \label{eq:rfsf:rfsf_1dim_alpha_norm}
    \end{align}
    where in (a) we used Corollary \ref{lem:krffsig_bound}, in (b) we used \eqref{eq:rfsf:rfsf_1dim_alpha_norm}.
    We are almost ready to use Theorem \ref{thm:alpha_subexp_concentration}, but it requires the $\norm{\cdot}_{\Psi_{1/m}}$ bound in terms of centered random variables.
    Although $\norm{\cdot}_{\Psi_\alpha}$ does not satisfy the triangle inequality for $\alpha \in (0, 1)$, it obeys that (see \cite[Lemma~A.3.]{gotze2021concentration}) $\norm{X + Y}_{\Psi_\alpha}\leq 2^{1/\alpha}\pars{\norm{X}_{\Psi_\alpha} + \norm{Y}_{\Psi_\alpha}}$ for any random variables $X$ and $Y$. For a constant $c \in \bbR$, we have $\norm{c}_{\Psi_{1/m}} = \frac{\abs{c}}{\ln^m 2}$, and hence by Lemma \ref{lem:ksig_bound} we have that
    \begin{align}
        \norm{\sigkernel[m](\bx, \by)}_{\Psi_{1/m}} \leq \frac{\pars{L^2\norm{\bx}_\onevar \norm{\by}_\onevar / \ln2}^m}{(m!)^2},
    \end{align}
    where $L = \norm{\bbE_{\bw \sim \Lambda}\bracks{\bw\bw^\top}}_2$ is the Lipschitz constant of the kernel $\kernel: \cX \times \cX \to \bbR$. This gives
    \begin{align}
        \norm{\rffsigkernelhat[m]^{(k)}(\bx, \by) - \sigkernel[m](\bx, \by)}_{\Psi_{1/m}}
        &\leq
        2^m\pars{\norm{\rffsigkernelhat[m]^{(k)}(\bx, \by)}_{\Psi_{1/m}} + \norm{\sigkernel[m](\bx, \by)}_{\Psi_{1/m}}}
        \\
        &\leq
        \frac{\pars{2\norm{\bx}_\onevar\norm{\by}_\onevar}^m}{(m!)^2}\pars{\pars{2d\max(S, R)}^m + \pars{L^2/\ln 2}^m}.
    \end{align}
    Finally, we have the required Orlicz norm bound for invoking Theorem \ref{thm:alpha_subexp_concentration}, so that we get
    \begin{align}
        \bbP&\bracks{\abs{\rffsigkernelDP[m](\bx, \by) - \sigkernel[m](\bx, \by)} \geq \epsilon}
        \\&\leq
        2\exp\pars{-\frac{1}{4}\min\curls{
        \pars{\frac{\sqrt{\dimRFF} \epsilon}{2C_{d, m}}}^2,
        \pars{\frac{\dimRFF \epsilon}{\sqrt{8}C_{d, m}}}^{\frac{1}{m}}}},
    \end{align}
    where the constant $C_{d, m} > 0$ is defined as
    \begin{align}
    C_{d, m}
    =
    \sqrt{8} e^4 (2\pi)^{1/4} e^{1/24} \frac{(4em \norm{\bx}_\onevar \norm{\by}_\onevar )^m}{(m!)^2} \pars{\pars{2d\max(S, R)}^m + \pars{L^2/\ln 2}^m},
    \end{align}
    and invoking Stirling's approximation $\frac{1}{m!} \leq \pars{\frac{e}{m}}^m$ gives the stated result.
\end{proof}

Now, we prove the analogous result for $\rffsigkernelTRP[m]$.
\begin{theorem}[Concentration inequality for \RFSFT{} kernel] \label{thm:rfsf_trp_approx}
    Let $\kernel: \bbR^d \times \bbR^d \to \bbR$ be a continuous, bounded, translation-invariant kernel with spectral measure $\Lambda$, which satisfies for some $S, R > 0$ that
    \begin{align} \label{eq:rfsf:rfsf_trp_approx_cond}
        \bbE_{\bw \sim \Lambda}\bracks{\abs{w_i}^{2k}} \leq \frac{k! S^2 R^{k-2}}{2} \quad \text{for all} \spc i \in [d] \spc \text{and} \spc k \geq 2.
    \end{align}
    Then, for the level-$m$ \RFSFT{} kernel as defined in \eqref{eq:rfsf:rffsigtrpkernel_def}, we have for $\bx, \by \in \Seq(\cX)$ and $\epsilon > 0$
    \begin{align}
        \bbP\bracks{\abs{\rffsigkernelTRP[m](\bx, \by) - \rffsigkernel[m](\bx, \by)} \geq \epsilon}
        \leq
        C_{d, \Lambda}
        \exp\pars{- \pars{\frac{m^2 \dimTRP^{\frac{1}{2m}} \epsilon^{\frac{1}{m}}}{2\sqrt{2}e^3 R \norm{\bx}_\onevar \norm{\by}_\onevar}}^\frac{1}{2}},
    \end{align}
    where $C_{d, \Lambda} = 2\pars{1 + \frac{S}{2R} + \frac{S^2}{4R^2}}^d$.
\end{theorem}
\begin{proof}
    First, we consider the conditional probability $\prob{\abs{\rffsigkernelTRP[m](\bx, \by) - \rffsigkernel[m](\bx, \by)} \geq \epsilon \middle\vert \bW}$ by conditioning on the \RFSF{} weights $\bW = (\bW^{(1)}, \dots, \bW^{(m)})$, so that the only source of randomness comes from the \TRP{} weights $\bP = (\bP^{(1)}, \dots, \bP^{(m)})$. The idea is to call Theorem \ref{thm:hyper_concentration} to estimate the conditional probability, and then take expectation over $\bW$. Since Theorem \ref{thm:hyper_concentration} quantifies the concentration of a Gaussian polynomial around its mean in terms of its variance, we first compute the conditional statistics of $\rffsigkernelTRP[m](\bx, \by)$.
    
    \emph{Conditional expectation.}
    Recall the definition of $\rffsigTRP[m](\bx)$ \eqref{eq:rfsf:rffsigtrpdef}, where $\bp^{(1)}_1, \dots, \bp^{(m)}_{\dimTRP} \stackrel{\iid}{\sim} \cN(0, \b I_{2\dimRFF})$, and $\bx \in \Seq(\cX)$, so that
    \begin{align}
        \rffsigTRP[m](\bx) &= \frac{1}{\sqrt{\dimTRP}} \pars{\sum_{\bi \in \Delta_m(\len{\bx})} \prod_{p=1}^m \inner{\bp_i^{(p)}}{\delta \rff_p(\bx_{i_p})}}_{i=1}^{\dimTRP}.
    \end{align}
    
    Hence, $\rffsigkernelTRP[m](\bx, \by)$ is written for $\bx, \by \in \Seq(\cX)$ as
    \begin{align}
        \rffsigkernelTRP[m](\bx, \by) = \frac{1}{\dimTRP} \sum_{i=1}^{\dimTRP} \underbrace{\sum_{\substack{\bi \in \Delta_m(\len{\bx}) \\ \bj \in \Delta_m(\len{\by})}} \prod_{p=1}^m \inner{\bp_i^{(p)}}{\delta \rff_p(\bx_{i_p})} \inner{\bp_i^{(p)}}{\delta \rff_p(\by_{i_p})}}_{A_i},
    \end{align}
    which is a sample average of $\dimTRP$ $\iid$ terms, i.e.~$\rffsigkernelTRP[m](\bx, \by) = \frac{1}{\dimTRP} \sum_{i=1}^{\dimTRP} A_i$. We only have to verify that $A_i$ is conditionally an unbiased approximator of $\rffsigkernel[m]$ given $\bW$. 
    \begin{align}
        \expe{A_i \given \bW} &\stackrel{\text{(a)}}{=} \sum_{\substack{\bi \in \Delta_m(\len{\bx}) \\ \bj \in \Delta_m(\len{\by})}} \prod_{p=1}^m  \expe{\inner{\bp^{(p)}_i}{\delta \rff_p(\bx_{i_p})} \inner{\bp^{(p)}_i}{\delta \rff_p(\by_{j_p})} \given \bW}
        \\
        &\stackrel{\text{(b)}}{=}
        \sum_{\substack{\bi \in \Delta_m(\len{\bx}) \\ \bj \in \Delta_m(\len{\by})}} \prod_{p=1}^m \inner{\expe{\bp^{(p)}_i \otimes \bp^{(p)}_i}}{\delta \rff_p(\bx_{i_p}) \otimes \delta \rff_p(\by_{j_p})}
        \\
        &\stackrel{\text{(c)}}{=}
        \sum_{\substack{\bi \in \Delta_m(\len{\bx}) \\ \bj \in \Delta_m(\len{\by})}} \prod_{p=1}^m \inner{I_{\dimRFF}}{\delta \rff_p(\bx_{i_p}) \otimes \delta \rff_p(\by_{j_p})}
        \\
        &\stackrel{\text{(d)}}{=}
        \sum_{\substack{\bi \in \Delta_m(\len{\bx}) \\ \bj \in \Delta_m(\len{\by})}} \prod_{p=1}^m \inner{\delta \rff_p(\bx_{i_p})}{\delta \rff_p(\by_{j_p})},
    \end{align}
    where (a) follows from linearity of expectation and independence of the $\bp_i^{(p)}$'s for $p \in [m]$, (b) from bilinearity of inner product, and independence of $\bP$ and $\bW$, (c) from substituting the covariance, (d) is since the outer product is projected onto the diagonal.

    \emph{Conditional variance.} We compute the conditional variance of $A_i$ given $\bW$:
    {\small\begin{align}
        \bbE&\bracks{A_{m, i}^2 \given \bW}&&
        \\
        &\stackrel{\text{(e)}}{=} \sum_{\substack{\bi,\bk \in \Delta_m({\len{\bx}})\\\bj,\bl \in \Delta_m({\len{\by}})}} \prod_{p=1}^m
        &&\bbE\bracks{\inner{\bp^{(p)}_i}{\delta \rff_p(\bx_{i_p})} \inner{\bp^{(p)}_i}{\delta \rff_p(\by_{j_p})}\inner{\bp^{(p)}_i}{\delta \rff_p(\bx_{k_p})} \inner{\bp^{(p)}_i}{\delta \rff_p(\by_{l_p})} \given \bW}
        \\
        &\stackrel{\text{(f)}}{=}
        \sum_{\substack{\bi,\bk \in \Delta_m({\len{\bx}})\\\bj,\bl \in \Delta_m({\len{\by}})}} \prod_{p=1}^m
        &&\bigg(
        \bbE\bracks{\inner{\bp^{(p)}_i}{\delta \rff_p(\bx_{i_p})} \inner{\bp^{(p)}_i}{\delta \rff_p(\by_{j_p})} \given \bW} \bbE\bracks{\inner{\bp^{(p)}_i}{\delta \rff_p(\bx_{k_p})} \inner{\bp^{(p)}_i}{\delta \rff_p(\by_{l_p})} \given \bW}
        \\
        &&&+
        \bbE\bracks{\inner{\bp^{(p)}_i}{\delta \rff_p(\bx_{i_p})} \inner{\bp^{(p)}_i}{\delta \rff_p(\bx_{k_p})}\given \bW} \bbE\bracks{\inner{\bp^{(p)}_i}{\delta \rff_p(\by_{j_p})} \inner{\bp^{(p)}_i}{\delta \rff_p(\by_{l_p})} \given \bW}
        \\
        &&&+
        \bbE\bracks{\inner{\bp^{(p)}_i}{\delta \rff_p(\bx_{i_p})} \inner{\bp^{(p)}_i}{\delta \rff_p(\by_{l_p})} \given \bW} \bbE\bracks{\inner{\bp^{(p)}_i}{\delta \rff_p(\bx_{k_p})} \inner{\bp^{(p)}_i}{\delta \rff_p(\by_{j_p})} \given \bW}\bigg)
        \\
        &\stackrel{\text{(g)}}{=}
        \sum_{\substack{\bi,\bk \in \Delta_m({\len{\bx}})\\\bj,\bl \in \Delta_m({\len{\by}})}} \prod_{p=1}^m
        &&\bigg(
        \inner{\delta \rff_p(\bx_{i_p})}{\delta \rff_p(\by_{j_p})} \inner{\delta \rff_p(\bx_{k_p})}{\delta \rff_p(\by_{l_p})}
        \\
        &&&+
        \inner{\delta \rff_p(\bx_{i_p})}{\delta \rff_p(\bx_{k_p})} \inner{\delta \rff_p(\by_{j_p})}{\delta \rff_p(\by_{l_p})}
        \\
        &&&+
        \inner{\delta \rff_p(\bx_{i_p})}{\delta \rff_p(\by_{l_p})} \inner{\delta \rff_p(\bx_{k_p})}{\delta \rff_p(\by_{j_p})} \bigg)
        \\
        &\stackrel{\text{(h)}}{\leq} \mathrlap{
        \sum_{\substack{\bi,\bk \in \Delta_m({\len{\bx}})\\\bj,\bl \in \Delta_m({\len{\by}})}} 3^m \prod_{p=1}^m
        \norm{\delta \rff_p(\bx_{i_p})} \norm{\delta \rff_p(\by_{j_p})} \norm{\delta \rff_p(\bx_{k_p})}\norm{\delta \rff_p(\by_{l_p})}}
        \\
        &\stackrel{\text{(i)}}{=} \mathrlap{
        3^m \pars{\sum_{\bi \in \Delta_m({\len{\bx}})} \prod_{p=1}^m \norm{\delta \rff_p(\bx_{i_p})}}^2
        \pars{\sum_{\bj \in \Delta_m({\len{\by}})} \prod_{p=1}^m
        \norm{\delta \rff_p(\by_{j_p})}}^2}
        \\
        &\stackrel{\text{(j)}}{\leq} \mathrlap{
        \frac{1}{(m!)^4} {\pars{\frac{3 \norm{\bx}_\onevar^2 \norm{\by}_\onevar^2}{\dimRFF^2}}}^m \prod_{p=1}^m \norm{\bW^{(p)}}^4_2,}
    \end{align}}
    where (e) follows from linearity of expectation and independence of the $\bp_i^{(p)}$'s for $p \in [m]$, (f) from Isserlis' theorem \cite{isserlis1918formula}, (g) is the same as (a)-(d), (h) is the Cauchy-Schwarz inequality, (i) from factorizing the summation, (j) is the same as Lemma \ref{lem:1var_rffsig_norm}. 
    
    Therefore, we have due to Lemma \ref{lem:krffsig_bound} for the variance that
    \begin{align}
        \bbV\bracks{A_{m, i} \,\vert\, \bW} 
        = \expe{A_{m, i}^2 \,\vert\, \bW} - \bbE^2\bracks{A_{m, i} \,\vert\, \bW}
        \leq 
        \frac{3^m + 1}{(m!)^4} {\pars{\frac{\norm{\bx}_\onevar^2 \norm{\by}_\onevar^2}{\dimRFF^2}}}^m \prod_{p=1}^m \norm{\bW^{(p)}}^4_2.
    \end{align}
    Let $\beta_m(\bx, \by) = \frac{3^m + 1}{(m!)^4} \norm{\bx}_\onevar^{2m} \norm{\by}_\onevar^{2m}$. Then, as $\rffsigkernelTRP[m](\bx, \by) \,\vert\, \bW$ is a sample average,
    \begin{align}
        \bbV\bracks{\rffsigkernelTRP[m](\bx, \by) \given \bW} \leq \frac{\beta_m(\bx, \by)}{\dimTRP \dimRFF^{2m}} \prod_{p=1}^m \norm{\bW^{(p)}}^4_2. \label{eq:rfsf:cond_variance}
    \end{align}

    \emph{Conditional bound.} Since $\rffsigkernelTRP[m](\bx,\by) \,\vert\, \bW$ is a Gaussian polynomial of degree-$2m$, with expectation $\rffsigkernel[m](\bx, \by)$, and variance \eqref{eq:rfsf:cond_variance}, we have by Theorem \ref{thm:hyper_concentration} for $\epsilon > 0$ that
    \begin{align}
        \bbP\bracks{\abs{\rffsigkernelTRP[m](\bx, \by) - \rffsigkernel[m](\bx, \by)} \geq \epsilon \given \bW} 
        &\leq 2 \exp\pars{-\frac{\epsilon^{\frac{1}{m}}}{2\sqrt{2}e \bbV^{\frac{1}{2m}}\bracks{\rffsigkernelTRP[m](\bx, \by) \given \bW}}}
        \\
        &\leq 2 \exp\pars{- \frac{\dimTRP^{\frac{1}{2m}} \dimRFF \epsilon^{\frac{1}{m}}}{2\sqrt{2}e \beta_m^{\frac{1}{2m}}(\bx, \by) \prod_{p=1}^m \norm{\bW^{(p)}}^\frac{2}{m}_2}}. \label{eq:rfsf:conditional_bound}
    \end{align}
    
    \emph{Undoing the conditioning.} We take the expectation in \eqref{eq:rfsf:conditional_bound} so that
    \begin{align}
        \bbP\bracks{\abs{\rffsigkernelTRP[m](\bx, \by) - \rffsigkernel[m](\bx, \by)} \geq \epsilon} &= \bbE\bracks{\prob{\abs{\rffsigkernelTRP[m](\bx, \by) - \rffsigkernel[m](\bx, \by)} \geq \epsilon \given \bW}}
        \\
        &\leq 2\expe{\exp\pars{- \frac{\dimTRP^{\frac{1}{2m}} \dimRFF \epsilon^{\frac{1}{m}}}{2\sqrt{2}e \beta_m(\bx, \by)^{\frac{1}{2m}} \prod_{p=1}^m \norm{\bW^{(p)}}^\frac{2}{m}_2}}}
        \\
        &\stackrel{\text{(k)}}{\leq} 2\expe{\exp\pars{- \frac{\lambda \dimTRP^{\frac{1}{2m}} \epsilon^{\frac{1}{m}}}{2\sqrt{2}e \beta_m(\bx, \by)^{\frac{1}{2m}}} \frac{\dimRFF}{  \lambda \prod_{p=1}^m \norm{\bW^{(p)}}^\frac{2}{m}_2}}} 
        \\
        &\stackrel{\text{(l)}}{\leq} 2\expe{\exp\pars{- 2\pars{\frac{\lambda \dimTRP^{\frac{1}{2m}} \epsilon^{\frac{1}{m}}}{2\sqrt{2}e \beta_m(\bx, \by)^{\frac{1}{2m}}}}^\frac{1}{2} + \frac{\lambda \prod_{p=1}^m \norm{\bW^{(p)}}^\frac{2}{m}_2}{\dimRFF}}}
        \\
        &\stackrel{\text{(m)}}{\leq} 2\expe{\exp\pars{- \pars{\frac{\sqrt{2}\lambda \dimTRP^{\frac{1}{2m}} \epsilon^{\frac{1}{m}}}{e \beta_m(\bx, \by)^{\frac{1}{2m}}}}^\frac{1}{2} + \frac{\lambda \sum_{p=1}^m \norm{\bW^{(p)}}^2_2}{m\dimRFF}}}, \label{eq:rfsf:expected_trp_prob_decomp}
    \end{align}
    where (k) follows form multiplying and dividing with a $\lambda > 0$, (l) from applying Lemma \ref{lem:reverse} with $p=\frac{1}{2}$ and $q = 1$, and (m) from the arithmetic-geometric mean inequality. 

    Bounding the MGF of ${w_{i,j}^{(p)}}^2$ for $p \in [m], i \in [d], j \in [\dimRFF]$, we have that 
    \begin{align}
    \bbE\bracks{\exp\pars{\lambda{w^{(p)}_{i,j}}^2}}
    &\stackrel{\text{(n)}}{\leq} \sum_{k \geq 0} \bbE\bracks{{w_{i,j}^{(p)}}^k} \frac{\lambda^k}{k!}
    \stackrel{\text{(o)}}{\leq} 1 + \lambda S + \frac{\lambda^2 S^2}{2} \sum_{k \geq 0} (\lambda R)^k
    \\
    &\stackrel{\text{(p)}}{=} 1 + \lambda S + \frac{\lambda^2 S^2}{2} \frac{1}{1 - \lambda R} \stackrel{\text{(q)}}{=} 1 + \frac{S}{2R} + \frac{S^2}{4R^2},
    \end{align}
    where (n) is the Taylor expansion, (o) is the condition \eqref{eq:rfsf:rfsf_trp_approx_cond} and applying Jensen inequality to the degree-$1$ term, (p) is the geometric series for $\lambda < \frac{1}{R}$, and (q) is choosing $\lambda = \frac{1}{2R}$. Hence,
    \begin{align}
        \expe{\exp\pars{\frac{\lambda \sum_{p=1}^m \norm{\bW^{(p)}}^2_2}{m\dimRFF}}}
        &=
        \expe{\exp\pars{\frac{\lambda \sum_{p=1}^m \sum_{i=1}^d \sum_{j=1}^{\dimRFF} {w_{i, j}^{(p)}}^2}{m \dimRFF}}}
        \\
        &\stackrel{\text{(r)}}{\leq} 
        \bbE^{1/(m \dimRFF)}\bracks{\exp\pars{\lambda \sum_{p=1}^m \sum_{i=1}^d \sum_{j=1}^{\dimRFF} {w_{i, j}^{(p)}}^2}}
        \\
        &\stackrel{\text{(s)}}{\leq} 
        \pars{1 + \frac{S}{2R} + \frac{S^2}{4R^2}}^d,
    \end{align}
    where (r) is due to the Jensen inequality (Lemma \ref{lem:jensen}), and (s) follows from the independence of the $w_{i,j}^{(p)}$'s for $p \in [m], i \in [d], j \in [\dimRFF]$.
    Finally, plugging this into \eqref{eq:rfsf:expected_trp_prob_decomp}, we get that
    \vspace{-10pt}
    \begin{align}
        \bbP\bracks{\abs{\rffsigkernelTRP[m](\bx, \by) - \rffsigkernel[m](\bx, \by)} \geq \epsilon}
        \leq
        2\pars{1 + \frac{S}{2R} + \frac{S^2}{4R^2}}^d
        \exp\pars{- \pars{\frac{ \dimTRP^{\frac{1}{2m}} \epsilon^{\frac{1}{m}}}{\sqrt{2}e R \beta_m(\bx, \by)^{\frac{1}{2m}}}}^\frac{1}{2}}.
    \end{align}
    \vspace*{-10pt}
    
    Finally, note that $\beta_m(\bx, \by)^\frac{1}{2m} = \pars{\frac{3^m + 1}{(m!)^4}}^\frac{1}{2m} \norm{\bx}_{\onevar} \norm{\by}_{\onevar} \leq \frac{2e^2}{m^2} \norm{\bx}_{\onevar} \norm{\by}_{\onevar}$, since $3^m + 1 \leq 4^m$ for $m \geq 1$, and $\frac{1}{m!} \leq \pars{\frac{e}{m}}^m$ due to Stirling's approximation.
\end{proof}

\section{Algorithms} \label{apx:algs}

\begin{algorithm}[H]
    \begin{footnotesize}
	\caption{Computing the \RFSF{} map $\rffsig[\leq M]$.}
	\label{alg:rfsf:rfsf}
	\begin{algorithmic}[1]
		\STATE {\bfseries Input:}  Sequences $\bX=(\bx_i)_{i=1}^N \subset \Seq(\cX)$, measure $\Lambda$, truncation $M \in \bbZ_+$, \RFF{} sample size $\dimRFF \in \bbZ_+$
		\STATE Optional: Add time-parameterization $\bx_i \gets (\bx_{i, t}, t / \ell_{\bx_i})_{t=1}^{\ell_{\bx_i}}$ for all $i \in [N]$
		\STATE Tabulate to uniform length $\ell = \max_{j \in [N]} \ell_{\bx_j}$ by $\bx_i \gets (\bx_{i, 1}, \ldots, \bx_{i, \ell_{\bx_i}}, \ldots, \bx_{i, \ell_{\bx_i}})$ for all $i \in [N]$
		\STATE Sample independent \RFF{} weights $\bW^{(1)}, \dots, \bW^{(M)} \stackrel{\iid}{\sim} \Lambda^{\dimRFF}$
		\STATE Initialize an array $U$ with shape $[M, N, \ell-1, 2\dimRFF]$
		\STATE Compute increments $U[m, i, t, :] \gets \rff_m(\bx_{i,t+1}) - \rff_m(\bx_{i, t})$ for $m \in [M]$, $i \in [N]$, $t \in [\ell-1]$
        \STATE Initialize array $V \gets U[1, :, :, :]$
        \STATE Collapse into level-$1$ features $P_1 \gets V[:, \Sigma, :]$
		\FOR{$m=2$ {\bfseries to} $M$}
		\STATE Update with next increment $V \gets V[:, \boxplus+1, :] \boxtimes_{3} U[m, :, :, :]$
		\STATE Collapse into level-$m$ features $P_m \gets V[:, \Sigma, :]$
		\ENDFOR
		\STATE {\bfseries Output:} Arrays of \RFSF{} features per signature level $P_1, \dots, P_M$.
	\end{algorithmic}
    \end{footnotesize}
\end{algorithm}

\vspace{-5pt}

\begin{algorithm}[H]
    \begin{footnotesize}
	\caption{Computing the \RFSFD{} map $\rffsigDP[\leq M]$.}
	\label{alg:rfsf:rsfsdp}
	\begin{algorithmic}[1]
		\STATE {\bfseries Input:} Sequences $\bX=(\bx_i)_{i=1}^{N} \subset \Seq(\cX)$, measure $\Lambda$, truncation $M \in \bbZ_+$, \RFSFD{} sample size $\dimRFF \in \bbZ_+$
		\STATE Optional: Add time-parameterization $\bx_i \gets (\bx_{i, t}, t / \ell_{\bx_i})_{t=1}^{\ell_{\bx_i}}$ for all $i \in [N]$
		\STATE Tabulate to uniform length $\ell = \max_{j \in [N]} \ell_{\bx_j}$ by $\bx_i \gets (\bx_{i, 1}, \ldots, \bx_{i, \ell_{\bx_i}}, \ldots, \bx_{i, \ell_{\bx_i}})$ for all $i \in [N]$
		\STATE Sample independent \RFF{} weights $\bW^{(1)}, \dots, \bW^{(M)} \stackrel{\iid}{\sim} \Lambda^{\dimRFF}$
		\STATE Initialize an array $U$ with shape $[M, N, \ell-1, \dimRFF, 2]$
		\STATE Compute increments $U[m, i, t, k, :] \gets \hat\kernelfeatures_{m, k}(\bx_{i,t+1}) - \hat\kernelfeatures_{m, k}(\bx_{i, t})$ for $m \in [M]$, $i \in [N]$, $t \in [\ell-1]$, $k \in [\dimRFF]$
        \STATE Initialize array $V \gets \frac{1}{\sqrt{\dimRFF}} U[1, :, :, :, :]$
        \STATE Collapse into level-$1$ features $P_1 \gets V[:, \Sigma, :, :]$
		\FOR{$m=2$ {\bfseries to} $M$}
		\STATE Update with next increment $V \gets V[:, \boxplus+1, :, :] \boxtimes_4 U[m, :, :, :, :]$
		\STATE Collapse into level-$m$ features $P_m \gets V[:, \Sigma, :, :]$
		\ENDFOR
		\STATE {\bfseries Output:} Arrays of \RFSFD{} features per signature level $P_1, \dots, P_M$.
	\end{algorithmic}
    \end{footnotesize}
\end{algorithm}

\vspace{-5pt}

\begin{algorithm}[H]
    \begin{footnotesize}
	\caption{Computing the \RFSFT{} map $\rffsigTRP[\leq M]$.}
	\label{alg:rfsf:rsfstrp}
	\begin{algorithmic}[1]
		\STATE {\bfseries Input:} Sequences $\bX=(\bx_i)_{i=1}^N \subset \Seq(\cX)$, measure $\Lambda$, truncation $M \in \bbZ_+$, \RFSF{} and \TRP{} sample size $\dimRFF \in \bbZ_+$
		\STATE Optional: Add time-parameterization $\bx_i \gets (\bx_{i, t}, t / \ell_{\bx_i})_{t=1}^{\ell_{\bx_i}}$ for all $i \in [N]$
		\STATE Tabulate to uniform length $\ell = \max_{j \in [N]} \ell_{\bx_j}$ by $\bx_i \gets (\bx_{i, 1}, \ldots, \bx_{i, \ell_{\bx_i}}, \ldots, \bx_{i, \ell_{\bx_i}})$ for all $i \in [N]$
		\STATE Sample independent \RFF{} weights $\bW^{(1)}, \dots, \bW^{(M)} \stackrel{\iid}{\sim} \Lambda^{\dimRFF}$
		\STATE Sample standard normal matrices  $\bP^{(1)}, \dots, \bP^{(M)} \stackrel{\iid}{\sim} \cN^{\dimRFF}(0, \b I_{2\dimRFF})$
            \STATE Initialize an array $U$ with shape $[M, N, \ell-1, \dimRFF]$
		\STATE Compute projected increments $U[m, i, t, :] \gets {\bP^{(m)}}^\top \pars{\rff_m(\bx_{i,t+1}) - \rff_m(\bx_{i, t})}$ for $m \in [M]$, $i \in [N]$, $t \in [\ell-1]$
        \STATE Initialize array $V \gets \frac{1}{\sqrt{\dimRFF}} U[1, :, :, :]$
        \STATE Collapse into level-$1$ features $P_1 \gets V[:, \Sigma, :]$
		\FOR{$m=2$ {\bfseries to} $M$}
		\STATE Update with next increment $V \gets V[:, \boxplus+1, :] \odot U[m, :, :, :]$
		\STATE Collapse into level-$m$ features $P_m \gets V[:, \Sigma, :]$
		\ENDFOR
		\STATE {\bfseries Output:} Arrays of \RFSFT{} features per signature level $P_1, \dots, P_M$.
	\end{algorithmic}
    \end{footnotesize}
\end{algorithm}
\end{subappendices}

\endgroup
\begingroup
\chapter{Recurrent Sparse Spectrum Signature Gaussian processes} \label{ch:vfsf}
\section{Introduction}
Time series forecasting plays a central role in various domains, including finance \cite{sezer2020financial}, renewable energy \cite{wang2019review, adachi2023bayesian}, and healthcare \cite{bui2018time}. Ongoing research faces several challenges, such as non-linear data domains, ordered structure, time-warping invariance, discrete and irregular sampling, and scalability.
Among the numerous existing approaches, we focus on signature approach, which address the first four challenges and have demonstrated state-of-the-art performance as a kernel for sequential data \cite{toth2023random}. However, the trade-off between accuracy and scalability remains unresolved. While the exact computation of the signature kernel requires quadratic complexity for sequence length $L$ \cite{kiraly2019kernels}, existing approximate methods reduce complexity to linear, at the cost of degraded performance on large-scale datasets. Moreover, these methods treat the signature as a global feature, which limits their ability to capture recent local information effectively. This is particularly crucial for time series forecasting, where capturing near-term correlations are often as important as identifying long-term trends.

\paragraph{Contributions} 
We propose Random Fourier Decayed Signature Features with Gaussian Processes, which dynamically adjust its context length based on the data, prioritizing more recent information, and call the resulting model Recurrent Sparse Spectrum Signature Gaussian Process (RS\textsuperscript{3}GP). By incorporating variational inference, our model learns the decay parameters in a data-driven manner, allowing it to transform time series data into a joint predictive distribution at scale while maintaining accuracy. Our approach outperforms other \texttt{GP} baselines for time series on both small and large datasets and demonstrates comparable performance to state-of-the-art deep learning diffusion models while offering significantly faster training times.

\subsection{Related work}


\paragraph{Gaussian process approach} 
There are two prominent approaches: (1) specialized kernels for sequences \cite{lodhi2002text, cuturi2011fast, cuturi2011autoregressive, al2017learning}, and (2) state-space models \cite{frigola2013bayesian, frigola2014variational, mattos2016recurrent, eleftheriadis2017identification, doerr2018probabilistic, ialongo2019overcoming}. These approaches can complement each other; modeling the latent system as a higher-order Markov process allows sequence kernels to capture the influence of past states. The random projection approach, also known as sparse spectrum approximation in the Gaussian process community \cite{lazaro2010sparse, wilson2014fast, gal2015improving, dao2017gaussian, li2024trigonometric}.

\paragraph{Deep learning approach} 
From classic LSTMs \cite{schmidhuber1997long} to transformers \cite{vaswani2017attention}, deep learning methods have been widely applied to sequential data. In the context of probabilistic forecasting, diffusion models \cite{sohl2015deep, ho2020denoising, kollovieh2024predict} are considered among the state-of-the-art techniques. Although deep learning models can approximate any continuous function, they often require a large number of parameters, suffer from high variance, and offer limited interpretability. This creates opportunities for alternative approaches, which can serve not only as competitors but also as complementary components within larger models. 

\section{Background} \label{sec:vfsf:back}

\subsection{Signature Features and Kernels}
\vspace{-0.5em}
The path signature $\S(\bx)$ provides a graded description of a path $\bx:[0,T]\to \bbR^d$ by mapping it to a hierarchy of tensors $\S(\bx)= \pars{1, \S_1(\b x), \S_2(\b x), \ldots} \in \prod_{m \geq 0} (\bbR^d)^{\otimes m}$ of increasing degrees.
Among its attractive properties are that: it maps from a nonlinear domain (there is not natural addition of paths of different length) to a linear space $\prod (\bbR^d)^{\otimes m}$; $\bx \mapsto \S(\bx)$ is injective up to time-parameterization (thus naturally factoring out time-warping) and it can be made injective by adding time as a path coordinate (i.e.~setting $x^0_t=t)$; it linearizes path functionals, that is we can approximate non-linear path functionals $f(\bx)$ as linear functionals of signatures $f(\bx) \approx \langle \ell, \S(\bx)\rangle$, see Section \ref{subsec:back:props} for details. Additionally, if the path is random, the expected path signature characterizes the distribution, $\mu \mapsto \mathbb{E}_{\bx \sim \mu}[\S(\bx)]$ is injective, see \cite{chevyrev2022signature}.
All this makes it a viable transformation for extracting information from time series and paths in machine learning.

\paragraph{Discrete-time signature features}
In practice, we do not have access to continuous paths but are only given a sequence $\bx = (\bx_{0}, \dots, \bx_L) \in \Seq(\bbR^d)$  in $\bbR^d$. Here, we focus on the higher-order discretized signature variation from Section \ref{sec:back:discrete}, where the order parameter is set to the truncation level $p=M$. Hence, in this case we can simply identify $\bx$ with a continuous time path by its linear interpolation for any choice of time parameterization.
To compute signatures, denote with $\bar\Delta_m(n)$ the collection of all ordered $m$-tuples such that
\begin{align} \label{eq:vfsf:delta}
   \bar\Delta_m(n) = \curls{1 \leq i_1 \leq \cdots \leq i_m \leq n \setgiven i_1, \dots, i_m \in [n]}.
\end{align}
 Then, it can be shown using Chen's relation (Proposition~\ref{prop:chen}) that its signature can be computed by the recursion for $1 \leq m$, and $1 \leq l \leq L$:
\begin{align} \label{eq:vfsf:sig_recursion}
    \S_m(\bx_{0:l}) = \S_m(\bx_{0:l-1}) + \sum_{p=1}^m \S_{m-p}(\bx_{0:l-1}) \otimes \frac{(\delta \bx_l)^{\otimes p}}{p!}
\end{align}
with the identification $S_m(\bx_{0:0}) \equiv 0$, and the first-order difference operator is defined as $\delta \bx_k = \bx_k - \bx_{k-1}$. Unrolling this along the length of the time series with respect to the first term, we get
\begin{align} \label{eq:vfsf:sig_recursion_unrolled}
    \S_m(\bx_{0:l}) = \sum_{k=1}^l \sum_{p=1}^m \S_{m-p}(\bx_{0:k-1}) \otimes \frac{(\delta \bx_k)^{\otimes p}}{p!}.
\end{align}
Recall from Section \ref{sec:back:discrete} that this expression can be evaluated in closed form:
\begin{align} \label{eq:vfsf:sigexplicit}
    \S_m(\bx_{0:k}) = \sum_{\b i \in \bar\Delta_m(\len{\bx})} \frac{1}{\b i!} \delta \bx_{i_1} \otimes \cdots \otimes \delta \bx_{i_m},
\end{align}
where $\b i! = k_1! \cdots k_q!$ such that there are $q \in \bbZ_+$ unique indices in the multi-index $\b i$ and $k_1, \cdots, k_q$ are the number of times they are repeated.
It will be important later on that \eqref{eq:vfsf:sig_recursion} allows to compute the signature up to all time steps using a paralellizable scan operation. 

\paragraph{Discrete-time signature kernels} 
As $\S_m(\bx)$ is tensor-valued, it has $d^m$ coordinates, which makes its computation infeasible for high-dimensional state-spaces or high $m$. 
Kernelization allows to alleviate this by first lifting sequences and paths from $\bbR^d$ into a RKHS and subsequently computing inner products of signature features in this RKHS. This alleviates the computational bottleneck and further adds more expressivity by first lifting the paths into the RKHS; see Section \ref{sec:back:sigkernels}. 
Concretely, the signature kernel $\sigkernel: \Seq(\bbR^d) \times \Seq(\bbR^d) \rightarrow \bbR$ is defined as 
\begin{align} \label{eq:vfsf:sig_kernel}
    \sigkernel(\bx, \by) = \sum_{m=0}^M \kernel_{S_m}(\bx, \by),
\end{align}
where $M \in \bbZ_+ \cup \{\infty\}$ is the truncation level. Although truncating the series \eqref{eq:vfsf:sig_kernel} may seem like a limitation, it often happens that a careful selection of finite truncation level outperforms the untruncated case, since it can mitigate overfitting.

If we identify two sequences $\bx, \by \in \Seq(\bbR^d)$ with paths given by their linear interpolations for some choice of time parameterization, then we get the analogous expression to \eqref{eq:vfsf:sigexplicit}:
\begin{align} \label{eq:vfsf:discr_sig_kernel}
    \sigkernel[m](\bx, \by) = \sum_{\substack{\bi \in \bar\Delta(\len{\bx})\\ \bj \in \bar\Delta(\len{\by})}} \frac{1}{\bi! \bj!} \delta_{i_1, j_1} \kernel(\bx_{i_1}, \by_{j_1}) \cdots \delta_{i_m, j_m} \kernel(\bx_{i_m}, \by_{j_m}),
\end{align}
where $\bi!$ is as previously, and $\delta_{i, j}$ is a second-order differencing operator, such that $\delta_{i, j} \kernel(\bx_i, \by_j) = \kernel(\bx_{i+1}, \by_{j+1}) - \kernel(\bx_{i+1}, \by_j) - \kernel(\b x_i, \b y_{j+1}) + \kernel(\b x_i, \b y_j).$ A similar recursion to \eqref{eq:vfsf:sig_recursion} exists to compute \eqref{eq:vfsf:discr_sig_kernel}, see Algorithm \ref{alg:back:sigkernel_p} in Section \ref{sec:back:algs}. For further details, see Sections \ref{sec:back:sigkernels} and \ref{sec:back:discrete}.

\subsection{Random Fourier Features}
Kernelization circumvents the computational burden of a high- or infinite-dimensional feature space, but the cost is the associated quadratic complexity in the number of samples which is due to the evaluation of the Gram matrix. 
Random Fourier Features \cite{rahimi2007random} address this by constructing for a given kernel $\kernel$ on $\bbR^d$ a random feature map for which the inner product is a good random approximation to the original kernel. 
Many variations proposed, see \cite{liu2021random, chamakh2020orlicz}.

Existing variations are based on Bochner's theorem, which allows for a spectral representation of stationary kernels. Let $\kernel: \bbR^d \times \bbR^d \to \bbR$ be a continuous, bounded, stationary kernel. Then, Bochner's theorem states that there exists a non-negative finite measure $\Lambda$ over $\bbR^d$, such that $\kernel$ is its Fourier transform. In other words, for $\bx, \by, \in \bbR^d$, $\kernel$ can be be represented as
\begin{align}
    \kernel(\bx, \by) 
    &= \int_{\bbR^d} \exp(i \bomega^\top(\bx-\by)) \d \Lambda(\bomega)\\ 
    &= \int_{\bbR^d} \cos(\bomega^\top (\bx - \by) \d\Lambda(\bomega), \label{eq:vfsf:bochner_cos}
\end{align}
since the kernel is real-valued. Without loss of generality, we may assume that $\Lambda(\bbR^d) = 1$ so $\Lambda$ is a probability measure, since it amounts to rescaling the kernel. 

Now, we may draw Monte Carlo samples to approximate \eqref{eq:vfsf:bochner_cos}, $\bomega_1, \dots, \bomega_{D} \stackrel{\iid}{\sim} \Lambda$ for $D \in \bbZ_+$: 

\begin{align} \label{eq:vfsf:rff1}
    \kernel(\bx, \by) \approx \frac{1}{D} \sum_{i=1}^{D} \cos(\bomega_i^\top(\bx - \by)). 
\end{align}

There are two ways to proceed from here: one is to use the cosine identity $\cos(x - y) = \cos(x)\cos(y) + \sin(x)\sin(y)$, or to further approximate the sum in \eqref{eq:vfsf:rff1} using the identity
{\small
\begin{align} \label{eq:vfsf:cos_prob_id}
    \cos(x - y) = \bbE_{b \sim \cU(0, 2\pi)}\bracks{\sqrt{2} \cos(x + b) \sqrt{2} \cos(y + b)}
\end{align}}
see \cite[App.~A]{gal2015improving}. We approximate the expectation \eqref{eq:vfsf:rff1} by drawing $1$ sample for each term:
\begin{align}
    \kernel(\bx, \by) \approx \frac{2}{D} \sum_{i=1}^{D} \cos(\bomega_i^\top \bx + b_i) \cos(\bomega_i^\top \by + b_i),
\end{align}
where $b_1, \dots, b_{D} \sim \cU(0, 2\pi)$. This represents each sample using a $1$-dimensional feature, which will be important for us later on in Section \ref{sec:vfsf:rfsf_rev} when constructing our variation of the Random Fourier Signature Feature map. Now, let $\Omega = (\bomega_1, \dots, \bomega_{D}) \in \bbR^{d \times D}$ and $\bb = (b_1, \dots, b_{D}) \in \bbR^{D}$. Then, the \RFF{} feature map can be written concisely as
\begin{align} \label{eq:vfsf:rff_feature}
	\tilde\varphi(\b x) = \sqrt{\frac{2}{p}} \cos\pars{\Omega^\top \b x + \b b}.
\end{align}
This offers to approximate a stationary kernel using a random, finite-dimensional feature map, and consequently to reformulate downstream algorithms in weight-space, which avoids the usual cubic costs in the number of data points. Probabilistic convergence of the \RFF{} is studied in the series of works: \cite{rahimi2007random}, optimal rates were derived in \cite{sriperumbudur2015optimal}, extended to the kernel derivatives in \cite{szabo2019kernel}, conditions on the spectral measure relaxed in \cite{chamakh2020orlicz}.

\section{Recurrent Sparse Spectrum Signature Gaussian processes} \label{sec:vfsf:rs3gp}
\subsection{Revisiting Random Fourier Signature Features} \label{sec:vfsf:rfsf_rev}

Now, we revisit how \RFF s can be combined with signatures to extract random features from a time series. To do so, we adapt the construction of Random Fourier Signature Features (\RFSF) from Chapter \ref{ch:rfsf}. We provide an adapted version of the diagonally projected \RFSF{} variant (\RFSFD), with a twist, which allows us to devise even more computationally efficient features. 


We will focus on the discrete-time setting, but the construction applies to continuous-time in an analogous manner. Firstly, we will build a $1$-dimensional random estimator for the signature kernel. Let $\kernel: \bbR^d \times \bbR^d \to \bbR$ be a continuous, bounded, stationary kernel with spectral measure $\Lambda$, $\bomega^{(1)}, \dots, \bomega^{(m)} \sim \Lambda$ and $b^{(1)}, \dots, b^{(m)} \sim \cU(0, 2\pi)$, and $\phi^{(p)}(\bx) \coloneqq \cos({\bomega^{(p)}}^\top \bx + b^{(p)})$ for $\bx \in \bbR^d$. Then, for sequences $\bx, \by \in \Seq(\bbR^d)$, it holds for the signature kernel $\sigkernel[m]$ as defined in \eqref{eq:vfsf:discr_sig_kernel} that
{
\begin{align} \label{eq:vfsf:sig_estimator}
    &\sigkernel[m](\bx, \by) = \bbE\bracks{\sum_{\substack{\bi \in \bar\Delta_m(\len{\bx})\\ \bj \in \bar\Delta_m(\len{\by})}} \frac{1}{\bi!\bj!} \prod_{p=1}^m \sqrt{2} \delta \phi^{(p)}(\bx_{i_1}) \sqrt{2} \delta \phi^{(p)}(\by_{j_1})}.
\end{align}}
The unbiasedness of the expectation is straightforward to check, since due to linearity and independence of $\phi^{(p)}$'s, it reduces to verifying the relations \eqref{eq:vfsf:cos_prob_id} and \eqref{eq:vfsf:bochner_cos}, which in turn recovers the formulation of the signature kernel from \eqref{eq:vfsf:discr_sig_kernel}. Next, we draw $D \in \bbZ_+$ Monte Carlo samples to approximate \eqref{eq:vfsf:sig_estimator}, hence, let $\bomega^{(1)}_1, \dots, \bomega^{(m)}_D \stackrel{\iid}{\sim} \Lambda$ and $b^{(1)}_1, \dots, b^{(m)}_D \stackrel{\iid}{\sim} \cU(0, 2\pi)$, and $\phi^{(p)}_k (\bx) = \cos({\bomega^{(p)}_k}^\top \bx + b^{(p)}_k)$ for $\bx \in \bbR^d$. Then,
\begin{align} \label{eq:vfsf:sig_approx}
    \sigkernel[m](\bx, \by) \approx \frac{1}{D} \sum_{k=1}^D \sum_{\substack{\bi \in \bar\Delta_m(\len{\bx})\\ \bj \in \bar\Delta_m(\len{\by})}} \frac{1}{\bi!\bj!} \prod_{p=1}^m \sqrt{2} \delta \phi^{(p)}_k(\bx_{i_1}) \sqrt{2} \delta \phi^{(p)}_k(\by_{j_1}).
\end{align}
We call this approximation Random Fourier Signature Features (\RFSF), and the features corresponding to the kernel on the \texttt{RHS} of \eqref{eq:vfsf:sig_approx} can be represented in a finite-dimensional feature space. Let us collect the frequencies $\Omega^{(p)} = (\bomega^{(p)}_1, \dots, \bomega^{(p)}_D) \in \bbR^{d \times D}$ and phases $\bb^{(p)} = (b^{(p)}_1, \dots, b^{(p)}_D)^\top \in \bbR^D$, and define the \RFF{} maps as in \eqref{eq:vfsf:rff_feature} so we have $\varphi^{(p)}(\bx) = \cos({\Omega^{(p)}}^\top \bx + \bb^{(p)})$. Then, the \RFSF{} map $\Phi_{m}: \Seq(\bbR^d) \to \bbR^D$ is defined for $\bx \in \Seq(\bbR^d)$ as
\begin{align} \label{eq:vfsf:rfsf}
    \Phi_{m}(\bx) = \sqrt{\frac{2^m}{D}} \sum_{\bi \in \bar\Delta_m(K)} \frac{1}{\bi!} \bigodot_{p=1}^m \delta \varphi^{(p)}(\bx_{i_p}),
\end{align}
where $\odot$ refers to the Hadamard product. Although \eqref{eq:vfsf:rfsf} looks difficult to compute, we can establish a recursion across time and signature levels. The update rule analogous to \eqref{eq:vfsf:sig_recursion} is
\begin{align} \label{eq:vfsf:rfsf_recursion}
    &\Phi_{m}(\bx_{0:l}) = \Phi_{m}(\bx_{0:l-1}) 
    + \sum_{p=1}^m \frac{1}{p!} \Phi_{m-p}(\bx_{0:l-1}) \bigodot_{q = m-p+1}^m \delta \varphi^{(q)}(\bx_l),
\end{align}
which can be computed efficiently using parallelizable scan operation. Additionally, we overload the definition of the differencing operator $\delta$ to mean fractional differencing for some learnable differencing order $\alpha \in (0, 1)$. Thus, in \eqref{eq:vfsf:rfsf}, we learn a separate differencing order parameter for each feature channel, see Appendix \ref{app:vfsf:fracdiff}.

In practive, we use the first $M \in \bbZ_+$ levels in conjunction, and normalize each \RFSF{} level to unit norm, so that the full \RFSF{} map $\Phi: \Seq(\bbR^d) \to \bbR^{M D + 1}$ is
\begin{align} \label{eq:vfsf:rfsf_full}
    \Phi(\bx) = \pars{1, \frac{\Phi_1(\bx)}{\norm{\Phi_1(\bx)}_2}, \dots, \frac{\Phi_M(\bx)}{\norm{\Phi_M(\bx)}_2}}.
\end{align}
We note that $\Phi(\bx)$ can be computed with complexity $O((M+W) L D + M L D d)$, where $W \in \bbZ_+$ is the window size for fractional differencing, see Appendix \ref{app:vfsf:algs}. This is more efficient than the scalable variants of Chapter \ref{ch:rfsf}, i.e. it avoids the $2^M$ factor as in \RFSFD, and it also avoids the $D^2$ factor as in \RFSFT.

\subsection{A forgetting mechanism for signatures} \label{sec:vfsf:forgetting}  The recursive step in \RFSF{} \eqref{eq:vfsf:rfsf_recursion} allows to compute the feature map up to all time steps of a time series, similarly to a recurrent neural network (\texttt{RNN}), in one forward pass. This suggests to perform sequence-to-sequence regression (including time series forecasting, when the target sequence is future values of the time series) by considering the \RFSF{} map over an expanding window. However, signatures have no built-in forgetting mechanism, and it is well-known that the level-$m$ signature feature has magnitude $O\pars{\nicefrac{\norm{\bx}_\onevar^m}{m!}}$, and the same is true for \eqref{eq:vfsf:rfsf}, which shows that \RFSF{} only accumulates information without a way of forgetting. This is in contrast to modern \texttt{RNN}s, which have built-in gating mechanisms that allow or disallow the flow of historical information, allowing them to focus on more recent information.

In this section, we propose a novel forgetting mechanism tailored for \RFSF, but which can be applied to signature features in general. We tackle this by introducing time step dependent decay factors, which multiply each increment in the formulation \eqref{eq:vfsf:rfsf}. We assume exponential decay in time, so let $\blambda \in \bbR^D$ be channel-wise decay factors. Then, we define Random Fourier Decayed Signature Features (\texttt{RFDSF}) as
\begin{align} \label{eq:vfsf:rfdsf}
    \Phi_{m}(\bx_{0:l}) = \sqrt{\frac{2^m}{D}} \sum_{\bi \in \bar\Delta_m(K)} \frac{1}{\bi!} \bigodot_{p=1}^m \blambda^{\odot(l-i_p)} \odot \delta \varphi^{(p)}(\bx_{i_p}).
\end{align}
A recursion analogous to \eqref{eq:vfsf:rfsf_recursion} allows to compute \eqref{eq:vfsf:rfdsf}:
\begin{align} \label{eq:vfsf:rfdsf_recursion}
    &\Phi_{m}(\bx_{0:l}) = \blambda^{\odot m} \odot \Phi_{m}(\bx_{0:l-1}) + \sum_{p=1}^m \frac{1}{p!} \blambda^{\odot(m-p)} \odot \Phi_{m-p}(\bx_{0:l-1}) \bigodot_{q = m-p+1}^m \delta \varphi^{(q)}(\bx_l),
\end{align}
which can be unrolled over the length of a sequence so
\begin{align} \label{eq:vfsf:rfdsf_recursion_unrolled}
    &\Phi_{m}(\bx_{0:l}) = \sum_{k=1}^l \blambda^{\odot m(l-k)} \sum_{p=1}^m \frac{1}{p!} \blambda^{\odot(m-p)} \odot \Phi_{m-p}(\bx_{0:l-1}) \bigodot_{q = m-p+1}^m \delta \varphi^{(q)}(\bx_k),
\end{align}
which is an exponential decay over time steps of an appropriately chosen sequence depending on signature levels lower than $m$. Hence, \eqref{eq:vfsf:rfdsf_recursion_unrolled} is still computable by a parallelizable scan operation, allowing for sublinear time computations. This makes the feature map well-suited for long time series, due to the adjustment of context length by the decay factors, and the possibility for parallelization across time. The theoretical complexity is the same as previously, i.e.~$O((M+W)MLD + MLDd)$, see Appendix \ref{app:vfsf:algs}, but we emphasize that a work-efficient scan algorithm allows to compute the recursion on a \texttt{GPU} log-linearly.

\subsection{Recurrent Sparse Spectrum Signature Gaussian processes}
Next, we construct our (Variational) Recurrent Sparse Spectrum Signature Gaussian Process (\texttt{(V)RS\textsuperscript{3}GP}) model. We define a Gaussian process, which treats the hyperparameters of the random covariance function in a probabilistic fashion \cite{gal2015improving} by incorporating them into inference. We formulate our variational \texttt{GP} model in the feature space, as opposed to in the function space. 


\paragraph{Bayesian formulation}
We focus on supervised learning, and assume our dataset consists of multiple input time series $\bx \in \Seq(\bbR^d)$ each with a corresponding univariate output time series $\by \in \Seq(\bbR)$, such that $\len{\bx} = \len{\by}$. We note that the multivariate case can also be handled by stacking multiple \texttt{GP} priors, one for each output coordinate. 

As the signature is a universal feature map, we expect that a linear layer on top of the computed \texttt{RFDSF} features will work well for approximating functions of time series. We assume a probabilistic model, which models the prediction process as a linear layer on top of \texttt{RFDSF}, corrupted by Gaussian noise, that is,
\begin{align}
    y_l = \b w^\top \Phi(\bx_{0:l}) + \epsilon_l \quad \text{for } l = 0, \dots, L,
\end{align}
where $\b w \in \bbR^{MD+1}$, such that $M$ is the signature truncation level and $D$ is the \RFF{} dimension. We place an $\iid$ standard Gaussian prior on $\bw$, $p(\b w) = \c N(0, I_{MD+1})$. The model noise is $\iid$ Gaussian with learnable variance $\sigma_y^2 > 0$, so that $p(y \given \b w, \bOmega, \b B, \bx) = \c N(\b w^\top \Phi(\b x), \sigma_y^2)$, where we explicitly denoted the dependence on the hyperparameters of the \RFSF{} map $\Phi$, frequencies $\bOmega = (\Omega^{(1)}, \dots, \Omega^{(M)}) \in \bbR^{M \times d \times D}$, and phases $B = (\b b^{(1)}, \dots, \b b^{(M)}) \in \bbR^{M \times D}$. Next, we specialize to the case of the ARD Gaussian kernel, which has $\Lambda = \cN(\b 0, D^{-1})$, where $D \in \bbR_+^{D \times D}$ is a diagonal matrix of lengthscales. We assign independent lengthscales to each random feature map in \eqref{eq:vfsf:rfdsf}.
Then, following previous work \cite{gal2015improving, cutajar2017random}, we place the priors on the \RFF{} parameters
\begin{align} \label{eq:vfsf:rfsf_priors}
&p(\bOmega) = \prod_{m=1}^M p(\Omega^{(m)}) = \prod_{m=1}^m \c N^D(\b 0, D_m^{-1}),\\
&p(B) = \prod_{m=1}^M p(\bb^{(m)}) = \prod_{m=1}^M \cU^D(0, 2\pi),
\end{align}
where $D_m = \diag(\ell_{m, 1}, \dots, \ell_{m, d})$ refers to the diagonal lengthscale matrix of the $m^{th}$ \RFF{} map.


As noted in \cite{cutajar2017random}, this model corresponds to a 2-layer Bayesian neural network, where the first layer is given by the \texttt{RFDSF} activations with priors on it as given by \eqref{eq:vfsf:rfsf_priors}, while the second layer is a linear readout layer with a Gaussian prior on it. Inference in this case is analytically intractable. For a Gaussian likelihood, although $\b w$ can be marginalized out in the log-likelihood, it is not clear how to handle the integrals with respect to $\bOmega$  and $B$.  Further, disregarding the latter issue of additional hyperparameters, it only admits full-batch training, and this limits scalability to datasets with large numbers of input-output pairs. To this end, we introduce a variational approximation, which allows for both scalability and flexibility.

\paragraph{Variational treatment}
We variationally approximate the posterior to define an evidence lower bound (\texttt{ELBO}).
We define factorized variational distributions over the parameters $\b w \in \bbR^{MD+1}$, frequencies $\bOmega \in \bbR^{M \times d \times D}$, and phases $B \in \bbR^{M \times D}$. The variational over $\b w \in \bbR^{MD+1}$ is as usual given by a Gaussian $q(\b w) = \c N(\bmu_{\b w}, \Sigma_{\b w})$, where $\bmu_{\b w} \in \bbR^{MD+1}$ is the variational mean and $\Sigma_{\b w} \in \bbR^{(MD+1) \times (MD+1)}$ is the variational covariance matrix, which is symmetric and positive definite, represented in terms of its Cholesky factor, $L_{\b w} \in \bbR^{(MD+1) \times (MD+1)}$, such that $\Sigma_{\b w} = L_{\b w}L_{\b w}^\top$. The other parameters get the factorized variationals
\begin{align}
	&q(\bOmega) = \prod_{m=1}^M q(\Omega^{(m)}) = \prod_{m=1}^M \prod_{i=1}^D \prod_{j=1}^d \c N (\mu_{mij}, \sigma^2_{mij}),\\
    &q(B) = \prod_{m=1}^M q(\b b^{(m)}) =  \prod_{m=1}^M \prod_{i=1}^D \c B_{[0, 2\pi]}(\alpha_{mi}, \beta_{mi}),
\end{align}
where $\alpha_{mi}, \beta_{mi} > 0$, so that the posteriors over the component frequencies of $\omega$ are independent Gaussians, while the posteriors over the phases $b$ are independent beta distributions on $[0, 2\pi]$.

The \texttt{KL}-divergence to the posterior is then minimized:
{
\begin{align}
&\KL{q(\bw, \bOmega, B)}{p(\bw, \bOmega, B \given \by)} = \int q(\b w, \bOmega, B) \log \frac{q(\b w, \bOmega, B)}{p(\b w, \bOmega, B \given \b y)} \d \b w \d \bOmega \d B.
\end{align}}

Then, by the usual calculations, we get the \texttt{ELBO}, $\c L_{\texttt{ELBO}} \leq \log p(\b y)$, such that
{\footnotesize
\begin{align} \label{eq:vfsf:elbo}
    \c L_{\texttt{ELBO}} = &\underbrace{\sum_{i=1}^N \bbE_q\bracks{\log p(y_i \given \b w, \bOmega, B)}}_{\text{data-fit term}} - \underbrace{\KL{q(\b w)}{p(\b w)} - \KL{q(\bOmega)}{p(\bOmega)} - \KL{q(B)}{p(B)}}_{\text{\texttt{KL} regularizers}}
\end{align}}

The predictive distribution for a set of (possibly overlapping) sequences $\b X = (\bx_1, \dots, \bx_N) \subset \Seq(\bbR^d)$ is then, denoting $\Phi(\b X) = \bracks{\Phi(\bx_1), \ldots, \Phi(\bx_N)}^\top \in \bbR^{N \times (MD+1)}$ and $\b f = \Phi(\b X) \b w \in \bbR^N$,
\begin{align} \label{eq:vfsf:variational_pred}
    q(\b f) = \c N \left( \Phi(\b X) \bmu_{\b w}^\top, \,\, \Phi(\b X) L_{\b w} L_{\b w}^\top \Phi(\b X)^\top \right).
\end{align}

In order to evaluate the data-fit term in \eqref{eq:vfsf:elbo}, and make inference about unseen points in \eqref{eq:vfsf:variational_pred}, we require a way to handle the randomness in $\bOmega$ and $B$. Two solutions proposed in \cite{cutajar2017random} are to perform Monte Carlo sampling or to sample them once at the start of training and keep them fixed while learning their distributional hyperparameters using the reparameterization trick \cite{kingma2013auto}. Our observations align with theirs, and we found that resampling leads to high variance and non-convergence. The reason for this is likely that the factorized variational distributions do not model correlations between $\bw$ and $\bOmega, B$. Hence, we use a fixed random outcome throughout with the reparameterization trick, see Appendix \ref{app:vfsf:random} for details.

\paragraph{Training objective}
Although the \texttt{ELBO} \eqref{eq:vfsf:elbo} lower bounds the log-likelihood, it suffers from well-known pathologies. One of these is that it greatly underestimates the latent function uncertainty, and mainly relies on the observation noise for uncertainty calibration, which leads to overestimation. This is clearly unideal for time series forecasting, where heteroscedasticity is often present, and since useful information carried by the kernel about the geometry of the data space is lost. There are several solutions proposed by \cite{jankowiak2020parametric}, and we adopt the approach called Parametric Predictive \texttt{GP} Regression (\texttt{PPGPR}), which restores symmetry between the training objective and the test time predictive distribution by treating both latent function uncertainty and observation noise on equal footing. The key idea is to modify the data-fit term in \eqref{eq:vfsf:elbo} taking the $\log$ outside the $q$-expectation so we have
{\small
\begin{align}
    \c L_{\texttt{PPGPR}} =& \sum_{i=1}^N \log\bbE_q\bracks{p(y_i \given \bw, \bOmega, B}
    - \KL{q(\b w)}{p(\b w)} - \KL{q(\bOmega)}{p(\bOmega)} - \KL{q(B)}{p(B)}
\end{align}}
which although is not a lower-bound anymore to the log-likelihood, leads to better calibrated uncertainties. However, as noted by \cite{jankowiak2020parametric}, this objective can lead to underfitting in data regions, where a good fit is harder to achieve, and instead relying on uncertainty overestimation. We account for this by modifying the objective function by introducing a penalty term for the latent function variance. Let $q(f_i) = \cN(\mu_i, \sigma_i^2)$. Then, we modify the objective as
{\small
\begin{align}
    \c L =& \sum_{i=1}^N \log\bbE_q\bracks{p(y_i \given \bw, \bOmega, B}
    - \KL{q(\b w)}{p(\b w)} - \KL{q(\bOmega)}{p(\bOmega)} - \KL{q(B)}{p(B)} - \alpha \sum_{i=1}^N \sigma_i^2,
\end{align}}
where $\alpha > 0$ is a regularization hyperparameter. This promotes the optimizer to escape local optima, and rely on a better fit of the predictive mean, rather than on uncertainty overestimation to fit the data. We give further details in Appendix \ref{app:vfsf:obj}.

\section{Experiments} \label{sec:vfsf:exp}
\paragraph{Implementation} We implemented our models and \texttt{GP} baselines using \texttt{PyTorch} \cite{paszke2019pytorch} and \texttt{GPyTorch} \cite{gardner2018gpytorch}.
The computing cluster used has 4 NVIDIA RTX 3080 TI \texttt{GPU}s.

\subsection{Synthetic Dataset}

\paragraph{Dataset}
We handcrafted a dataset to test the ability of our \texttt{(V)RS\textsuperscript{3}GP} model of adapting its context length to the data on a task which requires reasoning on multiple time horizons. The dataset consists of a multi-sinusoidal wave with multiple frequencies, which includes both large and low frequencies, requiring reasoning over both short and long time periods.

\paragraph{Methods} We compare three models: Variational Recurrent Sparse Spectrum Gaussian Process (\texttt{VRS\textsuperscript{3}GP}), our model constructed in the previous section; \texttt{RS\textsuperscript{3}GP}, which ablates the previous model as it does not learn variationals over the covariance parameters and sets them equal to the prior;  Sparse Variational Gaussian Process (\texttt{SVGP}) using the \texttt{RBF} kernel and a fixed hand-tuned context length. \texttt{(V)RS\textsuperscript{3}GP} uses $D = 200$ and $M = 5$, while \texttt{SVGP} uses $100$ inducing points.

\paragraph{Qualitative analysis} 
Figure~\ref{fig:demo} qualitatively illustrates the predictive mean and uncertainty of the different approaches. \texttt{SVGP} achieves a perfect fit of the training data, but in the testing regime fails to properly capture the underlying dynamics. \texttt{RS\textsuperscript{3}GP} properly captures both short and long horizon dynamics, but in certain data regions underfits the dataset. \texttt{VRS\textsuperscript{3}GP} on the other hand fits the data perfectly, and also perfectly captures the temporal dynamics, being more flexible than \texttt{RS\textsuperscript{3}GP} due to variational parameter learning. 

\begin{figure}[t]
    \centering
    \includegraphics[width=0.75\textwidth]{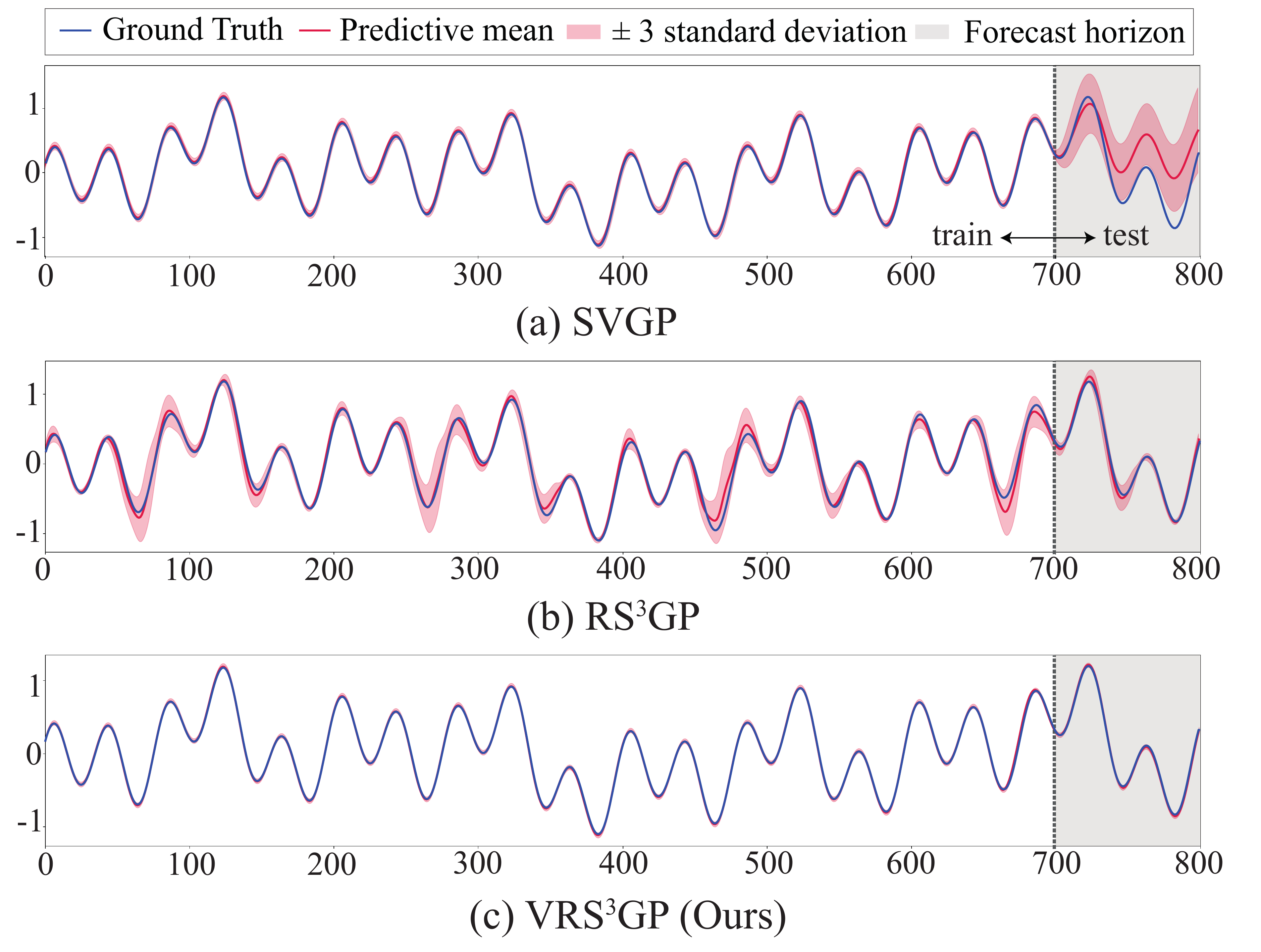}
    \caption{Predictive mean and uncertainty on a toy dataset composed of multi-sinusoidal waves with four distinct components including low and high frequencies, comparing various \texttt{GP} approaches. The true function is depicted in blue, while the predictive mean $\pm 3$ standard deviations are shown in red. The dataset consists of 700 training points, with a context window of 100 for \texttt{SVGP}, and a prediction horizon of 100 for all.}
    \label{fig:demo}
\end{figure}


\subsection{Real-World Datasets}

\begin{table*}
    \centering
    \caption{Forecasting results on eight benchmark datasets ranked by \texttt{CRPS}. The best and second best models have been shown as \textbf{bold} and \underline{underlined}, respectively.}
    \resizebox{1.0\textwidth}{!}{
    \begin{sc}
    \begin{tabular}{lcccccccc}
        \toprule
         method&  Solar & Electricity & Traffic & Exchange & M4 & UberTLC & KDDCup & Wikipedia\\
         \midrule
         \texttt{SeasonalNaïve} & 0.512 $\pm$ 0.000 & 0.069 $\pm$ 0.000 & 0.221 $\pm$ 0.000 & 0.011 $\pm$ 0.000 & 0.048 $\pm$ 0.000 & 0.299 $\pm$ 0.000 & 0.561 $\pm$ 0.000 & 0.410 $\pm$ 0.000\\
         \texttt{ARIMA} & 0.545 $\pm$ 0.006 & - & - & \textbf{0.008 $\pm$ 0.000} & 0.044 $\pm$ 0.001 & 0.284 $\pm$ 0.001 & 0.547 $\pm$ 0.004 & - \\
         \texttt{ETS} & 0.611 $\pm$ 0.040 & 0.072 $\pm$ 0.004 & 0.433 $\pm$ 0.050 &  \textbf{0.008 $\pm$ 0.000} & 0.042 $\pm$ 0.001 & 0.422 $\pm$ 0.001 & 0.753 $\pm$ 0.008 & 0.715 $\pm$ 0.002 \\
         Linear & 0.569 $\pm$ 0.021 & 0.088 $\pm$ 0.008 & 0.179 $\pm$ 0.003 & 0.011 $\pm$ 0.001 & 0.039 $\pm$  0.001 & 0.360 $\pm$ 0.023 & 0.513 $\pm$ 0.011 & 1.624 $\pm$ 1.114\\
         \midrule
         \texttt{DeepAR} & 0.389 $\pm$ 0.001 & \underline{0.054 $\pm$ 0.000} & \underline{0.099 $\pm$ 0.001} & 0.011 $\pm$ 0.003 & 0.052 $\pm$ 0.006 & \textbf{0.161} $\pm$ \textbf{0.002} & 0.414 $\pm$ 0.027 & 0.231 $\pm$ 0.008 \\
         \texttt{MQ-CNN} & 0.790 $\pm$ 0.063 & 0.067 $\pm$ 0.001 & - & 0.019 $\pm$ 0.006 & 0.046 $\pm$ 0.003 & 0.436 $\pm$ 0.020 & 0.516 $\pm$ 0.012 & 0.220 $\pm$ 0.001 \\
         \texttt{DeepState} & 0.379 $\pm$ 0.002 & 0.075 $\pm$ 0.004 & 0.146 $\pm$ 0.018 & 0.011 $\pm$ 0.001 & 0.041 $\pm$ 0.002 & 0.288 $\pm$ 0.087 & - & 0.318 $\pm$ 0.019 \\
         \texttt{Transformer} & 0.419 $\pm$ 0.008 & 0.076 $\pm$ 0.018 & 0.102 $\pm$ 0.002 & \underline{0.010 $\pm$ 0.000} & 0.040 $\pm$ 0.014 & 0.192 $\pm$ 0.004 & 0.411 $\pm$ 0.021 & \textbf{0.214} $\pm$ \textbf{0.001} \\
         \texttt{TSDiff} & \underline{0.358 $\pm$ 0.020} & \textbf{0.050} $\pm$ \textbf{0.002} & \textbf{0.094} $\pm$ \textbf{0.003} & 0.013 $\pm$ 0.002 &  0.039 $\pm$ 0.006 & \underline{0.172 $\pm$ 0.008} & 0.754 $\pm$ 0.007 & \underline{0.218 $\pm$ 0.010} \\
         \midrule
         \texttt{SVGP} & \textbf{0.341 $\pm$ 0.001} & 0.104 $\pm$ 0.037 & - & 0.011 $\pm$ 0.001 & 0.048 $\pm$ 0.001 & 0.326 $\pm$ 0.043 & 0.323 $\pm$ 0.007 & - \\
         \texttt{DKLGP} & 0.780 $\pm$ 0.269 & 0.207 $\pm$ 0.128 & - & 0.014 $\pm$ 0.004 & 0.047 $\pm$ 0.004 & 0.279 $\pm$ 0.068 & 0.318 $\pm$ 0.010 & - \\
         \midrule
         \texttt{RS\textsuperscript{3}GP} & 0.377 $\pm$ 0.004 & $0.057 \pm 0.001$ & $0.165 \pm 0.001$ & $0.012 \pm 0.001$ & \underline{0.038 $\pm$ 0.003} & $0.354 \pm 0.016$ & \underline{0.297 $\pm$ 0.007} & $0.310 \pm 0.012$ \\
         \texttt{VRS\textsuperscript{3}GP} & $0.366 \pm 0.003$ & 0.056 $\pm$ 0.001 & $0.160 \pm 0.002$ & 0.011 $\pm$ 0.001 & \textbf{0.035} $\pm$ \textbf{0.001} & $0.347 \pm 0.009$ & \textbf{0.291} $\pm$ \textbf{0.015} & $0.295 \pm 0.005$ \\
         \bottomrule
    \end{tabular}
    \end{sc}
    }
    \label{tab:CRPS}
\end{table*}
\begin{table}
    \centering
    \caption{Training time in hours}
    \setlength\tabcolsep{3.0pt}
    \resizebox{0.45\textwidth}{!}{
    \begin{sc}
    \begin{tabular}{lccccc}
        \toprule
          & \texttt{SVGP} & \texttt{DKLGP} & \texttt{TSDiff} & RS\textsuperscript{3}GP & VRS\textsuperscript{3}GP\\
         \midrule
         Solar & 1.90 & 2.41 & 2.75 & 0.27 & 0.45\\
         Electricity & 2.23 & 1.91 & 4.09 & 0.62 & 0.98 \\
         Traffic & - & - & 4.44 & 2.31 & 2.70 \\
         Exchange & 0.75 & 0.65  & 2.00 & 0.23 & 0.29 \\
         M4 & 1.68 & 1.94 & 1.96 & 0.86 & 1.02 \\
         UberTLC & 1.89 & 2.19  & 3.19 & 0.53 & 0.71 \\
         KDDCup & 2.03 & 1.55  & 2.10 & 0.65 & 0.93 \\
         Wikipedia & - & - & 6.17 & 3.52 & 6.10 \\
         \bottomrule
    \end{tabular}
    \end{sc}
    }
    \label{tab:time}
\end{table}
In this section, we present empirical results on several real-world datasets.

\paragraph{Datasets} 
We conducted experiments on eight univariate time series datasets from different domains, available in \texttt{GluonTS} \cite{alexandrov2020gluonts}---Solar \cite{lai2018modeling}, Electricity \cite{asuncion2007uci}, Traffic \cite{asuncion2007uci}, Exchange \cite{lai2018modeling}, M4 \cite{makridakis2020m4}, UberTLC \cite{gasthaus2019probabilistic}, KDDCup \cite{godahewa2021monash}, and Wikipedia \cite{gasthaus2019probabilistic}. We use GluonTS \cite{alexandrov2020gluonts} to load the datasets, which has pre-specified train-test splits for each. See \cite[App.~B.1]{toth2024learning} for details. 

\paragraph{Metric} We employed the continuous ranked probability score (\texttt{CRPS}) \cite{gneiting2007strictly} for evaluating probabilistic forecasts. We approximate the \texttt{CRPS}
by the normalized mean quantile loss, and report means and standard deviations over three independent runs, see \cite[App.~B.3]{toth2024learning}.

\paragraph{Baselines}
We extend the list of baselines from \cite{kollovieh2024predict}. As such, we have included comprehensive baselines from three groups: classical statistics, deep learning, and other \texttt{GP} methods. For statistics, we included \texttt{SeasonalNaïve}, \texttt{ARIMA}, \texttt{ETS}, and a \texttt{RidgeRegression} model from the statistical literature \cite{hyndman2018forecasting}. Additionally, we compared against deep learning models that represent various architectural paradigms such as the \texttt{RNN}-based \texttt{DeepAR} \cite{salinas2020deepar}, the \texttt{CNN}-based \texttt{MQ-CNN} \cite{wen2017multi}, the state space model-based \texttt{DeepState} \cite{rangapuram2018deep}, the self-attention-based \texttt{Transformer} \cite{vaswani2017attention}, and diffusion-model-based \texttt{TSDiff} \cite{kollovieh2024predict}. For \texttt{GP} models, we add \texttt{SVGP} \cite{hensman2013gaussian}, uses 500 inducing points with the \texttt{RBF} kernel; and \texttt{DKLGP} \cite{wilson2016deep} which augments \texttt{SVGP} with a deep kernel using a $2$-layer \texttt{NN} with $64$ units. See \cite[App.~B.4]{toth2024learning} for further details.

\paragraph{Results} 
Table~\ref{tab:CRPS} shows the results of our models \texttt{VRS\textsuperscript{3}GP} and \texttt{RS\textsuperscript{3}GP} compared to baselines. We omit results for \texttt{SVGP} and \texttt{DKLGP} on Traffic and Wikipedia since inference time takes longer than 12 hours. Overall, our models achieve top scores on 2 datasets, outperforming the state-of-the-art diffusion baseline \texttt{TSDiff}, and provide comparable performance to deep learning baselines on the remaining 4/6 datasets. Importantly, they outperform linear baselines. They also outperform the \texttt{GP} baselines on 4/6 datasets. One of the datasets our model achieves the top score on, KDDCup, is the one with longest time series length, $L \sim 10000$, which suggests that this is due to its ability to adapt its context memory to long ranges. Moreover, Table~\ref{tab:time} illustrates that our models \texttt{VRS\textsuperscript{3}GP} and \texttt{RS\textsuperscript{3}GP} are significantly quicker to train \texttt{TSDiff}, and both \texttt{SVGP} and \texttt{DKLGP}. This speed up is especially pronounced for datasets with long time series such as Solar, Electricity, Exchange, UberTLC and KDDCup, since our model is able to process a long time series in sublinear time due to parallelizability as detailed in Section \ref{sec:vfsf:rs3gp}. We further investigate the scalability in Appendix \ref{app:vfsf:scale}.

\section{Conclusion and Limitations}
In this work, we introduced the Random Fourier Decayed Signature Features for time series forecasting, incorporating a novel forgetting mechanism into signature features, which addresses the need to adaptively prioritize local, recent information in long time series. Our proposed model, the Recurrent Sparse Spectrum Signature Gaussian Process (\texttt{RS\textsuperscript{3}GP}), leverages variational inference and recurrent structure to efficiently transform time series data into a joint predictive distribution. We demonstrated that our approach outperforms traditional \texttt{GP} models and achieves comparable performance to state-of-the-art deep learning methods, while offering significantly faster training times. Limitations include reliance on the Gaussian likelihood, which may be unsuitable for non-Gaussian or heavy-tailed distributions, where asymmetric noise models are warranted, or in cases where explicit target constraints can be used to refine predictions. Additionally, while we introduced a decay mechanism for handling the trade-off between global and local information, more sophisticated forgetting mechanisms might be warranted for capturing complex dependencies in non-stationary time series. Future work could explore these contexts, and extend to the multivariate forecasting regime.

\begin{subappendices}

\section{Model details} \label{app:vfsf:model}

\subsection{Fractional Differencing} \label{app:vfsf:fracdiff}
Fractional differencing \cite{granger1980introduction} is a technique used in time series analysis to transform non-stationary data into stationary data while preserving long-term dependencies in the series. Unlike traditional differencing methods, which involve subtracting the previous value (integer differencing), fractional differencing provides a more flexible way of modelling persistence in time series by introducing a fractional order parameter. Hence, fractional differencing generalizes the concept of differencing by allowing the differencing parameter 
$q$ to take non-integer values.

The general form of fractional differencing is expressed through a binomial expansion:
\begin{align} \label{eq:vfsf:fracdiff}
    \delta^q X_t = (1 - B)^q X_t = \sum_{k=0}^\infty {q \choose k}(-1)^k X_{t-k},
\end{align}
where $B$ is the backshift operator, $q > 0$ is the fractional differencing order parameter, and $q \choose k$ is the generalized binomial coefficient, i.e.
\begin{align} \label{eq:vfsf:gen_binom}
    {q \choose k} = \frac{\Gamma(q + 1)}{\Gamma(k+1) \Gamma(q - k + 1)},
\end{align}
where $\Gamma$ is the Gamma function. This formula creates a weighted sum of past values, where the weights decay gradually depending on 
$q$. Note that for integer values of $q$, the formula collapses to standard differencing.

In our case in equation \eqref{eq:vfsf:rfsf} and \eqref{eq:vfsf:rfdsf}, we apply fractional differencing to each channel in the lifted time series it is applied to, such that the \RFF{} maps $\varphi^{(p)}$ each have $D$ channels, and each channel is convolved with an individual filter as in \eqref{eq:vfsf:fracdiff} depending on the channel-wise fractional differencing parameter. In practice, we limit the summation in \eqref{eq:vfsf:fracdiff} to a finite window $W \in \bbZ_+$, since the sequence decays to zero fast this does not lose much.

\subsection{Algorithms} \label{app:vfsf:algs}
Additionally to the notation defined in Section \ref{sec:back:notation}, we extend the possible array operations. Recall that for arrays $1$-based indexing is used. Elements outside the bounds of an array are treated as zeros as opposed to circular wrapping.
Let $A$ and $B$ be $k$-fold arrays with shape $(n_1 \times \dots \times n_k)$, and let $i_j \in [n_j]$ for $j \in [k]$. We define the following channelwise array operations:
\begin{enumerate}[label=(\roman*)]
    \item Let $\blambda \in \bbR^{n_k}$. The channelwise geometric scan along axis $j$:
    \begin{align}
        A[\dots, :, \boxplus^{\blambda}, :, \dots][\dots, i_{j-1}, i_j,, i_{j+1} \dots] := \sum_{\kappa=1}^{i_j} \lambda_{i_k}^{i_j - \kappa} A[\dots, i_{j-1}, \kappa, i_{j+1}, \dots].
    \end{align}
    \item Let $\b q \in \bbR^{n_k}$ and $W \in \bbZ_+$. The channelwise fractional difference along axis $j$:
    \begin{align}
        A[\dots, :, \boxminus^{\b q}_W, :, \dots][\dots, i_{j-1}, i_j,, i_{j+1} \dots] := \sum_{\kappa=0}^{W-1} {q_{i_k} \choose \kappa} A[\dots, i_{j-1}, i_j - \kappa, i_{j+1}, \dots],
    \end{align}
    where ${q_{i_k} \choose \kappa}$ is as in \eqref{eq:vfsf:gen_binom}.
\end{enumerate}

\begin{algorithm}[H]
\caption{Computing the \RFSF{} map.}
\begin{footnotesize}
\label{alg:vfsf:rfsf}
\begin{algorithmic}[1]
    \STATE {\bfseries Input:} Time series $\bx = (\bx_1, \dots, \bx_L) \in \Seq(\bbR^d)$, spectral measure $\Lambda$, truncation $M \in \bbZ_+$, \RFF{} dimension $D \in \bbZ_+$, frac.~diff.~orders $\b q \in \bbR^D$, frac.diff.~window $W \in \bbZ_+$
    \STATE Sample independent \RFF{} frequencies $\Omega^{(1)}, \dots, \Omega^{(M)} \stackrel{\iid}{\sim} \Lambda^{D}$ and phases $\bb^{(1)}, \dots, \bb^{(M)} \stackrel{\iid}{\sim} \cU(0, 2\pi)^D$
    \STATE Initialize an array $U$ with shape $[M, L, D]$
    \STATE Compute \RFF s per time step $U[m, l, :] \gets \cos({\Omega^{(m)}}^\top \bx_l + \bb^{(m)})$ for $m \in [M]$ and $l \in [L]$
    \STATE Compute fractional differences $U \gets U[:, \boxminus_W^{\b q}, :]$
    \STATE Accumulate into level-$1$ features $P_1 \gets U[1, \boxplus, :]$
    \STATE Initialize list $R \gets [U[1, :, :]]$
    \FOR{$m=2$ {\bfseries to} $M$}
        \STATE Update with next step $P^\prime \gets P_{m-1}[+1, :] \odot U[m, :, :]$
        \STATE Initialize new list $R^\prime \gets [P^\prime]$
        \FOR{$p=2$ to $m$}
            \STATE Update with current step $Q \gets \frac{1}{p} R[p-1] \odot U[m, :, :]$
            \STATE Append to $R^\prime += [Q]$
        \ENDFOR
        \STATE Aggregate into level-$m$ features $P_m \gets R^\prime[\Sigma][\boxplus, :]$
        \STATE Roll list $R \gets R^\prime$
    \ENDFOR
    \STATE {\bfseries Output:} Arrays of \RFSF{} features per signature level and time step $P_1, \dots, P_M$.
\end{algorithmic}
\end{footnotesize}
\end{algorithm}

\begin{algorithm}[t]
\begin{footnotesize}
\caption{Computing the \texttt{RFDSF} map.}
\label{alg:vfsf:rfdsf}
\begin{algorithmic}[1]
    \STATE {\bfseries Input:} Time series $\bx = (\bx_1, \dots, \bx_L) \in \Seq(\bbR^d)$, spectral measure $\Lambda$, truncation $M \in \bbZ_+$, \RFF{} dimension $D \in \bbZ_+$, frac.~diff.~orders $\b q \in \bbR^D$, frac.diff.~window $W \in \bbZ_+$, channelwise decay factors $\blambda \in \bbR^d$
    \STATE Sample independent \RFF{} frequencies $\Omega^{(1)}, \dots, \Omega^{(M)} \stackrel{\iid}{\sim} \Lambda^{D}$ and phases $\bb^{(1)}, \dots, \bb^{(M)} \stackrel{\iid}{\sim} \cU(0, 2\pi)^D$
    \STATE Initialize an array $U$ with shape $[M, L, D]$
    \STATE Compute \RFF s per time step $U[m, l, :] \gets \cos({\Omega^{(m)}}^\top \bx_l + \bb^{(m)})$ for $m \in [M]$ and $l \in [L]$
    \STATE Compute fractional differences $U \gets U[:, \boxminus_W^{\b q}, :]$
    \STATE Accumulate into level-$1$ features $P_1 \gets U[1, \boxplus^{\blambda}, :]$
    \STATE Initialize list $R \gets [U[1, :, :]]$
    \FOR{$m=2$ {\bfseries to} $M$}
        \STATE Update with next step $P^\prime \gets \blambda^{\odot(m-1)} \odot P_{m-1}[+1, :] \odot U[m, :, :]$
        \STATE Initialize new list $R^\prime \gets [P^\prime]$
        \FOR{$p=2$ to $m$}
            \STATE Update with current step $Q \gets \frac{1}{p} R[p-1] \odot U[m, :, :]$
            \STATE Append to $R^\prime += [Q]$
        \ENDFOR
        \STATE Aggregate into level-$m$ features $P_m \gets R^\prime[\Sigma][\boxplus^{\blambda^{\odot m}}, :]$
        \STATE Roll list $R \gets R^\prime$
    \ENDFOR
    \STATE {\bfseries Output:} Arrays of \texttt{RFDSF} features per signature level and time step $P_1, \dots, P_M$.
\end{algorithmic}
\end{footnotesize}
\end{algorithm}

\subsection{Treatment of Random Parameters} \label{app:vfsf:random}

\paragraph{Reparameterization}
\begin{figure}
    \centering
    \includegraphics[width=0.6\textwidth]{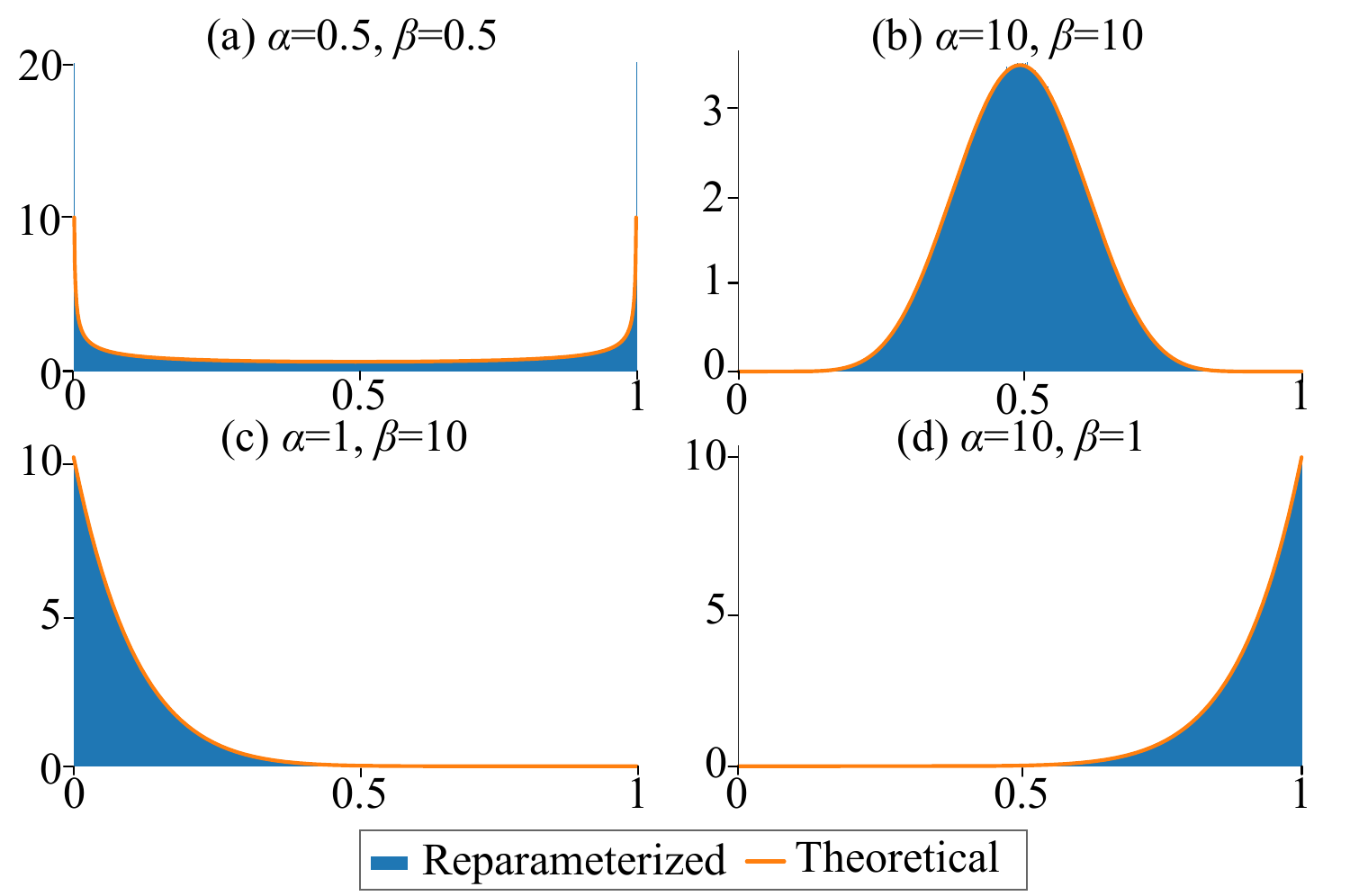}
    \caption{Reparameterizing the beta distribution for various shape parameters given fixed random outcomes.}
    \label{fig:beta_reparam}
\end{figure}
First of all, in order to be able to learn the distributional parameters of the random parameters, we need to reparameterize them \cite{kingma2013auto}. That is, we reparametrize random variable $X$, in terms of a simpler random variable, say $\epsilon$, such that the distributional parameters of $X$ can be captured by a continuous function.  For a Gaussian distribution, this is easy to do, since for $X \sim \c N(\mu, \sigma^2)$, we have the reparameterization $X = \sigma \epsilon + \mu$, where $\epsilon \sim \c N(0, 1)$. For a beta-distributed random variable, this is more tricky, and it is unfeasible to do exact reparameterization. The problem is equivalent to reparameterizing a Gamma random variable, since for $X = Y_1 / (Y_1 + Y_2)$ such that $Y_1 \sim \Gamma(\alpha, 1)$, $Y_2 \sim \Gamma(\beta, 1)$, it holds that $X \sim \c B(\alpha, \beta)$. Previous work \cite{naesseth2017reparameterization} combined accepted samples from the Gamma acceptance-rejection sampler with the shape augmentation trick for Gamma variables, so that the probability of acceptance is close to 1 in order to reparameterize Gamma random variables. We numerically validated this approach, and for fixed randomness, it is able to continuously transport between beta distributions with different shape parameters. The below figure shows a density plot of the reparameterization for shape augmentation parameter $B = 10$ and $n = 10^7$ samples. We see that the theoretical densities are reproduced exactly by the reparameterized distributions.

\paragraph{Resample or not to resample}
The evaluation of the \texttt{KL} divergences can be done analytically, and the question becomes how to sample from the data-fit term, and how to treat the integrals with respect to the random parameters during inference. There are two approaches proposed in \cite{cutajar2017random}. First, is to simply apply vanilla Monte Carlo sampling to the random hyperparameters, that is, resample them at each step optimization step. This leads to high variance and slow convergence of the \texttt{ELBO}, or it may not even converge within optimization budgets. The explanation given for this is that the since the variational factorizes across the ``layers'' of the model, it does not maintain the proper correlations needed to propagate the uncertainty. Hence, instead of resampling the random parameters during training, we retort to sampling them once at the start of training, and then keeping them fixed using the reparameterization trick throughout.

\subsection{Fixing Pathologies in the \texttt{ELBO}} \label{app:vfsf:obj}
Next, we will focus on why the \texttt{ELBO} is unsuitable for capturing latent uncertainty, and relies solely on the observation noise for uncertainty calibration. Let us recall the \texttt{ELBO} from \eqref{eq:vfsf:elbo}:
{\begin{align} \label{eq:vfsf:app_elbo}
    \c L_{\texttt{ELBO}} = &\underbrace{\sum_{i=1}^N \bbE_q\bracks{\log p(y_i \given \b w, \bOmega, B)}}_{\text{data-fit term}} -\underbrace{\KL{q(\b w)}{p(\b w)} - \KL{q(\bOmega)}{p(\bOmega)} - \KL{q(B)}{p(B)}}_{\text{\texttt{KL} regularizers}}.
\end{align}}
Clearly, the \texttt{KL} regularizers do not play a role in uncertainty calibration to the data, and we will focus on a datafit term $\bbE_q\bracks{\log p(y_i \given \bw, \bOmega, B)}$. Let $f_i = \bw^\top \Phi(\bx_i)$, and $q(f_i) = \cN(\mu_i, \sigma_i)$ be given as in \eqref{eq:vfsf:variational_pred}. Now, note that $p(y_i \given \bw, \bOmega, B) = \cN(y_i \given f_i, \sigma^2_y)$. Hence,
\begin{align}
    \bbE_q\bracks{\log p(y_i \given \bw, \bOmega, B)} = - \frac{1}{2} \log(2\pi \sigma_y^2) - \frac{1}{2 \sigma_{y}^2} \bbE_q[y_i - f_i]^2. 
\end{align}
The first term clearly plays no role in the data calibration, but serves as a penalty, while the second term matches the predictive distribution to the data. Calculating the expectation, 
\begin{align}
    \frac{1}{2\sigma_y^2}\bbE_q\bracks{y_i - f_i}^2 = \frac{1}{2\sigma_y^2}\pars{(y_i - \mu_i)^2 + \sigma_i^2}.
\end{align}
Putting these two together, we get that
\begin{align} \label{eq:vfsf:elbo_expanded}
    \bbE_q\bracks{\log p(y_i \given \bw, \bOmega, B)} = - \frac{1}{2} \log(2\pi \sigma_y^2) - \frac{1}{2 \sigma_{y}^2} \pars{(y_i- \mu_i)^2 + \sigma_i^2}.
\end{align}
Hence, the data-fit is solely determined by the closeness of the predictive mean to the underlying observation weighted the inverse of the observation noise level, while the latent function uncertainty is simply penalized without playing a role in the data-fit calibration. This explains why the \texttt{ELBO} underestimates the latent function uncertainty and only relies on the observation noise for uncertainty calibration.

Next, we consider the \texttt{PPGPR} loss. With the above notation, consider 
{\begin{align}
    \c L_{\texttt{PPGPR}} = &\underbrace{\sum_{i=1}^N \log \bbE_q\bracks{p(y_i \given \b w, \bOmega, B)}}_{\text{data-fit term}} -\underbrace{\KL{q(\b w)}{p(\b w)} - \KL{q(\bOmega)}{p(\bOmega)} - \KL{q(B)}{p(B)}}_{\text{\texttt{KL} regularizers}}.
\end{align}}
Now, the data-fit term is given by the $\log$ outside of the $q$-expectation so that $\bbE_q[p(y_i \given \bw, \bOmega, B)] = \bbE_q[\cN(y_i \given f_i, \sigma_y^2)]$, which is a convolution of two Gaussians, hence we get
\begin{align} \label{eq:vfsf:ppgpr_expanded}
    \log& \bbE_q[p(y_i \given \bw, \bOmega, B)] \\=& \log \cN(y_i \given \mu_i, \sigma_f^2 + \sigma_y^2) = -\frac{1}{2}\log(2\pi(\sigma_i^2 + \sigma_y^2)) - \frac{1}{2(\sigma_i^2 + \sigma_y^2)} (y_i - \mu_i)^2. 
\end{align}
Contrasting this with \eqref{eq:vfsf:elbo_expanded}, we see that \eqref{eq:vfsf:ppgpr_expanded} treats the latent function variance $\sigma_i^2$ and $\sigma_y^2$ on equal footing, hence restoring symmetry in the objective function. This is why \texttt{PPGPR} often leads to better predictive variances, since both the latent function uncertainty and observation noise are taken into account during uncertainty calibration, which makes use of the expressive power contained in the kernel function regarding the geometry of the data-space.

\subsection{Scalability} \label{app:vfsf:scale}

\begin{figure}[h]
\begin{minipage}{0.49\textwidth}
    \includegraphics[width=\textwidth]{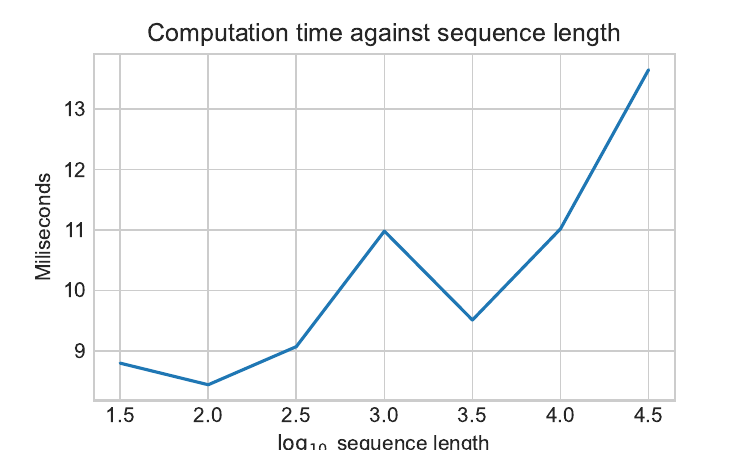}
\end{minipage}
\begin{minipage}{0.49\textwidth}
    \includegraphics[width=\textwidth]{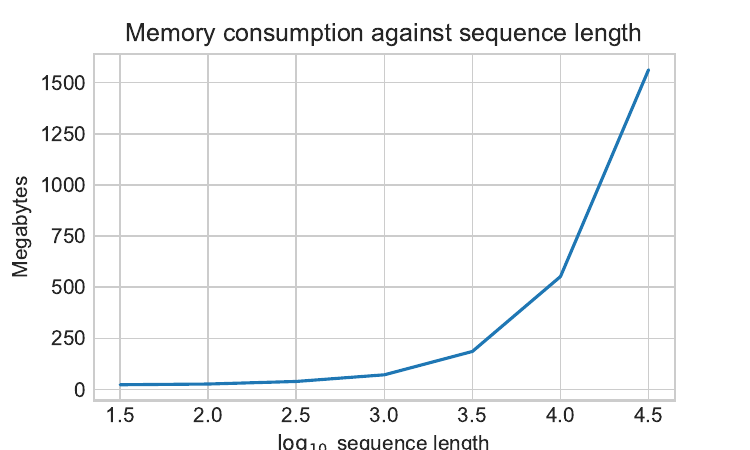}
\end{minipage}
\caption{Computation time (left) and memory consumption (right) of \texttt{VRS\textsuperscript{3}GP} measured against scaling the sequence length on a logarithmic scale. The hyperparameters of the model are $D = 200$ and $M = 5$, and the input time series is univariate augmented with $l = 9$ lags.}
\label{fig:benchmark}
\end{figure}

\end{subappendices}

\endgroup
\begingroup
\chapter*{Conclusion}
\addcontentsline{toc}{chapter}{Conclusion}
\section*{Summary}
\addcontentsline{toc}{section}{Summary}
The thesis explores the integration of path signatures into machine learning frameworks to address challenges in modeling sequential and structured data. Path signatures, rooted in rough path theory, serve as a powerful mathematical tool that generalizes polynomial features and captures complex dynamics in sequential data. Despite their theoretical strengths, computational challenges such as the combinatorial growth of tensor dimensions have limited their scalability, while its kernelized variation suffers from quadratic complexity in the sequence length. This thesis addresses these limitations while preserving the expressiveness of signatures.

Chapter \ref{ch:back} introduces path signatures by laying the necessary mathematical foundation through tensor algebras and their properties. It highlights path signatures' universality in approximating continuous functions and their robustness to reparameterization, which makes them particularly effective for sequence modelling. The connection between path signatures and classical polynomial features motivates their suitability for extracting meaningful representations from sequential data, although scalability remains a concern. Additionally, oscillatory behaviour may deteriorate performance on complex datasets analogously to how polynomial regression is often outperformed by other nonlinearities, which motivates the use of the lifted signature kernel as introduced in \cite{kiraly2019kernels} and discussed in Section \ref{sec:back:sigkernels}.

Building upon this foundation, Chapter \ref{ch:gpsig} applies path signatures to Gaussian Processes (\texttt{GP}s), embedding signature features into covariance functions. This approach enhances \texttt{GP}s' ability to model temporal dependencies in sequential data within a Bayesian framework. The introduction of sparse variational inference mitigates computational bottlenecks, enabling the model to scale effectively to large datasets of high-dimensional sequence data. The empirics demonstrate the advantages of the method, particularly in time series classification tasks where classic sequence kernels underperform. Synergy with sequence modelling methods in deep learning is also demonstrated, which can further enhance performance on real-world tasks.

Motivated by the success of combining signatures with deep learning architectures in Chapter 2, Chapter \ref{ch:s2t} introduces the \texttt{Seq2Tens} framework, which bridges path signatures with deep learning architectures. By employing low-rank linear functionals, \texttt{Seq2Tens} mitigates the computational expense of signature features while preserving their expressiveness. Stacking multiple layers of low-rank transformations efficiently enables deep learning-style flexibility, making \texttt{Seq2Tens} both efficient and powerful for tasks such as time series classification, mortality prediction, and generative modeling. The framework's success underscores its ability to combine the mathematical richness of signatures with the scalability of deep learning. Further on the theoretical side, the chapter formalizes the theory behind the discretized order-$1$ signature features introduced in Section \ref{sec:back:discrete} proving universality under mild conditions. This explains the 
practical usefulness of the order-$1$ signature, that we often make use of in the other chapters.

In Chapter \ref{ch:g2t}, the application of path signatures is extended to graph learning by connecting expected signatures on graphs to hypo-elliptic diffusion processes. This approach introduces a hypo-elliptic graph Laplacian that encodes both local and long-range graph structures via diffusion dynamics. By integrating these features into neural network architectures, the method effectively models hierarchical and global characteristics in graph-structured data. This development demonstrates the versatility of signature methods beyond sequential data, achieving strong performance in tasks such as node and graph-level classification. This connects to Chapter \ref{ch:s2t} in the deep learning-style iteration of layers, since we employ the same low-rank representation in the weight-space as introduced in there.

Chapter \ref{ch:rfsf} tackles the scalability limitations of signature kernels by introducing Random Fourier Signature Features (\RFSF). This novel method approximates signature kernels using random feature maps, significantly reducing computational complexity while preserving kernel properties with high probability. Through dimensionality reduction techniques, the approach allows signature methods to handle large-scale datasets efficiently both in sample size and sequence length. Experiments confirm that this retains high performance across time series tasks, establishing it as a practical tool for large-scale machine learning applications. This work can be treated as an extension of the low-rank algorithms proposed in \cite{kiraly2019kernels}, which had a similar aim of reformulating the computations from the dual space to the feature space, reducing the computational complexity while preserving the representational power of kernelization.

Finally, Chapter \ref{ch:vfsf} focuses on time series forecasting and introduces the Recurrent Sparse Spectrum Signature Gaussian Process (\texttt{RS\textsuperscript{3}GP}). This model combines Random Fourier Signature Features (\RFSF) with a novel decay-based forgetting mechanism to dynamically adjust its context length allowing it fo focus on more recent data segments. By leveraging variational inference and efficient parallelization, \texttt{RS\textsuperscript{3}GP} captures both short-term and long-term dependencies efficiently, making it well-suited for complex high-dimensional forecasting tasks. The model achieves competitive results against traditional \texttt{GP}s and state-of-the-art deep learning methods, highlighting its scalability and interpretability. This chapter can be treated as an extension of Chapters \ref{ch:gpsig} and \ref{ch:rfsf}, since we combine the finite-rank signature kernel called Random Fourier Signature Features introduced in Chapter \ref{ch:rfsf} with the variational \texttt{GP} framework of Chapter \ref{ch:gpsig}. 

The thesis demonstrates an evolution of path signatures from a theoretical construct to a versatile tool in machine learning. By integrating path signatures with Gaussian processes, deep learning, and kernel methods, the work highlights how the core principles of path signatures can address longstanding challenges in sequential and structured data modeling. Furthermore, the adaptation of these principles to graph learning, scalable low-rank and random feature techniques, extends the applicability of path signatures beyond their traditional scope, offering solutions to real-world problems in both temporal and graph-structured domains. This exploration bridges the gap between theoretical mathematical constructs and practical machine learning methodologies. The findings pave the way for future advancements by emphasizing the balance between computational scalability and expressiveness, making the contributions of this thesis a promising foundation for applications of signatures in scalable AI systems.


\section*{Outlook}
\addcontentsline{toc}{section}{Outlook}
This thesis demonstrates that path signatures provide a powerful and mathematically grounded framework for learning from sequential and structured data. However, there are avenues for further exploration. In particular, next we detail open questions and further ideas that might be interesting to pursue in the future relating to each chapter. 

In Section, \ref{app:gpsig:gpsig} we show on the theoretical side that there exists a continuous modification of the Gaussian process with the unlifted truncated signature kernel as covariance function. We do this by bounding the metric entropy of the space of uniformly bounded variation paths.
There are several possible extensions for this result, such as considering the case of the lifted signature kernel by e.g.~the RBF kernel as static kernel, generalizing to unbounded variation paths using tools from rough path theory, and extending the result to the untruncated signature kernel to provide theoretical support for the \texttt{GP} approach of \cite{lemercier2021siggpde}.

In Section, \ref{sec:3} we introduce \texttt{LS2T} layers as (order-$1$) signature features projected by low-rank weight tensors. Although the low-rank constraint limits the expressiveness of these features, we empirically demonstrate that stacking such layers allows to achieve state-of-the-art performance. Some algebraic results regarding compositions of such layers are given in \cite[App.~C]{toth2021seq2tens}, but it would be interesting to provide an explicit characterization of the functions that can be recovered by compositions of low-rank sequence-to-sequence layers, and answer the question theoretically whether the iteration of such low-rank layers can recover universality. Another observation, is that if the components vectors of the low-rank tensor projections are initialized as Gaussians, then the resulting activations at initialization correspond to the Tensor Random Projection map (\texttt{TRP}) from Section \ref{sec:rfsf:trp} applied to unlifted signature features. The result in Theorem \ref{thm:main3} can be easily generalized to this case to show that in the infinite-width limit the activations of the \texttt{LS2T} layer at initialization correspond to a feature representation of the unlifted signature kernel. Consequently, a 2-layer neural network with an \texttt{LS2T} layer followed by a linear projection layer, both with Gaussian initializations, corresponds in the infinite-width limit to the Gaussian process with the signature covariance function from Chapter \ref{ch:gpsig}. Inspired by this, in the setting when the low-rank projections are learnt, it would be interesting to characterize the Neural Tangent Kernel (\texttt{NTK}) \cite{jacot2018neural} of \texttt{LS2T} layers, and establish its relation to the signature kernel, potentially sheding light on the training dynamics. Finally, the extension of signature features with a forgetting mechanism from Section \ref{sec:vfsf:forgetting} can be incorporated into \texttt{LS2T} layers to make them amenable to the learning of both short- and long-range interactions, potentially ameliorating the need for the preprocessing \texttt{CNN} layers employed in our experiments.

In Section \ref{sec:g2t:algos}, we introduce hypo-elliptic diffusions on graphs corresponding to expected signatures of random walks on nodes. We apply low-rank tensor projections similarly to the predecing section to construct scalable graph learning architectures. These are called \texttt{G2TN} layers, and they generalize the \texttt{LS2T} layers of the previous chapter to the domain of graphs. The exact same questions about \texttt{LS2T} layers follow through, e.g.~whether compositions of such low-rank layers can recover the characteristicness property of expected signatures. Further, it can analogously be derived that in the infinite-limit the layer activations converge to a feature representation for the expected signature kernel (i.e.~a kernel for probability measures on path-space by considering the inner product of expected signatures), which motivates to investigate the \texttt{NTK} associated to \texttt{G2TN} layers. Additionally, there are two main issues associated with classical \texttt{GNN}s: oversquashing related to structural bottlenecks on graphs \cite{topping2022understanding}; and oversmoothing caused by node features converging to a constant node vector as we increase the network depth \cite{rusch2023survey}. Whether these issues are apparent in the context of \texttt{G2TN} layers is an open question. However, we conjecture that due to the noncommutativity of the sequence features used to represent random walks on graphs, our node features might be less plagued by oversmoothing, and that graph rewiring techniques may be equally useful in this setting for ameliorating oversquashing.

In Chapter \ref{ch:rfsf}, we construct various scalable, random approximations for the order-$1$ signature kernel that converge with high probability. The extension of this construction to the approximation of  higher-order signature kernels (see Section~\ref{sec:back:discrete}) is straightforward, but the theoretical results require further work, especially the key lemma in the proof of Theorem \ref{thm:main}, i.e.~Lemma \ref{lem:RFSF_approx}, where the recursive step requires to consider the higher-order terms (i.e.~the repetitions). Further, extension to continuous-time signature kernels including the unbounded variation setting is an interesting future direction to consider. On the applications' side, \texttt{RFSF} features can also be used to construct efficient approximations to signature \texttt{MMD}s, which could come handy in several contexts, such as hypothesis testing for large time series datasets \cite{chevyrev2022signature}, or as an efficient loss function for the training of generative models for time series \cite{issa2023non}.

Finally, Chapter \ref{ch:vfsf} constructs a parametric Gaussian process model called \texttt{RS\textsuperscript{3}GP} for efficient forecasting of long time series with adaptive context horizon, that utilizes \texttt{RFSF} features combined with a forgetting mechanism. However, the model formulation can be applied in more general contexts, such as sequence-to-sequence regression and system dynamics identification. In particular, we carried out preliminary experiments on battery health prediction \cite{aitio2021predicting} with promising results, indicating that this could be an interesting application to consider. Finally, given the results in Section \ref{sec:vfsf:exp}, we conclude that our models are much more efficient but somewhat lacking in expressiveness compared to the state-of-the-art deep learning models for probabilistic time series forecasting, such as \texttt{TSDiff}. This could be enhanced by constructing a deep \texttt{GP} using probabilistic iterations of \texttt{RS\textsuperscript{3}GP} layers in the same spirit as done in \cite{cutajar2017random}.

Finally, a critical direction lies in the interpretability of models that employ path signatures. Developing methods to explain the contributions of signature features to predictions would enhance their applicability in fields such as healthcare and finance, where transparency is essential. By bridging theoretical rigour with computational scalability, this thesis establishes path signatures as a robust tool for modern machine learning. The innovations introduced here pave the way for scalable deployment of signature features in diverse large-scale real-world applications.
\endgroup


\addcontentsline{toc}{chapter}{Bibliography}
\bibliography{bib}        
\bibliographystyle{plain}  

\end{document}